\documentclass[12pt]{report}
\usepackage[utf8]{inputenc}
\DeclareUnicodeCharacter{2212}{-}
\usepackage{graphicx}
\usepackage{float}
\usepackage{hyperref}
\usepackage{amsmath}
\usepackage{amsthm}
\usepackage[linesnumbered, ruled, vlined, algo2e]{algorithm2e}
\usepackage{array}
\usepackage{xcolor}
\usepackage{soul}
\usepackage{amssymb}
\usepackage{caption}
\usepackage[labelsep=space,tableposition=top,font=small]{caption}
\usepackage[caption=false]{subfig}
\usepackage{multirow}

\newcommand{\nop}[1]{}
\newtheorem{lemma}{Lemma}

\usepackage{algorithm}\floatstyle{plaintop}
\usepackage{algpseudocode}

\begin{document}

\title{Sensor Data Validation and Driving Safety in Autonomous Driving Systems}
\author{Jindi Zhang}
\date{}

\begin{titlepage}
\begin{center}
    \vspace*{1pt}
    
    \Large
    CITY UNIVERSITY OF HONG KONG
    
    \vspace*{3cm}
    
    \LARGE
    \textbf{Sensor Data Validation and Driving Safety in Autonomous Driving Systems}
    
    \vspace*{2cm}
    
    \normalsize
    Submitted to \\
    Department of Computer Science \\
    in Partial Fulfillment of the Requirements \\
    for the Degree of Doctor of Philosophy
    
    \vspace*{4cm}
    by
    
    \vspace*{1cm}
    Jindi Zhang
    
    \vspace*{1cm}
    January 2022
\end{center}
\end{titlepage}

\newpage

\begin{center}
    \vspace*{4cm}
    \textit{To people I love.}
\end{center}

\newpage

\maketitle

\begin{abstract}
\label{abstract}

\addcontentsline{toc}{section}{Abstract}

Autonomous driving technology has drawn a lot of attention due to its fast development and extremely high commercial values. The recent technological leap of autonomous driving can be primarily attributed to the progress in the environment perception. Good environment perception provides accurate high-level environment information which is essential for autonomous vehicles to make safe and precise driving decisions and strategies. Moreover, such progress in accurate environment perception would not be possible without deep learning models and advanced onboard sensors, such as optical sensors (LiDARs and cameras), radars, GPS. However, the advanced sensors and deep learning models are prone to recently invented attack methods. For example, LiDARs and cameras can be compromised by optical attacks, and deep learning models can be attacked by adversarial examples. The attacks on advanced sensors and deep learning models can largely impact the accuracy of the environment perception, posing great threats to the safety and security of autonomous vehicles. In this thesis, we study the detection methods against the attacks on onboard sensors and the linkage between attacked deep learning models and driving safety for autonomous vehicles.

To detect the attacks, redundant data sources can be exploited, since information distortions caused by attacks in victim sensor data result in inconsistency with the information from other redundant sources. To study the linkage between attacked deep learning models and driving safety, the evaluation of the impact of attacks on driving safety in an end-to-end fashion is the key. Thus, we can leverage the data from different onboard sensors to detect attacks for single autonomous vehicle platforms. And we can use sensor data from multiple neighboring vehicles to achieve the attack detection for multiple connected autonomous vehicles. Furthermore, we can implement an end-to-end driving safety evaluation framework to help assess the attack impact on driving safety.

In this thesis, we first develop a data validation framework to detect and identify optical attacks against LiDARs and cameras for single autonomous vehicles. The greatest challenge lies in finding a type of redundant information which can be observed in both LiDAR point clouds and camera images. We tackle the challenge by leveraging depth information as the redundancy. Our main idea is to (1) use data from three sensors to obtain two versions of depth maps (i.e., disparity) and (2) detect attacks by analyzing the distribution of disparity errors. Based on the detection scheme, we further develop an identification model that is capable of identifying up to $n-2$ attacked sensors in a system with one LiDAR and $n$ cameras. We prove the correctness of our identification scheme and conduct experiments to show the accuracy of our identification method.

Second, as the countermeasures designed for single vehicles take no advantage of multiple connected vehicles, simply deploying them in a collaborative autonomous driving system does not produce more security bonus. To this end, we propose a new data validation method by leveraging data sources from multiple neighboring vehicular nodes to detect the optical attacks against LiDARs. The first challenge of designing the method is that no mobile network can bare the burden of transmitting all point clouds among connected autonomous vehicles, leading to limited size of data for validation, while the second challenge is that the scans of objects in point clouds are usually severely incomplete at the unlit side, causing barriers to accurate validation. To overcome the first challenge, we leverage a region proposal network to produce proposals as validation regions and propose to only transmit the scans within them, which heavily scales down the size of data for transmission and never overlooks potential attacks. We tackle the second challenge though concatenating the original scan of objects with a symmetrical copy of it to fill in the incomplete part. We perform preliminary experiments to examine our method. And the results show that our data validation method for multiple connected vehicles detects the attacks effectively with a fair accuracy.

Third, previous studies demonstrated that adversarial examples can hugely impact deep learning models for environment perception, and inaccurate perception results with no doubt may jeopardize the driving safety of autonomous vehicles. However, driving safety is a combined result of many factors, and the weakened model performance does not necessarily result in safety dangers. The linkage between the performance of a deep learning model under adversarial attacks and driving safety is still under-explored. To study such linkage and evaluate the impact of adversarial examples on driving safety in an end-to-end fashion, we propose an end-to-end driving safety evaluation framework with a set of driving safety performance metrics. With the framework, we investigate the impact of two primary types of adversarial attacks, perturbation attacks (as shown in Fig.~\ref{fig:plot_perturbation}) and patch attacks (as shown in Fig.~\ref{fig:plot_patch}), on the driving safety rather than only on the perception precision of deep learning models. In particular, we consider two state-of-the-art models in vision-based 3D object detection, Stereo R-CNN and DSGN. By analyzing the results of our extensive evaluation experiments, we find that the attack's impact on the driving safety of autonomous vehicles and the attack's impact on the precision of 3D object detectors are decoupled. In addition, we further investigate the causes behind the finding with an ablation study. Our finding provides a new perspective to evaluate adversarial attacks and guide the selection of deep learning models in autonomous driving.

To briefly summarize, in this thesis, we first propose a framework to detect optical attacks and identify the attacked sensors for single autonomous vehicles. Then, we propose a data validation method to detect the optical attacks against LiDARs using point clouds from multiple connected vehicle sources. At last, we propose an end-to-end driving safety evaluation framework to investigate the impact of adversarial attacks on the driving safety of autonomous vehicles. All the presented research in this thesis would greatly advance the safety and security of autonomous driving technology and eventually benefit our future life.

\end{abstract}

\newpage

\begin{center}
    {\textbf{CITY UNIVERSITY OF HONG KONG\\ Qualifying Panel and Examination Panel}}
\end{center}

{ 

\begin{table}[h]
\begin{tabular}{ll}
\textbf{Surname:}                          & ZHANG   \\
\textbf{First name}:                       & Jindi   \\
\textbf{Degree:}                           & Doctor of Philosophy    \\
\textbf{College/Department:}                       & Department of Computer Science \\
                                  &                                                                                                                                    \\
\multicolumn{2}{l}{\textbf{The Qualifying Panel of the above student is composed of:}}                                                                                           \\
\textbf{Supervisor(s)}               &                                                                                                                                    \\
Prof. JIA Xiaohua    & \begin{tabular}[c]{@{}l@{}}Department of Computer Science,\\ City University of Hong Kong\end{tabular}                       \\
                                  &                                                                                                                                    \\
\textbf{Qualifying Panel Member(s)} &                                                                                                                                    \\
Prof. WANG Jianping                  & \begin{tabular}[c]{@{}l@{}}Department of Computer Science,\\ City University of Hong Kong\end{tabular}                       \\
Prof. WANG Cong                   & \begin{tabular}[c]{@{}l@{}}Department of Computer Science,\\ City University of Hong Kong\end{tabular}                       \\
                                  &                                                                                                                                    \\
\multicolumn{2}{l}{\textbf{This thesis has been examined and approved by the following}}                                                                              \\
\multicolumn{2}{l}{\textbf{examiners:}}                                                                              \\
Dr. LI Zhenjiang                        & \begin{tabular}[c]{@{}l@{}}Department of Computer Science,\\ City University of Hong Kong\end{tabular}                       \\
Prof. JIA Xiaohua                    & \begin{tabular}[c]{@{}l@{}}Department of Computer Science,\\ City University of Hong Kong\end{tabular}                       \\
Prof. LI Minming                     & \begin{tabular}[c]{@{}l@{}}Department of Computer Science,\\ City University of Hong Kong\end{tabular}                       \\
Dr. WANG Dan               & \begin{tabular}[c]{@{}l@{}}Department of Computing,\\ The Hong Kong Polytechnic University\end{tabular}
\end{tabular}
\end{table}} 

\newpage

\section*{Acknowledgement}
\label{acknowledgements}

\addcontentsline{toc}{section}{Acknowledgements}

The journey to a Ph.D. is never tedious, filled with ups and downs, adventures and challenges, stresses and breakthroughs, strains and joys, and failures and rewards. It is truly a great pleasure personally for me to have such an opportunity to experience life as a Ph.D. candidate, learning how to think and write, digging into scientific problems, exploring knowledge for humankind, and most importantly, getting to know myself and searching for the meaning of this life. But none of this would have been possible if I had relied only on myself. So, I would like to express my heartfelt thanks to everyone who has ever generously offered me help and support during this journey.

Looking back on the past four years, I would like to thank my supervisors Prof. Jianping Wang and Prof. Xiaohua Jia. They gave me tremendous help and guidance through these years, in both life and academia. Prof. Wang is such an optimistic and quick-thinking person who is full of passions. Her enthusiasm cheers me a lot. Whenever I was faced with problems and got lost and confused, she always shared her ideas and wisdom to enlighten me without reservation and encouraged me to try out new ideas. She also taught me critical thinking and essential writing skills, which are so helpful for my research. In addition, I would like to thank Prof. Kui Wu from the University of Victoria and Prof. Kejie Lv from the University of Puerto Rico at Mayagüez. They gave me a lot of valuable comments and suggestions on my work.

Collaboration is crucial for doing good research. I am so fortunate to have excellent collaborators. My gratitude goes to Ms. Yifan Zhang, Mr. Yang Lou, Ms. Jinghuai Zhang, and Ms. Weiwei Fu for their efforts and teamwork. I also want to thank my labmates, Dr. Luning Wang, Dr. Zhongzheng Tang, Dr. Rongwei Yang, Dr. Qian Xu, Dr. Yangbin Chen, Ms. Ruohan Wang, Mr. Yang Jin, Mr. Xinhong Chen, Mr. Zihao Wen, and Ms. Luyao Ye, for the joyful times we spent together. With their companion, research life is much more fun.

The internship experience at Huawei Hong Kong Research Center also benefited me a lot. It let me have a good taste of industry research. During my internship, I had the luck to get acquainted with Prof. Lei Chen, Dr. Chen Cao, Dr. Xueling Lin, Dr. Xiaohui Li, Mr. Ngai Fai Ng, Mr. Yuanwei Song, Mr. Han Gao, Ms. Yi Yang, Ms. Cong Wang, Mr. Haoyang Li, Ms. Luyu Qiu, Ms. Jingzhi Fang, and Ms. Ge Lv. It is a great pleasure to work with them.

Time flies. It has been five years since I started my life in Hong Kong. Creating a life in a new place is never easy, so I want to mention my friends who made things easier for me in this city. They are Ms. Yi Tang, Mr. Ruifang Yang, and Ms. Jiyou Teng.

Last but not least. I want to thank my family, my dearest brother and best friend, Mr. Jinxiao Zhang, my parents, Ms. Xiaoshu Zhang and Mr. Zhao Zhang, and my grandmother, Ms. Fengqiong Zhao, for their selfless love, caring, and support. And I also want to thank my beloved girlfriend, Ms. Dan Su, for everything.

\newpage

\renewcommand*\contentsname{Table of Contents}
\tableofcontents
\newpage
\listoffigures
\newpage
\listoftables

\newpage

\chapter{Introduction}
\label{ch1:introduction}

\section{Background}
\label{ch1:sec:background}

Human has dreamed of autonomous driving for years, which could be traced back to the 20th century~\cite{sentinel26phantom}. The modern development of autonomous driving technology started with the DARPA Grand Challenge in 2004~\cite{chen04ohio}. Since then, it has attracted significant attention from both academia and industry due to its scientific charisma and the extremely high commercial values.

A modern autonomous vehicle is a highly complex Artificial Intelligence (AI) system consisting of many functional pieces organized along a data-processing pipeline. These functional pieces can be roughly categorized into three major modules, i.e., perception module, planning module, and execution module~\cite{badue21self,reke20self}. The perception module produces high-level information about surrounding environment through environment perception. The planning module makes short-term driving decisions and long-term driving strategies based on the environment information from the perception module and user intentions. And the execution module carries out the commands issued by the planning module. It is worth noting that the recent advancement of autonomous driving technology primarily lies in the perception module.

The recent technological leap in the perception module is mainly due to advanced onboard sensors and deep learning models. Since the DARPA Grand Challenge in 2004, autonomous vehicles have been equipped with multiple types of sensors, such as cameras, LiDARs, Radars, and GPS~\cite{chen04ohio}. These sensors, in particular optical sensors (cameras and LiDARs),  provide essential measurements of the surrounding environment with better quality and higher quantity. Specifically, cameras provide RGB image data, which capture visual shapes and textures of real-world objects in detail and enable the vision-based solution for autonomous driving. And LiDARs emit and receive laser beams to precisely measure the distances between objects and the ego autonomous vehicle, creating the possibility for accurate object detection in 3D world. In addition, as deep learning developed so fast and received great success these years, it has been widely applied to almost all computer vision applications. Deep learning methods relying on powerful generalization capability and large datasets have significantly boosted the performance of environment perception tasks in autonomous driving, such as 2D/3D object detection, 2D/3D object tracking, and semantic segmentation. And there emerged a lot of excellent learning based algorithms, e.g., YOLO~\cite{redmon16you}, SSD~\cite{liu16ssd}, and R-CNN based algorithms~\cite{girshick14rich,girshick15fast,ren15faster}.

Despite the great progress in autonomous driving technology, the top concern for its public adoption is still over safety and security. It is hard to convince people to put their lives in the hands of a driving machine for a period of time, without showing them that the machine is flawless.

However, every coin has two sides. While advanced sensors and deep learning make significant improvements for environment perception, they also suffer from recently reported security threats.

For onboard sensors, they can be compromised to make erroneous measures or even dysfunctional by physical attacks. Petit et al.~\cite{petit15remote} demonstrated that cameras can be blinded by strong light beams and are dysfunctional under continuous quick exposure changes caused by a flickering light source. They also used laser beams with the same frequency as used by LiDARs to spoof fake points in point clouds. Shin et al.~\cite{shin17illusion} systematically studied the attack methods against LiDARs and introduced the saturation attack to incapacitate LiDARs. And Yan et al.~\cite{yan16can} focused on other onboard sensors and successfully launched jamming and spoofing attacks against ultrasonic sensors and radars. Based on these previous studies, the research community kept digging into the flaws and vulnerabilities of onboard sensors and proposed more comprehensive studies on compromising these sensor~\cite{cao19adversarial,sun20towards,cao21invisible}.

Since the sensors for measuring the surrounding environment are at the beginning of the entire data-processing pipeline of autonomous driving systems, errors in sensory measurements could lead to incorrect object detection results, and further result in much more severe mistakes in driving decision making of the planning module, jeopardizing the driving safety.

As for deep learning models, adversarial examples are believed as one of their major securities threats~\cite{xiao20adversarial}. Adversarial attacks refer to the malicious use of adversarial examples to attack deep learning models. The examples are created by incorporating the normal input data of deep learning models with carefully designed perturbations, and can confuse deep learning models to make prediction errors. Another major trait of adversarial examples is that they can be in variety of forms. Some adversarial examples are with perturbations that are even too tiny to be perceivable to human eyes~\cite{goodfellow14explaining}, while some adversarial examples with high transferability are universally effective to a range of different deep learning models~\cite{moosavi17universal}. Some of them look like stickers and tapes, while some are designed to simulate natural lighting effects~\cite{eykholt18robust}. The unpredictability, concealment, and harmfulness of adversarial examples make them extremely hazardous to deep learning models.

Moreover, since deep learning models are incorporated in almost all the processes in the autonomous driving pipeline, erroneous inferences caused by adversarial examples can happen everywhere, especially in the perception module whose input is from the environment, also imposing significant risks on driving safety.

\section{Research Objectives and Motivations}

In this thesis, to help overcome the threats to onboard sensors and deep learning models in autonomous vehicles, we are motivated to study their impacts and design detection methods.


First, we aim to detect optical attacks against cameras and LiDARs. Specifically, we design a detection method and an identification method for single autonomous vehicles by validating the data produced by multiple onboard optical sensors (camera and LiDAR). And we design the detection method for the connected autonomous vehicles by validating sensor data from multiple neighboring vehicular nodes. Our motivations for these two research objectives are: (1) since cameras and LiDARs are the most critical sensors for environment perception, they should have the top priority for protection; (2) since attacks cause information distortions in sensor data, resulting in the information inconsistency among sensor data from different sources, redundant data sources can be used for sensor data validation.

Second, we evaluate the impact of adversarial attacks on the driving safety of autonomous driving systems in an end-to-end fashion. Our motivations for this research objective are: (1) adversarial attacks can do direct harms to the performance of deep learning models, but the data-processing pipeline of autonomous vehicles consists of many functional pieces and driving safety is a combined effort of many factors, so weakened deep learning performance does not necessarily lead to dangers in driving safety; (2) the linkage between the weakened performance of a deep learning model caused by adversarial attacks and driving safety of autonomous vehicles has not been explored in previous studies.

\section{Challenges}
\label{ch1:sec:challenges}

To detect and identify optical attacks for the single autonomous vehicles by validating sensor data from different onboard sensors, there exist two major challenges. First, the effects of optical attacks on camera images and LiDAR point clouds are different, so the information distortions appearing respectively in images and point clouds are different. It is challenging to find an appropriate type of redundant feature where the information distortions can be observed for both images and point clouds. Second, the challenge lies in that the scope and the position of the information distortions appearing in both types of data are unpredictable, so the data validation method must perform in fine granularity for whole sensor views. In our data validation countermeasure framework, we use depth information as redundancy to reveal information distortion, which can be extracted from both camera images and LiDAR point clouds, to tackle the first challenge. As for the second challenge, we compare the depth information distortion among different sensor data in pixel-level. Based on these two techniques, we design an algorithm to detect optical attacks using a three-sensor system, and further derive a method to identify the attacked sensors.

Moreover, we also face challenges when detecting optical attacks by validating LiDAR point clouds for multiple connected vehicles. First, the point clouds from every vehicular node are harvested in a fixed perspective each time, resulting in one side of objects not being scanned by LiDARs, so it is challenging to accurately validate the redundant information in severely incomplete scans. Second, the network connecting to autonomous vehicles cannot bear big data transmission overhead, indicating that only limited size of point clouds can be transmitted. But the information distortions caused by attacks could exist anywhere in point clouds, so the challenge lies in how to scale down the data for validation while maintaining information distortions still detectable at the same time. In our proposed detection method, we use a mirroring technique to fill the unscanned part of objects with points according to the object symmetry, to overcome the first challenge. And we use a region proposal network (RPN), e.g., PointNet++~\cite{qi17pointnet++}, to process the point clouds and use the produced proposals as the validation regions, so that the scale of data for validation is significantly reduced and no potential information distortions are overlooked. As for the data validation, we first use the scans in the validation regions to generate surface meshes. And then we discretize them and measure the distances among them. Finally, we detect attacks by analyzing the distance distributions.

In terms of assessing the impact of adversarial attacks on driving safety of autonomous vehicles in an end-to-end fashion, we implement an end-to-end driving safety evaluation framework. To this end, we face two nontrivial technical challenges. First, along the data-processing pipeline, the 3D object detection results in the perception module only contain static information, such as position and dimension, and include no dynamic information. Second, to realistically generate a planned trajectory for the self-driving ego-vehicle, real driving constraints, such as speed limits for different road types and dynamics models for different vehicles, must be provided to the planning module. To tackle the first challenge, we train a classifier with manually labeled ground truth to categorize whether an object is moving or static. For the second challenge, we train another classifier with road type labels to classify the road type of each scenario, so as to select appropriate driving constraints.

\section{Contributions}
\label{ch1:sec:contributions}

Our main contributions in this thesis are as follows:

\begin{itemize}
\item
We propose a data validation framework against the optical attacks for single autonomous vehicles by detecting distortions in depth information in a three-sensor system. And based on this, we further develop a method capable of identifying up to $n − 2$ attacked sensors in a system with one LiDAR and $n$ cameras and give mathematical proof of its correctness.

\item
We propose a data validation scheme against the optical attacks on LiDARs for connected autonomous vehicles, by analyzing the distance distributions of surface meshes generated from LiDAR point clouds which are collected from multiple neighboring autonomous vehicles.

\item
We propose an end-to-end driving safety evaluation framework and use it to evaluate the impact of adversarial attacks on the driving safety of vision-based autonomous vehicles in an end-to-end fashion. We investigate the linkage between attacked deep learning models and driving safety and find that, the attack impact on the driving safety and the attack impact on the precision of 3D object detectors are decoupled.
\end{itemize}

\section{Thesis Outline}
\label{ch1:sec:thesis}

The rest of this thesis is organized as follows. In Chapter~\ref{ch2:literature}, we review the related literature to this thesis. In Chapter~\ref{ch3:sensor}, we introduce our data validation framework for single autonomous vehicles. In Chapter~\ref{ch4:sensor}, we present the proposed data validation scheme for multiple connected autonomous vehicles. In Chapter~\ref{ch5:impact}, we introduce the evaluation of the impact of adversarial attacks on driving safety. Finally, we conclude and discuss the future work in Chapter~\ref{ch6:conclusion}.

\chapter{Literature Review}
\label{ch2:literature}

In this chapter, we review the related work from the perspectives of optical attacks, existing countermeasures against optical attacks, depth estimation, adversarial attacks, vision-based 3D object detection, and motion planning.

\section{Attacks on Optical Sensors}

The methods for attacking optical sensors (LiDAR and camera) gradually become more and more advanced. In surveys of studies~\cite{wyglinski13security,ren19security}, researchers introduced the vulnerability that perceptual sensors of AVs could be compromised via physical channels at a close distance. In~\cite{petit15remote}, the authors showed several effective and realistic methods to compromise a 2D LiDAR and a camera. Particularly, in their experiments, they managed to relay and spoof LiDAR signals, as well as blind the camera using strong light beams. Adversary attacks against camera with intensive lights were also studied in~\cite{yan16can} and even caused irrecoverable damages to the camera. Later, Shin~\textit{et al.} demonstrated the attacks against Velodyne VLP-16, one of the most popular top-sale 3D LiDARs in the market, by producing fake signals~\cite{shin17illusion}. Based on the previous work, the researchers dug even deeper in~\cite{cao19adversarial}, in which the authors designed an optimization-based strategy to produce more bogus dots to compromise a 3D LiDAR with a much higher success rate, and they constructed new attacking scenarios to study the impact on the decision making of AVs. We refer to these attacks using physical signals against optical sensors as optical attacks. Despite their importance, existing studies in optical attacks only give some rough countermeasure ideas, such as redundancy of sensors and randomization of LiDAR pulse.

\section{Existing Countermeasures against Optical Attacks}

In the literature, there are only a few studies for systematically defending optical sensors of AVs, but these studies focus on other types of sensor attack. For example, the authors of~\cite{rofail19multi} claimed that the attacks against a camera could be wise enough to erase only objects from pictures or modify their positions. By using an additional LiDAR as a reference, they proposed to extract object features from images and LiDAR point cloud, and then detect the attacks via mismatches of the two sets of features. In~\cite{changalvala19lidar}, Changalvala~\textit{et al.} investigated an internal attack that can tamper a point cloud from the inside of an AV system, and they tackled the detection problem by adding a watermark to the LiDAR point cloud. Different from~\cite{rofail19multi} and~\cite{changalvala19lidar}, our work targets at defending against the optical attacks on LiDAR and cameras of AVs. We follow the idea of redundancy of sensors and design a countermeasure framework that not only can detect optical attacks via analyzing the inconsistency of depth information (i.e., disparity) from different sources, but also can identify the compromised sensors of an AV system.

\section{Depth Estimation}

Depth estimation is a task to estimate the distance to objects at every pixel using image data. As for estimating depth using images, there are two main categories of algorithms: monocular-vision based and stereo-vision based. The current methods of the two categories all adopt deep neural networks, but the monocular ones consider the task as a dense matrix regression problem and focus on minimizing the error of predictions, while the stereo-vision based algorithms formulate it as a problem of matching pixels from two images~\cite{bhoi19monocular}. As a result, DORN~\cite{fu18deep}, one of the best monocular methods, can only predict depth with an error of around $9\%$. In contrast, as a representative algorithm of the latter category, the state-of-the-art PSMNet~\cite{chang18pyramid} achieves the task with an error that is less than $2\%$ by leveraging the spatial pyramid pooling structure. In this thesis, we choose PSMNet over other methods because of its better accuracy.

\section{Adversarial Attacks}

A survey of general vulnerabilities in autonomous driving can be found in~\cite{wyglinski13security, ren19security}. Among all the vulnerabilities, the perturbation attack and the patch attack are the most dangerous threats of vision-based autonomous driving systems, since they can directly cause damages to input images.

Both the perturbation attack and the patch attack fall into the domain of adversarial attacks. The main idea of adversarial attacks is to leverage small changes in the input to trigger significant errors in the output of deep learning models. According to~\cite{yuan19adversarial}, adversarial attacks can be either universal (effective to all valid inputs) or input-specific (only effective to a specific input). There are mainly two categories of methods to achieve adversarial attacks, namely, optimization-based methods and fast gradient step method (FGSM)-based approach. Optimization-based methods can be used for either universal attacks or input-specific attacks. An example is L-BFGS proposed by Szegedy \textit{et al.}~\cite{szegedy14intriguing}. FGSM-based methods include FGSM~\cite{goodfellow15explaining} and its extensions, such as I-FGSM~\cite{kurakin17adversarial}, MI-FGSM~\cite{dong18boosting}, and PGD~\cite{madry18towards}. These methods are usually used only for input-specific attacks.

The perturbation attack and the patch attack work in different ways. The perturbation attack usually affects all pixels in an input image but the changes in pixel values are very small, while the patch attack only affects a small number of pixels but the changes in pixel value are larger. Both the attacks were studied concerning different functional modules needed in vision-based autonomous driving. For example, the perturbation attack was studied regarding sign classification in~\cite{eykholt18robust}, 2D object detection in~\cite{lu17adversarial}, semantic segmentation in~\cite{metzen17universal, xiao18characterizing}, and monocular depth estimation in~\cite{zhang20adversarial, mathew20monocular}, while the patch attack was studied regarding lane keeping in~\cite{zhou18deep}, optical flow estimation in~\cite{ranjan19attacking}, 2D object detection in~\cite{song18physical, chen18shapeshifter}, and monocular depth estimation in~\cite{mathew20monocular}. None of these studies, however, focus directly on the attacks' impact on driving behavior and driving safety of autonomous vehicles.

\section{Vision-Based 3D Object Detection}

Vision-based 3D object detection provides a more budget-friendly approach to perform object detection in 3D space by mainly leveraging stereo cameras instead of expensive LiDARs. It is the core of vision-based autonomous driving. Traditional approaches, e.g., 3DOP~\cite{chen183d} and Pseudo-LiDAR~\cite{wang19pseudo}, first generate a pseudo point cloud with depth estimation and then perform 3D object detection with similar methods used in LiDAR-based 3D object detection. As a result, they are usually not comparable to LiDAR-based methods in terms of accuracy and efficiency. Different from traditional approaches, Stereo R-CNN~\cite{li19stereo} and DSGN~\cite{chen20dsgn} are the two leading methods in this area. The network of Stereo R-CNN consists of a Region Proposal Network (RPN) and a regression part. The 2D bounding box candidates generated by the RPN are fed to the regression part where keypoints of 3D bounding boxes are predicted. The network of DSGN includes a single-stage pipeline which extracts pixel-level features for stereo matching and high-level features for object recognition. Both methods can achieve comparable performance to LiDAR-based methods.

\section{Motion Planning}

Motion planning is a key task for autonomous vehicles. Given an initial vehicle state, a goal state region, a cost function, and vehicle dynamics, motion planning finds collision-free trajectories. Currently, sampling-based motion planning algorithms are the mainstream methods~\cite{zhang20novel, zhang22integrating}. They can be viewed as a discrete planner, such as RRT~\cite{lavalle98rapidly}, greedy BFS, and A*~\cite{hart68a}, in combination with a C-space sampling scheme.

\section{Existing Autonomous Driving Systems}

\begin{figure*}[!t]
    \centering
    \includegraphics[width=0.98\textwidth]{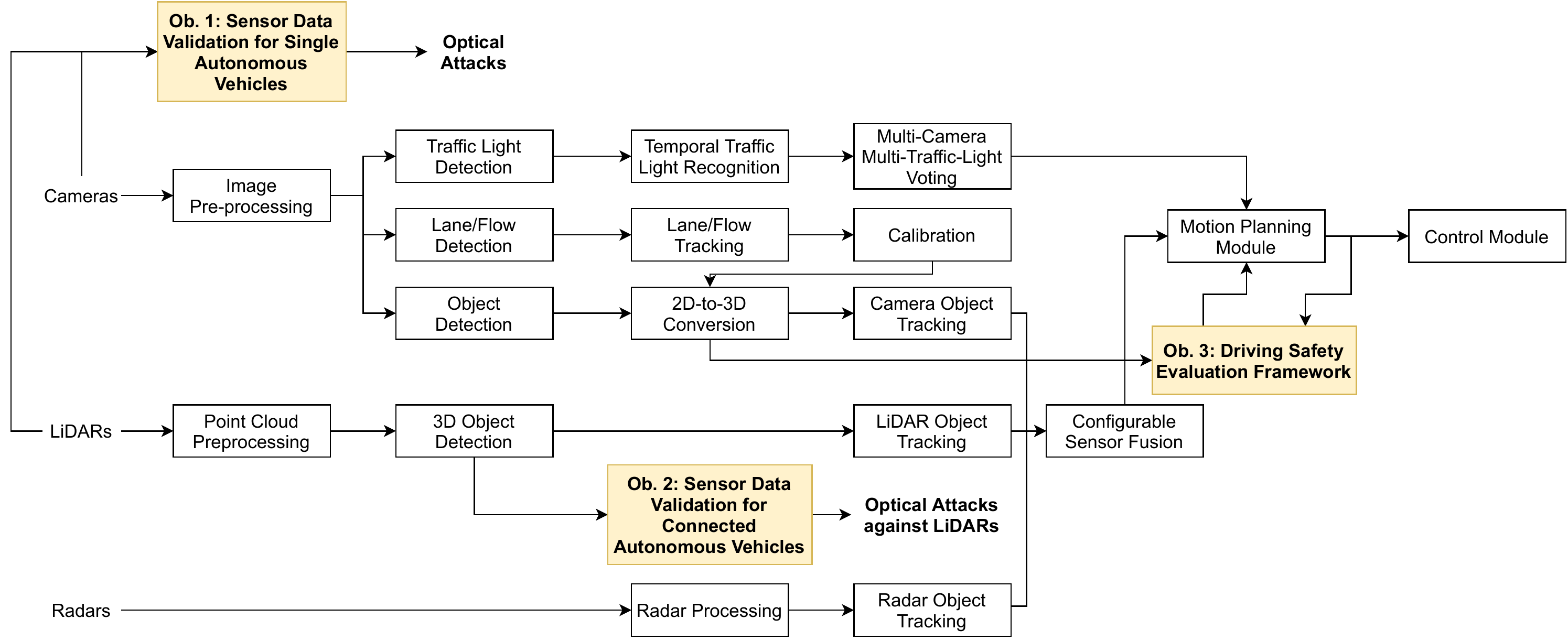}
    \caption{Take Baidu Apollo~\cite{apollo19perception} for example, this is where the three research objectives (highlighted boxes) in this thesis locate in a real autonomous driving system.}
    \label{fig:studies}
\end{figure*}

There are only a limited number of open source autonomous driving system structures. As for real autonomous driving system structures, there is only one available, Baidu Apollo~\cite{apollo19perception}. In Fig.~\ref{fig:studies}, we illustrate the system structure of Baidu Apollo. As we can see, for sensors, the Baidu Apollo vehicle is primarily equipped with three types of sensors for environment perception, cameras, LiDARs, and Radars. And the tasks of environment perception can be categorized as sign/traffic light detection, lane/flow detection, object detection, and object tracking. After the environment perception, the object tracking results are fused together, and fed to the motion planning module along with the results of sign/traffic light detection. At last, the results of motion planning is fed to the control module. Here, we also take the system structure of Baidu Apollo for example, to show where our three research objectives locate in a real autonomous driving system.

\chapter{Sensor Data Validation for Single Autonomous Vehicles}
\label{ch3:sensor}

\section{Introduction}
\label{ch3:introduction}

In the past few years, autonomous driving has attracted significant attention from both academia and industry. Recent advances in artificial intelligence and machine learning technologies enable accurate object and event detection and response (OEDR)~\cite{SAE18taxonomy}. The technology advances, together with great commercial potentials and incentives, quickly pushed the adoption of autonomous driving. For instance, Waymo launched a driverless taxi service in Arizona in 2018~\cite{davies18waymo}. Tesla has announced the beta version of their full self-driving feature of its products~\cite{hawkins21tesla}.

To facilitate accurate OEDR tasks, autonomous vehicles (AVs) are usually equipped with a number of sensors, including GPS, inertial measurement unit, radar, sonar, camera, light detection and ranging (LiDAR), etc.~\cite{apollo19perception}. Among these sensors, optical devices (LiDAR and camera) have become more and more important because they can provide object detection in a large range and also because many emerging machine learning models proposed in the past few years can accurately measure the depth of objects and detect objects. Due to the importance of these optical devices, in this chapter, we focus on their security aspects, particularly on the mitigation of potential attacks on these optical devices.

Despite the importance of the security of optical devices in an autonomous driving system, it was investigated only in a few previous studies. In~\cite{wyglinski13security,ren19security}, researchers summarized several categories of vulnerabilities in autonomous vehicles. In~\cite{petit15remote, yan16can, shin17illusion}, researchers demonstrated through experiments that LiDAR can be attacked by sending spoofed and/or delaying optical pulses. They also demonstrated that a camera can be blinded if it receives an intense light beam.

Although these pioneer studies are important, there is a lack of a comprehensive mechanism to detect and identify such attacks. In this chapter, we propose a novel framework to tackle this important issue by (1) detecting the optical attacks using data from multiple sensors and (2) identifying the sensors that are under attack. To achieve accuracy in both detection and identification, there are two major challenges:

\begin{itemize}
\item
The optical signals can be processed by many advanced machine learning models, each of which can generate various features. Moreover, optical signal attack causes different consequences on camera images and LiDAR point clouds. Therefore, an appropriate type of feature needs to be chosen as the common ground where both attacks can be detected.

\item
The size and the position of the damaged area caused by optical signal attack in images and point clouds are unpredictable. The damaged area can appear anywhere in the sensor view. Detection method must perform fine-grained detection across the whole sensor view in order to be invariant to the size and position of the damaged area and distinguish the feature differences in non-attack scenarios and attack scenarios.

\end{itemize}

To address the first challenge, the proposed framework includes an optical attack detection method that extracts depth information, which is referred to as the distance from the ego-vehicle to the surrounding environment, from two sets of sensor data respectively and then uses depth information as the common ground to detect attacks on both images and point clouds. To address the second challenge, our method detects attacks by analyzing the distribution of disparity errors that measure \textbf{pixel-level} disparity inconsistencies in the whole sensor view. Thus, the detection method is robust.

The main contributions of this study can be summarized as follows:

\begin{itemize}
\item
We develop a new technique to detect optical attacks on a system that consists of three sensors, including two possible cases (1) one LiDAR and two cameras, or (2) three cameras. Three optical sensors are the minimal number of sensors to obtain two depth maps in general sensor setups for autonomous vehicles. Specifically, we first use data from three sensors to obtain two versions of depth maps (i.e., disparity) and then detect attacks by analyzing the distribution of disparity errors. In our study, we use real datasets of KITTI~\cite{geiger12are, geiger13vision} and the state-of-the-art machine learning model PSMNet~\cite{chang18pyramid} to evaluate our attack detection scheme and the results confirm the effectiveness of our detection method.

\item
Based on the detection scheme, we further develop an identification model that is capable of identifying up to $n-2$ attacked sensors in a system with one LiDAR and $n$ cameras. In our study, we prove the correctness of our identification scheme and conduct experiments to show the accuracy of our identification method.

\item
At last, we investigate the sensitivity of our framework to optical attacks with more diverse settings. We use experiments to show its excellence in this aspect.

\end{itemize}

The rest of this chapter is organized as follows. In Section~\ref{ch3:sec.system}, we first discuss the system models, including the optical sensors and attack models, and our attack mitigation framework. In Section~\ref{ch3:sec.detection}, we elaborate on the attack detection schemes. In Section~\ref{ch3:sec.identification}, we further investigate the attack identification issue. In Section~\ref{ch3:sec.sensitivity}, we examine the overall sensitivity of our framework. Finally, we conclude this chapter in Section~\ref{ch3:sec.conclusion}.

\section{System Models}
\label{ch3:sec.system}

In this section, we first explain the main optical sensors in an autonomous driving system and their normal operations. We then elaborate on the attack models on LiDAR and camera, with some numerical results that illustrate the impacts of optical attacks. Finally, we briefly explain the main idea of the proposed framework for attack detection and identification.

\subsection{Optical Sensors}
\label{subsec:architecture}

In this chapter, we consider a general autonomous driving system, and we focus on the optical devices, in particular, LiDAR and camera.

\subsubsection{LiDAR}

A LiDAR sensor can send and receive specific optical pulses in certain directions. By comparing the incoming reflected signals with the transmitted ones, LiDAR can provide an accurate estimation of the distance between the LiDAR and an object in a specific direction. The output of LiDAR consists of a set of points in 3D space, which is known as a point cloud. By clustering these points, the object detection models applied in AV systems can locate obstacles in the real world.

\subsubsection{Camera}

Cameras are very common in existing autonomous driving systems. AVs are usually equipped with more than two of them for covering the 360-degree view of surrounding environment. The produced images are useful to several perception functions, such as obstacle detection and road/lane detection.

Specifically, similar to human eyes, two cameras can be used to form a stereo vision system that can estimate the depth of an object. As a simple example, if a real-world point is captured at pixel $P_{l}=(x_{l},y)$ in the left image and at pixel $P_{r}=(x_{r},y)$ in the right image, then the \textit{disparity} $d$ is defined as $d=x_{l}-x_{r}$. We can calculate the depth $z$ using
\begin{equation}\label{eq:triangulation}
z=f\times\frac{b}{d},
\end{equation}
where $f$ is the focal length and $b$ is the distance between the two cameras. 

In general, we can obtain the depth of a real-world point using disparity, once the pair of corresponding pixels ($P_{l}$ and $P_{r}$) are located in two images. Therefore, the main goal of depth estimation algorithms, such as PSMNet~\cite{chang18pyramid}, is to identify pairs of pixels in two images that are corresponding to the same real-world points. Finally, a \textit{disparity map} is generated by computing disparity for every pixel in an image.

\subsection{Data-Processing Pipeline of Autonomous Vehicles}

We illustrate the general data-processing pipeline for autonomous vehicles in Fig.~\ref{fig:system_model}. This pipeline is an abstract of existing autonomous driving systems, such as Baidu Apollo~\cite{apollo19perception}. As we can see first, the perception model directly takes the raw sensor data and outputs the high-level environment information, such as object detection results, self-localization results, traffic sign information, for the planning module. After receiving the information delivered by the perception module, the planning module makes driving decisions based on the received information, user intentions, and other information. Finally, the action module sends commands to vehicular hardware (gas, steer, brake, etc.) according to the decisions made by the planning module. Therefore, if the optical attacks cause wrong measurements in raw data of optical sensors, errors may propagate along this pipeline and cause risks to the safety of autonomous vehicles.

\begin{figure}[!t]
    \centering
    \includegraphics[width=0.6\textwidth]{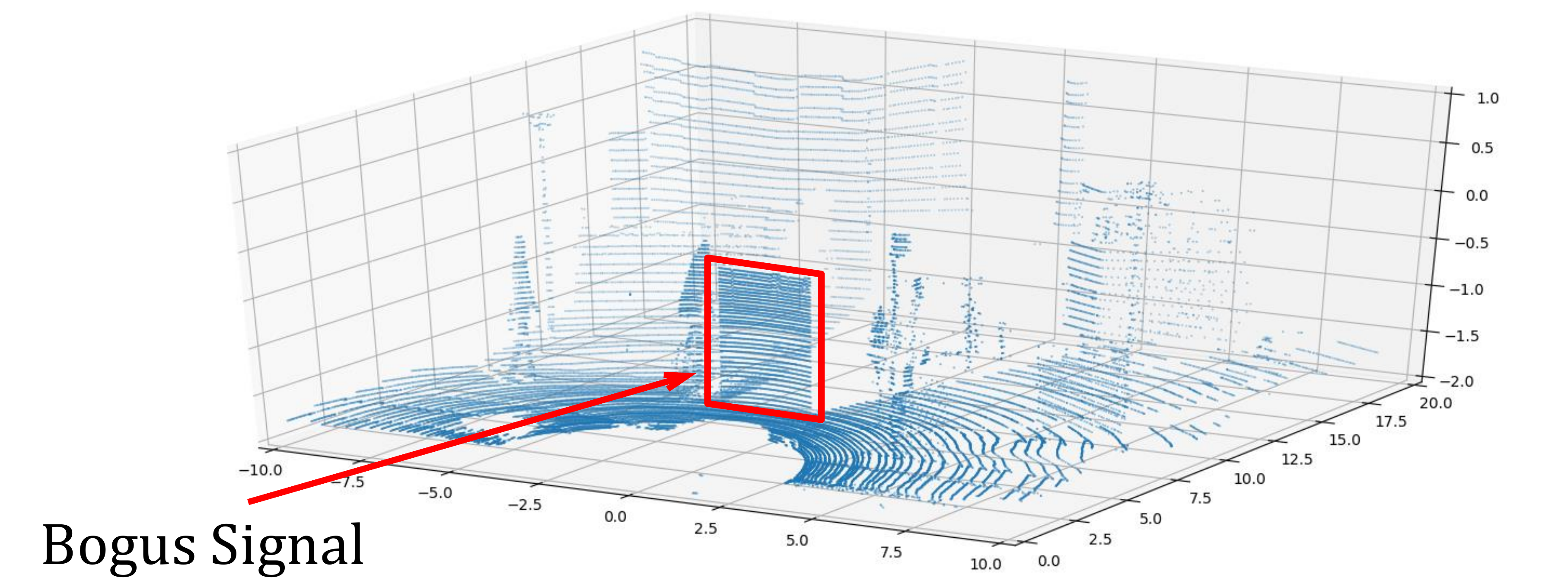}
    \caption{An example of a compromised point cloud that contains bogus signals in a region.}
    \label{fig:lidar_attack}
\end{figure}

\begin{figure}[!t]
    \centering
    \includegraphics[width=0.6\textwidth]{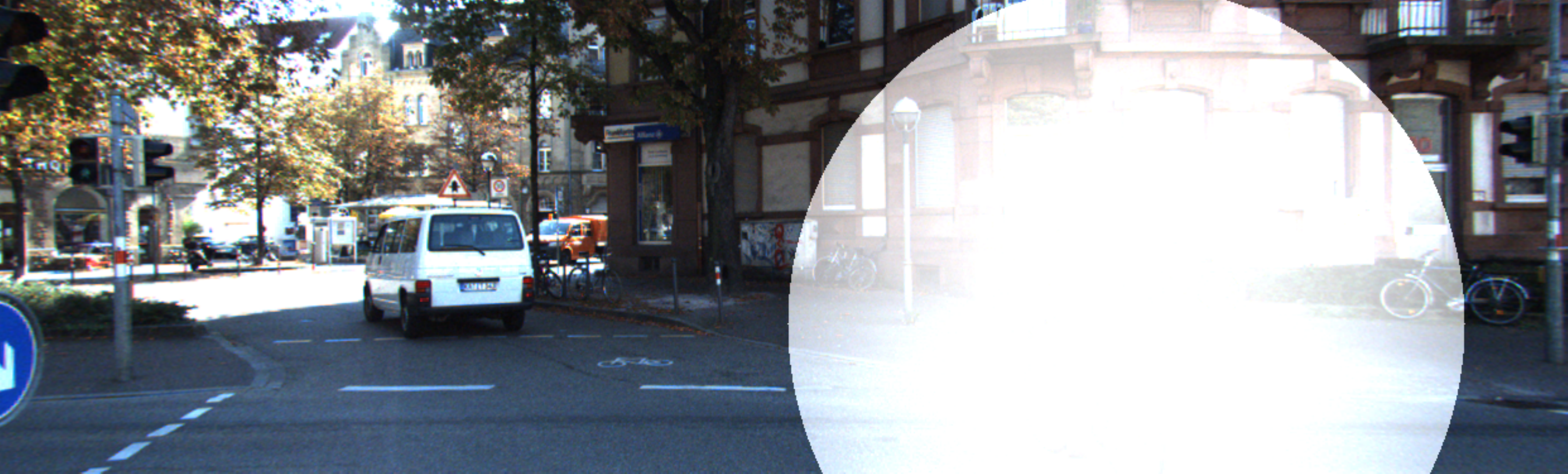}
    \caption{An example of a compromised camera image that contains a round Gaussian facula.}
    \label{fig:camera_attack}
\end{figure}

\subsection{Attacks on Optical Sensors}

\subsubsection{Attacks against LiDAR}

In~\cite{petit15remote} and~\cite{shin17illusion}, the authors demonstrated several methods to attack LiDARs. The main idea in these attacks is to generate or relay the legitimate optical pulses so as to mislead the perception module in the system model.

To attack LiDARs, attackers must know in advance the laser frequency and the time sequence of the victim LiDAR, and use photodiode, signal generator, and a laser at the same frequency as used by the victim LiDAR to launch the attack. Although the existing LiDAR spoofing attacks can only generate a limited number of fake points, we believe that a more powerful attacker can generate a larger number of spoofed points in the point cloud with several advanced attacking sources. Therefore, in this chapter, we produce the compromised point clouds by generating spoofed signals for a region, so that the perception module of AVs may detect a fake object, as shown in Fig.~\ref{fig:lidar_attack}.

\subsubsection{Attacks against Camera}

The attacks against camera have been studied in~\cite{petit15remote} and~\cite{yan16can}. The main idea in these studies is to generate strong light signals so as to blind the cameras. According to~\cite{petit15remote}, to blind a camera, the power of light source must increase exponentially with the growth of the distance between the light source and the camera. The effectiveness of the attacks is also affected by the environment light conditions. Therefore, when LED is used, in order to form effective attacks, the distance between the light source and the camera must be within a few meters, and the attacks must be conducted in dark environments, which is less practical. By comparison, attacks using lasers seem to be more realistic.

In our study, we believe that the attackers do not need to completely blind the camera. Instead, their main objective is to mislead the perception module in the autonomous driving system. To this end, we consider that the attacking light source is a laser and the distance between the attacking source and the cameras can be sufficiently large. As a result, the attacks from a laser result in a contaminated area with certain size at a random position in images. Therefore, we generate the affected camera data by overlaying a Gaussian facula on them, as shown in Fig.~\ref{fig:camera_attack}. The affected images we generate are equivalent to the results in~\cite{petit15remote} and~\cite{yan16can}.

\subsection{Impact of the Optical Attacks}

\begin{figure}[!t]
    \centering
    \includegraphics[width=0.6\textwidth]{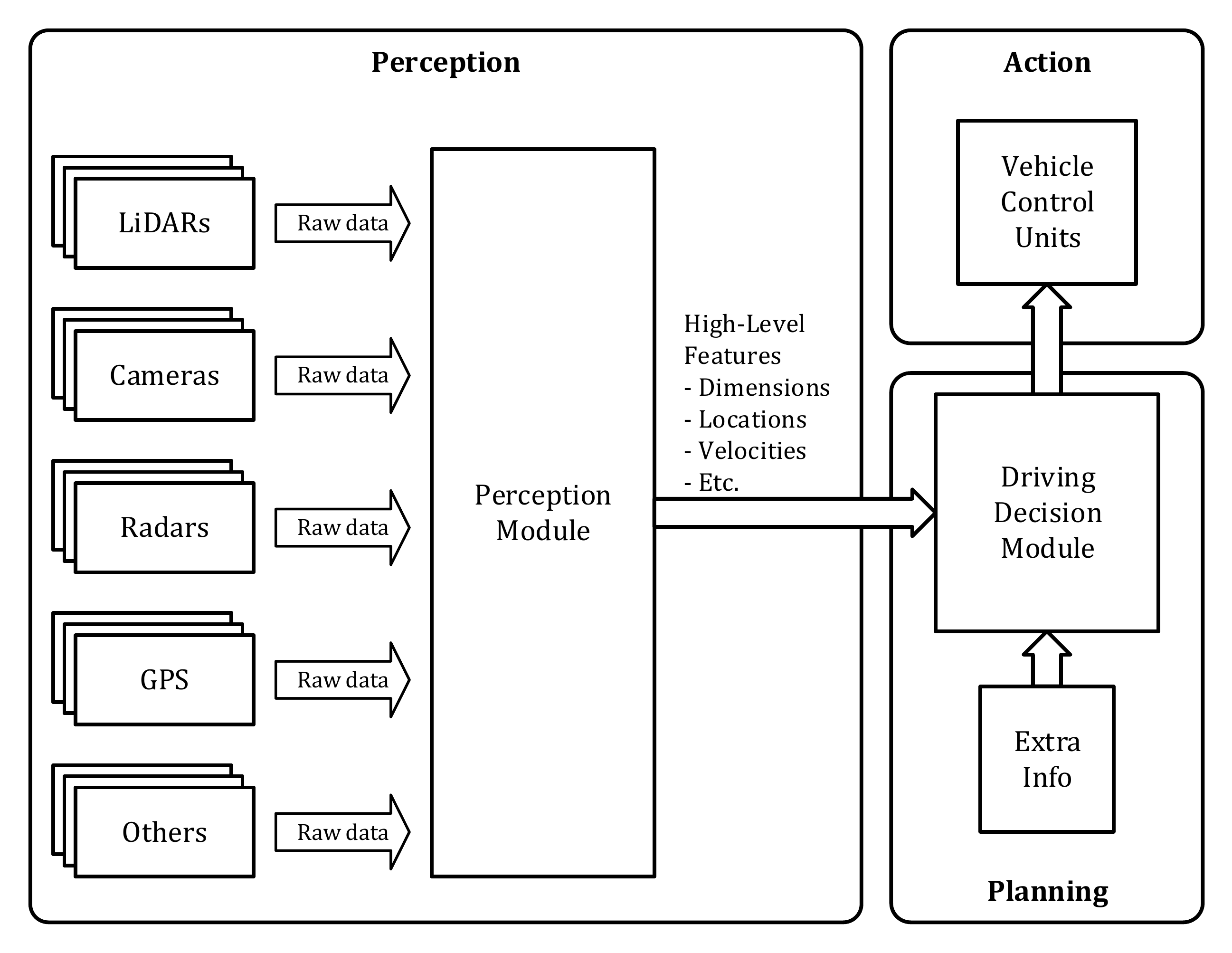}
    \caption{In AV systems, the perception module processes the raw sensor data and generates the environmental high-level information for driving decision making. Then, the driving commands are sent to control units.}
    \label{fig:system_model}
\end{figure}

To understand the impacts of the aforementioned optical attacks, we conduct extensive experiments testing the object detection algorithms for AVs with the compromised sensor data. Next, we briefly introduce the common experimental setup in this chapter, which is also used in the experiments of other sections.

\subsubsection{Common Experimental Setup}
\label{subsubsec:setup}

In this chapter, we use two datasets. The first one is the KITTI raw dataset~\cite{geiger13vision}, which includes data of one LiDAR and four cameras in different environmental conditions for autonomous driving, such as \textbf{City}, \textbf{Residential}, \textbf{Road}, \textbf{Campus}, etc. We customize it by selecting 1000 sets of sensor data.

The second dataset we use is the one provided in the KITTI object detection benchmark~\cite{geiger12are}, which contains sensor data of one LiDAR and two cameras. We divide the labeled part of the dataset into a training set and a test set according to~\cite{chen153d}. The two sets include 3712 sets of sensor data and 3769 sets of sensor data, respectively.

To produce the compromised LiDAR data, we generate a bogus signal with a height of $1.5$ meters and a width of $2.5$ meters, which is equal to the width of typical highway lanes, at a random distance of $6$ to $10$ meters away from the LiDAR sensor in point clouds.

To generate the compromised camera images, we overlay a round Gaussian facula with a random radius of $187$ pixels to $375$ pixels on images that have a size of $1242$ pixels by $375$ pixels.

\subsubsection{Attack Experiments and Results}

Here, we first conduct experiments on the customized KITTI raw dataset~\cite{geiger13vision}. To evaluate the impact of the optical attacks on LiDAR, we use a pre-trained model based on PIXOR~\cite{yang18pixor}, which is a 3D object detection method using LiDAR data. In our experiments, we generate a compromised point cloud for each one of the 1000 sets of sensor data and feed it to the PIXOR model. We observe that the model falsely considers the bogus signals as obstacles in 986 cases out of 1000.

To measure the impact of the attacks on camera, we use the standard performance metric, \textbf{average precision} (AP), where the prediction is considered accurate if and only if the \textbf{Intersection over Union} (IoU) is larger than $50\%$. In addition, we use a pre-trained model provided in TensorFlow API~\cite{huang17speed} that can detect vehicles from images. Specifically, the model is based on Faster R-CNN~\cite{ren15faster} with a ResNet-101 architecture~\cite{he16deep}. In our experiments, we produce a compromised image for each set of sensor data in the dataset and feed it to the Faster R-CNN model. The numerical results show that the AP for detecting cars is $84.62\%$ when there are no attacks. By comparison, the AP drops dramatically to $61.53\%$ when there are optical attacks against the camera.

To briefly summarize, we observe that the aforementioned attacks on optical devices can significantly compromise the accuracy of object detection, which is a fundamental task of perception module in autonomous driving. As shown in Fig.~\ref{fig:system_model}, the results of environment perception are passed to the driving decision module that directly sends commands to the vehicle control units, such as the engine and brake. Therefore, we believe that optical attacks are extremely hazardous because it is highly possible that an inaccurate perception due to optical attacks can lead to wrong driving decisions and can cause catastrophic outcomes.

\subsection{A Mitigation Framework Against Optical Attacks}

To defend against such attacks, in this chapter, we propose a framework to mitigate optical attacks. The main idea of our framework is to detect optical attacks and then identify the affected sensors. In this manner, the autonomous driving system can choose to use signals from sensors that are not under attack to perform accurate perception.

Specifically, our framework consists of two main procedures. The first procedure is for attack detection. To this end, we consider a system that consists of three sensors in two scenarios, (1) one LiDAR and two cameras, and (2) three cameras. In both cases, we use data from three sensors to obtain two versions of disparity maps and then detect attacks by analyzing the distribution of disparity errors. Based on the first procedure, we design the second procedure to identify up to $n-2$ affected sensors in a system that consists of one LiDAR and $n$ cameras. In Section~\ref{ch3:sec.detection} and Section~\ref{ch3:sec.identification}, we introduce the two procedures in more detail.

\begin{figure*}[!t]
    \centering
    \includegraphics[width=0.98\textwidth]{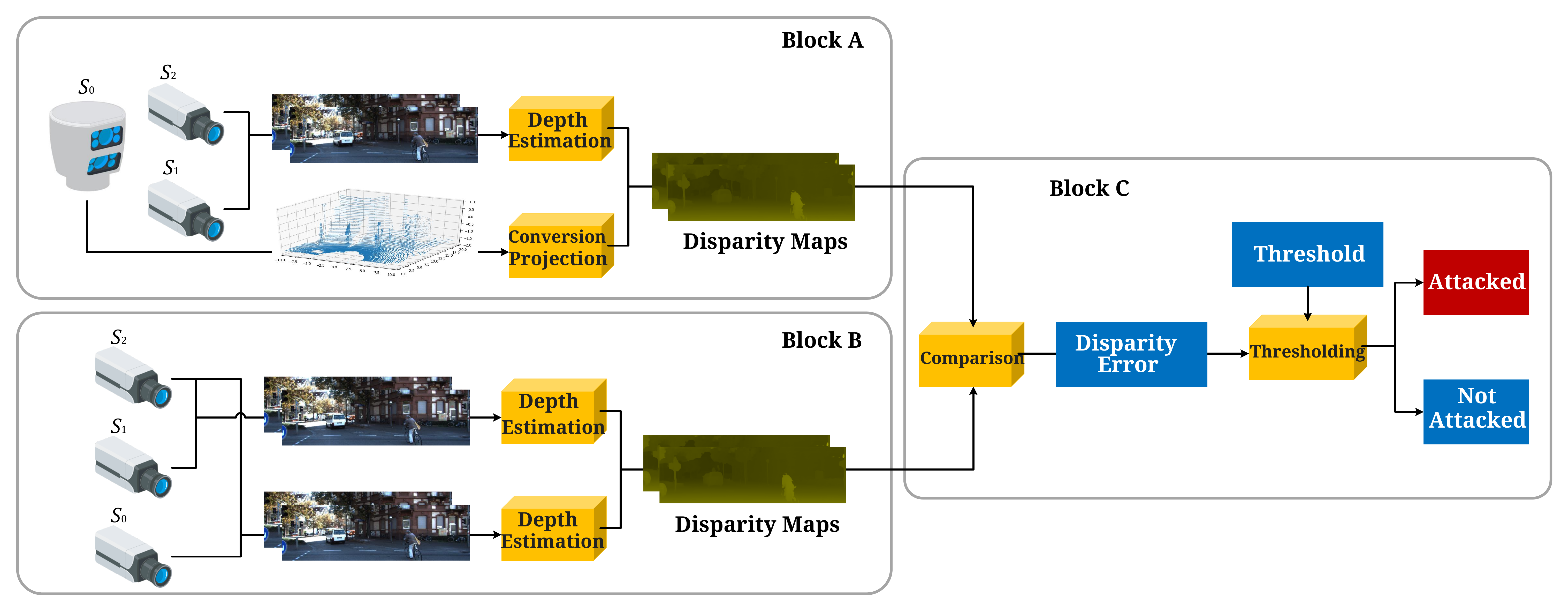}
    \caption{The detection method is designed for two three-sensor systems. For Scenario 1, the detection method structure involves Block A and Block C. For Scenario 2, the detection method structure involves Block B and Block C.}
    \label{fig:detection}
\end{figure*}

\section{Attack Detection}
\label{ch3:sec.detection}

In this section, we first explain why we target a system that consists of three sensors and then make a hypothesis about the feasibility of the detection task in a three-sensor system. Next, we focus on the disparity error and how to detect attacks by analyzing the disparity error distribution. Specifically, we elaborate on the calculation of disparity error for two main scenarios of the three-sensor system and conduct extensive experiments on the real dataset to prove the hypothesis and show the accuracy of our method.

\subsection{Three-Sensor Systems and a Hypothesis}

For attack detection, we aim to detect attacks with sufficient accuracy using the \textbf{smallest} number of sensors. Due to the trade-off between cost and performance of object detection, usually, there is one primary LiDAR mounted on the roof of an autonomous vehicle which is also equipped with multiple cameras~\cite{apollo19perception}. To obtain two versions of a depth map from an AV system like this, we at least need three sensors: one LiDAR and two cameras, or three cameras.

For the first case, we notice that LiDAR can produce accurate depth maps in point clouds. On the other hand, stereo-vision based depth estimation algorithms can also generate depth maps out of stereo images. Intuitively, if we compare a depth map produced by LiDAR and another generated by two stereo images, we may be able to detect the distortion of depth information caused by optical attacks on such a three-sensor system. Consequently, the first three-sensor system that we consider consists of one LiDAR and two cameras.

For the second case, it is obvious that we can use the image taken by one camera as the reference, and then use images taken by two other cameras to produce two depth maps using a depth estimation model. By comparing the two depth maps, we may be able to detect attacks on the second three-sensor system that consists of three cameras.

To briefly summarize, in this chapter, we consider two three-sensor systems that are practical in autonomous driving systems. Furthermore, we make a \textbf{hypothesis} that, with appropriate design, we can accurately detect the optical attacks on both of the three-sensor systems. In the following discussions, we verify this hypothesis by elaborating on the mechanisms to detect attacks on the two three-sensor systems, respectively.

\subsection{Scenario 1: One LiDAR and Two Cameras}

For this scenario, we denote the LiDAR as sensor $S_{0}$, and denote two cameras from the right to the left as $S_{1}$ and $S_{2}$, respectively. Accordingly, the data generated by the sensors are denoted as $D_{0}$, $D_{1}$, and $D_{2}$. The detection system we design for this scenario is shown as the combination of Block A and Block C in Fig.~\ref{fig:detection}.

In our system illustrated in Fig.~\ref{fig:detection} (Block A \& Block C), we designate camera $S_2$ as the reference camera to generate two disparity maps. Specifically, we set the image taken by the reference camera (i.e., $D_{2}$) as the reference image, and then feed it with the image taken by the other camera (i.e., $D_{1}$) to a depth estimation model to produce the first disparity map, denoted as $DM_{1,2}$, in which we include the disparity information at each pixel of the reference image.

Here, we note that many algorithms have been developed in recent years that can generate accurate disparity maps with camera images. For instance, PSMNet~\cite{chang18pyramid} and DORN~\cite{fu18deep} are two recent algorithms based on deep learning. In this chapter, we use the former one to produce disparity maps, since the PSMNet model gives more accurate results over others.

In addition to $DM_{1,2}$, we also project the depth information (i.e., point cloud $D_{0}$) obtained by LiDAR onto the reference image $D_{2}$ to generate the second disparity map, denoted as $DM_{0,2}$. In this procedure, the depth information in the point cloud $D_{0}$ is converted to disparities by using Eqn.~(\ref{eq:triangulation}). To generate all disparity maps in the same scale, $f$ in the equation is set to be the same as the focal length of the cameras, and $b$ is set to be equal to the baseline of the stereo vision formed by $S_{1}$ and $S_{2}$. Then, disparities are projected onto $D_{2}$.

\begin{figure*}[t]
	\centering
	\subfloat[]{
		\label{fig:s0_1_d}
		\includegraphics[width=0.23\textwidth]{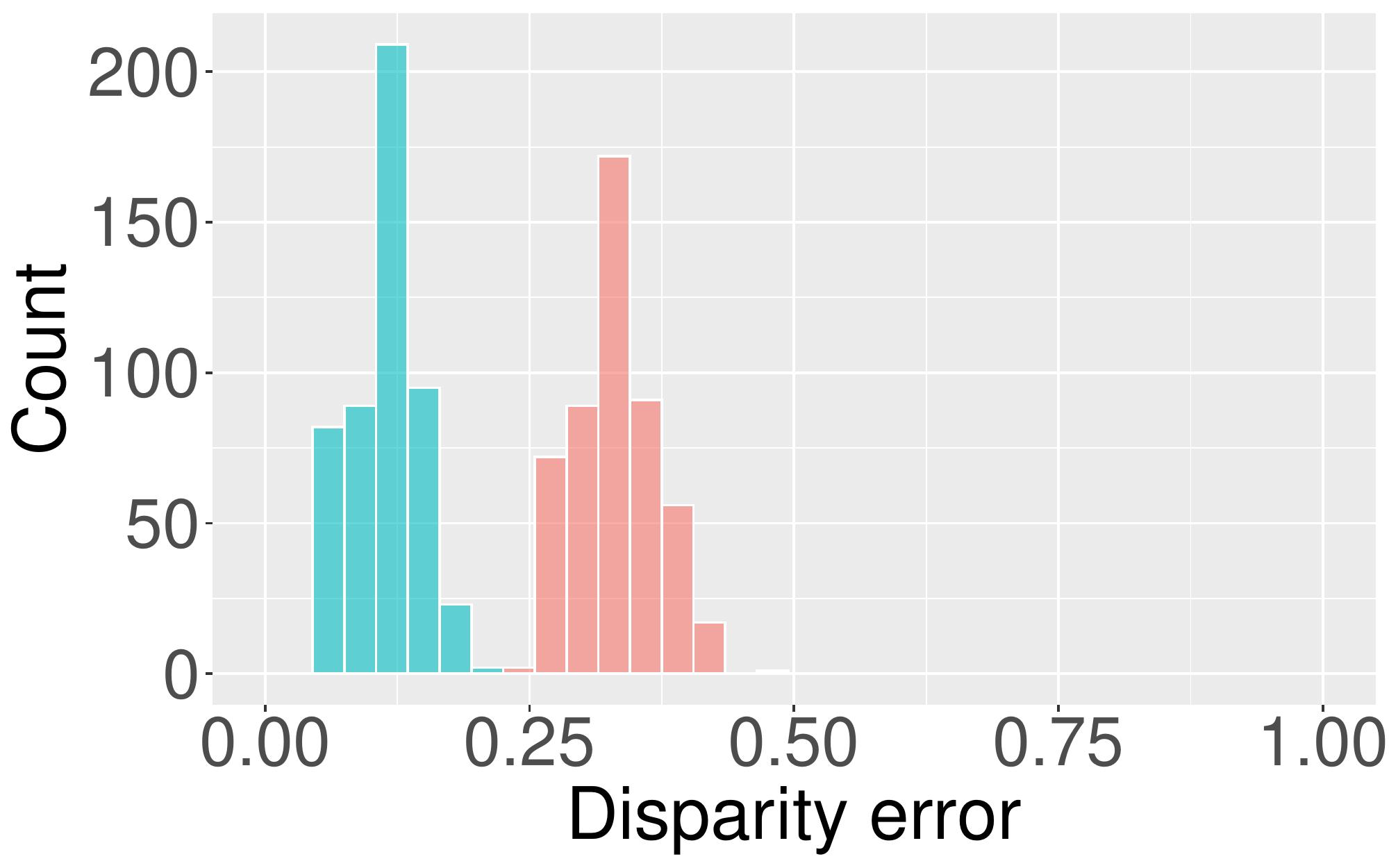}}
	\hfil
	\subfloat[]{
		\label{fig:s1_1_d}
		\includegraphics[width=0.23\textwidth]{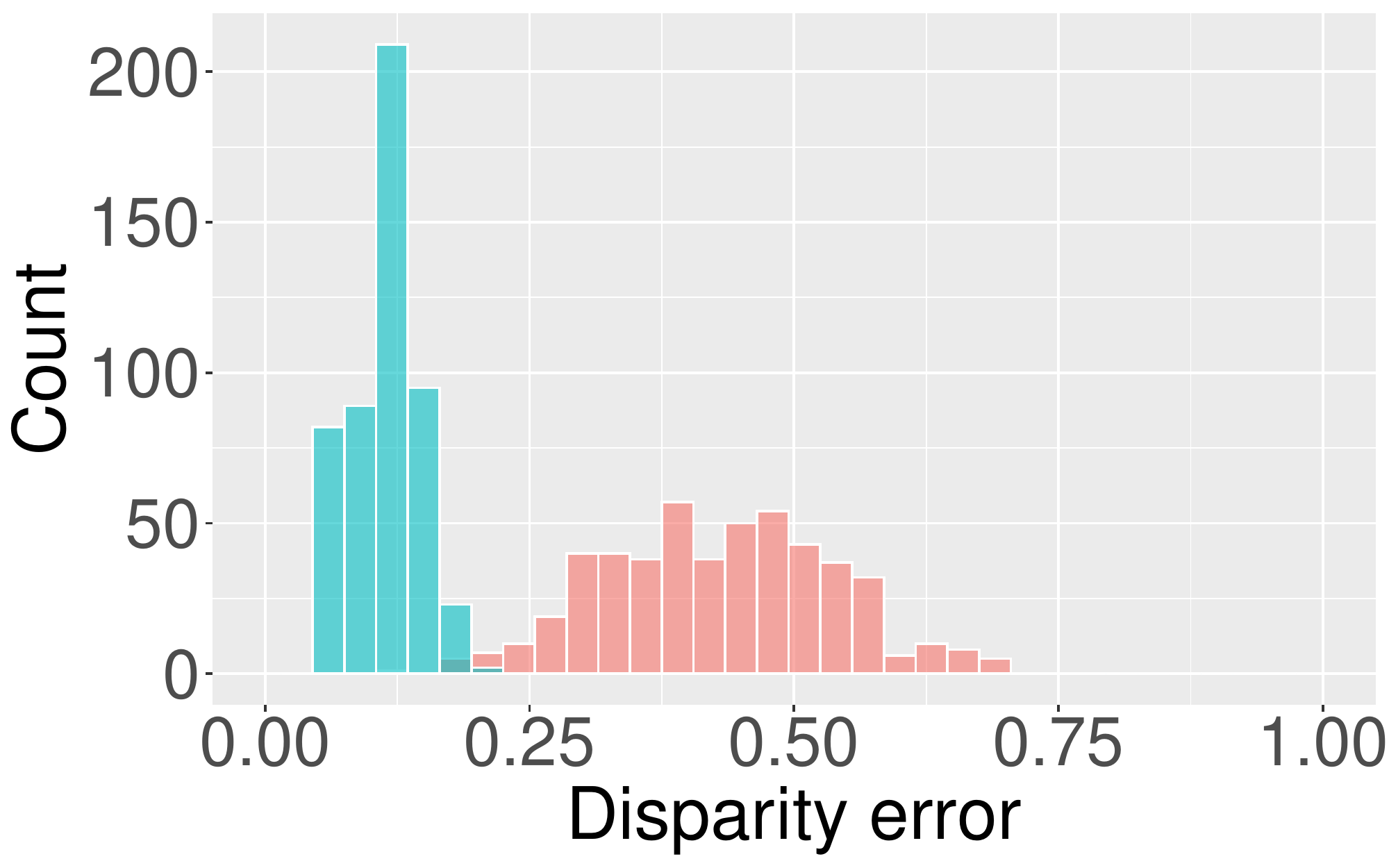}}
	\hfil
	\subfloat[]{
		\label{fig:s2_1_d}
		\includegraphics[width=0.23\textwidth]{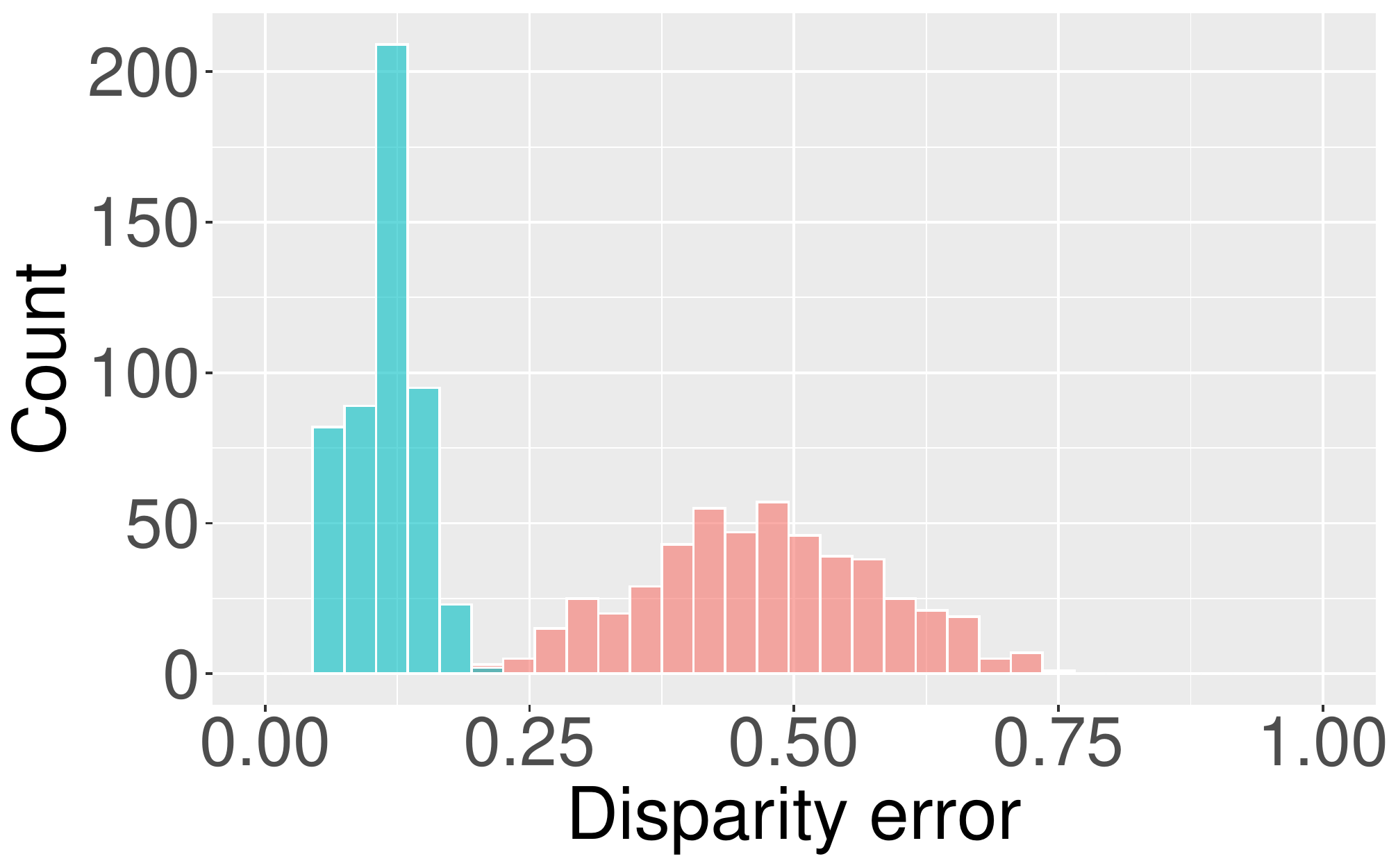}}
	\hfil
	\subfloat[]{
		\label{fig:s0_s1_1_d}
		\includegraphics[width=0.23\textwidth]{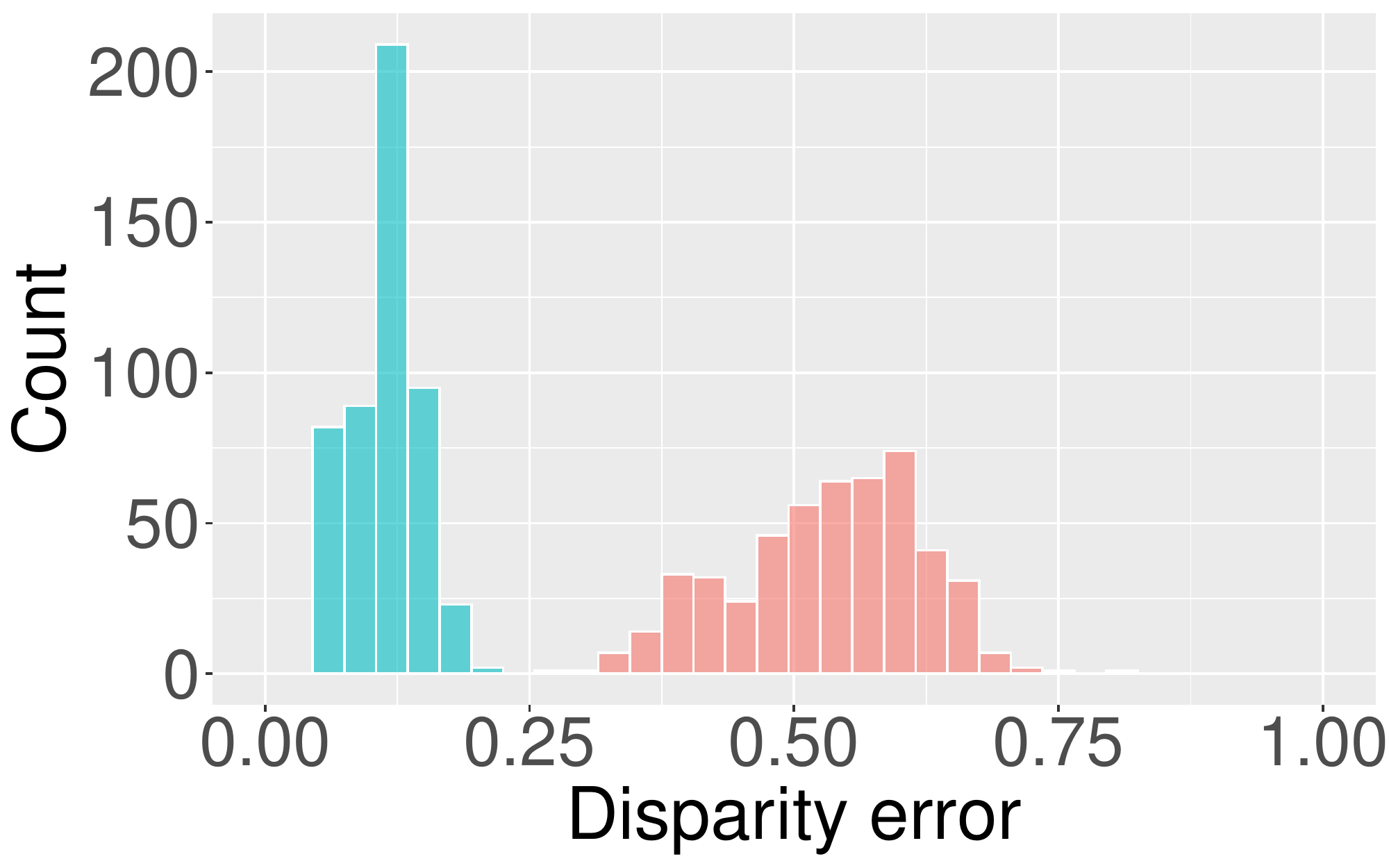}}
	\hfil

	\subfloat[]{
		\label{fig:s0_s2_1_d}
		\includegraphics[width=0.23\textwidth]{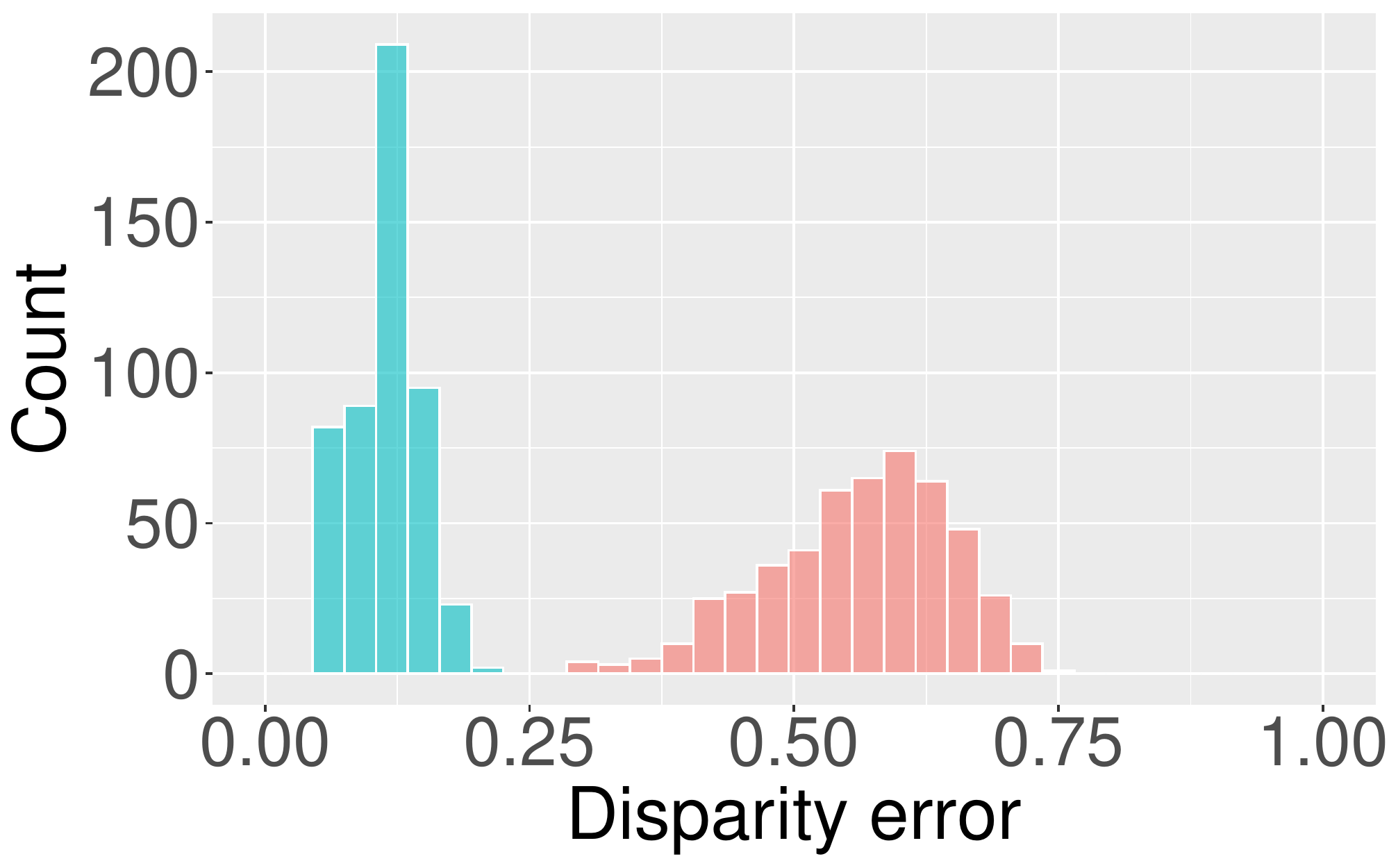}}
	\hfil
	\subfloat[]{
		\label{fig:s1_s2_1_d}
		\includegraphics[width=0.23\textwidth]{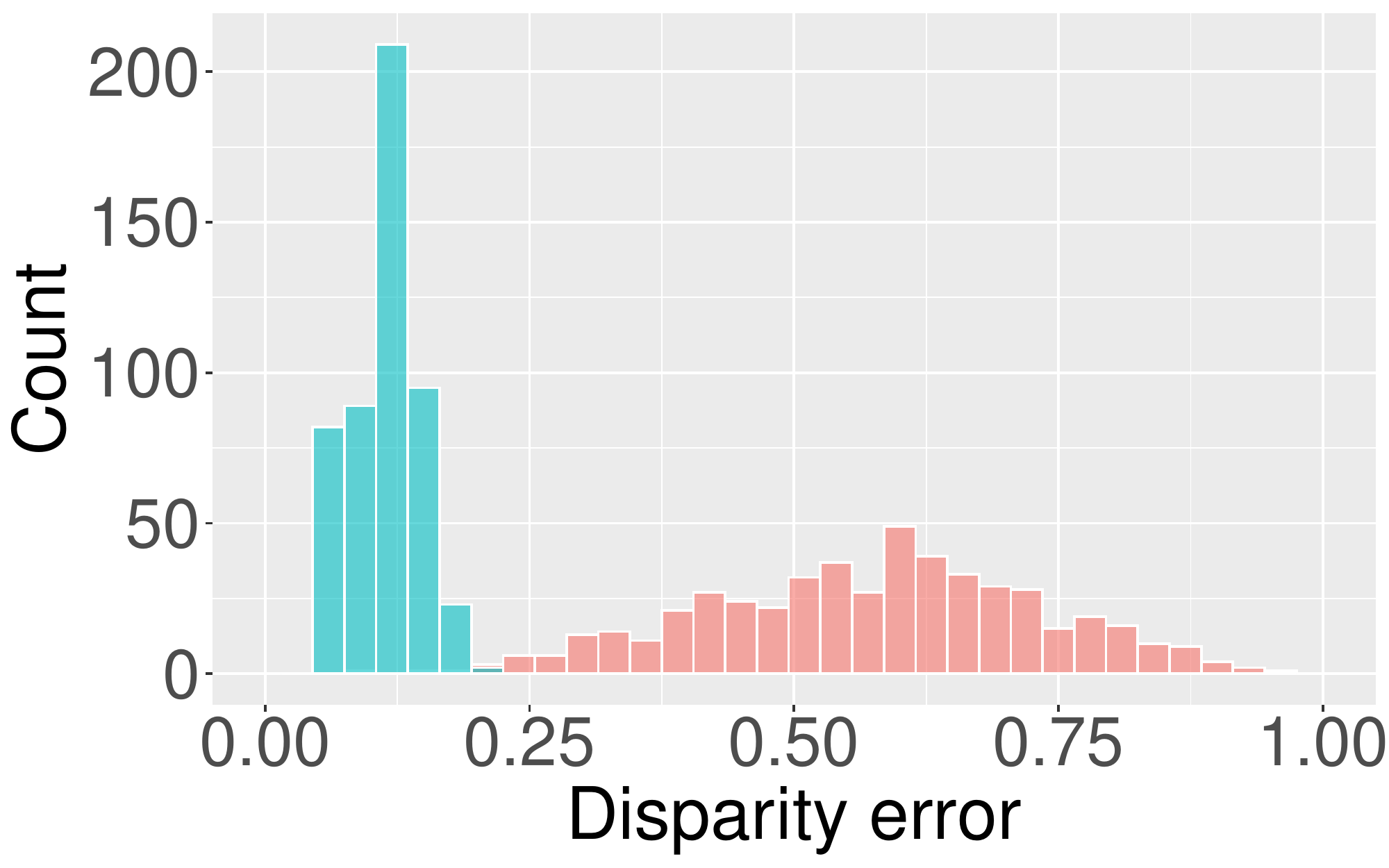}}
	\hfil
	\subfloat[]{
		\label{fig:s0_s1_s2_1_d}
		\includegraphics[width=0.23\textwidth]{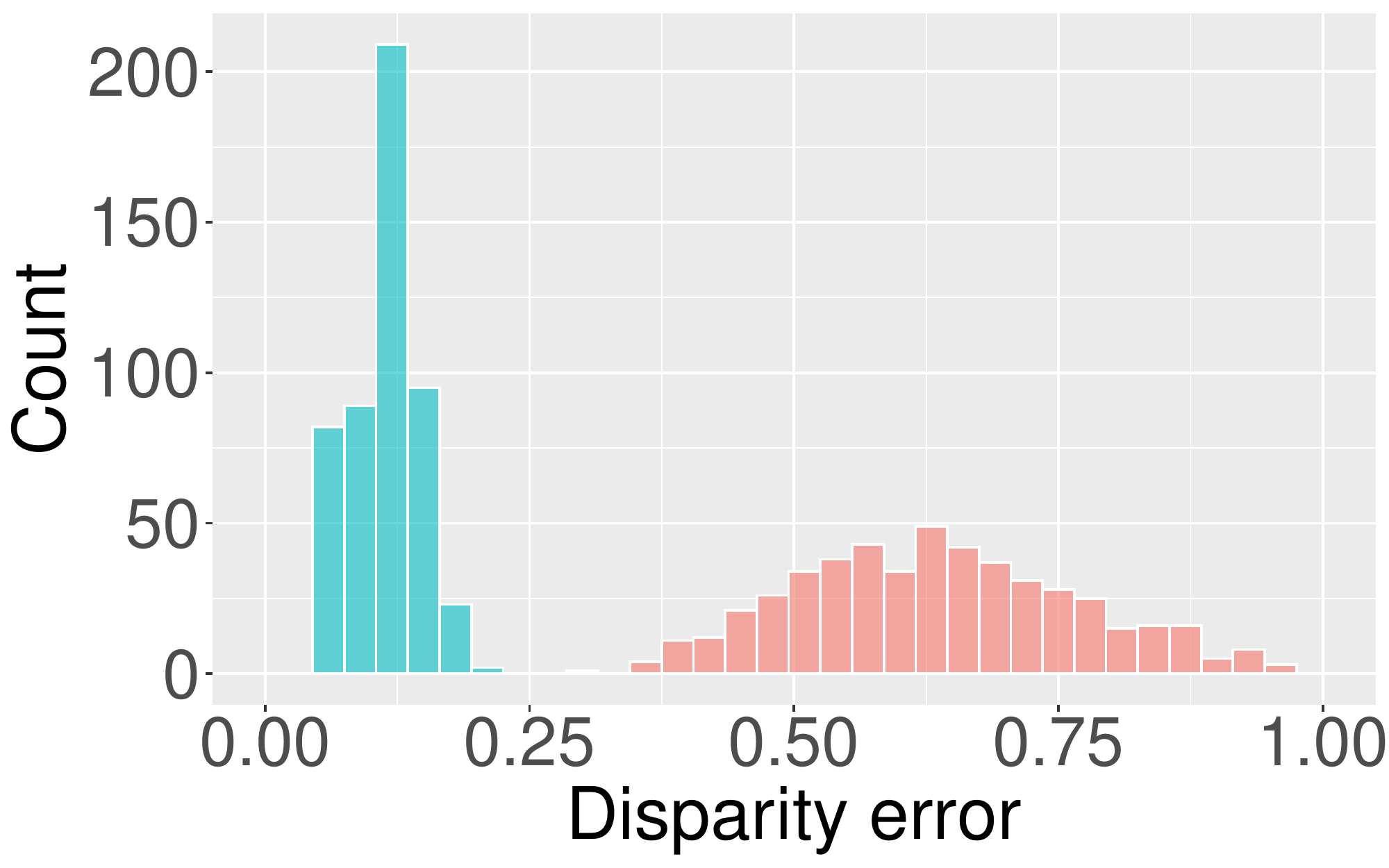}}
	\caption{Distributions of disparity error $E_{0,1,2}$ in normal case (cyan bars) and in attack cases (red bars) for Scenario 1. (a)~No attack vs. $S_{0}$ attacked; (b)~No attack vs. $S_{1}$ attacked; (c)~No attack vs. $S_{2}$ attacked; (d)~No attack vs. $S_{0},S_{1}$ attacked; (e)~No attack vs. $S_{0},S_{2}$ attacked; (f)~No attack vs. $S_{1},S_{2}$ attacked; (g)~No attack vs. $S_{0},S_{1},S_{2}$ attacked.}
	\label{fig:scen_1_exp_d}

\end{figure*}

\begin{figure*}
	\centering
	\subfloat[]{
		\label{fig:s0_1_r}
		\includegraphics[width=0.23\textwidth]{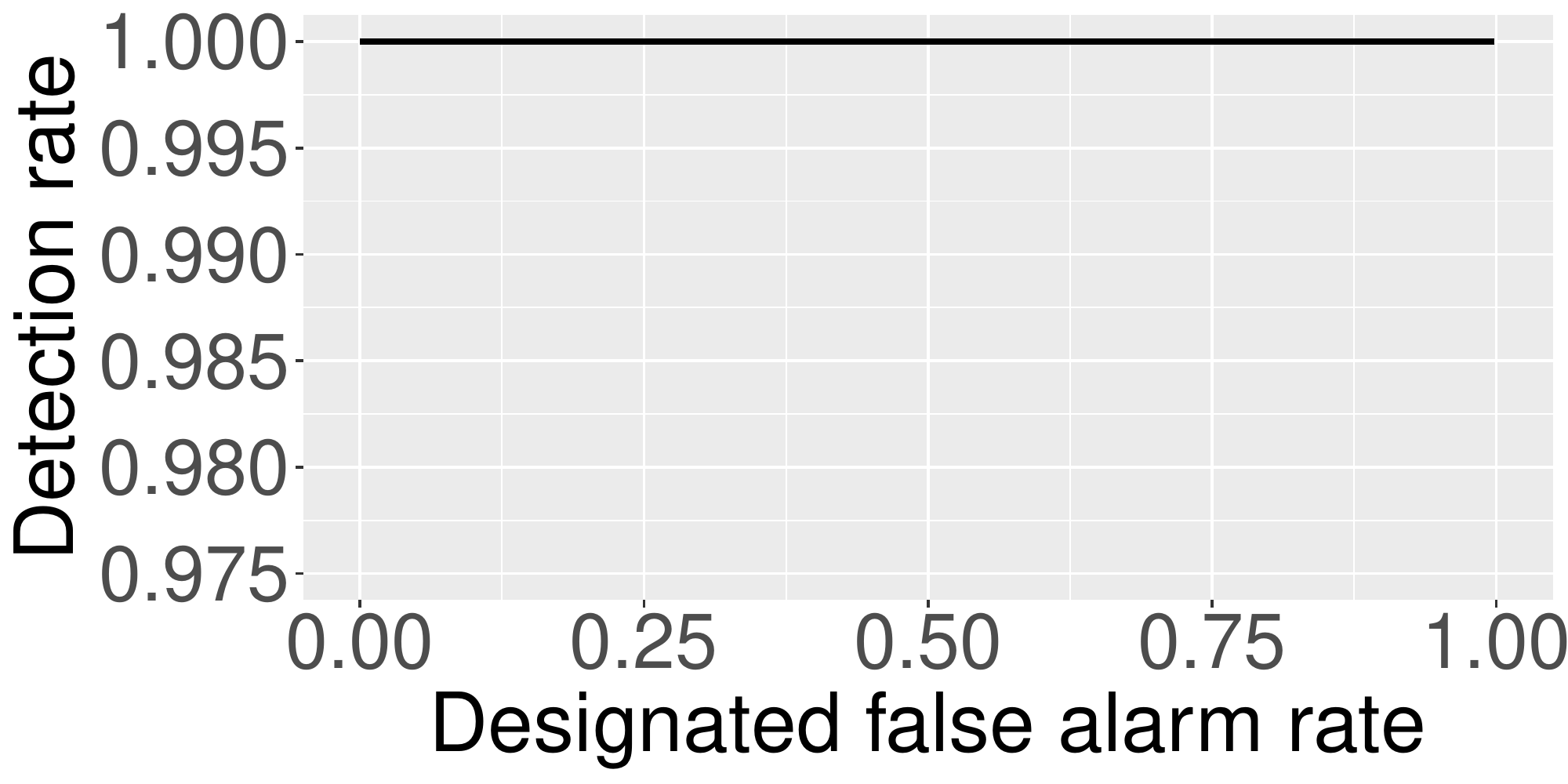}}
	\hfil
	\subfloat[]{
		\label{fig:s1_1_r}
		\includegraphics[width=0.23\textwidth]{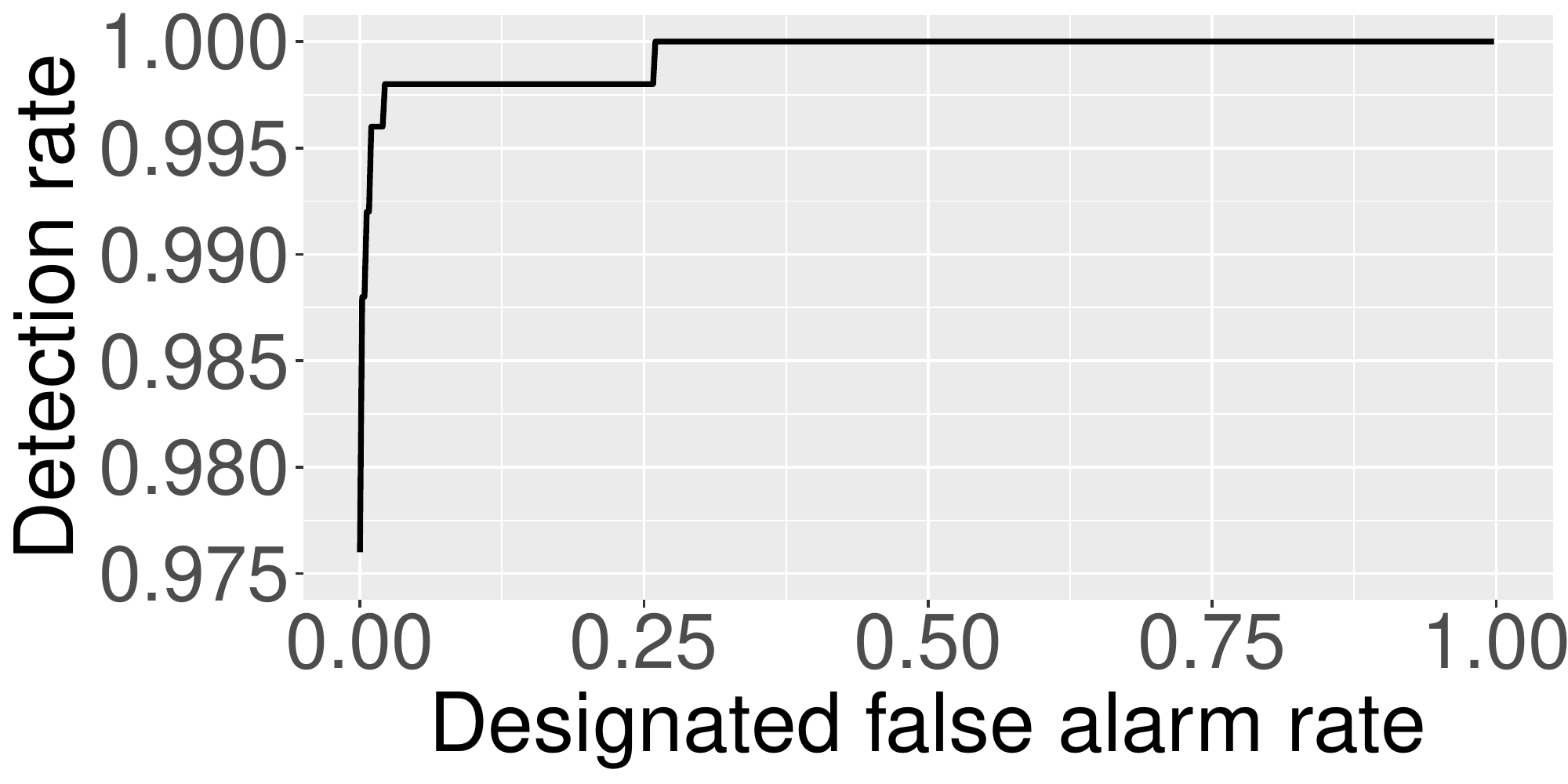}}
	\hfil
	\subfloat[]{
		\label{fig:s2_1_r}
		\includegraphics[width=0.23\textwidth]{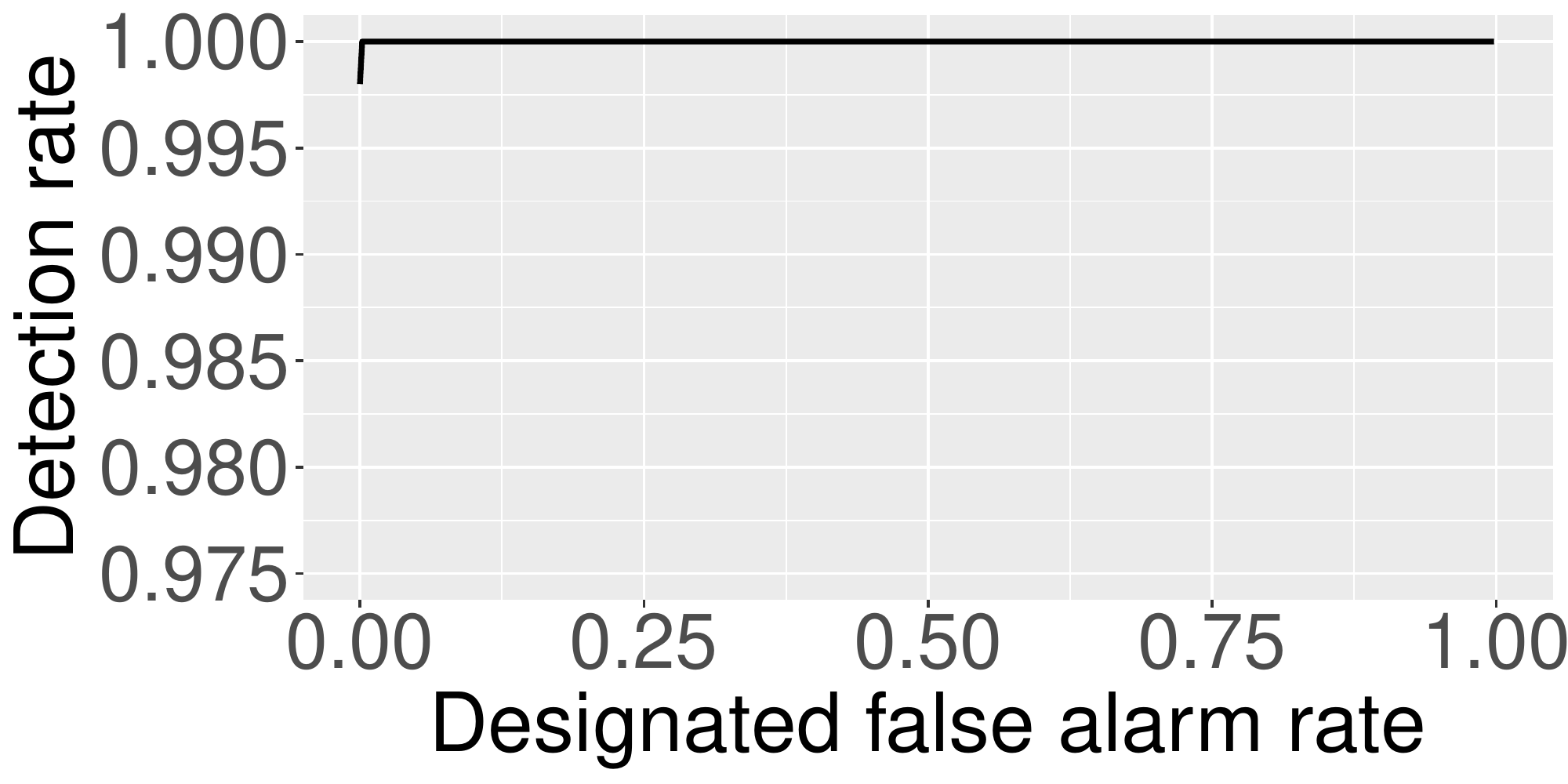}}
	\hfil
	\subfloat[]{
		\label{fig:s0_s1_1_r}
		\includegraphics[width=0.23\textwidth]{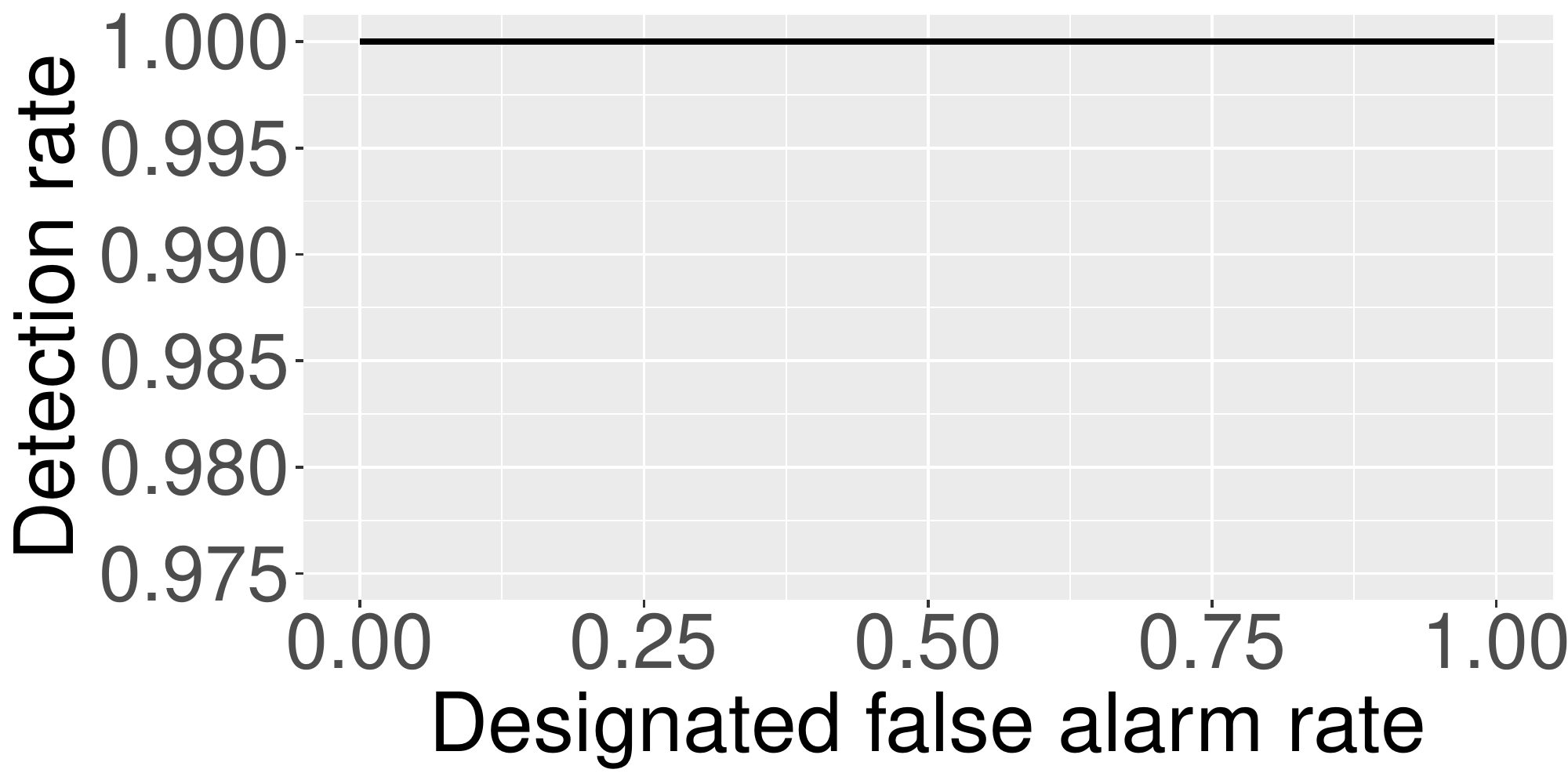}}
	\hfil

	\subfloat[]{
		\label{fig:s0_s2_1_r}
		\includegraphics[width=0.23\textwidth]{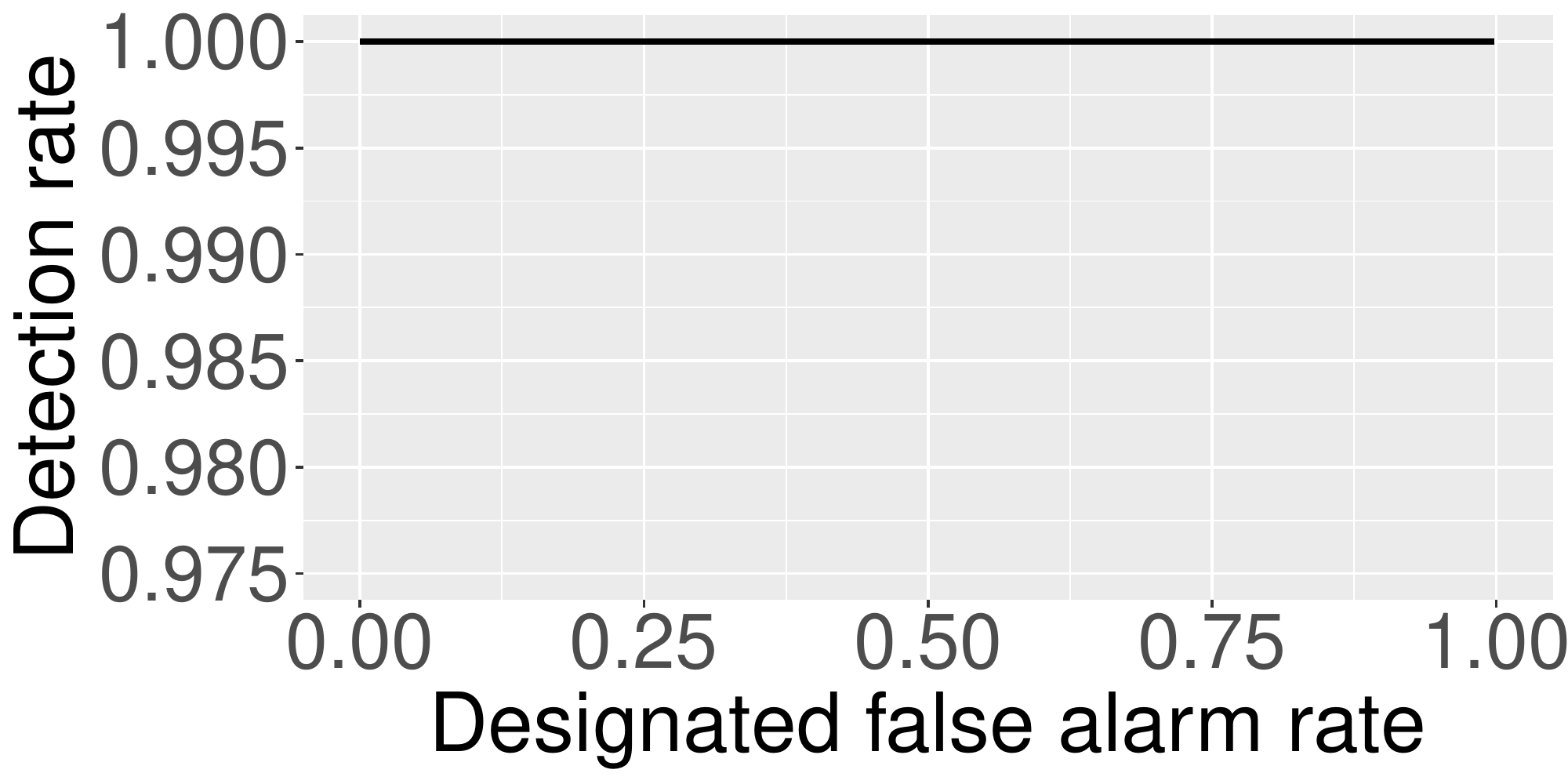}}
	\hfil
	\subfloat[]{
		\label{fig:s1_s2_1_r}
		\includegraphics[width=0.23\textwidth]{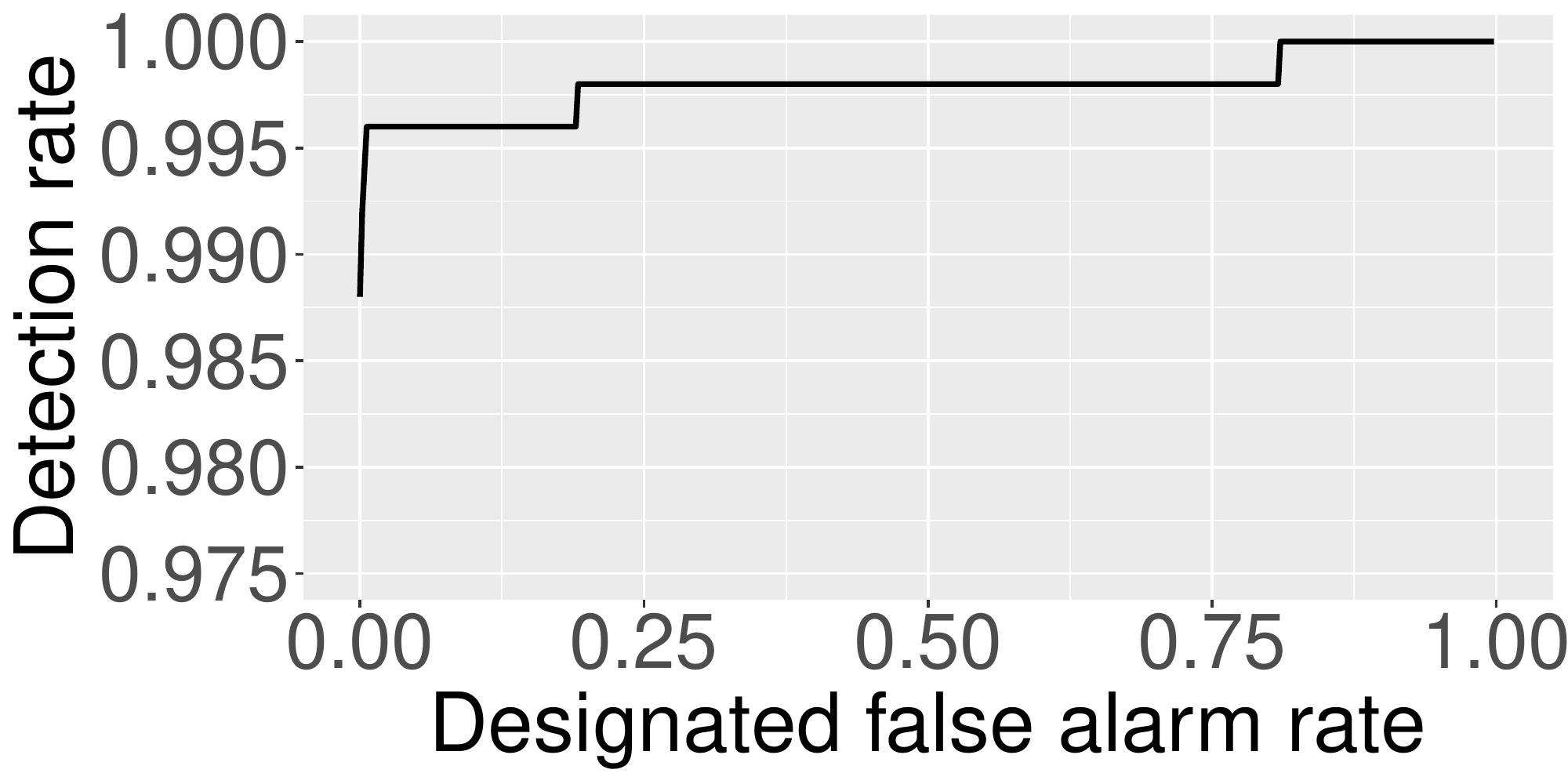}}
	\hfil
	\subfloat[]{
		\label{fig:s0_s1_s2_1_r}
		\includegraphics[width=0.23\textwidth]{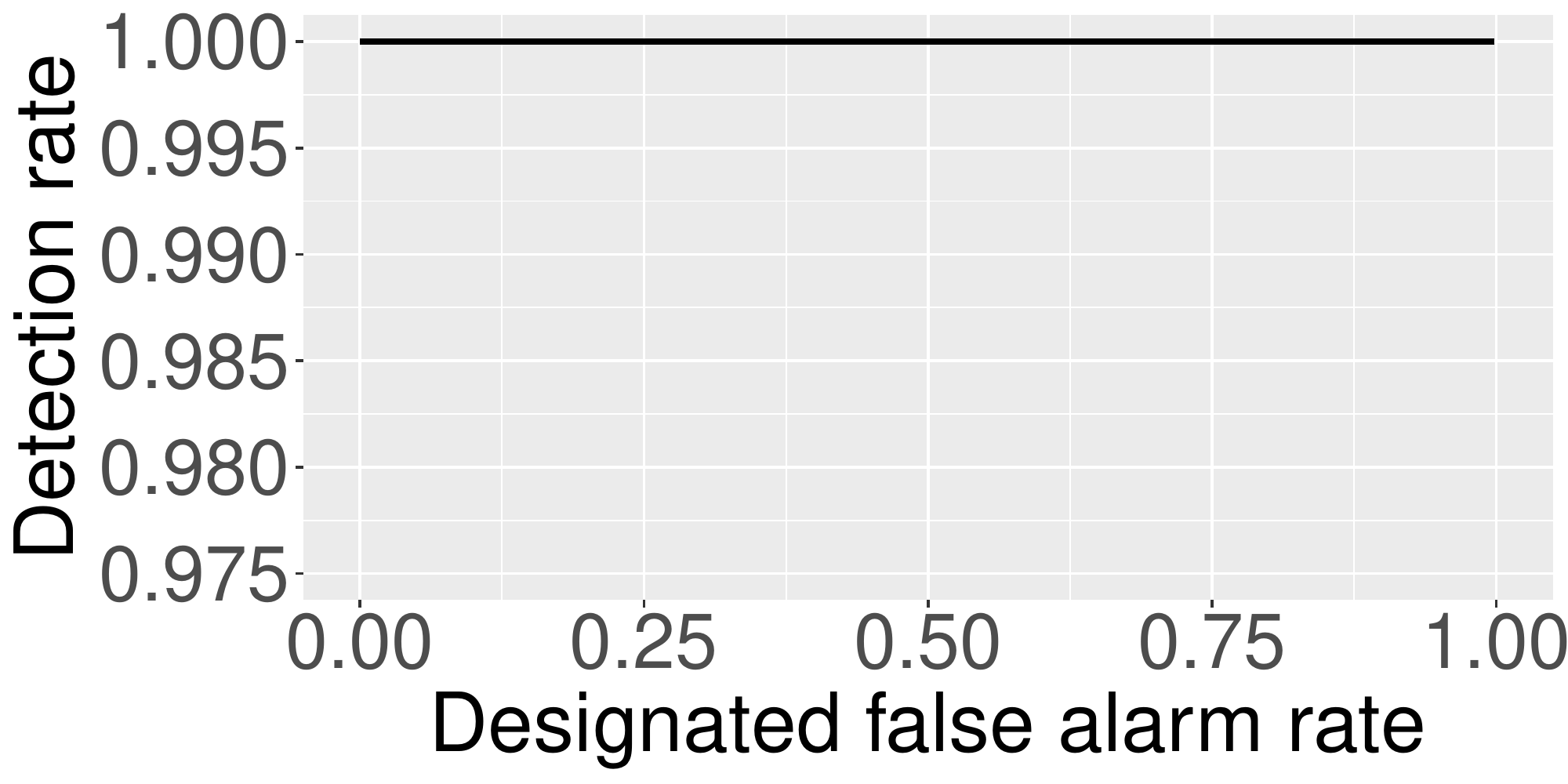}}
	\caption{Detection rate varies with the designated false alarm rate $r$ in each attack case for Scenario 1. (a)~$S_{0}$ is attacked; (b)~$S_{1}$ is attacked; (c)~$S_{2}$ is attacked; (d)~$S_{0}$ and $S_{1}$ are attacked; (e)~$S_{0}$ and $S_{2}$ are attacked; (f)~$S_{1}$ and $S_{2}$ are attacked; (g)~$S_{0}$, $S_{1}$ and $S_{2}$ are attacked.}
	\label{fig:scen_1_exp_r}

\end{figure*}

To detect optical attacks, we compare the two disparity maps: $DM_{1,2}$ and $DM_{0,2}$. Since in the two procedures described above we use $D_2$ as the reference image, the two produced disparity maps have the same scale and share the same view. Thus, we can compare them directly. Here, it shall be noted that $DM_{0,2}$ contains sparse disparity information, since the distances in the point cloud $D_0$ are not densely measured. Therefore, in this comparison procedure, we only compare pixels that have valid disparity in $DM_{0,2}$. For valid pixels, we take the KITTI stereo benchmark~\cite{menze15object} as a reference and design our own standard, in which a disparity inconsistency for pixel $a$ is counted if and only if
\begin{equation}\label{eq:disparity_error}
\begin{cases}
    |DM_{0,2}(a)-DM_{1,2}(a)|>3,\\
    \frac{|DM_{0,2}(a)-DM_{1,2}(a)|}{min(DM_{0,2}(a),DM_{1,2}(a))}>0.05.
\end{cases}
\end{equation}

Based on the \textbf{pixel-level} disparity inconsistencies, we evaluate the disparity error, denoted as $E_{0,1,2}$, between $DM_{0,2}$ and $DM_{1,2}$. In particular, the \textbf{disparity error} is defined as the ratio of the quantity of \textbf{pixel-level} disparity inconsistencies over the total number of valid pixels.

Finally, to detect an optical attack, we need to evaluate the ranges of disparity errors in normal cases and attack cases. We believe that the two ranges are distinguishable and we can then use a threshold $\theta_{0,1,2}$ to determine whether there is an optical attack. In particular, we determine that one of the three sensors is under attack if and only if $E_{0,1,2}>\theta_{0,1,2}$.

The threshold $\theta_{0,1,2}$ is determined offline based on the value distribution of $E_{0,1,2}$ when the three-sensor system is in a safe environment, since only the correct data of optical sensors are available on an autonomous vehicle in normal conditions. According to the statistical law, we use a large number of samples of the disparity error to represent its real distribution in normal cases and define a designated false alarm rate $r$ to arbitrarily set $r \times 100\%$ of them as virtual outliers, where $r \in [0,1]$. Then, the critical value separating inliers from outliers is $\theta_{0,1,2}$:
\begin{equation}\label{eq:detection_threshold}
    \frac{\text{\# samples of } E_{0,1,2} > \theta_{0,1,2}}{\text{\# samples of } E_{0,1,2}} = r.
\end{equation}
In this manner, the threshold, which only moves within the value range of disparity error samples, is determined by the value of $r$. Intuitively, to obtain the best detection performance, we should maximize the detection rate and minimize the designated false alarm rate. Hence, we show how the detection rate varies when adjusting the threshold via $r$.

\subsection{Experiments for Scenario 1}

\begin{table}[t]
    \centering
	\caption{Detection rate comparison between the method in our framework and the baseline}\label{tbl:dete_exp_r}
	\begin{tabular}{|c|c|c|c|}
	    \hline
	        \multirow{2}{*}{Method} & \multirow{2}{*}{Granularity} & \multicolumn{2}{c|}{Avg. Detection Rate} \\
		\cline{3-4}
		    && Scenario 1 & Scenario 2 \\
		\hline
    	\hline
    	Ours ($r=0\%$) & Pixel-level & $99.46\%$ & $99.94\%$ \\
    	Ours ($r=1\%$) & Pixel-level & $99.89\%$ & $99.97\%$ \\
    	Ours ($r=2\%$) & Pixel-level & $99.89\%$ & $100\%$ \\
    	Ours ($r=3\%$) & Pixel-level & $99.91\%$ & $100\%$ \\
    	Ours ($r=5\%$) & Pixel-level & $99.91\%$ & $100\%$ \\
    	\hline
    	Baseline (IoU $=0.5$) & Object-level & $65.32\%$ & $67.17\%$ \\
    	Baseline (IoU $=0.7$) & Object-level & $59.54\%$ & $63.39\%$\\
    	\hline
	\end{tabular}

\end{table}

\subsubsection{Setup}

To validate the hypothesis for this scenario, we conduct extensive experiments. We consider all possible attack cases where any sensor or any combination of the three sensors gets attacked. We use the data of one LiDAR and two cameras from the customized KITTI raw dataset~\cite{geiger13vision} to produce affected sensor data for each attack case. The production scheme is described in Section~\ref{subsubsec:setup}. The PSMNet model used in the experiments is provided by Wang \textit{et al.}~\cite{wang19pseudo}, which is trained on Scene Flow dataset~\cite{mayer16large} and KITTI object detection dataset~\cite{geiger12are}. As for metrics, we measure the disparity error distribution and the rate of correct detection for each attack case.

In the literature, there is no existing solution for optical attack detection. Therefore, to compare our scheme with possible solutions, we implement a possible baseline solution to optical attack detection that first extracts \textbf{object-level} features from the data of two individual sensors respectively, and then detects attacks by measuring the mismatches between the two sets of features. Specifically for Scenario 1, we implement the baseline with the backbone of PIXOR~\cite{yang18pixor} for extracting \textbf{object-level} features from point clouds and the backbone of Faster R-CNN~\cite{ren15faster} for extracting from images. We set IoU to $0.5$ and $0.7$ for determining feature mismatches.

\begin{figure*}[!t]
	\centering
	\subfloat[]{
		\label{fig:s0_2_d}
		\includegraphics[width=0.23\textwidth]{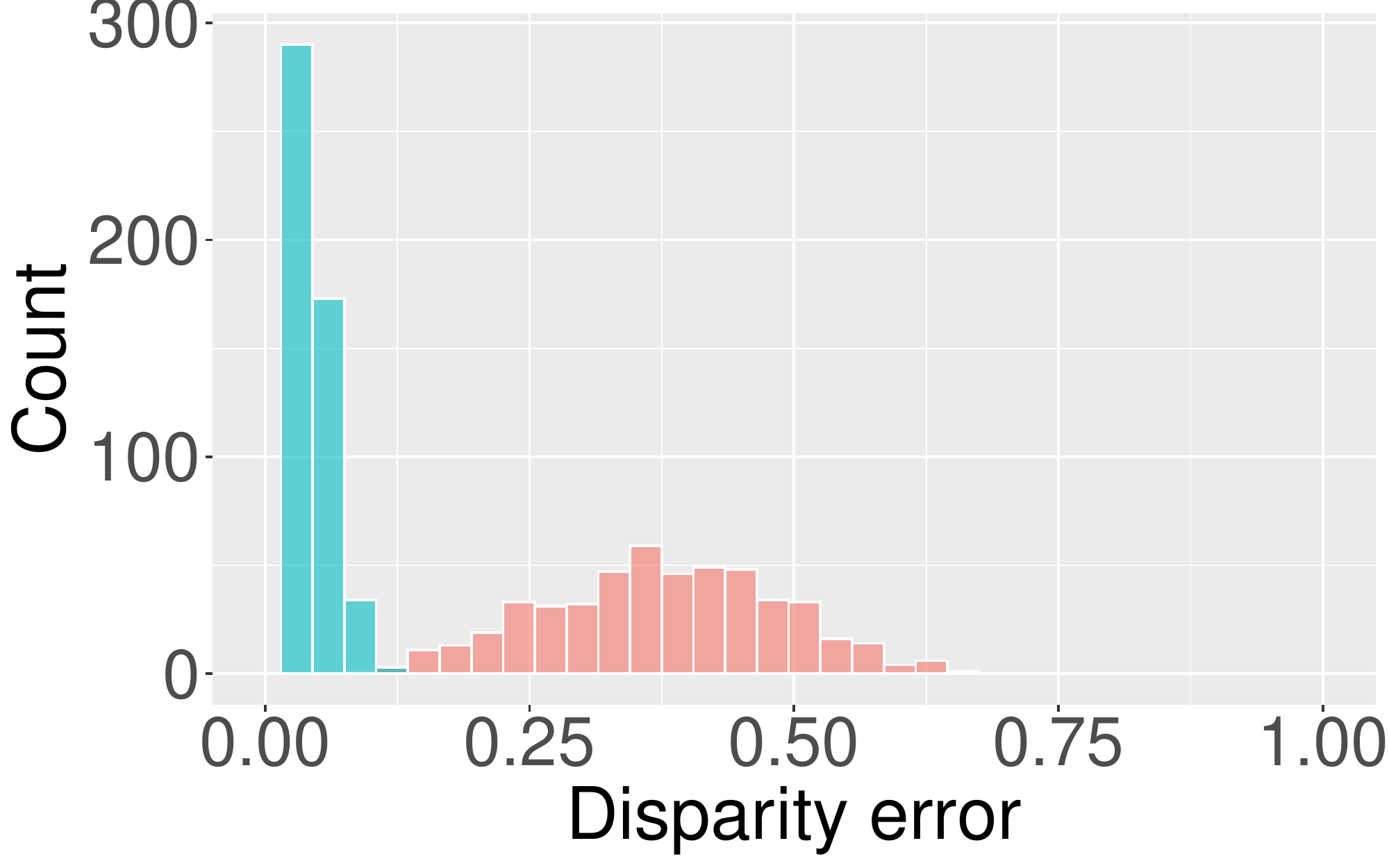}}
	\hfil
	\subfloat[]{
		\label{fig:s1_2_d}
		\includegraphics[width=0.23\textwidth]{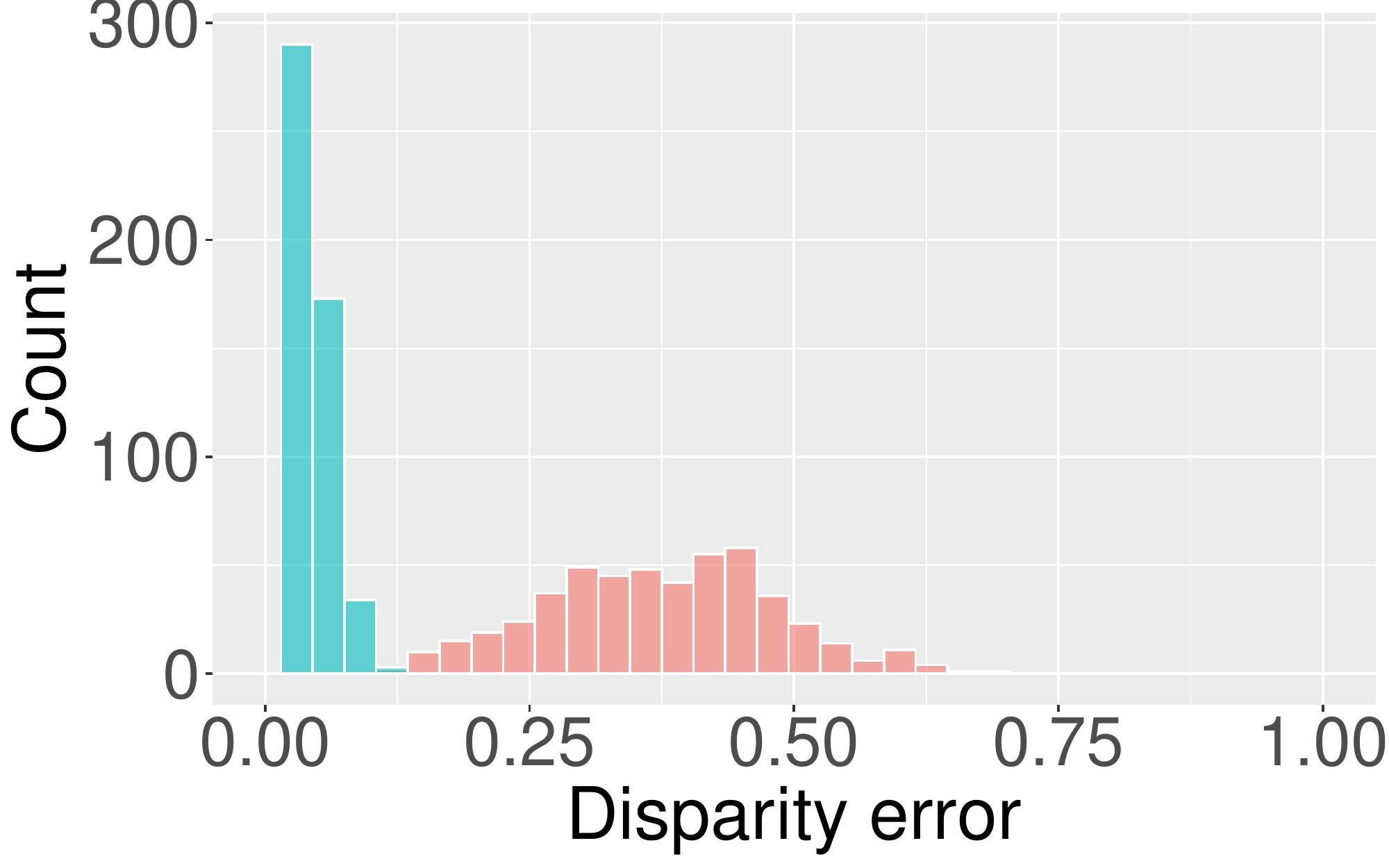}}
	\hfil
	\subfloat[]{
		\label{fig:s2_2_d}
		\includegraphics[width=0.23\textwidth]{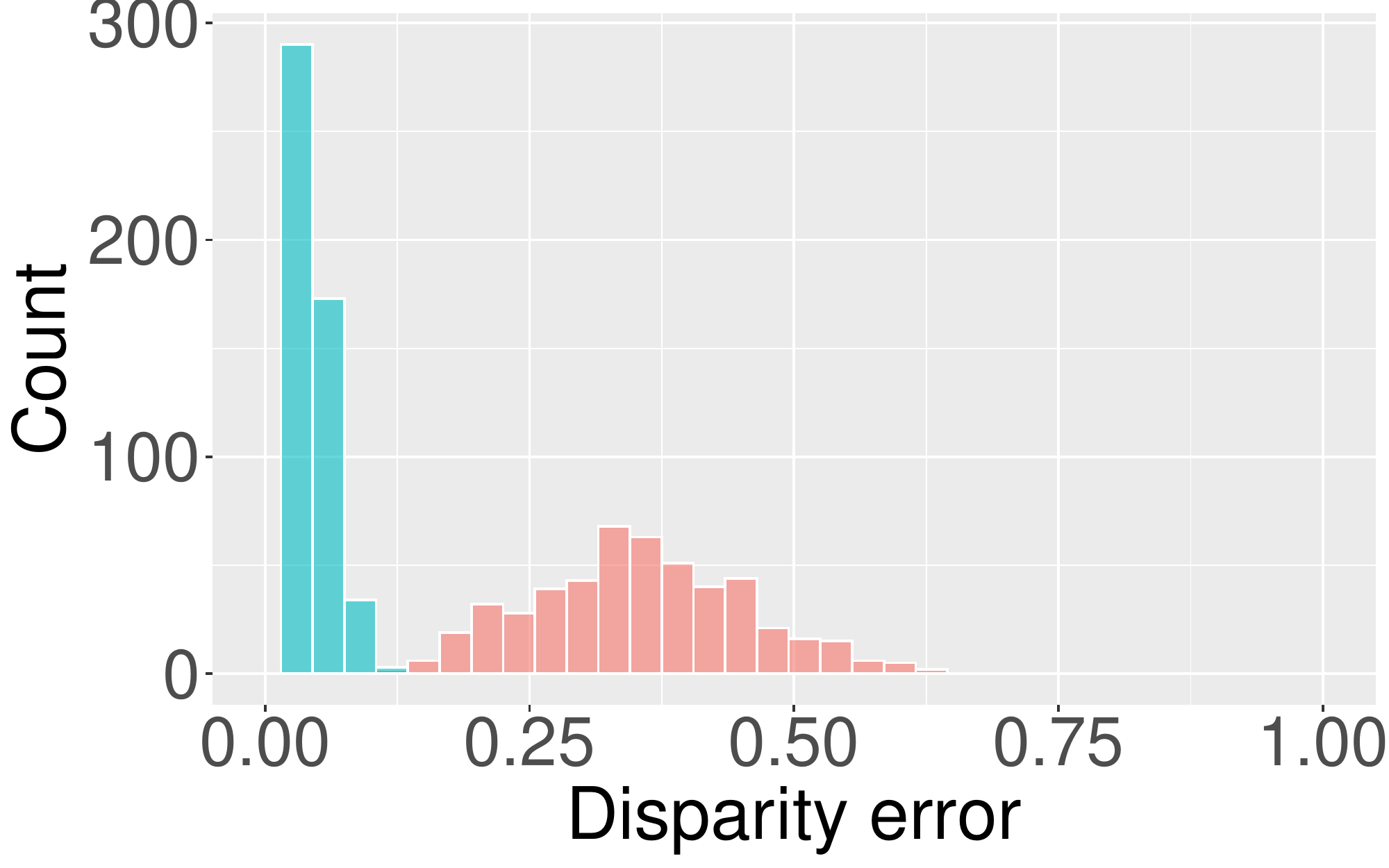}}
	\hfil
	\subfloat[]{
		\label{fig:s0_s1_2_d}
		\includegraphics[width=0.23\textwidth]{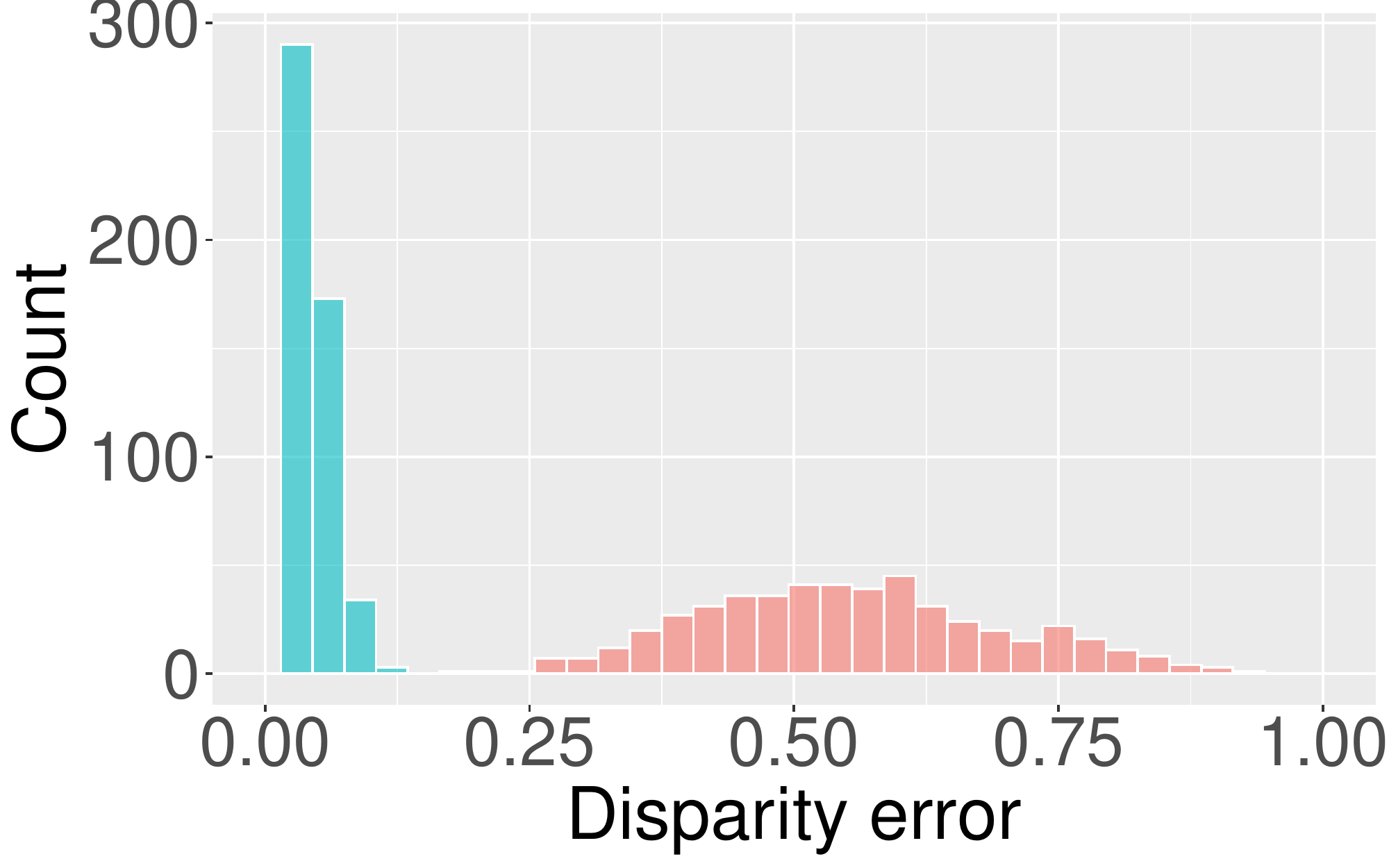}}
	\hfil

	\subfloat[]{
		\label{fig:s0_s2_2_d}
		\includegraphics[width=0.23\textwidth]{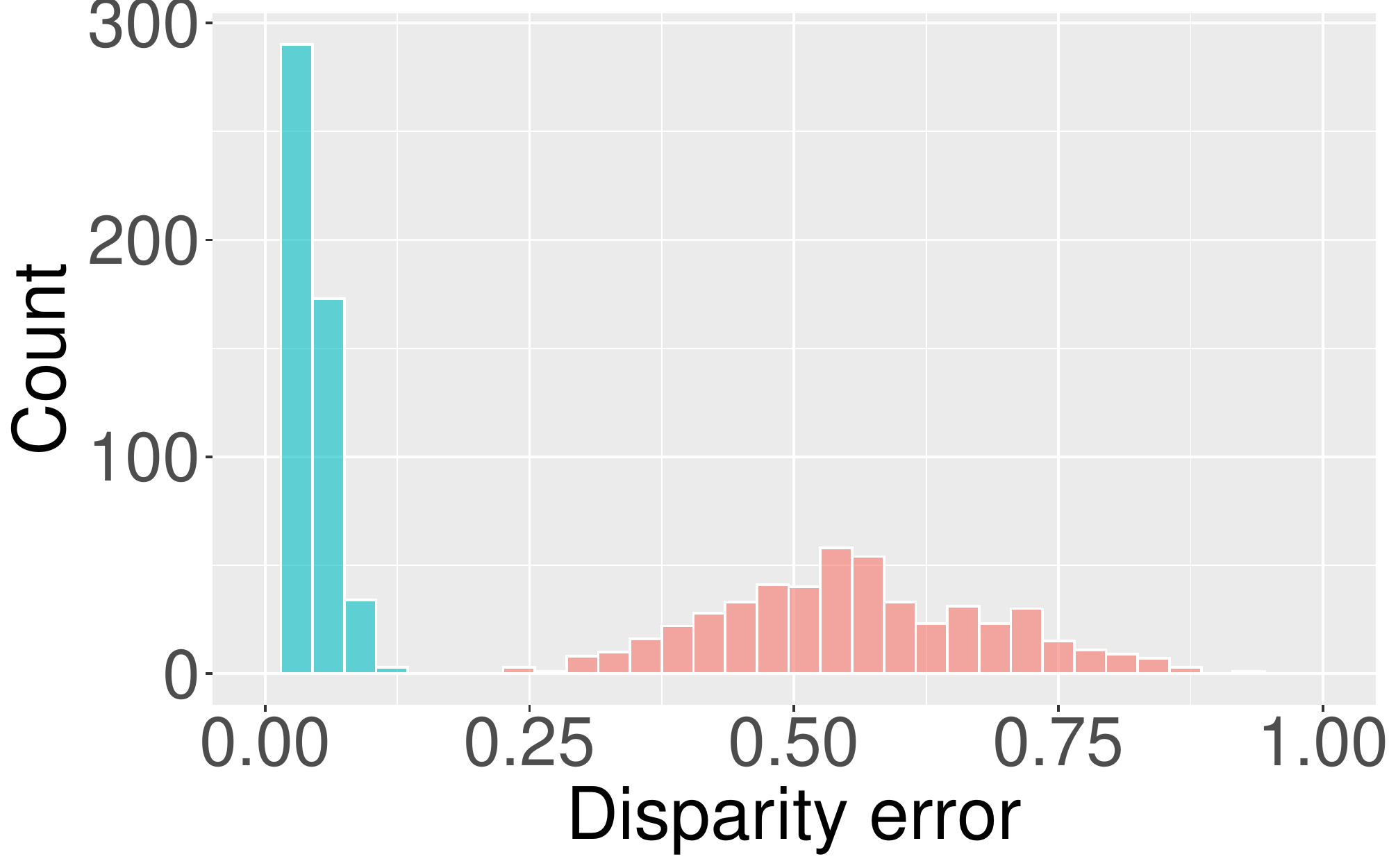}}
	\hfil
	\subfloat[]{
		\label{fig:s1_s2_2_d}
		\includegraphics[width=0.23\textwidth]{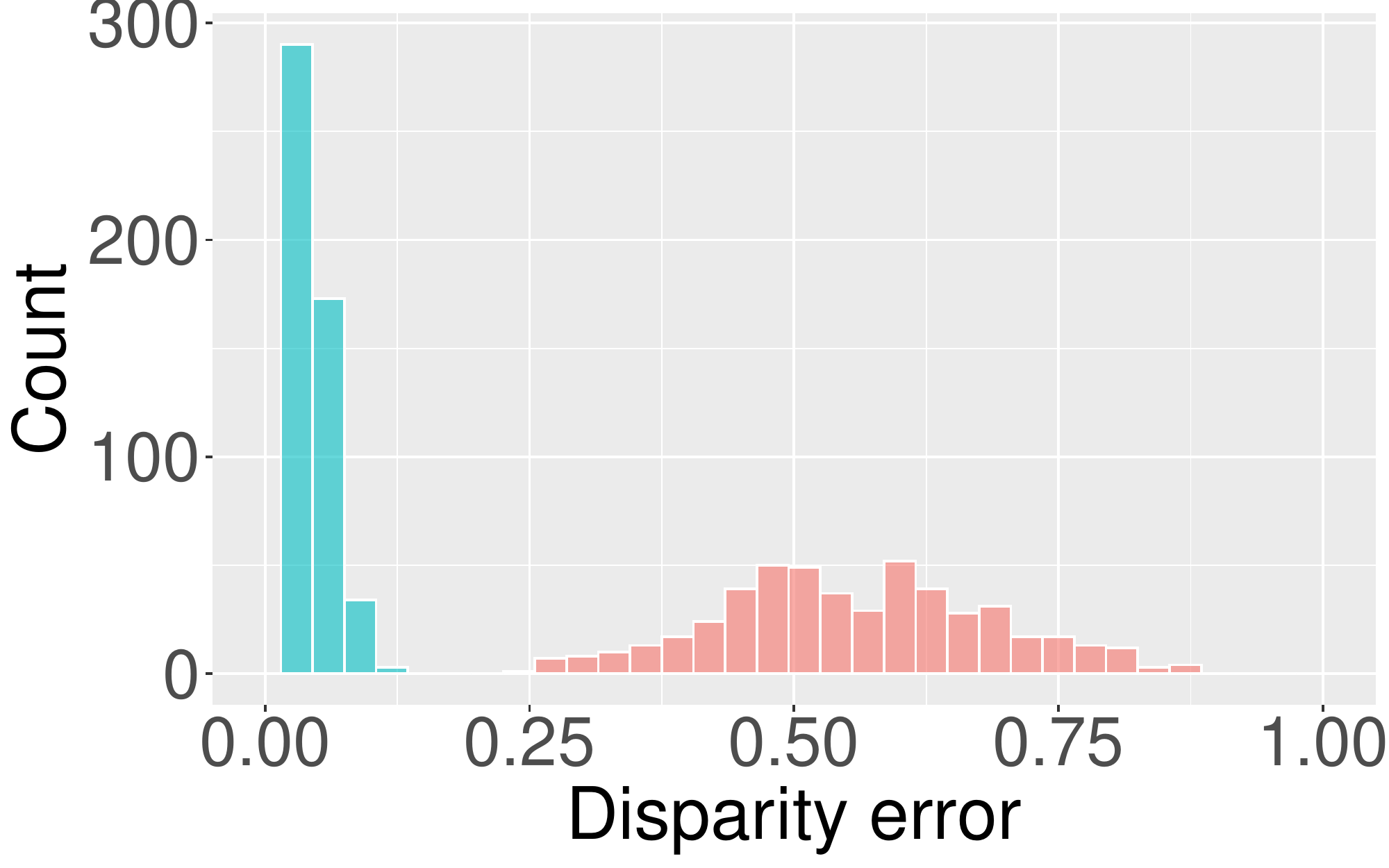}}
	\hfil
	\subfloat[]{
		\label{fig:s0_s1_s2_2_d}
		\includegraphics[width=0.23\textwidth]{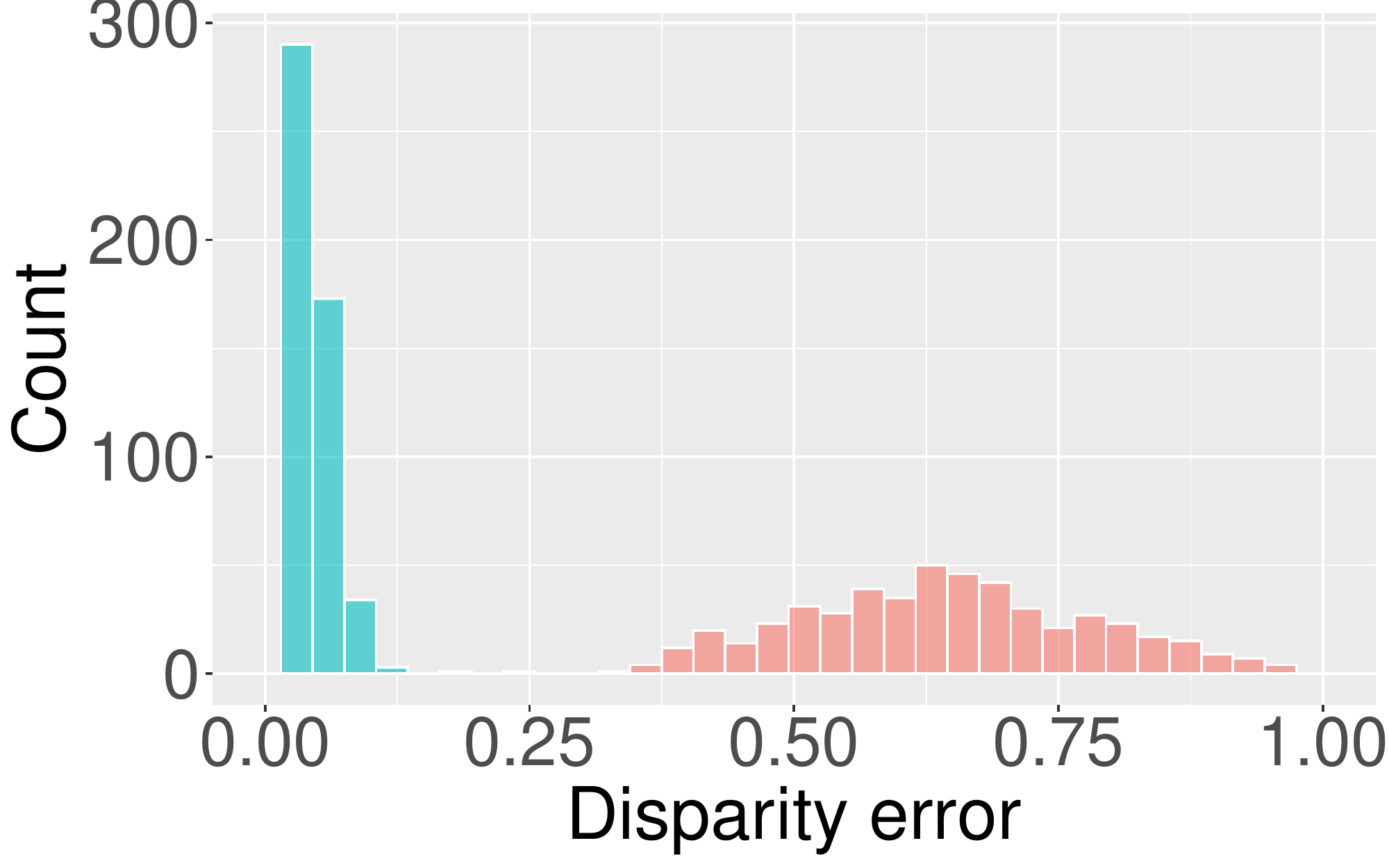}}
	\caption{Distributions of disparity error $E_{0,1,2}$ in normal case (cyan bars) and in attack cases (red bars) for Scenario 2. (a)~No attack vs. $S_{0}$ attacked; (b)~No attack vs. $S_{1}$ attacked; (c)~No attack vs. $S_{2}$ attacked; (d)~No attack vs. $S_{0},S_{1}$ attacked; (e)~No attack vs. $S_{0},S_{2}$ attacked; (f)~No attack vs. $S_{1},S_{2}$ attacked; (g)~No attack vs. $S_{0},S_{1},S_{2}$ attacked.}
	\label{fig:scen_2_exp_d}

\end{figure*}

\begin{figure*}
	\centering
	\subfloat[]{
		\label{fig:s0_2_r}
		\includegraphics[width=0.23\textwidth]{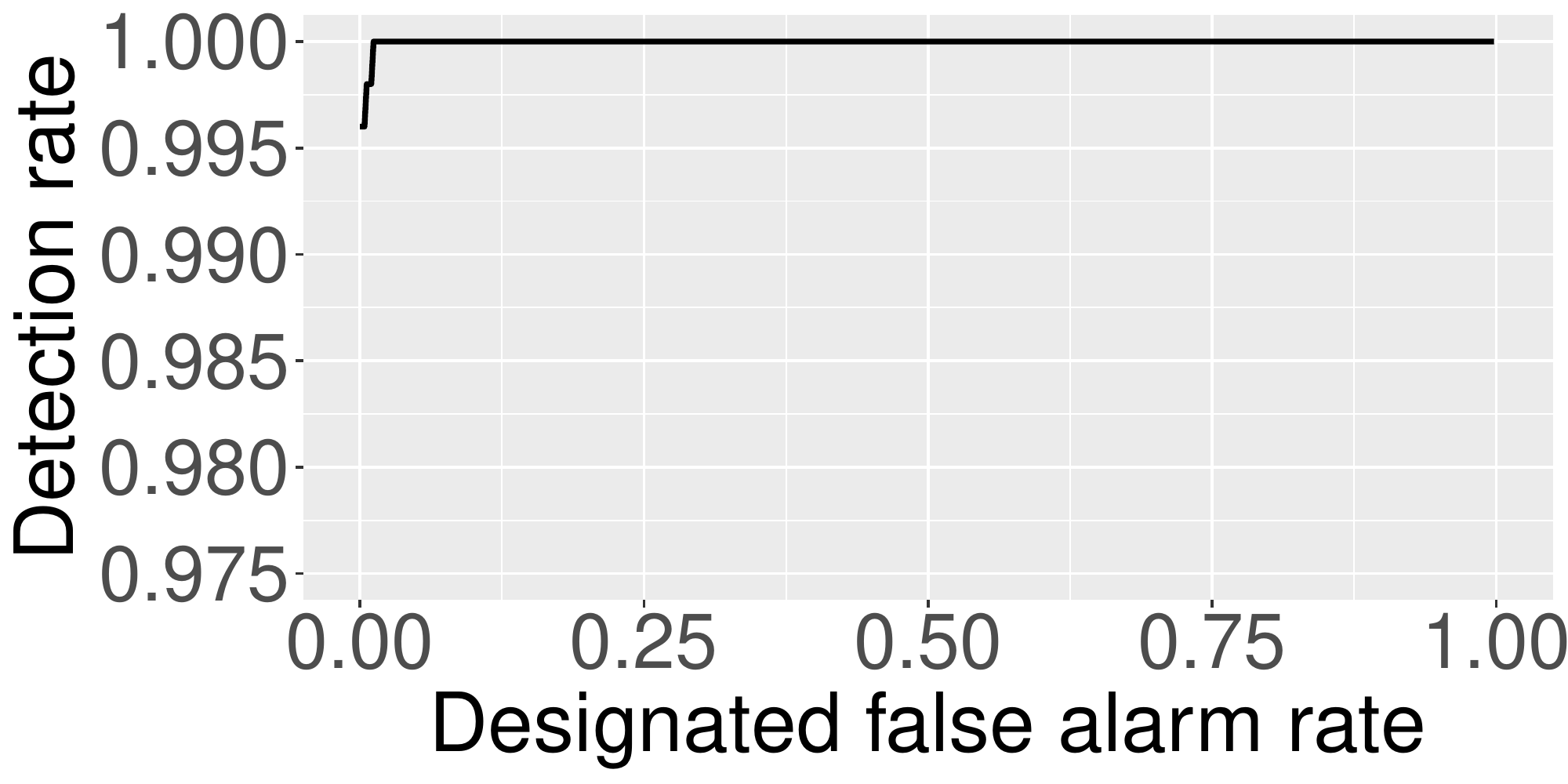}}
	\hfil
	\subfloat[]{
		\label{fig:s1_2_r}
		\includegraphics[width=0.23\textwidth]{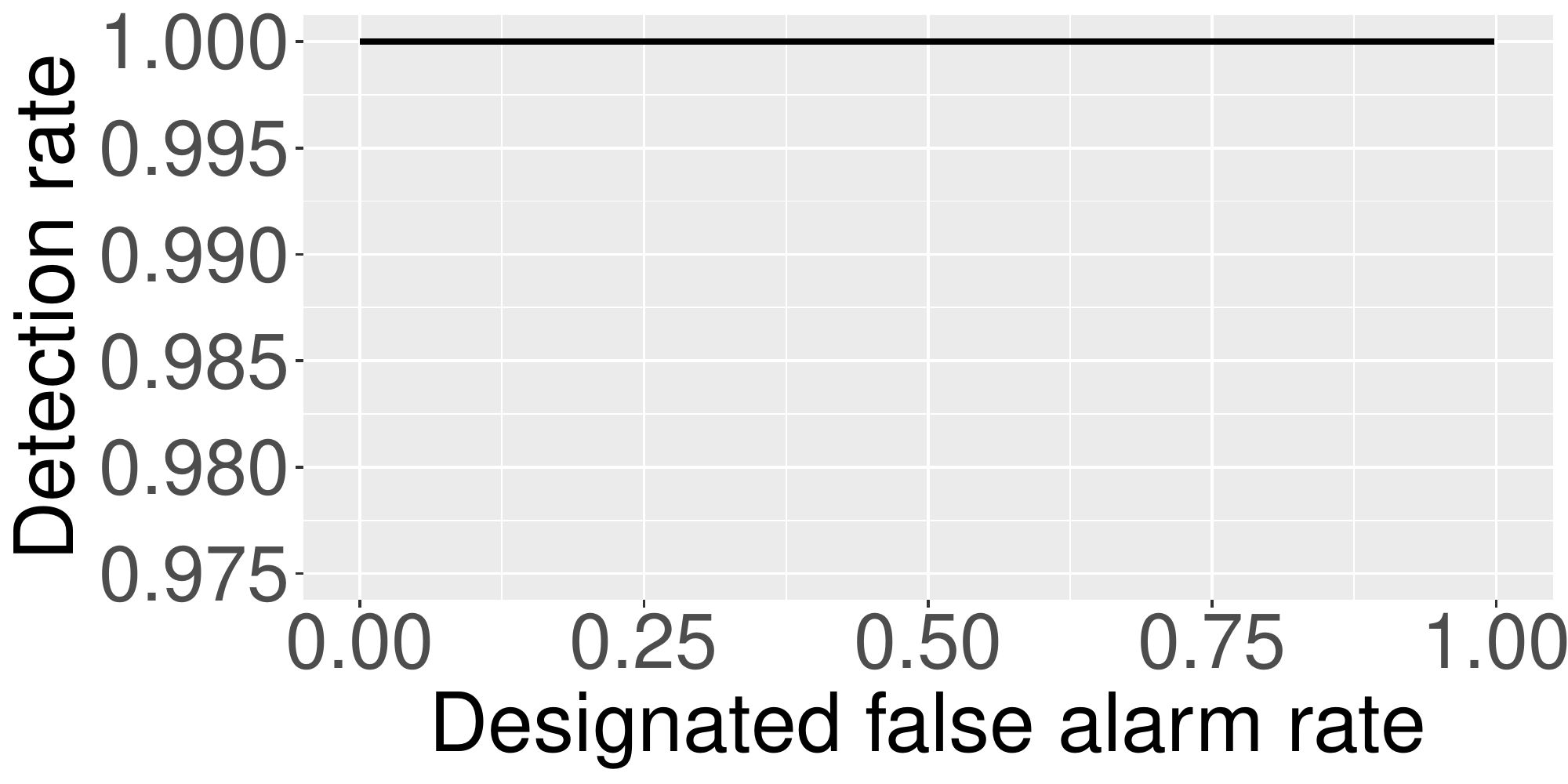}}
	\hfil
	\subfloat[]{
		\label{fig:s2_2_r}
		\includegraphics[width=0.23\textwidth]{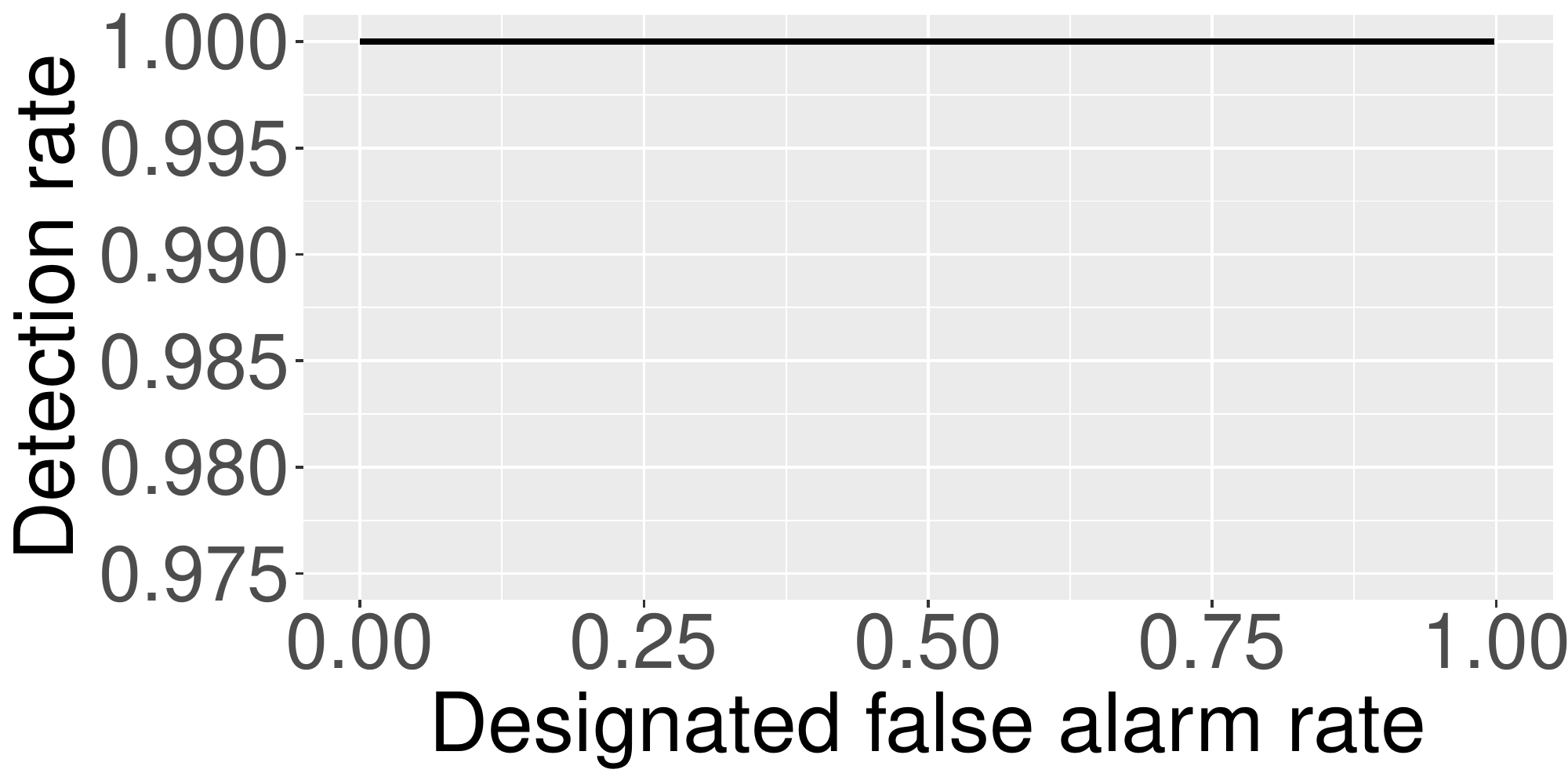}}
	\hfil
	\subfloat[]{
		\label{fig:s0_s1_2_r}
		\includegraphics[width=0.23\textwidth]{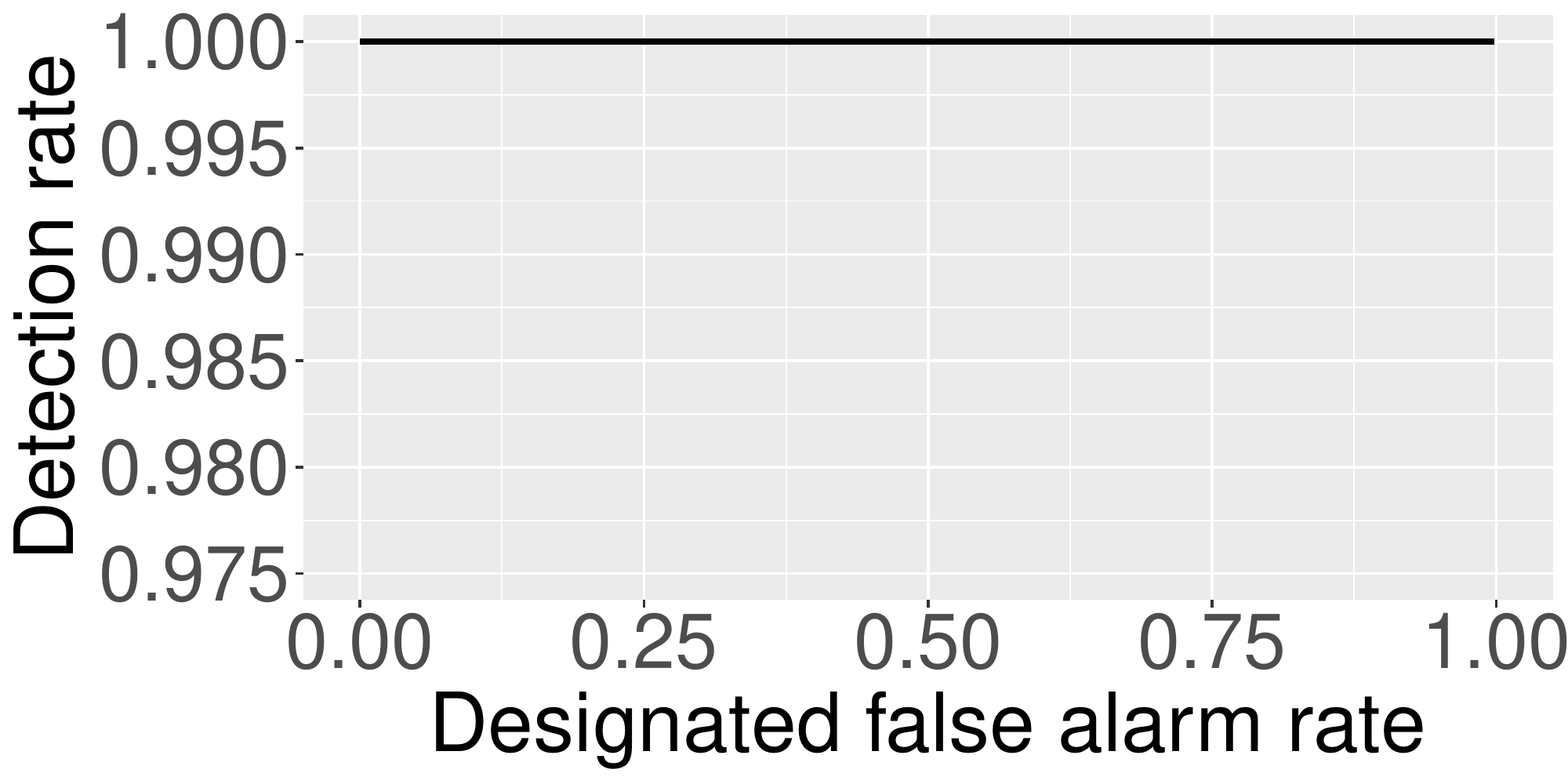}}
	\hfil

	\subfloat[]{
		\label{fig:s0_s2_2_r}
		\includegraphics[width=0.23\textwidth]{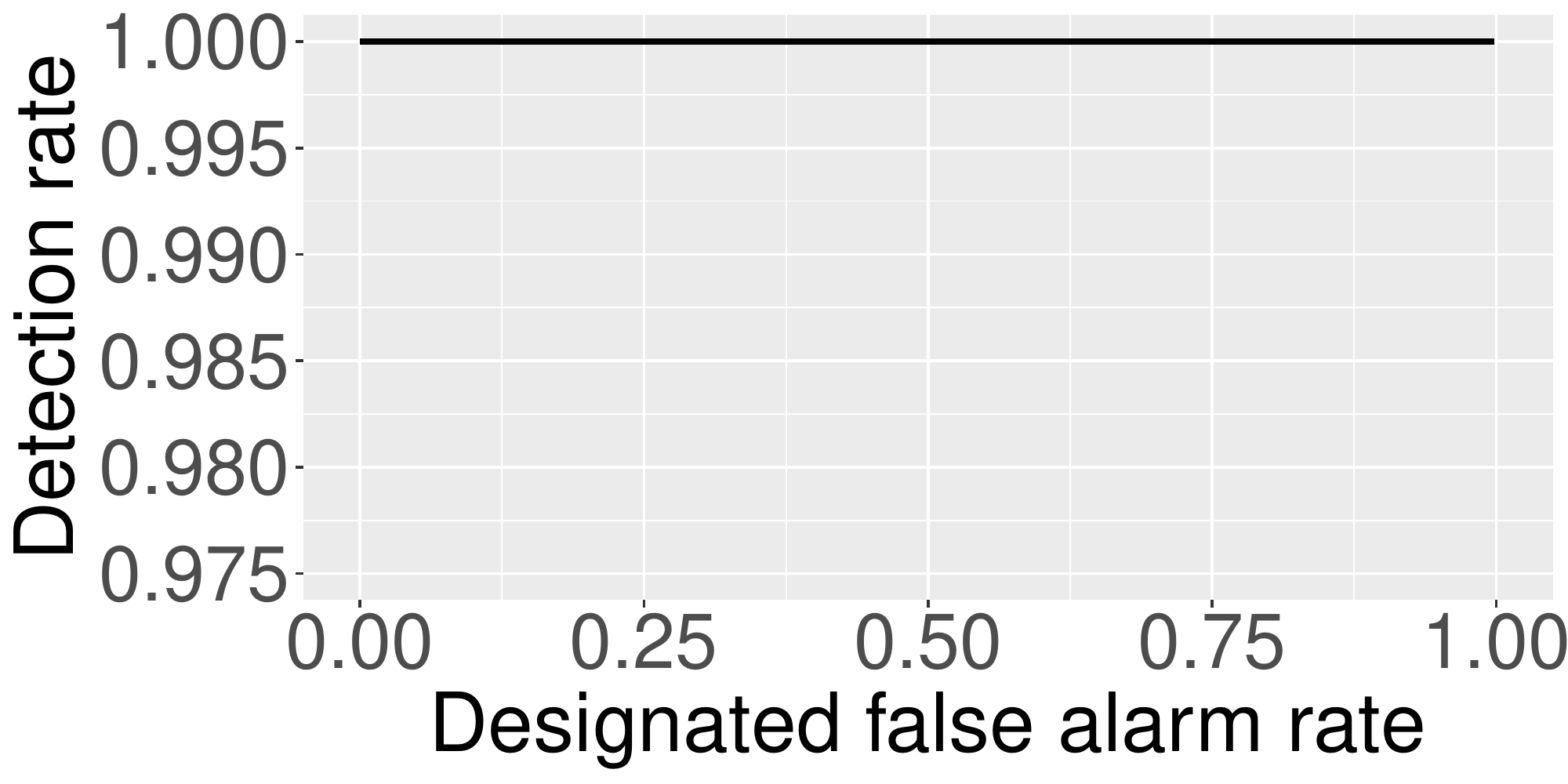}}
	\hfil
	\subfloat[]{
		\label{fig:s1_s2_2_r}
		\includegraphics[width=0.23\textwidth]{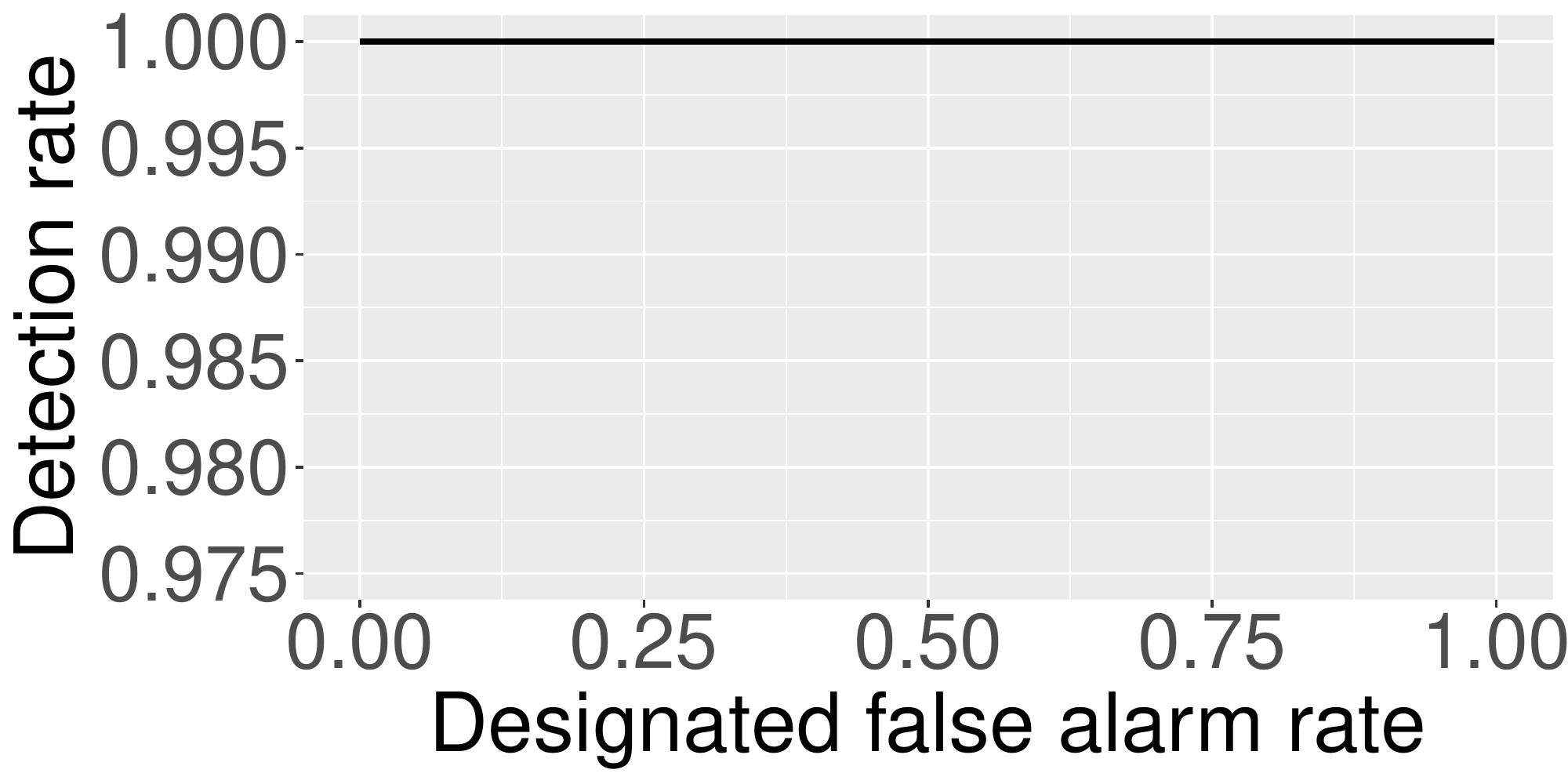}}
	\hfil
	\subfloat[]{
		\label{fig:s0_s1_s2_2_r}
		\includegraphics[width=0.23\textwidth]{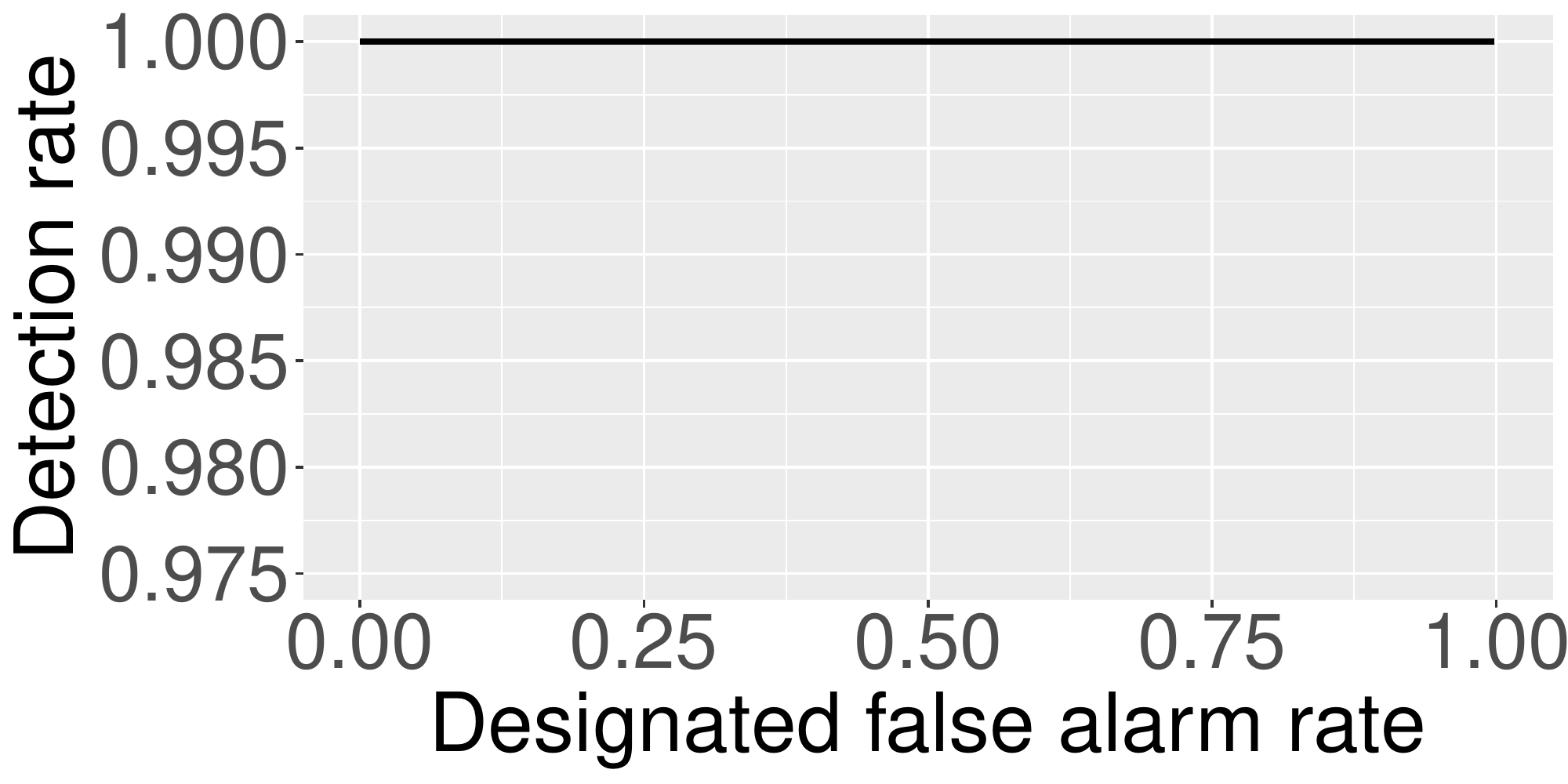}}
	\caption{Detection rate varies with the designated false alarm rate $r$ in each attack case for Scenario 2. (a)~$S_{0}$ is attacked; (b)~$S_{1}$ is attacked; (c)~$S_{2}$ is attacked; (d)~$S_{0}$ and $S_{1}$ are attacked; (e)~$S_{0}$ and $S_{2}$ are attacked; (f)~$S_{1}$ and $S_{2}$ are attacked; (g)~$S_{0}$, $S_{1}$ and $S_{2}$ are attacked.}
	\label{fig:scen_2_exp_r}

\end{figure*}

\subsubsection{Results}

The experimental results are shown in Fig.~\ref{fig:scen_1_exp_d} and Fig.~\ref{fig:scen_1_exp_r}. In Fig.~\ref{fig:scen_1_exp_d}, we compare the distributions of disparity errors between the normal case (i.e., no attack, in cyan bars) with each one of the seven possible attack cases (red bars). We first observe that the disparity errors are smaller than $20\%$ in most normal cases. By comparison, the disparity errors in most attack scenarios are larger than $25\%$. These results indicate that our detection scheme is sensitive enough, so that there is almost no overlap between the distribution in the normal case and the distributions in those attack cases.

In Fig.~\ref{fig:scen_1_exp_r}, we adjust the threshold $\theta_{0,1,2}$ used to declare attacks by varying the designated false alarm rate $r$, and evaluate the attack detection rate versus $r$. As we can see from the figures, among seven attack scenarios, the performance is perfect in five cases, where the detection rate hits $100\%$ for all possible values of $r$. And even in the non-perfect cases (i.e., (b) and (f)), the proposed detection system can obtain more than $99.5\%$ detection rate with less than $5\%$ false alarm. Such results confirm our hypothesis for such a three-sensor system.

We also show the detection rate comparison between our proposed detection method ($r \leq 5\%$) and the baseline for Scenario 1 in Table~\ref{tbl:dete_exp_r}, where we can observe that our method outperforms the baseline by about $35\%$ with IoU $=0.5$ and about $40\%$ with IoU $=0.7$. The reason is that our method detects optical attack by measuring \textbf{pixel-level} disparity inconsistencies, which is much denser and more fine-grained than the \textbf{object-level} detection used in the baseline. In most cases where the attack only partially occludes important objects or does not occlude them at all, the \textbf{object-level} detection is highly likely to fail, while our proposed method can still function normally.

\subsection{Scenario 2: Three Cameras}

For this scenario, we consider a three-sensor system that consists of three cameras, denotes as $S_{0}$, $S_{1}$, and $S_{2}$, from the right to the left. Similar to the previous scenario, we also consider the data generated by the sensors as $D_{0}$, $D_{1}$, and $D_{2}$. The detection system that we design for this scenario is shown as the combination of Block B and Block C in Fig.~\ref{fig:detection}.

In the system illustrated in Fig.~\ref{fig:detection} (Block B \& Block C), we designate camera $S_2$ as the reference camera to generate two disparity maps. In our experiments, we find it more convenient to implement our detection scheme when the leftmost or rightmost camera is used as the reference camera.

Since the sensor data of this scenario are all images, to generate the two disparity maps, we feed $D_{2}$ with $D_{0}$ and $D_{1}$ to the depth estimation model, respectively. It shall be noted that, since the distance between $S_0$ and $S_2$ is usually different from the distance between $S_1$ and $S_2$, we need to adjust the disparity in $DM_{0,2}$ by updating the baseline $b$ accordingly. After the disparity maps $DM_{0,2}$ and $DM_{1,2}$ are generated, the rest of the procedures in the detection method are the same as those in the previous scenario.

\subsection{Experiments for Scenario 2}

\subsubsection{Setup}

Here, we use the data of three cameras in the customized KITTI raw dataset~\cite{geiger13vision} and also consider all possible attack cases. The baseline method is implemented with the backbone of Faster R-CNN~\cite{ren15faster}. The rest of settings are the same as those in the experiments for Scenario~1.

\subsubsection{Results}

We show the experiment results in Fig.~\ref{fig:scen_2_exp_d} and Fig.~\ref{fig:scen_2_exp_r}. In Fig.~\ref{fig:scen_2_exp_d}, we compare the distributions of disparity errors between the normal case (i.e., no attack, in cyan bars) with each one of the seven possible attack cases (red bars). Similar to Scenario~1, we observe that the disparity errors are less than $10\%$ in virtually all normal cases. By comparison, the disparity errors in $99\%$ of attack scenarios are larger than $12.5\%$.

In Fig.~\ref{fig:scen_2_exp_r}, we vary the threshold $\theta_{0,1,2}$ used to declare attacks via $r$, and evaluate the attack detection rate versus the designated false alarm rate. We can observe that the detection performance is perfect in all cases, where the detection rate remains $100\%$ when $r$ varies from $0$ to $1$. Comparing the results in Scenario 2 with results in Scenarios 1, we notice that the detection performance in scenario 2 is slightly better. We believe this is due to the facts that the disparity maps in this scenario are generated using the same method and there are more valid pixels in the comparison.

In Table~\ref{tbl:dete_exp_r}, the performance comparison for this scenario still shows that our proposed detection method outperforms the baseline by a large margin (more than $30\%$), which again shows the merit of \textbf{pixel-level} detection.

\subsection{Empirical Findings}

To briefly summarize, the findings from the attack detection experiments for the two three-sensor systems are listed as follows:

\begin{itemize}
    \item 
    The experimental results confirm our hypothesis that there exists a detection system that can detect optical attacks on the two three-sensor systems with high accuracy and low false alarm rate.
    
    \item
    The detection rate is insensitive to the designated false alarm rate. As long as the detection rate is maintained at a high level, the designated false alarm rate should be set as low as possible, empirically less than $5\%$.
    
    \item
    In those two three-sensor systems, any sensor or any combination of sensors being attacked can cause the disparity error beyond the threshold.
\end{itemize}

Based on these findings, we further develop the identification approach for the proposed framework.

\section{Attack Identification}
\label{ch3:sec.identification}

In this section, we present the second procedure of our framework which identifies the compromised sensors in a system with one LiDAR $S_{0}$ and $n$ cameras, namely, $S_{1}$ to $S_{n}$ from the right to the left, where $n \geq 3$, based on the empirical findings from the detection method. This method is inspired by Error Correction Codes (ECC) and can achieve the identification as long as there are no more than $n-2$ sensors being attacked simultaneously. In addition, we also demonstrate the proof of the correctness of our identification method, as well as show its effectiveness and accuracy via experiments.

We now introduce a few definitions that are used in the rest of this section. For every sensor $S_{i}$, its state $s_{i}$ can switch between normal state and attack state
\begin{equation}\label{eq:sensor_state}
s_{i}=\begin{cases}
    1, \text{ if } S_{i} \text{ is attacked,}\\
    0, \text{ otherwise,}
\end{cases}
\end{equation}
where $i\in\{0,1,\cdots,n\},n \geq 3$. The sensor state vector in the system is defined as:
\begin{equation}\label{eq:sensor_state_vector}
\textbf{\textit{s}}:=[s_{0},s_{1},\cdots,s_{n}],
\end{equation}
which is the hidden ground truth that we try to identify.

For disparity error $E_{i,j,k}$ among sensors $S_{i},S_{j}$, and $S_{k}$, we use $e_{i,j,k}$ to indicate whether $E_{i,j,k}$ exceeds the corresponding threshold $\theta_{i,j,k}$ in which
\begin{equation}\label{eq:error_state}
e_{i,j,k}=\begin{cases}
    1, \text{ if }E_{i,j,k}>\theta_{i,j,k},\\
    0, \text{ otherwise},
\end{cases}
\end{equation}
where $i,j,k \in \{0,1,\cdots,n\},i<j<k$. And similarly, we use the disparity error state vector $\textbf{\textit{e}}$ to represent the states of disparity errors of three-sensor combinations in the system.

Since the system consists of one LiDAR and $n$ cameras, the combination of any three sensors from it must be either one LiDAR with two cameras or three cameras. According to the empirical finding drawn from the experiments in the previous section, any sensor or any combination of sensors from such three-sensor sets being compromised leads to the corresponding disparity error higher than the threshold. Therefore, based on the definitions in Eqn.~(\ref{eq:sensor_state}) and Eqn.~(\ref{eq:error_state}), we have
\begin{equation}\label{eq:finding}
e_{i,j,k}=(s_{i} \lor s_{j} \lor s_{k}),
\end{equation}
where $i,j,k \in \{0,1,\cdots,n\},i<j<k$, and $\lor$ is logical OR operation.

\subsection{Calculation of Disparity Error State Vector}

Given $n+1$ sensors, we use the leftmost sensor $S_{n}$ as the reference camera and calculate disparity error state vector
\begin{equation}\label{eq:error_state_vector}
\textbf{\textit{e}}:=[e_{0,1,n},e_{0,2,n},\cdots,e_{n-2,n-1,n}],
\end{equation}
where $n \geq 3$. In the calculation, the disparity maps generated by $S_{n}$ with every remaining sensor are compared with each other using the same standard described by Eqn.~(\ref{eq:disparity_error}). Then, by the definition in Eqn.~(\ref{eq:error_state}), $\textbf{\textit{e}}$ is obtained via thresholding the resulted disparity errors from comparison. We also show this calculation process in the form of pseudocode in Algorithm~\ref{algo:cal_e} that takes the data of one LiDAR and $n$ cameras and a list of thresholds as inputs and outputs the disparity error state vector $\textbf{\textit{e}}$. Specifically, it first generates $n$ disparity maps with the sensor data and compares each two of them to obtain disparity errors, and then calculates $\textbf{\textit{e}}$ by encoding disparity errors with thresholds.

Note that the thresholds for calculating $\textbf{\textit{e}}$ are also determined offline using one designated false alarm rate $r$. The approach is similar as in the detection procedure. For every disparity error, we collect sufficient samples when the system is safe and consider $r \times 100\%$ of samples as virtual outliers. The thresholds are then set to the maximal values of inliers. Hence, Eqn.~(\ref{eq:detection_threshold}) can be rewritten as:
\begin{equation}
    \frac{\text{\# samples of } E_{i,j,n} > \theta_{i,j,n}}{\text{\# samples of } E_{i,j,n}} = r,
\end{equation}
where $r \in [0,1]$, $i,j \in \{0,1,\cdots,n-1\},i<j$. Unlike the detection rate which is insensitive to $r$, our subsequent experiments indicate that the identification rate drops linearly with $r$ increasing and the best identifying performance is achieved when $r=1\%$.

\subsection{Identification of Sensor State Vector}

\setlength{\textfloatsep}{6pt}

\begin{algorithm}[t]
\caption{\textit{Calculation of disparity error state vector $\textbf{\textit{e}}$ }}\label{algo:cal_e}
\KwIn{$\textbf{\textit{D}}$: a list of sensor data with length equal to $|\textbf{\textit{D}}|$, where $D_{0}$ is point cloud and the rest are images; $\boldsymbol{\theta}$: a list of thresholds.}
\KwOut{$\textbf{\textit{e}}$: the disparity error state vector.}
select $D_{n}$ as reference image, where $n=|\textbf{\textit{D}}|-1$\;

get disparity map $DM_{0,n}$ using $D_{0}$ and $D_{n}$ (Scenario 1 in Section~\ref{ch3:sec.detection})\;

\For{$i\gets{1}$ \KwTo $n-1$}{
    get disparity map $DM_{i,n}$ using $D_{i}$ and $D_{n}$ (Scenario 2 in Section~\ref{ch3:sec.detection})\;
}
\For{$i\gets{0}$ \KwTo $n-2$}{
    \For{$j\gets{i+1}$ \KwTo $n-1$}{
        get disparity error $E_{i,j,n}$ by comparing $DM_{i,n}$ with $DM_{j,n}$ (Section~\ref{ch3:sec.detection})\;
        
        \If{\text{\normalfont $E_{i,j,n}$ exceeds threshold $\theta_{i,j,n}$}}{
            assign $1$ to disparity error state $e_{i,j,n}$\;
        }
        \Else{
            assign $0$ to $e_{i,j,n}$\;
        }
    }
}
\Return{$\textbf{\textit{e}}$}\;

\end{algorithm}

We now elaborate on how to infer $\textbf{\textit{s}}$ according to $\textbf{\textit{e}}$. We consider the following three cases of $\textbf{\textit{e}}$:
\begin{itemize}
    \item If all elements in $\textbf{\textit{e}}$ are $0$s, according to Eqn.~(\ref{eq:finding}) and Eqn.~(\ref{eq:error_state_vector}), $s_i=0$ for $0\leq i\leq n$. In other words, no sensor is attacked. 
    \item
    If only some elements in $\textbf{\textit{e}}$ are $0$s, we have the following Lemma~\ref{thrm:lem_1} to identify all attacked sensors.
    \item
    If all elements in $\textbf{\textit{e}}$ are $1$s and no more than $n-2$ sensors are attacked simultaneously, we can repeatedly use Lemma~\ref{thrm:lem_1} and Lemma~\ref{thrm:lem_2} to identify all attacked sensors.
\end{itemize}

\begin{lemma}\label{thrm:lem_1}
In a system with $n+1$ sensors, if there exist $i_{0},j_{0}$ such that $e_{i_{0},j_{0},n}=0$, then $s_i=e_{i,i_{0},n}$ for  $0 \leq i < i_{0}$, and $s_{i}=e_{i_{0},i,l-1}$ for $i_0 <i < n$.
\end{lemma}

\begin{figure*}[!t]
    \centering
	\includegraphics[width=0.98\textwidth]{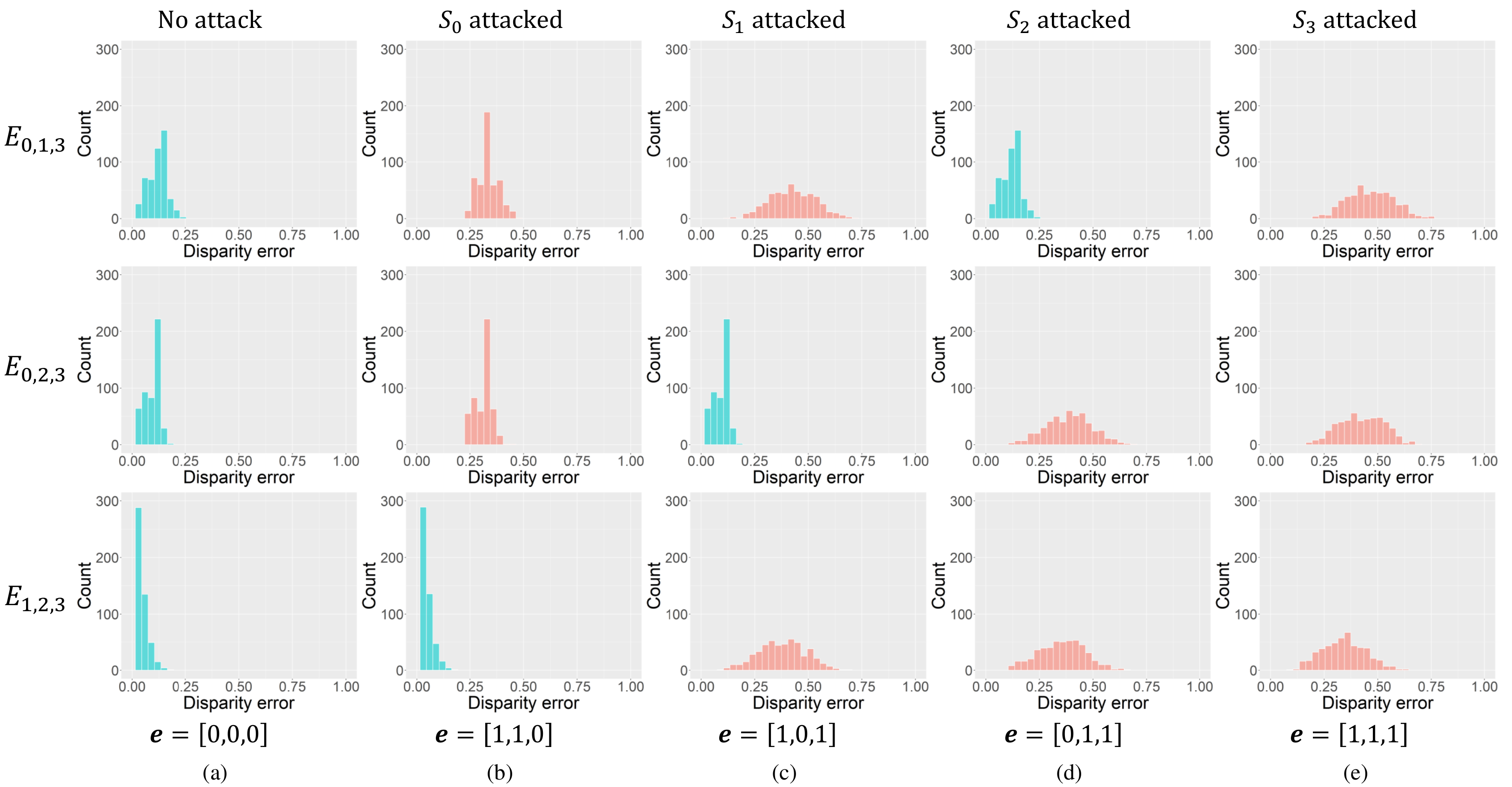}
	\caption{Distributions of disparity error $E_{0,1,3}$ (1st row), disparity error $E_{0,2,3}$ (2nd row) and disparity error $E_{1,2,3}$ (3rd row) in normal scenario (1st column) and in attack scenarios (2nd--5th columns). Cyan bars mean that the disparity errors in a scenario involve no attacked sensor, while red bars indicate that the disparity errors in a scenario involve attacked sensor. (a)~No disparity errors exceed thresholds, so $\textbf{\textit{e}}=[0,0,0]$ indicating no optical attack; (b)~$E_{0,1,3}$ and $E_{0,2,3}$ exceed thresholds, so $\textbf{\textit{e}}=[1,1,0]$ indicating that $S_{0}$ is attacked; (c)~$E_{0,1,3}$ and $E_{1,2,3}$ exceed thresholds, so $\textbf{\textit{e}}=[1,0,1]$ indicating that $S_{1}$ is attacked; (d)~$E_{0,2,3}$ and $E_{1,2,3}$ exceed thresholds, so $\textbf{\textit{e}}=[0,1,1]$ indicating that $S_{2}$ is attacked; (e)~All three disparity errors exceed thresholds, so $\textbf{\textit{e}}=[1,1,1]$ indicating that $S_{3}$ is attacked.}
	\label{fig:iden_exp_d}

\end{figure*}

\begin{figure*}
	\centering
	\subfloat[]{
		\label{fig:s0_i_r}
		\includegraphics[width=0.23\textwidth]{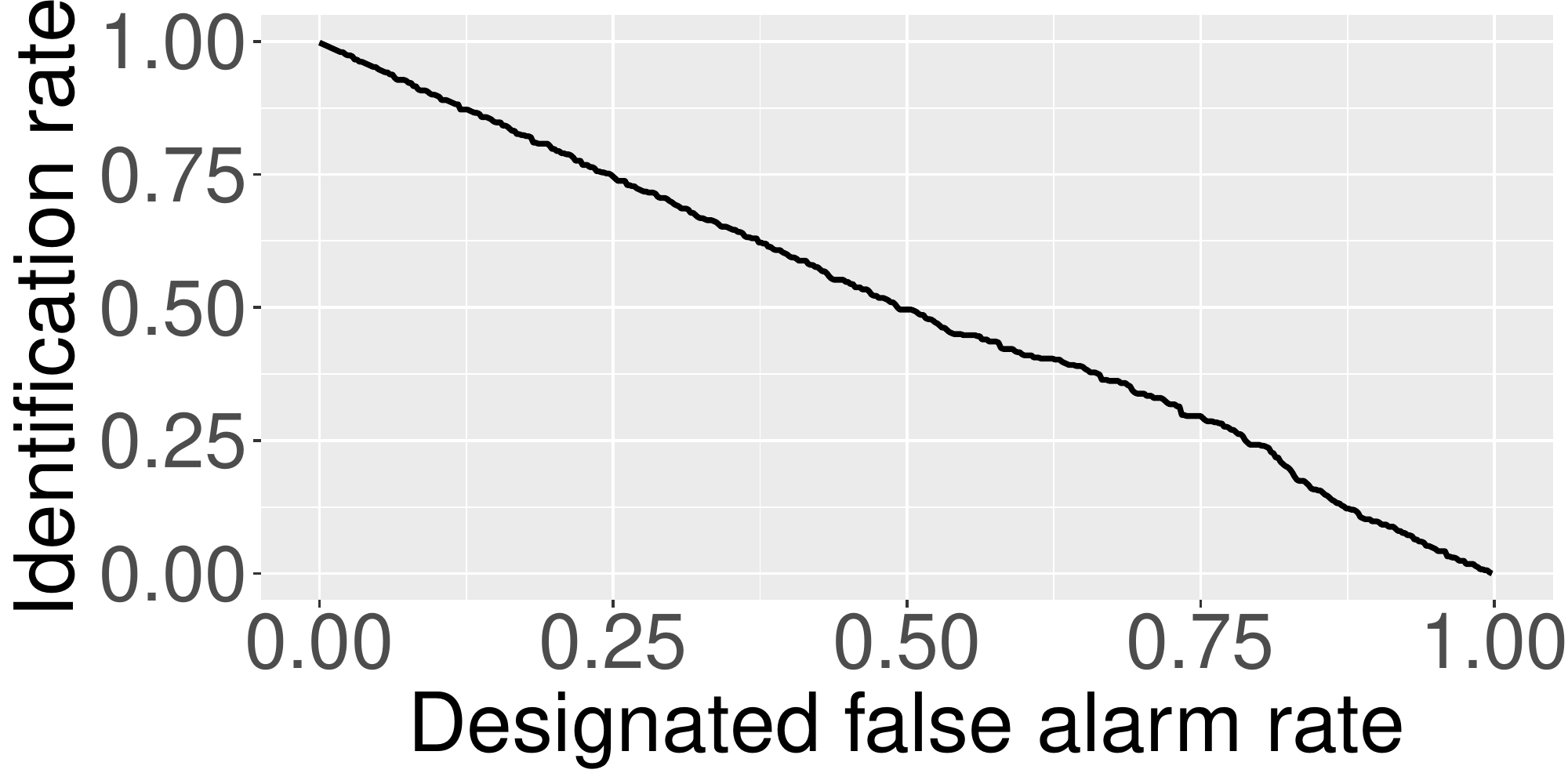}}
	\hfil
	\subfloat[]{
		\label{fig:s1_i_r}
		\includegraphics[width=0.23\textwidth]{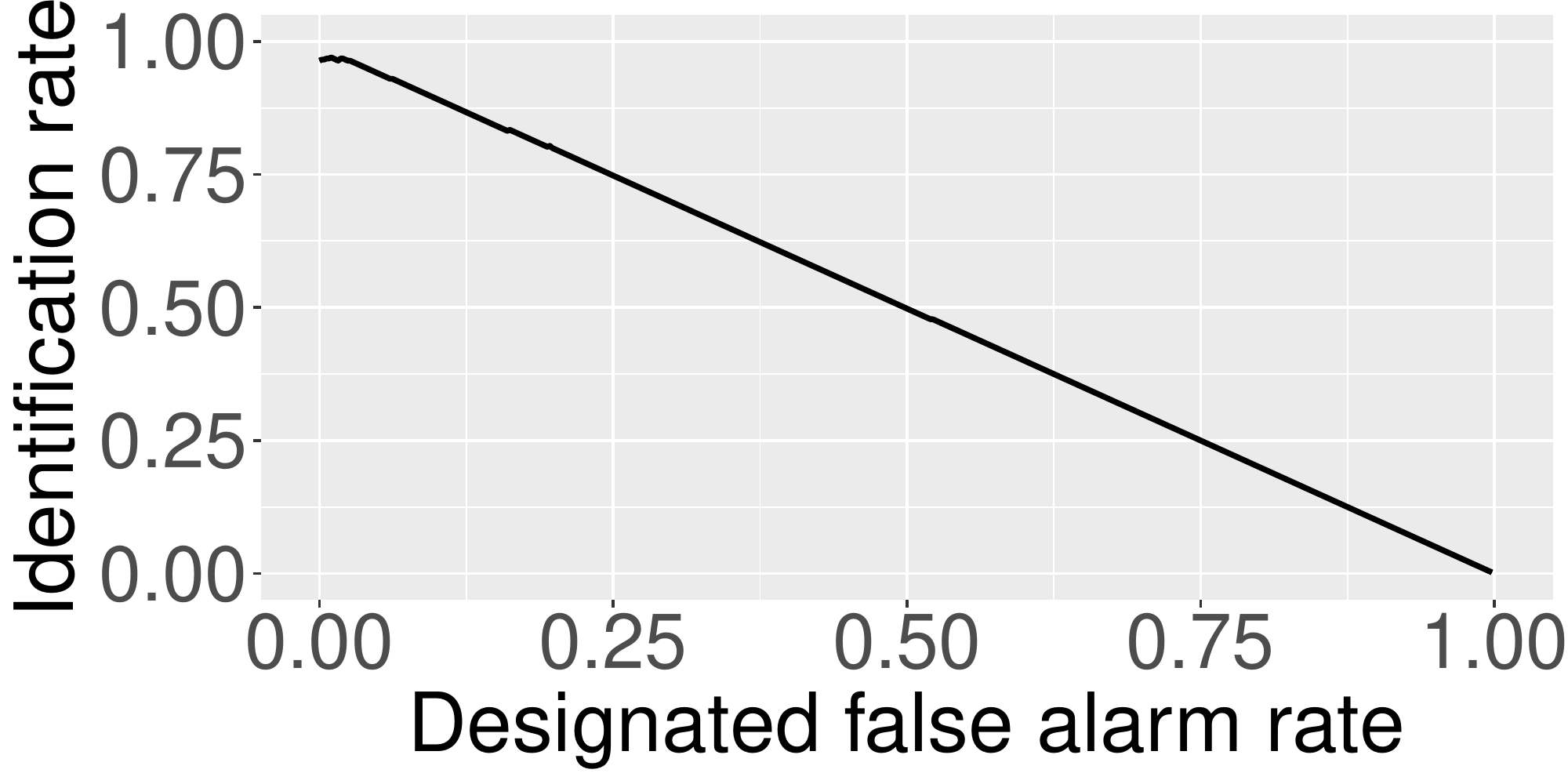}}
	\hfil
	\subfloat[]{
		\label{fig:s2_i_r}
		\includegraphics[width=0.23\textwidth]{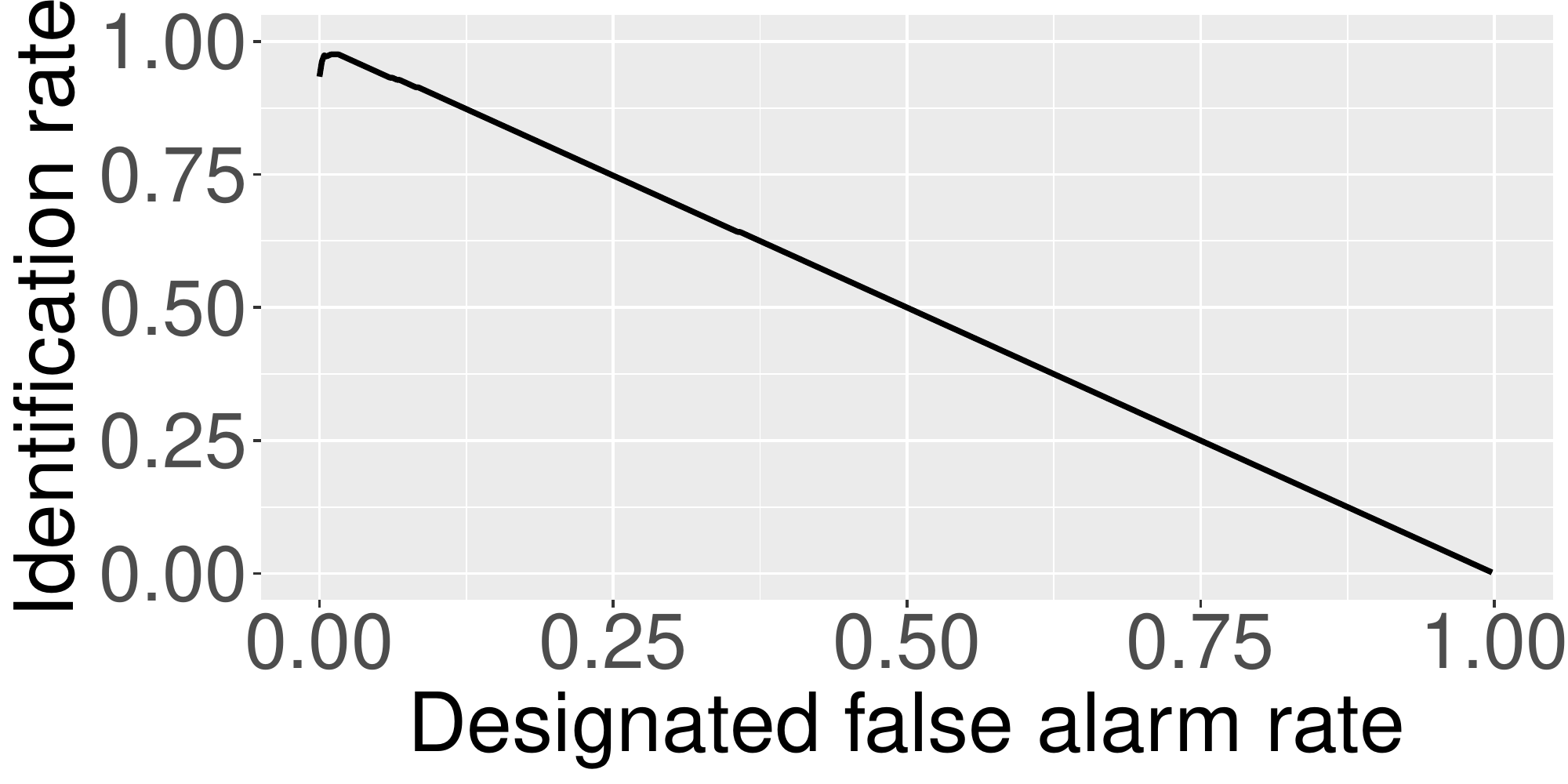}}
	\hfil
	\subfloat[]{
		\label{fig:s3_i_r}
		\includegraphics[width=0.23\textwidth]{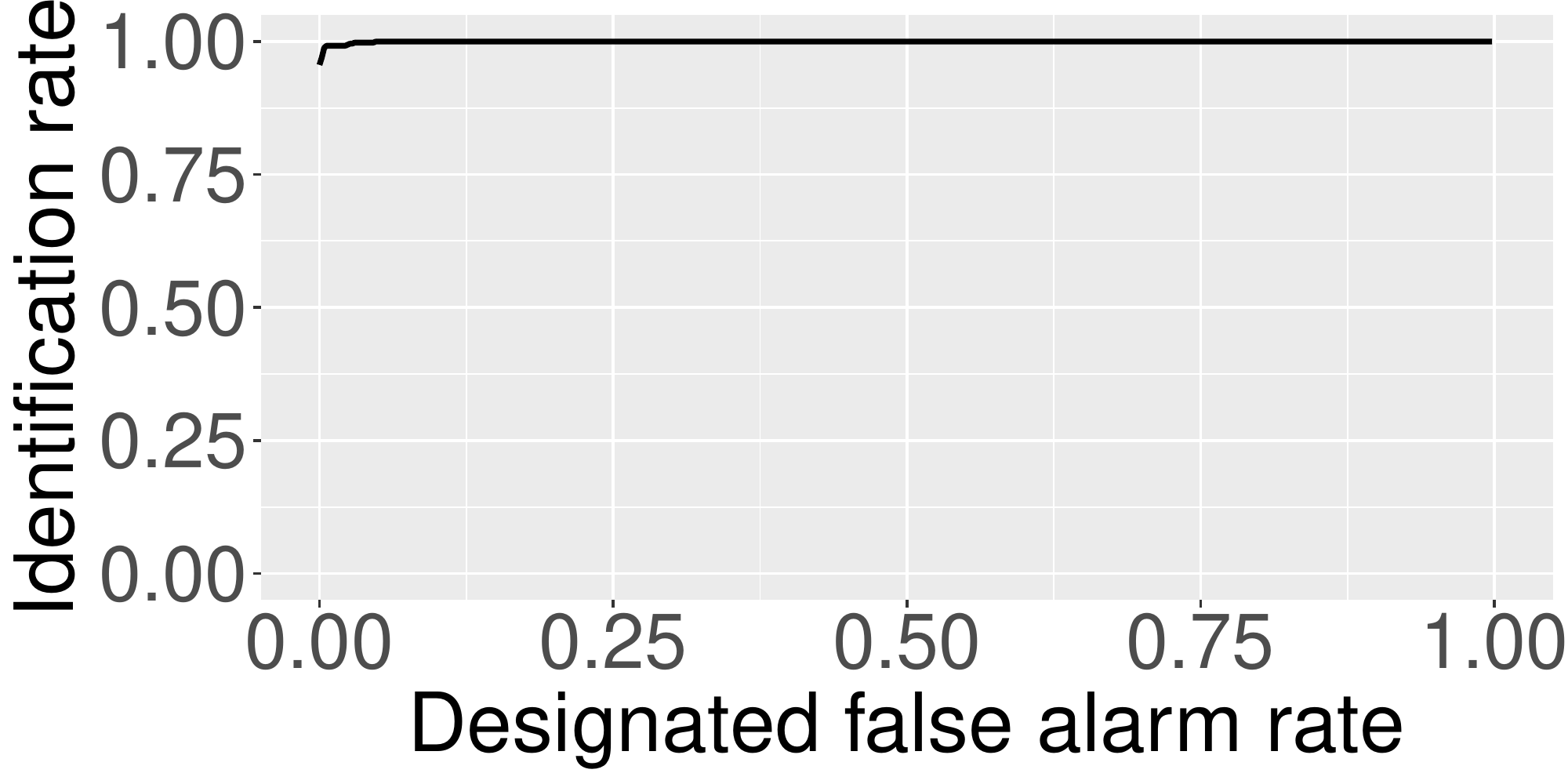}}
	\caption{Identification rate varies with the designated false alarm rate $r$ in each attack scenario. (a)~$S_{0}$ is attacked; (b)~$S_{1}$ is attacked; (c)~$S_{2}$ is attacked; (d)~$S_{3}$ is attacked.}
	\label{fig:iden_exp_r}

\end{figure*}

\begin{proof}
If there exist $i_{0},j_{0}$ such that $e_{i_{0},j_{0},n}=0$,
according to Eqn.~(\ref{eq:finding}), we can have:
\begin{equation}
e_{i_{0},j_{0},n}=(s_{i_{0}} \lor s_{j_{0}} \lor s_{n})=0
\end{equation}
which implies
\begin{equation}
s_{i_{0}}=s_{j_{0}}=s_{n}=0
\end{equation}
Now we only need to focus on the state of $s_{0},s_{1}, \ldots, s_{n-1}$.

For any $ i \in \{0,1,\cdots,i_{0}-1\}$,
\begin{equation}
e_{i,i_{0},l-1} = (s_{i} \lor s_{i_{0}} \lor s_{n}) = (s_{i} \lor 0 \lor 0) = s_{i},
\end{equation}
and  for any $ i \in \{i_{0}+1,i_{0}+2,\cdots,n-1\}$,
\begin{equation}
e_{i_{0},i,n} = (s_{i_{0}} \lor s_{i} \lor s_{n}) = (0 \lor s_{i}\lor 0) = s_{i}.
\end{equation}
\end{proof}

Lemma~\ref{thrm:lem_1} shows that we can identify the sensor state vector $\textbf{\textit{s}}$ if at least one element in $\textbf{\textit{e}}$ is $0$. For the case where all elements in $\textbf{\textit{e}}$ are $1$s, we have the following lemma.

\begin{lemma}\label{thrm:lem_2}
In a system with $n+1$ sensors, when there are no more than $n-2$ sensors being compromised simultaneously, if the elements of $\textbf{\textit{e}}$ are all $1$s, then $s_{n}=1$.
\end{lemma}
 
\begin{proof}
Since there are no more than $n-2$ sensors being attacked, the states of at least three sensors are $0$s. If $S_{n}$ is normal, then there exist $i^*$ and $j^*$, where $i^*<j^*$, such that sensors $S_{i^*}$ and $S_{j^*}$ are normal, i.e., $s_{i^*}=s_{j^*}=s_{n}=0$. In this case, $e_{i^*,j^*,n}=0$, which contradicts the fact that all the elements of $\textbf{\textit{e}}$ are $1$s. Therefore, $s_{n}=1$.
\end{proof}

\begin{algorithm}[t]
\caption{\textit{Identification of sensor state vector $\textbf{\textit{s}}$ }}\label{algo:iden_s}
\KwIn {$\textbf{\textit{D}}$: a list of sensor data with length equal to $|\textbf{\textit{D}}|$, where $D_{0}$ is point cloud and the rest are images; $\boldsymbol{\theta}$: a list of thresholds.}
\KwOut{$\textbf{\textit{A}}$: the list of compromised sensors.}
\setcounter{AlgoLine}{0}

select $D_{n}$ as reference image, where $n=|\textbf{\textit{D}}|-1$\;

calculate $\textbf{\textit{e}}$ using Algorithm~\ref{algo:cal_e} with $\textbf{\textit{D}}$ and $\boldsymbol{\theta}$\;

\If{\text{\normalfont there exists 1 in $\textbf{\textit{e}}$}}{
    \If{\text{\normalfont there exists 0 in $\textbf{\textit{e}}$}}{
        \tcp{use Lemma~\ref{thrm:lem_1} to infer $\textbf{\textit{s}}$}
        find $i_{0},j_{0}$ which satisfy $e_{i_{0},j_{0},n}=0$\;
        
        \For{$i\gets{0}$ \KwTo $n-1$}{
            assign $\min(i, i_0)$ to $k_1$, $\max(i, i_0)$ to $k_2$\;
            
            \If{\text{$e_{k_1,k_2,n}=1$ {\normalfont \textbf{and}} $k_1 \neq k_2$}}{
                sensor state $s_{i}$ is 1\;
                
                push attacked sensor $S_{i}$ into $\textbf{\textit{A}}$\;
            }
        }
    }
    \Else{
        \tcp{use Lemma~\ref{thrm:lem_2} to infer $s_{n}$}
        sensor state $s_{n}$ is 1\;
        
        push attacked sensor $S_{n}$ into $\textbf{\textit{A}}$\;
        
        \tcp{infer rest sensors recursively}
        \If{$n>3$}{
            remove $D_{n}$ from $\textbf{\textit{D}}$\;
            
            go to line $1$ to rerun with updated $\textbf{\textit{D}}$\;
        }
    }
}
\Return{$\textbf{\textit{A}}$}\;

\end{algorithm}

When all elements of $\textbf{\textit{e}}$ in a system with $n+1$ sensors are $1$, though we cannot directly find out the states of all sensors, Lemma~\ref{thrm:lem_2} can identify the last sensor's state. After that, we can virtually remove the last sensor and consider a system with $n$ sensors. We recalculate the $\textbf{\textit{e}}$ by Algorithm~1 for such $n$ sensors, then determine $\textbf{\textit{s}}$ according to Lemma~\ref{thrm:lem_1} and Lemma~\ref{thrm:lem_2}. We repeat this process until the states of all sensors are identified. We also present this identification algorithm as pseudocode in Algorithm~\ref{algo:iden_s} that takes the same inputs as Algorithm~\ref{algo:cal_e} and outputs a list of attacked sensors. Specifically, it begins with calculating $\textbf{\textit{e}}$ using Algorithm~\ref{algo:cal_e} with the inputs, and then infers the sensor state vector $\textbf{\textit{s}}$ using Lemma~\ref{thrm:lem_1} when some elements in $\textbf{\textit{e}}$ are $0$s. When all elements in $\textbf{\textit{e}}$ are $1$s, it first infers $s_{n}$, the state of the reference camera, using Lemma~\ref{thrm:lem_2}, and then excludes the data of the reference camera from the inputs and infers the rest of sensor states by rerunning Algorithm~\ref{algo:iden_s} with the updated inputs.

\subsection{Experiments}

\setlength{\textfloatsep}{18pt}

\begin{table}[t]
    \centering
	\caption{Identification rate at particular values of the designated false alarm rate in attack scenarios}\label{tbl:iden_exp_r}
	\begin{tabular}{|l|c c c c|l|}
	    \hline
	        \multirow{2}{1em}{$~r$}& \multicolumn{4}{c|}{Attacked Sensor} & \multirow{2}{0em}{Average} \\
		\cline{2-5}
		    & $S_{0}$ & $S_{1}$ & $S_{2}$ & $S_{3}$ & \\
		\hline
    	\hline
    	$0\%$ & $\textbf{99.80\%}$ & $96.40\%$ & $93.40\%$ & $95.60\%$ & $96.30\%$ \\
    	$1\%$ & $98.80\%$ & $\textbf{97.00\%}$ & $\textbf{97.60\%}$ & $99.20\%$ & $\textbf{98.15\%}$ \\
    	$2\%$ & $98.00\%$ & $96.80\%$ & $97.20\%$ & $99.20\%$ & $97.80\%$ \\
    	$3\%$ & $96.60\%$ & $96.00\%$ & $96.20\%$ & $99.80\%$ & $97.15\%$ \\
    	$5\%$ & $94.80\%$ & $94.00\%$ & $94.20\%$ & $\textbf{100\%}$ & $95.75\%$\\
    	\hline
	\end{tabular}

\end{table}

\subsubsection{Setup}

To verify the effectiveness and evaluate the precision of our identification scheme, we conduct substantial experiments. Since the identification scheme is designed to be in a form of recursion, the experiments here are for the base case where a system consists of one LiDAR ($S_{0}$) and $n=3$ cameras ($S_{1}$, $S_{2}$, and $S_{3}$). And according to the constraint, there is at most one attacked sensor in the system.

In the experiments, we consider the normal scenario and all possible attack scenarios where each one of the four sensors gets attacked. We use the data of one LiDAR and three cameras from the customized KITTI raw dataset~\cite{geiger13vision}. To generate the compromised sensor data for each attack scenario, we do the same as in Section~\ref{subsubsec:setup}. The depth estimation model is provided by~\cite{wang19pseudo}. As for metrics, we measure the disparity error distribution and the rate of correct identification for each attack scenario.

\subsubsection{Results}

We present the experimental results in Fig.~\ref{fig:iden_exp_d}, Fig.~\ref{fig:iden_exp_r}, and Table~\ref{tbl:iden_exp_r}. In Fig.~\ref{fig:iden_exp_d}, the three rows of sub-figures represent the distributions of disparity error $E_{0,1,3}$ (1st row), disparity error $E_{0,2,3}$ (2nd row), and disparity error $E_{1,2,3}$ (3rd row), respectively. And the columns in Fig.~\ref{fig:iden_exp_d} represent five scenarios.

In Fig.~\ref{fig:iden_exp_d}, we can first compare the disparity errors in each row. Similar to the results in the last section, the results clearly illustrate that the distribution of disparity errors can help to detect whether there is an attack. For example, in the first row for $E_{0,1,3}$, when there is no attack on $S_0$, $S_1$, and $S_3$, the disparity errors (cyan bars) are mostly less than $23\%$. By comparison, the disparity errors (red bars) are larger than $23\%$ when any of the sensors is attacked. Such results affirm the  feasibility and correctness of defining attacks according to the disparity error state. Three sub-figures in each of the five attack scenarios clearly show that there is a unique pattern of the combination of $E_{0,1,3}$, $E_{0,2,3}$, and $E_{1,2,3}$ for each attack scenario. For instance, when there is no attack launched, the three disparity errors are all within certain boundaries, representing the disparity error state vector $\textbf{\textit{e}}=[0,0,0]$. And if any one of the sensors is attacked, the disparity errors involving that sensor will exceed boundaries, leading to another unique $\textbf{\textit{e}}$.

In Fig.~\ref{fig:iden_exp_r}, we vary the thresholds by adjusting the designated false alarm rate $r$ and compute the corresponding identification rate in four attack scenarios. It is obvious that, if the attacked sensor is not the reference, then the identification rate drops linearly when $r$ increases from $0$ to $1$. On the other hand, when the reference sensor $S_3$ is attacked, then the identification rate remains very close to $100\%$. 

Such observations suggest that choosing a small $r$ may lead to the best performance overall. To identify the best $r$, we conduct some experiments to investigate the impact of $r$, when it is within the range of $0$ to $5\%$. The numerical results are shown in Table~\ref{tbl:iden_exp_r}. As we can see, the best average identification rate for the four attack scenarios occurs at $r=1\%$.

\subsection{Discussion}

Though the identification method of our framework can accurately identify attacked sensors, it is limited to the condition where no more than $n-2$ sensors are attacked simultaneously in a system with $n+1$ sensors. We plan to address this limitation through cross-vehicle sensing data validation in our future studies.

\begin{figure}[t]
	\centering
	\subfloat[]{
		\label{fig:l_dr}
		\includegraphics[width=0.45\textwidth]{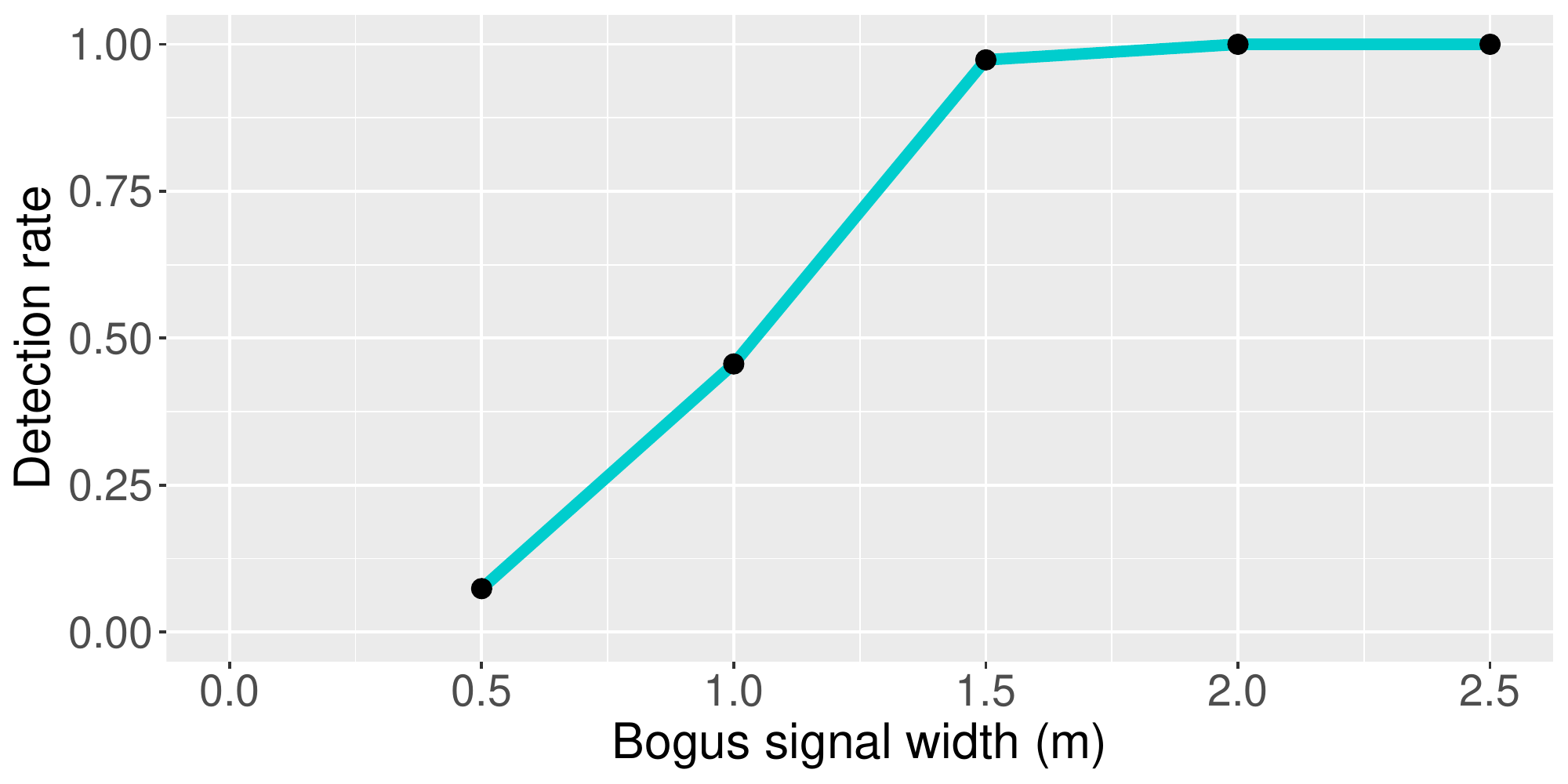}}
	\hfil
	\subfloat[]{
		\label{fig:l_ap}
		\includegraphics[width=0.45\textwidth]{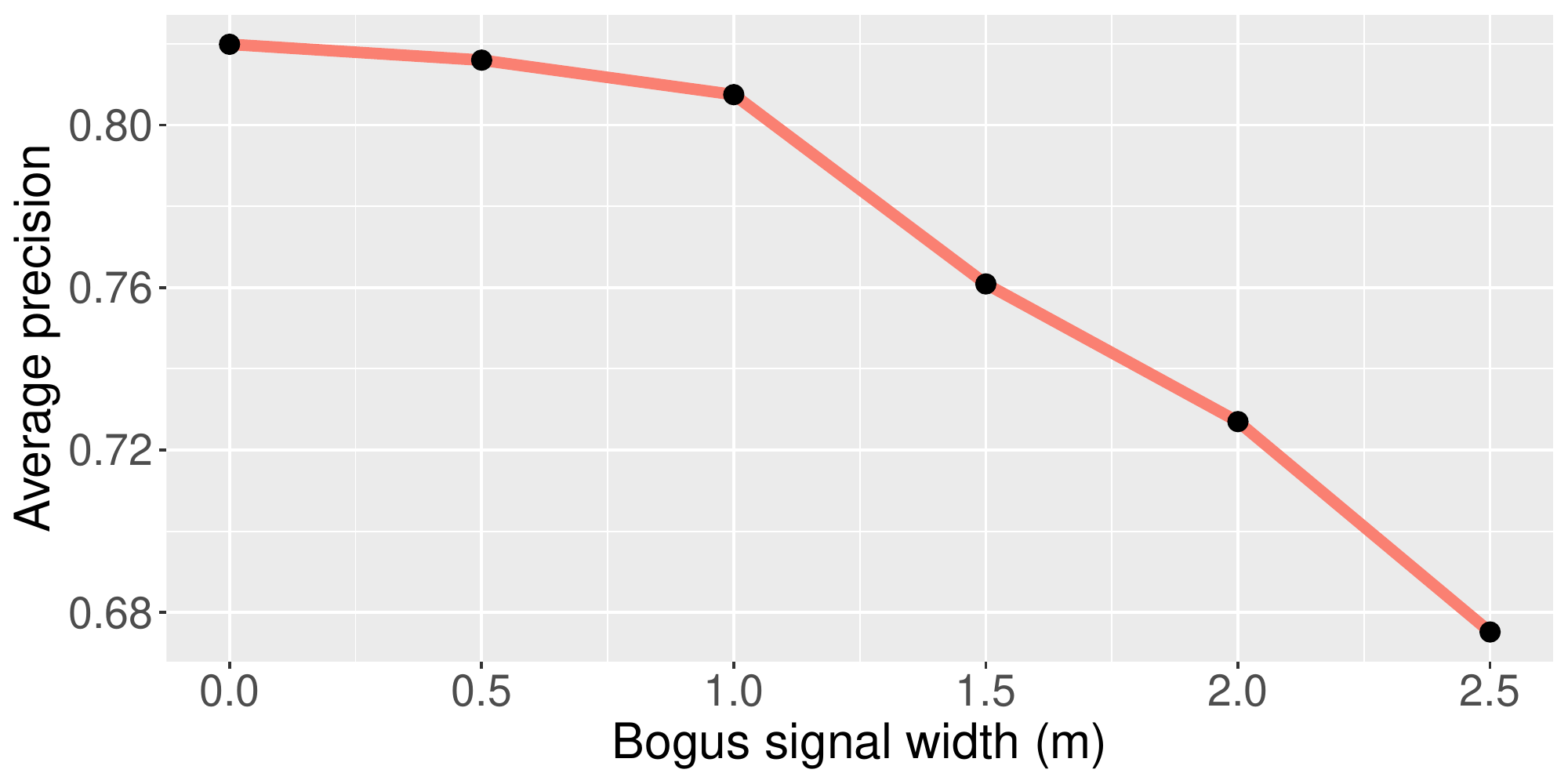}}
	\caption{Sensitivity of our framework and AP of PointRCNN when the attack is against LiDAR. (a)~Detection rate for attacks against LiDAR varies with the width of bogus signal; (b)~Average precision of PointRCNN varies with the width of bogus signal.}
	\label{fig:defense_l}

\end{figure}

\section{Framework Sensitivity}
\label{ch3:sec.sensitivity}

With the best designated false alarm rate $r$ determined, we conduct further experiments to investigate how sensitive our framework is, namely, for optical attacks with what range of settings (width of bogus signal, size of facula) our framework works effectively. Empirically, the milder optical attacks are, the more difficult they get detected. Meanwhile, we also measure how much the perception function of AVs is influenced by the optical attacks with different settings using state-of-the-art object detection algorithms. Those algorithms usually possess a certain degree of resistance to minor optical attacks, so our framework does not have to be universally sensitive.

In this section, we use experiments to demonstrate that our framework has excellent sensitivity to the attacks on LiDAR with settings that object detection model cannot overcome. As for the attacks on camera, our framework is also sensitive in most cases, but shows limit when the attack is too mild. The experiments consist of two parts: the first part is for the sensitivity to the attacks on LiDAR, and the second part for the sensitivity to the attacks on camera.

\subsection{Metrics}

To measure the sensitivity, we use the detection rate of our framework with $r=1\%$, since detection rate can also reflect the performance of identification procedure. As described in Section~\ref{ch3:sec.identification}, the identification procedure of our framework is directly developed upon multiple detection processes, so the identification rate is highly correlated with the detection rate.

As for evaluating the corresponding performance of object detection algorithms used in AVs, we follow the KITTI object detection benchmark~\cite{geiger12are} and calculate the average precision of vehicle detection with IoU threshold equal to $70\%$.

\begin{figure}[!t]
	\centering
	\subfloat[]{
		\label{fig:c_dr}
		\includegraphics[width=0.45\textwidth]{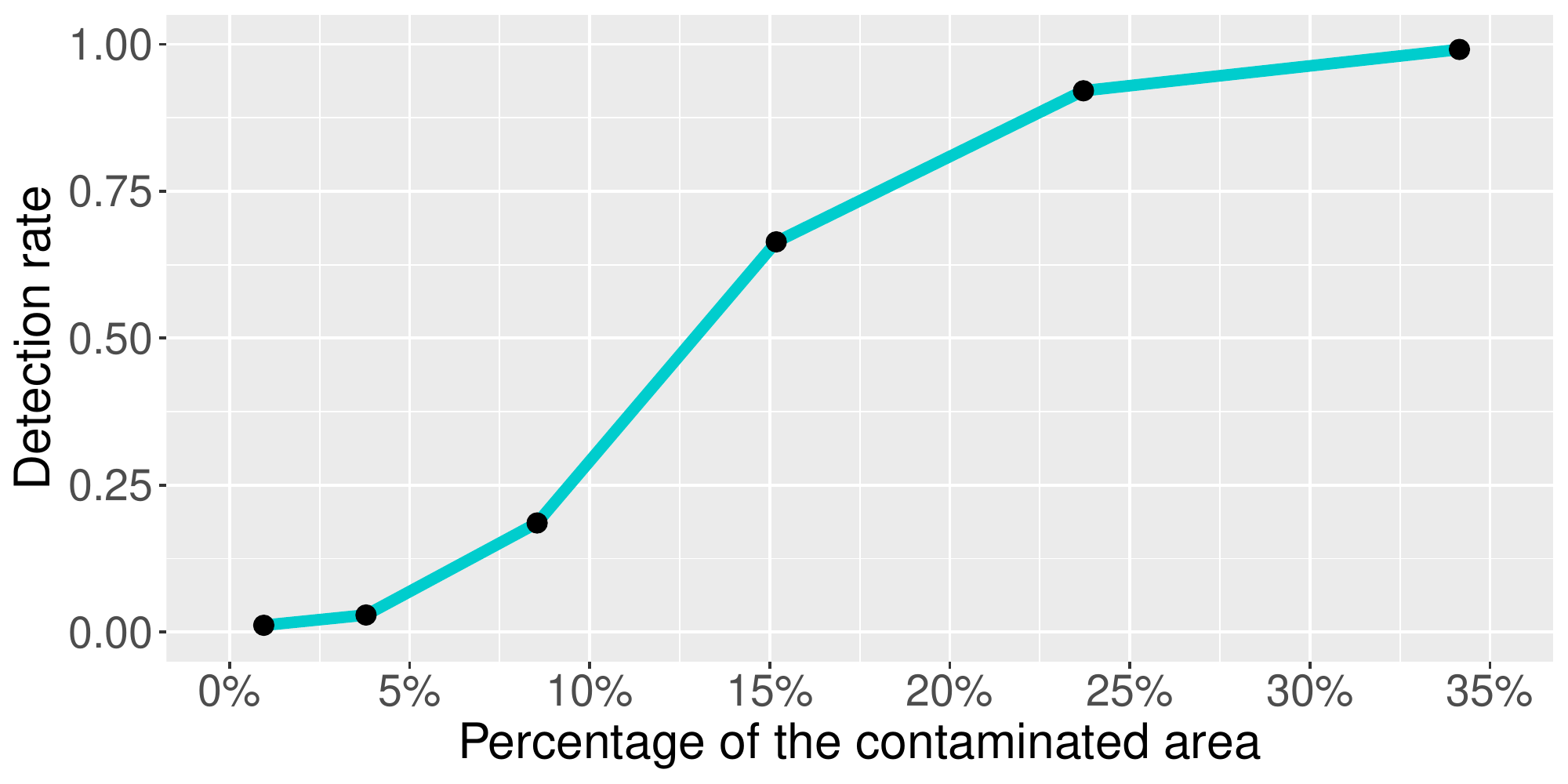}}
	\hfil
	\subfloat[]{
		\label{fig:c_ap}
		\includegraphics[width=0.45\textwidth]{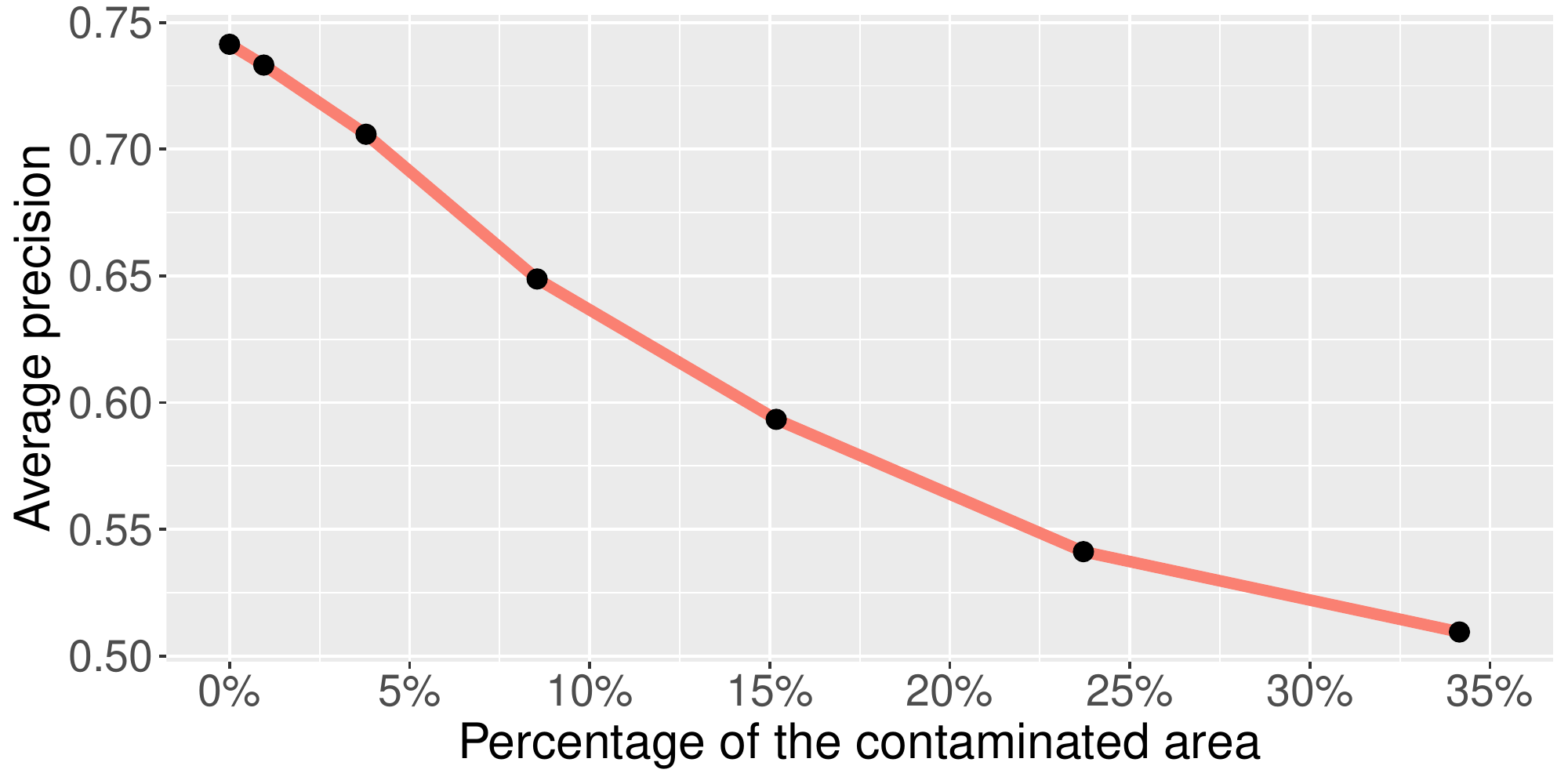}}
	\caption{Sensitivity of our framework and AP of YOLOv3 when the attack is against camera. (a)~Detection rate for attacks against camera varies with the percentage of the area contaminated by facula; (b)~Average precision of YOLOv3 varies with the percentage of the area contaminated by facula.}
	\label{fig:defense_c}

\end{figure}

\subsection{Experimental Setup}

We conduct our experiments on the dataset provided in the KITTI object detection benchmark~\cite{geiger12are} which contains sensor data of one LiDAR and two cameras. As described in Section~\ref{subsubsec:setup}, we divide the labeled part of the dataset into training set and test set. The training set is used to train the object detection models, while the test set is for generating compromised sensor data.

To find out the sensitivity of our framework to the optical attacks on LiDAR, we produce five affected point clouds for every set of sensor data in the test set. The five affected point clouds contain a bogus signal with a width of 0.5 meter, 1.0 meter, 1.5 meters, 2.0 meters, and 2.5 meters, respectively. And the object detection algorithm chosen for this part of experiments is PointRCNN~\cite{shi19pointrcnn}, a state-of-the-art 3D object detection algorithm that takes a point cloud as input.

In term of the experiment setup for evaluating sensitivity to the attacks on camera, we generate six pairs of compromised stereo pictures for each set of sensor data in the test set. The left pictures of the six pairs are overlaid with a Gaussian facula with radius of 37.5 pixels, 75 pixels, 112.5 pixels, 150 pixels, 187.5 pixels, and 225 pixels, respectively. And the corresponding percentages of the contaminated area in images are 0.95\%, 3.79\%, 8.54\%, 15.18\%, 23.71\%, and 34.15\%. The object detection algorithm for this part of experiments is YOLOv3~\cite{redmon18yolov3}, which is one of the most popular real-time object detection algorithms using images as input.

In the experiments, we feed the compromised sensor data to our framework and the selected object detection model, and then evaluate them via the aforementioned metrics. The PSMNet model used in the framework is provided by~\cite{wang19pseudo}.

\subsection{Experiment Results}

\subsubsection{Sensitivity to the Attacks on LiDAR}

As shown in Fig.~\ref{fig:defense_l}, with the increase of the width of the bogus signal, the detection rate of our framework surges, while the average precision of PointRCNN declines. The AP of PointRCNN decreases very slightly when the width of the bogus signal is within $1.0$ meter, which implies that PointRCNN exhibits some resistance to minor disturbing signals. On the other hand, when the size of the bogus signal is larger, the average precision starts dropping rapidly. However, it should be noted that when the width is at $1.5$ meters, the total decline of AP is only $0.06$, while the detection rate of our framework already reaches nearly $100\%$. These results show clearly that our framework is highly sensitive to the attacks on LiDAR that cannot be mitigated by object detection algorithms.

\subsubsection{Sensitivity to the Attacks on Camera}

The experiment results for this part are illustrated in Fig.~\ref{fig:defense_c}. The tendencies of the detection rate of our framework and the AP of object detection model are similar as those in the first part of the experiments. Particularly, when the percentage of the contaminated area is within the range of $5\%$ to $15\%$, although our framework has a small detection rate, the AP of YOLOv3 maintains at a high level, which means that the perception has not been compromised due to small attacks. When the percentage of the contaminated area in images increases to $23.71\%$, the detection rate of our framework surpasses $90\%$. In the meantime, the AP of YOLOv3 drops significantly to $54.12\%$. These results suggest that our framework has a strong sensitivity to the attacks on camera when the contaminated area in images is greater than $20\%$. Less than that, the framework may show some limitations.

\section{Conclusion}
\label{ch3:sec.conclusion}

In this chapter, we have systematically investigated the mitigation of attacks on optical devices (LiDAR and camera) that are essential to perform accurate object and event detection and response (OEDR) tasks in autonomous driving systems. Specifically, we proposed a framework to detect and identify sensors that are under attack. For the attack detection, we considered two common three-sensor systems, (1) one LiDAR and two cameras, and (2) three cameras, and we developed effective procedures to detect any attack on each of them. Using real datasets and the state-of-the-art machine learning model, we conducted extensive experiments confirming that our detection scheme can detect attacks with high accuracy and a low false alarm rate. Based on the detection models, we further developed an identification model that is capable of identifying up to $n-2$ attacked sensors in a system with one LiDAR and $n$ cameras. For the identification procedure, we proved its correctness and used experiments to validate its performance. At last, we investigated the sensitivity of our framework and showed its excellence in this aspect.

\chapter{Sensor Data Validation for Connected Autonomous Vehicles}
\label{ch4:sensor}

\section{Introduction}

In the foreseeable future, it is expected that autonomous vehicles will be a part of our daily traffic. By then, autonomous vehicles will have already started to evolve again, from single-vehicle systems to collaborative multi-vehicle systems. The combination of autonomous driving, Multi-access Edge Computing (MEC), and 5G network technology can be a strong candidate for the V2X-based collaborative autonomous driving system, which consists of connected vehicular nodes and base station nodes.

However, the collaboration of autonomous vehicles does not eliminate the security threats to individual autonomous vehicles, but raises even higher requirements for safety and security instead, since the environment perception results produced by one autonomous vehicular node can be shared among all the connected vehicles in the same collaborative system. In other words, one perception error could lead to disasters for the whole collaborative system. Therefore, security and safety threats to autonomous driving should be reconsidered in the multi-vehicle perspective.

In this study, we focus on detecting the threat of optical attacks to LiDARs for connected autonomous vehicles. Although there are previous detection methods designed for single autonomous vehicles, e.g., the method proposed by Zhang et al.~\cite{zhang21detecting}, they only suit single-vehicle systems and do not take advantages of the multi-vehicle collaborative system. Simply deploying the previous countermeasures on every vehicular node in the collaborative system does not produce more safety bonus.

In this chapter, we propose a data validation method by leveraging point clouds from multiple neighboring vehicular nodes in a collaborative autonomous driving system. To perform the data validation task for multiple connected autonomous vehicles, there are two major challenges: (1) since huge amount of point clouds are generated every second, transmitting them entirely would pose a great burden on the 5G network; (2) the scans of objects in point clouds, e.g., vehicles, are usually severely incomplete at the unlit side, causing inaccurate validation. To tackle the first challenge, we propose to only validate the scans within validation regions which are the proposals produced by point cloud region proposal networks, such as PointNet++\cite{qi17pointnet++}. This largely reduces the transmission overhead and does not overlook potential information distortions caused by optical attacks. In regard to the second challenge, we fill the scan of objects with a symmetrical copy of it according to the symmetry of the objects. Furthermore, we conduct preliminary experiments to evaluate our method. And the results show that our method can detect optical attacks against LiDARs with a fair accuracy and a relatively low false positive rate for connected autonomous vehicles.

The rest of this chapter is organized as follows. In Section~\ref{ch4:sec:system}, we first briefly introduce the collaborative autonomous driving system that we consider in this study. Then, we present the threat model in Section~\ref{ch4:sec:threat}. Next, we elaborate on our data validation method in Section~\ref{ch4:sec:method} and demonstrate our experiments in Section~\ref{ch4:sec:experiments}. Finally, we conclude this chapter in Section~\ref{ch4:sec:conclusion}.

\section{System Model}
\label{ch4:sec:system}

In this preliminary study, we consider a V2X-based collaborative autonomous driving system deployed with 5G-based Multi-access Edge Computing (MEC) technology. The system consists of multiple connected vehicular nodes and base station nodes. As for the vehicular nodes, they are equipped with full autonomous driving capability primarily based on cameras and LiDARs. In the aspect of navigation, the vehicular nodes are equipped with GPS and high definition map (HD map), which empower them with good self-localization capability. In terms of the base station nodes, they are equipped with abundant computation power and can handle big network throughput. And the 5G network enables low latency communication among the system nodes.

\section{Threat Model}
\label{ch4:sec:threat}

In this study, we mainly consider the optical attacks against LiDARs. Basically, this type of attack is launched by emitting laser beams with the same frequency as used by the victim LiDARs, so that the point clouds generated by the victim LiDARs are injected with spoofed points~\cite{petit15remote,shin17illusion}. To spoof a large amount of points and allocate them according to specific shapes, the time sequence of the victim LiDARs must be given in advance, and the timing to emit laser beams for attacking must be accordingly planned carefully and controlled very precisely~\cite{cao19adversarial}. Due to the difficulty in launching such attacks, we consider the attack cases where the attacker only attacks one of the neighboring vehicular nodes at a time. And the consequence of such optical attacks is that the victim autonomous vehicular node could detect ghost objects, possibly causing wrong driving decisions.

\section{Data Validation Method}
\label{ch4:sec:method}

In this section, we elaborate on the details of our proposed data validation method. Specifically, we first briefly explain the preprocessing phase for data preparation. Then, we elaborate on selecting validation regions and completing scans within validation regions. Next, we introduce surface mesh generation and discretization. Finally, we explain the detection of the optical attacks against LiDARs through comparing the discretized surface meshes and thresholding. The first two steps are executed on vehicular nodes, while the last two steps are done by the base station node. The structure of our proposed method is illustrated in Fig.~\ref{fig:structure}.

\begin{figure*}[!t]
    \centering
    \includegraphics[width=0.9\textwidth]{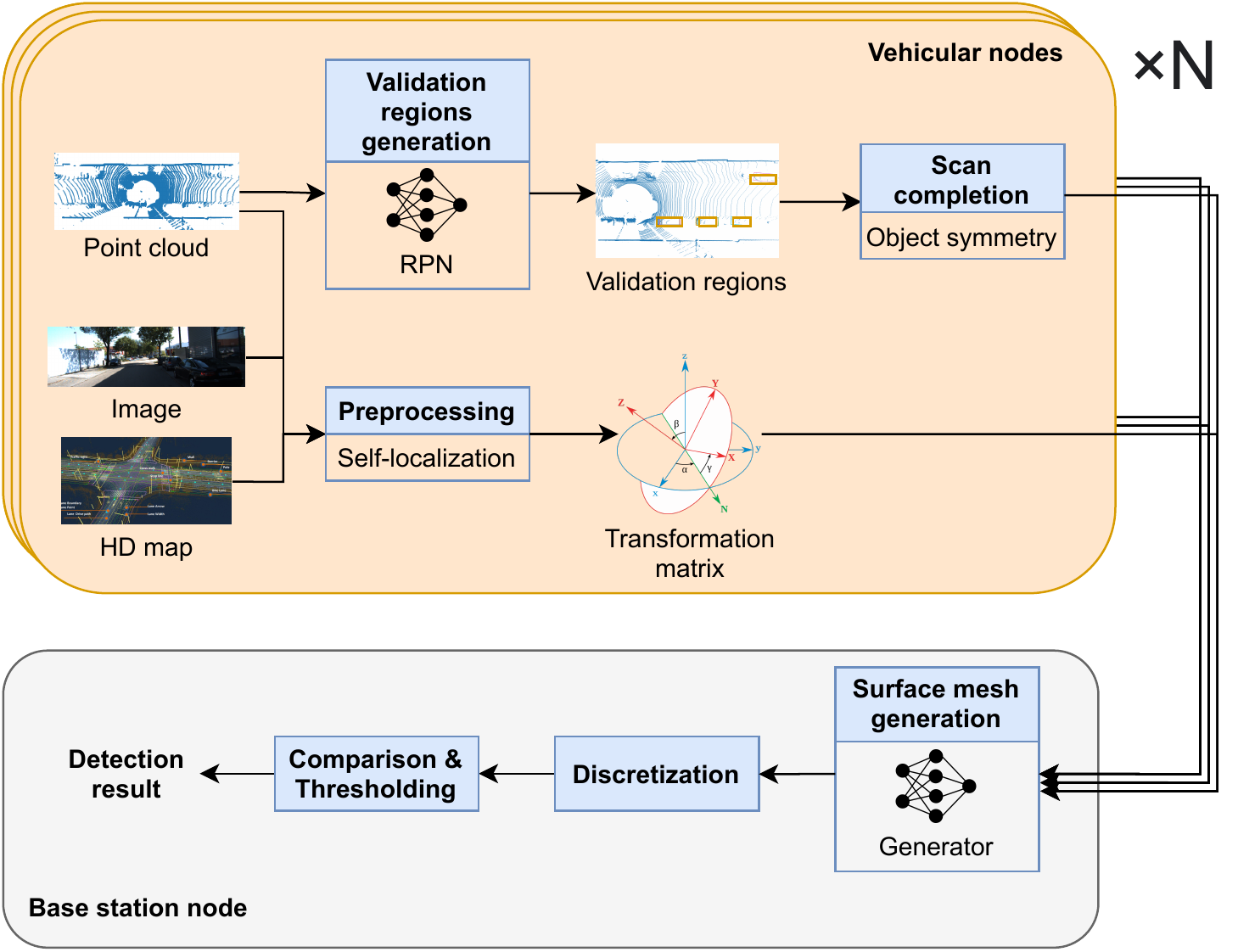}
    \caption{The structure of the data validation method for the multi-autonomous-vehicle scenario.}
    \label{fig:structure}
\end{figure*}

\subsection{Preprocessing}

As mentioned in Section~\ref{ch4:sec:system}, we consider connected autonomous vehicles in a 5G and MEC based collaborative system, in which vehicular nodes and base station nodes are equipped with computation capability. Since all of the harvested raw sensor data and the produced data from each autonomous vehicular node are based on the node's own coordinate system, the data must be aligned first before used for validation.

In this study, we convert all the coordinate systems of vehicular nodes and base station nodes of the collaborative system to the world coordinate system. We assume that autonomous vehicular nodes are equipped with high definition map (HD map) and capable of accurate self-localization. In other words, every autonomous vehicular node can calculate a transformation matrix $\textbf{Tr}=(\textbf{R}|\textbf{t})$ consisting of a translation vector $\textbf{t}$ and a rotation matrix $\textbf{R}$ using image data, point clouds, GPS data, and HD map, so that every vehicular coordinate system can be transformed to align with the world coordinate. In the subsequent phases, we use the transformation matrix $\textbf{Tr}$ to align LiDAR scans and validation regions.

\subsection{Validation Regions and Scan Completion}

\begin{figure}[!t]
    \centering
    \includegraphics[width=0.6\textwidth]{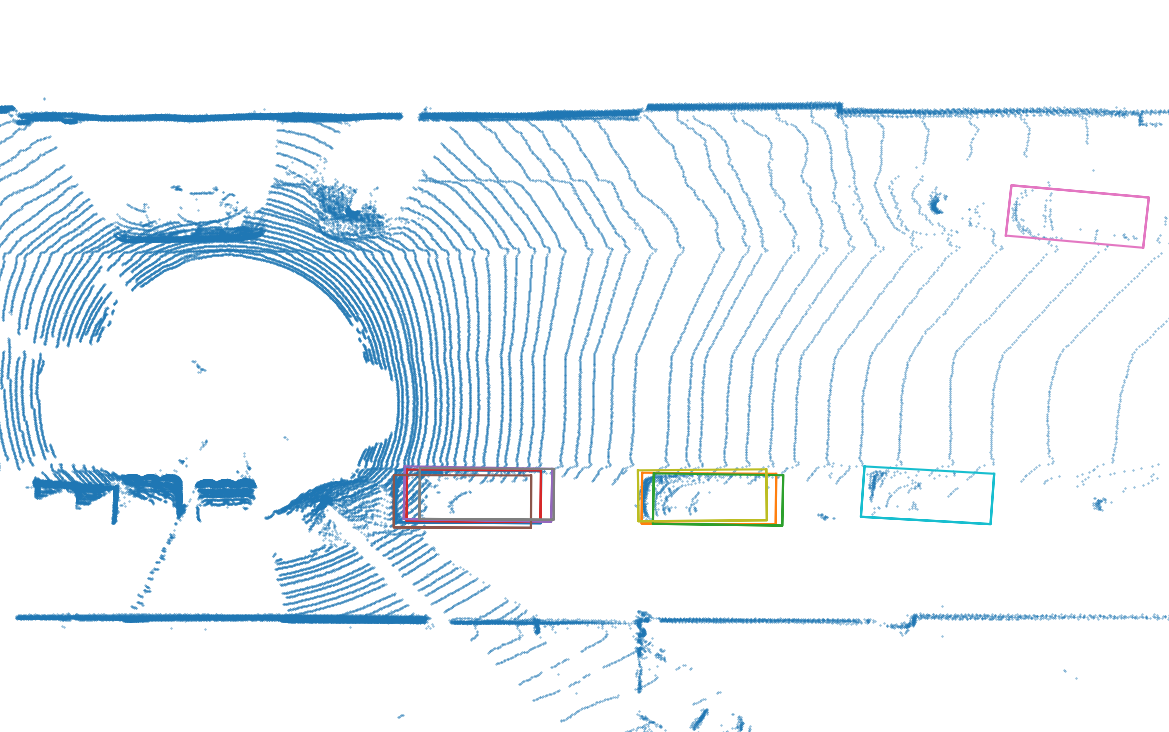}
    \caption{Validation regions (boxes) in the bird's eye view of a point cloud.}
    \label{fig:region}
\end{figure}

On one hand, to validate the genuineness of point clouds by leveraging sensor data sources of multiple neighboring vehicular nodes, the point clouds must be transmitted to a nearby MEC base station node to be processed together. On the other hand, transmitting entire frames of point clouds from each neighboring vehicular node to one nearby base station node causes great transmission overhead to the 5G network, considering the facts that huge amount of point clouds are generated every second. In addition, information distortions caused by optical attacks can appear anywhere in a point cloud, so we must scale down the data for validation and maintain information distortions still detectable at the same time.

To overcome this challenge, we use point cloud region proposal network (RPN), e.g., PointNet++~\cite{qi17pointnet++}, to generate proposals, and propose to only validate the scans within the proposals. The rationale of selecting the proposals as validation regions is three-fold. First, common paradigm of processing point clouds for autonomous driving is two-staged schemes consisting of a region proposal network and a regression network, so using proposals generated by an RPN for validation brings no extra computation overhead. Second, since the optical attacks can cause detection models to detect ghost objects or miss detecting real objects, and since all the final detected bounding boxes are chosen from generated proposals, validating scans only within proposals never overlooks potential information distortions caused by optical attacks. Third, generating proposals is very fast. According to Xu et al.~\cite{xu21soda}, the time for generating proposals with RPN is less than $10\%$ of the whole processing time of a two-staged point cloud 3D detection scheme.

\begin{figure}[!t]
    \centering
    \includegraphics[width=0.9\textwidth]{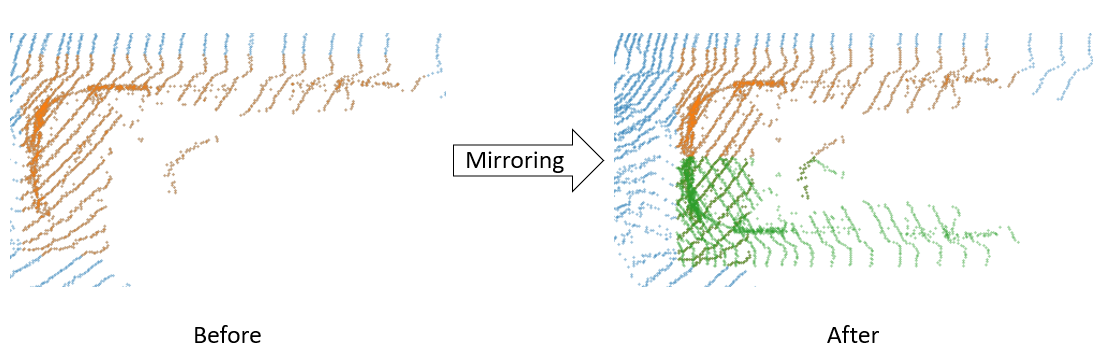}
    \caption{Points in the validation region before and after mirroring. The scan of the vehicle is severely incomplete at the unlit side (bottom right portion of the first figure) before the scan completion. After the mirroring operation according to the symmetry of the vehicle, a part of the void portion is filled with points (the second figure).}
    \label{fig:mirroring}
\end{figure}

In addition, when LiDARs scan objects, e.g., vehicles, only one side of objects is usually recorded in a frame of point cloud, resulting in zero information about the unlit side of objects, as shown in the first sub-figure in Fig.~\ref{fig:mirroring}. And this causes troubles in surface mesh generation and validation accuracy. The reason is that, due to the information vacuum of the unlit side of objects, the size of the surface meshes generated from the proposals is much smaller than the proposal size, so the area for comparison and thresholding (the overlap between two discretized surface meshes) is not large enough, causing inaccurate validation.

To overcome this challenge, we adopt a mirroring technique. Specifically, we make a symmetrical copy of the scan of objects in every proposal according to the object symmetry, and concatenate it with the original scan, which we call \emph{mirroring}, as shown in the second sub-figure in Fig.~\ref{fig:mirroring}. As we can observe, after the mirroring operation, a large portion of the information vacuum area is filled with points, which enlarges the validation area when comparing two discretized surface meshes, and in consequence boosts the validation accuracy.

To briefly summerize the first two steps, each vehicular node first performs preprocessing operation and calculate its transformation matrix. Then, they use point cloud region proposal network to generate proposals. Next, they send the proposal regions along with the transformation matrix to the base station node. After converting the received proposals to the world coordinate system and making a quick summary, the base station node distributes the complete set of proposal regions to each neighboring vehicular node. At last, each vehicular node crops the point cloud according to the complete set of proposal regions, performs mirroring operation within each proposal, and send all the proposals with cropped point cloud to the base station node.

\subsection{Surface Mesh Generation and Discretization}

After receiving the complete set of cropped point clouds from each neighboring vehicular node, the base station generate a surface mesh $sm_{i,j}$ for proposal region $pr_{i}$ and the vehicular node $vn_{j}$. There are many excellent algorithms for surface mesh generation from point cloud, e.g., Point2Mesh~\cite{hanocka20point2mesh}. The first reason why we converts point clouds to surface meshes is that surface meshes can approximately represent the semantic information contained in point clouds, which is the surface of objects measured by LiDARs, while a set of points in point clouds do not show semantics directly. And the second reason is that surface meshes can be transformed to distributions which can be efficiently compared with each other to obtain the distances among them, while directly comparing the point clouds poses great computational overhead to the base station node.

After the surface meshes are generated, we discretize the surface meshes into 2D grids with resolution of $0.1$ meter, the same resolution as used in~\cite{chen17multi}. For each grid, the value is computed as the height of the grid center. The purpose of the discretization is to further downscale the comparison space to save computation power.

\subsection{Comparison and Thresholding}

Finally, we perform the comparison and thresholding on the discretized surface meshes. The method is similar to that in Chapter~\ref{ch3:sensor}. First, for validating the data inside a proposal region $pr_{i}$, we need at least three discretized surface meshes, denoted as $dsm_{i,j_{1}}$, $dsm_{i,j_{2}}$, and $dsm_{i,j_{3}}$, from three different vehicular nodes, $vn_{j_{1}}$, $vn_{j_{2}}$, and $vn_{j_{3}}$. By comparing them with each other, we can obtain the distances $dist_{i,j_{1},j_{2}}$, $dist_{i,j_{1},j_{3}}$, $dist_{i,j_{2},j_{3}}$. For example, the distance between $dsm_{i,j_{1}}$ and $dsm_{i,j_{2}}$ is measured as:
\begin{equation}
    dist_{i,j_{1},j_{2}} = \frac{\sum_{g \in G} |dsm^{g}_{i,j_{1}} - dsm^{g}_{i,j_{2}}|}{|G|},
\end{equation}
\begin{equation}
    G = \{g : |dsm^{g}_{i,j_{1}} - dsm^{g}_{i,j_{2}}| > \epsilon\},
    \label{ch4:C}
\end{equation}
where $g$ denotes girds in $dsm$, $\epsilon$ is the average distance among discretized surface meshes of the ground, and $|G|$ denotes the size of $G$. Since every validation region contains scan of the ground, the equation~\ref{ch4:C} indicates that we only take into account the grids where the grid-level distance between two discretized surface meshes is larger than $\epsilon$, to avoid the final distance being severely normalized by the grids containing only scan of the ground. Then, we identify the vehicular node that is under attack by observing which distances surpass the threshold $\theta$. For example, if the vehicular node $vn_{j_{1}}$ is attacked, the distances $dist_{i,j_{1},j_{2}}$ and $dist_{i,j_{1},j_{3}}$ would be larger than $\theta$, while $dist_{i,j_{2},j_{3}}<\theta$. For the cases where more than three vehicular nodes are involved, we can apply the algorithms in Section~\ref{ch3:sec.identification} to identify the attacked nodes.

Similar as in Chapter~\ref{ch3:sensor}, the threshold $\theta$ is determined offline in attack-free cases, based on the value distribution of the distances and the designated false positive rate $r$. To this end, a large amount of samples of normal distances (generated in attack-free cases) are needed. And $r$ represents the percentage of the samples that are falsely considered as outliers. By setting $r$, we can control $\theta$. We can write the relation between the $\theta$ and $r$ as:
\begin{equation}
    \frac{\text{\# samples of } dist > \theta}{\text{\# samples of } dist} = r.
\end{equation}

\section{Experiments}
\label{ch4:sec:experiments}

In this section, we conduct some preliminary experiments to evaluate our data validation method. At first, we briefly introduce the dataset we use for the experiments and the experimental setups. Then, we present the preliminary experiment results in detail. And the preliminary results show that our data validation method can detect the optical attacks against LiDARs with a fair accuracy.

\subsection{Dataset}

In this preliminary study, we use the KITTI raw dataset~\cite{geiger13vision} to simulate the scenario of multiple connected autonomous vehicles. To this end, we specifically choose the recordings that only contain static objects from the dataset. And each time we select three frames of point clouds in temporal order from the recordings. The time interval between adjacent two frames is 0.5 seconds. Since only the ego-vehicle was moving when point clouds were recorded, we can use such three selected frames of point clouds to simulate the three frames of point clouds recorded at the same time from three vehicular nodes at different positions.

\subsection{Setups}

In this preliminary study, we consider the attack cases where the attacker attacks one of the vehicular nodes by injecting fake points into its point clouds, so that the victim vehicular node may detect ghost vehicles. To forge such compromised point clouds, we randomly select LiDAR scans of real vehicles and allocate them at random positions in attack-free point clouds. As for the point cloud region proposal network used in our method, we choose PointNet++~\cite{qi17pointnet++}. And in terms of the surface mesh generator, we select Point2Mesh~\cite{hanocka20point2mesh}. We sample the distances among the discretized surface meshes of the ground and set $\epsilon=0.04$. We conduct the experiments for both the attack-free cases and the attack cases, and report the results in the next subsection.

\subsection{Results}

\begin{figure*}
	\centering
	\subfloat[]{
		\label{distribution}
		\includegraphics[width=0.48\textwidth]{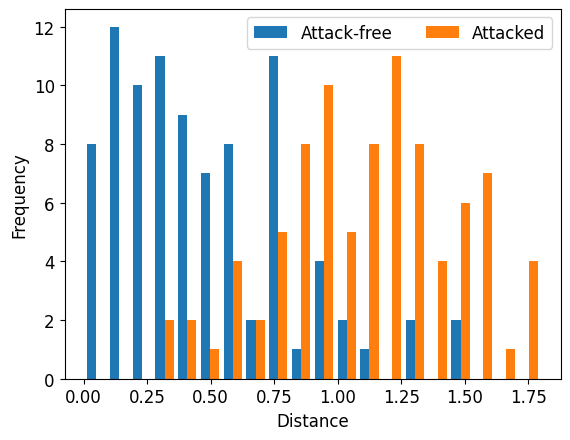}}
	\hfil
	\subfloat[]{
		\label{detection_rate}
		\includegraphics[width=0.48\textwidth]{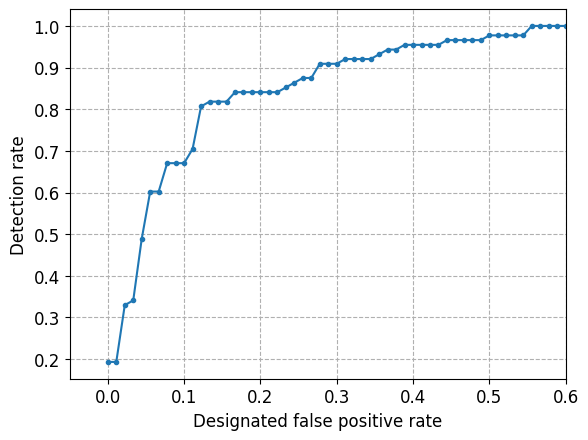}}
	\caption{Preliminary experiment results of the data validation for the the multi-autonomous-vehicle scenario. In (a), we show the distance distribution of the discretized surface meshes in the attack-free cases and the attacked cases. In (b), we illustrate how the detection rate varies with the designated false positive rate.}
	\label{experiment}
\end{figure*}

We present the preliminary experiment results in Fig.~\ref{experiment}. In Fig.~\ref{distribution}, we compare the distance distributions of the discretized surface meshes between the attack-free cases (blue bars) and the attacked cases (orange bars). We first observe that the distances among the discretized surface meshes are smaller than $0.8$ in the majority of attack-free cases. By comparisons, the distances in most attacked cases are larger than $0.8$. These results suggest that our data validation method demonstrates good sensitivity, so that there is only a relatively small overlap between the distribution of attack-free cases and the distribution of attacked cases.

In Fig.~\ref{distribution}, we adjust the threshold $\theta$ used for the declaration of attacks by changing the designated false positive rate $r$. As we can observe, when $r=0.1$, the detection rate is very close to $70\%$. And when we set $r$ to $0.3$, the detection rate hits over $90\%$. These results indicate that our data validation method can achieve a fair detection rate when the false positive rate is relatively low.

\section{Conclusion}
\label{ch4:sec:conclusion}

In this chapter, we addressed the data validation task for connected autonomous vehicles. Specifically, we proposed a data validation method which leverages multiple point clouds from different neighboring vehicular nodes. On one hand, we tackled the challenge of transmission overhead caused by large data size by only validating scans within validation regions which are the proposals produced by point cloud RPN. On the other hand, we overcome the challenge of severely incomplete scans of objects within proposals by filling them with points according to the symmetry of objects. At last, we conducted some preliminary experiments and demonstrated the results of comparing the distance distributions of discrete surface meshes between attack-free cases and attacked cases and showed how detection rate varies with the designated false positive rate. The results showed that our method can detect optical attacks against LiDARs with a fair accuracy and a relatively low false positive rate for multiple connected autonomous vehicles.

\chapter{Impact Evaluation of Adversarial Attacks on Driving Safety}
\label{ch5:impact}

\section{Introduction}
\label{ch5:sec.introduction}

Over the past decade, autonomous driving has gained significant developments and demonstrated its great commercial potentials~\cite{davies18waymo, hawkins19tesla}. The commercial potentials have attracted enormous investment as well as various malicious attacks~\cite{ren19security, wyglinski13security}, for example,  close-proximity sensor attacks, remote cyberattacks, perturbation attacks, and patch attacks.

Environment perception and other tasks of autonomous driving systems heavily rely on deep learning models. Researchers have demonstrated that adversarial examples, which are originally designed to affect general-purpose deep learning models, can also be used to cause malfunctions in autonomous driving tasks~\cite{zhou18deep, ranjan19attacking, eykholt18robust, lu17adversarial, metzen17universal, song18physical, chen18shapeshifter, zhang20adversarial, mathew20monocular, cao19adversarial}. In these studies, researchers usually use the decline of accuracy, or the erroneous rate increase of the deep learning models, to measure the effectiveness of attacks. Amplified by media reports, these attacks are casting cloud and posing psychological barriers to the broader adoption of autonomous driving~\cite{slovick2017security}. 

From the perspective of autonomous driving, however, the ultimate concern is driving safety. Without a doubt, the inaccurate detection results of a deep learning model in the presence of attacks may impact driving safety, and in some situations, misdetection of traffic signs~\cite{eykholt18robust} might have disastrous consequence. Nevertheless, \textit{driving safety is a combined effort of many factors in a dynamic environment}, and the deteriorated model performance does not necessarily lead to safety hazards. The linkage between the performance of a deep learning model under adversarial attacks and driving safety is not studied in the literature. In particular, there are no clear answers to the following questions: 1) Does the precision decline or the erroneous rate increase of the deep learning models under attacks represent their robustness in regard to driving safety of autonomous vehicles? In other words, does a larger decline in accuracy of an attacked deep learning model indicate a higher risk of driving safety? Similarly, does a slight decrease in accuracy of a deep learning model under attacks indicate mild risk of driving safety? 2) If the answers to the previous questions are all no, what are the reasons behind?

In this chapter, we aim to answer the aforementioned questions by evaluating the impact of two types of representative adversarial attacks, \textbf{perturbation attacks} and \textbf{patch attacks}, on driving safety of vision-based autonomous driving systems, rather than  the accuracy of deep learning models. We also investigate the reasons causing the decoupling between the detection precision of adversarial attacks and driving safety.

This study considers vision-based autonomous driving which mainly relies on stereo cameras for the task of environmental sensing. The vision-based object detectors that we consider in this chapter are Stereo R-CNN~\cite{li19stereo} and DSGN~\cite{chen20dsgn}, two state-of-the-art methods in this area.

To facilitate this study, we propose an \textbf{end-to-end} driving safety evaluation framework with a set of designed driving safety performance metrics, where the evaluation framework can directly take the results of the 3D object detector as input and outputs the scores of the driving safety performance metrics as the final assessment.

To implement such an evaluation framework, we are faced with two nontrivial technical challenges. First, the results of the 3D object detector only contain static information, such as position and dimension. Thus, we cannot determine which objects are moving and which are static. Second, to realistically generate a planned trajectory for the self-driving ego-vehicle, real driving constraints, such as speed limits for different road types and dynamics models for different vehicles, must be provided to the motion planning module. Considering that the driving scenarios change dynamically, we need to select appropriate real driving constraints accordingly for driving safety assessment.

To tackle the first challenge, we train a CNN-based classifier with manually labeled ground truth to categorize whether an object is moving or static. For the second challenge, we train another classifier with road type labels to classify the road type of each scenario, so as to select appropriate driving constraints.

To obtain comprehensive experiment results, we apply the aforementioned two types of adversarial attacks with different attack intensities in our evaluation framework and measure the rate that the motion planner successfully finds a trajectory, the rate of collision occurrence, and the rate that the ego-vehicle drives safely from the initial position to the goal region. In the meantime, we also measure the precision changes of the vision-based 3D object detectors when they are under attacks. By linking the impact of adversarial attacks on driving safety and on 3D object detection together, we manage to find the answers to our motivation questions. The main contributions of this chapter can be summarized as follows.

\begin{itemize}
\item
We propose an \textit{end-to-end} driving safety evaluation framework that directly takes the produced results of the 3D object detector as input and outputs driving safety performance scores as the evaluation outcome. With modular design, each individual module can be easily replaced so that the framework can be adapted to evaluate the driving safety of different self-driving systems threatened by various attacks.
\item
We conduct extensive experiments and observe that the changes in object detection precision and the changes in driving safety performance metrics caused by adversarial attacks are decoupled. Therefore, the answers to our motivation questions are all no. And we also observe that DSGN is more robust than Stereo R-CNN in terms of driving safety.
\item
We investigate the reasons behind our answers to those questions. We find that the reason for the decoupling is that it is easier for perturbation attacks to mislead object detectors to detect ghost objects at roadside which cause little influence on driving safety but huge impact on detection precision. We also find out that the reason why DSGN is more robust than Stereo R-CNN is that the latter is purely based on deep neural network, while DSGN adopts the Spatial Pyramid Pooling (SPP) in its architecture which can alleviate the attack effects.
\end{itemize}

The rest of this chapter is organized as follows. In Section~\ref{ch5:sec.attackmodels}, we first briefly introduce the attack model of the two adversarial attacks studied in this chapter. In Section~\ref{ch5:sec.system}, we elaborate on our proposed end-to-end evaluation framework and the driving safety performance metrics. In Section~\ref{ch5:sec.experiments}, we present the experiment design and results. In Section~\ref{ch5:sec.ablationstudy}, we investigate the causes of our observations with an ablation study. Finally, we conclude this chapter in Section~\ref{ch5:sec.conclusion}.

\section{Attack Models}
\label{ch5:sec.attackmodels}

\begin{figure*}[!t]
    \centering
    \includegraphics[width=1.0\textwidth]{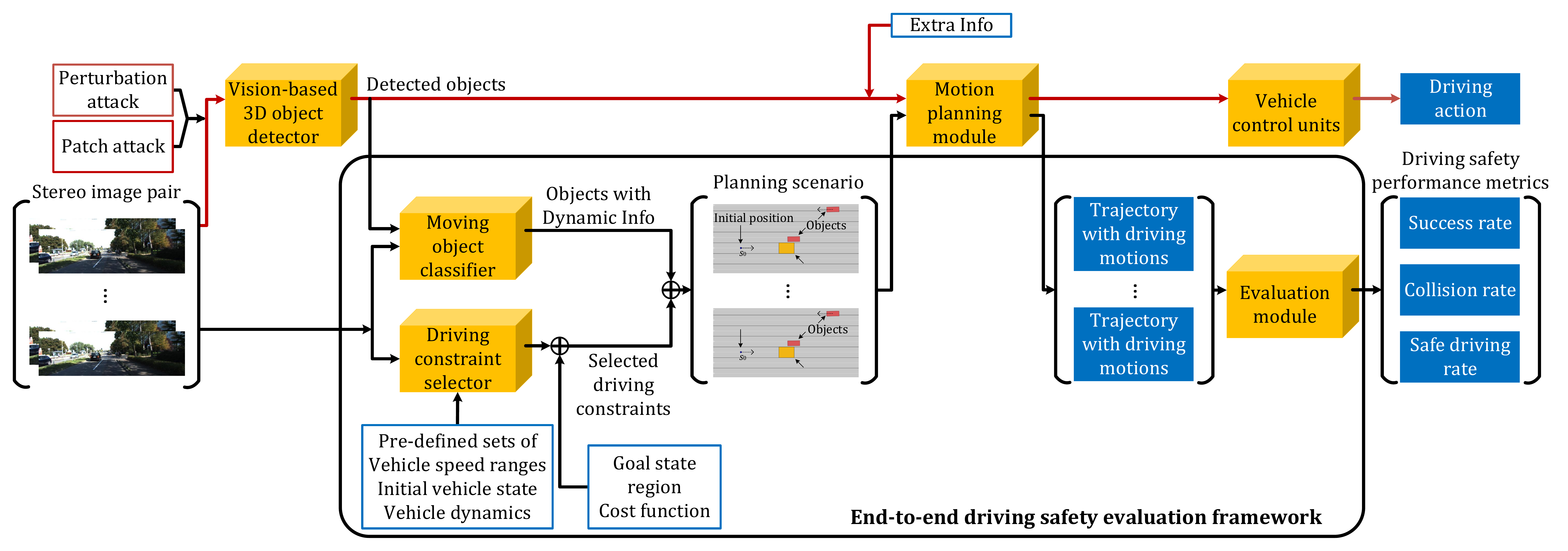}
    \caption{The end-to-end driving safety evaluation framework.}
    \label{ch5:fig:system_model}
\end{figure*}

We assess the impact of two types of adversarial attacks, perturbation attacks and patch attacks, on driving safety. Here, we briefly introduce the attack models.

\textbf{Perturbation Attack.} The goal of this type of adversarial attacks is to make the deep learning model dysfunctional by adding small changes to each pixel in the image that are imperceptible to human eyes. With prior knowledge of the deep learning model, attackers can launch perturbation attacks by tapping into the self-driving system and perturbing camera images. We consider the method of PGD~\cite{madry18towards} to achieve input-specific attacks. Consider a perturbation $\delta^{\text{per}}$ and an image pair $(I_{l}, I_{r})$, where $\delta^{\text{per}}$ has the same dimension as $I_{l}$ and $I_{r}$. Let the initial perturbed image pair $(\tilde{I}^{\text{per}}_{l, 0}, \tilde{I}^{\text{per}}_{r, 0}) = (I_{l}, I_{r})$. The attack is carried out by updating the perturbation using the projected loss gradient of the 3D object detector through multiple iterations with
\begin{equation}\label{eq:eq_2}
\delta^{\text{per}}_{n} = \text{Clip}_{\epsilon}\{\alpha \times \text{sign}(\nabla_{(I_{l}, I_{r})}L(O_{\theta}(\tilde{I}^{\text{per}}_{l, n}, \tilde{I}^{\text{per}}_{r, n}), b^{\text{true}}))\}
\end{equation}
and
\begin{equation}\label{eq:eq_3}
(\tilde{I}^{\text{per}}_{l, n+1}, \tilde{I}^{\text{per}}_{r, n+1}) = (\tilde{I}^{\text{per}}_{l, n} + \delta^{\text{per}}_{n}, \tilde{I}^{\text{per}}_{r, n} + \delta^{\text{per}}_{n})
\end{equation}
where $\text{Clip}_{\epsilon}\{\cdot\}$ ensures that the value is within $[-\epsilon, \epsilon]$, $\alpha$ is the parameter that controls the attack intensity, $\text{sign}(\cdot)$ denotes the sign function, $O_{\theta}(\cdot, \cdot)$ represents the vision-based 3D object detector parametrized by $\theta$, $L(\cdot, \cdot)$ is the loss function of $O_{\theta}(\cdot, \cdot)$, $b^{\text{true}}$ is the ground truth label paired with $(I_{l}, I_{r})$, and $0 \leqslant n \leqslant N-1$. For convenience, we denote the perturbation attack as $(\tilde{I}^{\text{per}}_{l}, \tilde{I}^{\text{per}}_{r}) = A^{\text{per}}(I_{l}, I_{r}, b^{\text{true}}, \epsilon, \alpha, N)$.

\textbf{Patch Attack.} The patch attack is designed to model the real-world poster-printing attack  in~\cite{eykholt18robust}. In the context of vision-based 3D object detection, a patch attack is launched to mislead the detector so that it detects ghost objects by including the patch in the view of the image. With prior knowledge of the deep learning model, attackers can train a patch offline, print it out, and put the physical patch inside the view of cameras to launch the attack. For example, the attacker can place the patch at the roadside where the vision-based self-driving car passes by. Since a real-world 3D point appears at different positions in two stereo images, we consider a patch $\delta^{\text{pat}}$ that is pasted onto $I_{l}$ at $loc_{l}$ and onto $I_{r}$ at $loc_{r}$, where $(loc_{l}, loc_{r}) \in \mathcal{L}$, $\mathcal{L}$ represents a set of random position pairs. Let $\lambda_{loc_{l}, loc_{r}} \in \Lambda$ be the disparity between $loc_{l}$ and $loc_{r}$, where $\Lambda$ denotes a set of valid disparities in the physical world. Let $\tau \in \mathcal{T}$ be a transformation that can be applied to $\delta^{\text{pat}}$, where $\mathcal{T}$ includes rotations. Then, the patched image pair can be represented as $(\tilde{I}^{\text{pat}}_{l}, \tilde{I}^{\text{pat}}_{r}) = A^{\text{pat}}(I_{l}, I_{r}, \delta^{\text{pat}}, loc_{l}, loc_{r}, \tau)$. To implement this attack, the patch is optimized as
\begin{equation}\label{eq:eq_4}
\mathop{\text{argmin}}_{\delta^{\text{pat}}} \mathbb{E}_{(I_{l}, I_{r}) \sim \mathcal{I}, (loc_{l}, loc_{r}) \sim \mathcal{L}, \tau \sim \mathcal{T}} L(O_{\theta}(\tilde{I}^{\text{pat}}_{l}, \tilde{I}^{\text{pat}}_{r}), b^{*}),
\end{equation}
where $b^{*}$ denotes the predefined 3D bounding boxes used for misleading the object detector and serves as the optimization target here.

\section{End-to-end Driving Safety Evaluation \\ Framework}
\label{ch5:sec.system}

As discussed in Chapter~\ref{ch2:literature}, many previous studied only showed that deep learning models of autonomous driving can be compromised by adversarial attacks, but they did not systematically assess the attack impact on driving safety. Our goal is to answer the questions raised in Section~\ref{ch5:sec.introduction} by investigating the impact of perturbation attacks and patch attacks on driving safety of vision-based autonomous vehicles. This investigation considers not only the performance of the attacked deep learning models but also their impact on the overall safety, which is a combined effect of different functional modules involved in autonomous driving.

To this end, we design an end-to-end driving safety evaluation framework. \textbf{End-to-end} means that our system directly takes 3D object detection results as input and outputs the driving safety scores. Moreover, our evaluation framework adopts a modular design, so that each module can be easily replaced with other methods to assess the driving safety of different autonomous driving systems. Note that  the existing simulators, such as Baidu Apollo~\cite{apollo19perception} and CARLA~\cite{dosovitskiy17carla}, either have a low level of customization or are not compatible with real autonomous driving datasets. Therefore, we implement our own evaluation framework with real autonomous driving dataset to evaluate the impact of adversarial attacks on driving safety. In this section, we first introduce our evaluation framework model for vision-based autonomous driving and the driving safety metrics, then elaborate on the framework implementation details.

\subsection{Framework}

Our evaluation framework works along with the data flow of vision-based autonomous driving systems. In Figure~\ref{ch5:fig:system_model}, the black lines represent the data flow of our evaluation framework, while the red lines are for the data flow of the autonomous driving system. Usually inside the vision-based autonomous vehicle, a pair of stereo images $(I_{l}, I_{r} \in \mathbb{R}^{h \times w \times 3})$ is first fed as the input to the 3D object detection module $O_{\theta}(\cdot, \cdot)$, which is parameterized by $\theta$ and generates detected objects in 3D bounding boxes $b$ (denoted as $b=O_{\theta}(I_{l}, I_{r})$) as the output. Next, the bounding boxes $b$ along with some extra information are passed to the the motion planning module $M(\cdot)$. At last, the vehicle control units execute the driving motion orders from the motion planning module. As depicted in Figure~\ref{ch5:fig:system_model}, our proposed end-to-end driving safety evaluation framework directly takes the detected objects of $O_{\theta}(\cdot, \cdot)$ as input, uses the same motion planning module of the autonomous vehicle, and outputs scores for driving safety metrics. This modular design makes it possible for our evaluation framework to be adopted by other self-driving systems which have different implementation of the aforementioned modules.

As described in Section~\ref{ch5:sec.introduction},  two main technical challenges need to be addressed after the object detection results are fed into our evaluation framework. First, the 3D bounding boxes $b$ as the object detection results only contain static information, i.e., object category, box dimensions, box center position in 3D space, and the confidence score. Base on the static information of one frame of data, we cannot distinguish between moving objects and static objects. However, the subsequent motion planning module requires dynamic information of objects as part of its input. Second, to realistically produce planned trajectory for the self-driving vehicle, driving constraints, including speed limits for different road types and dynamics constraints for different moving vehicles (acceleration, jerk, energy, etc.), must be considered to comply with the real driving scenarios. In addition, as the real driving scenarios can change dynamically, we must choose appropriate real driving constraints accordingly for driving safety evaluation.

\begin{figure*}
\centering

\subfloat[Clean image input.]{
    \includegraphics[width=0.35\textwidth]{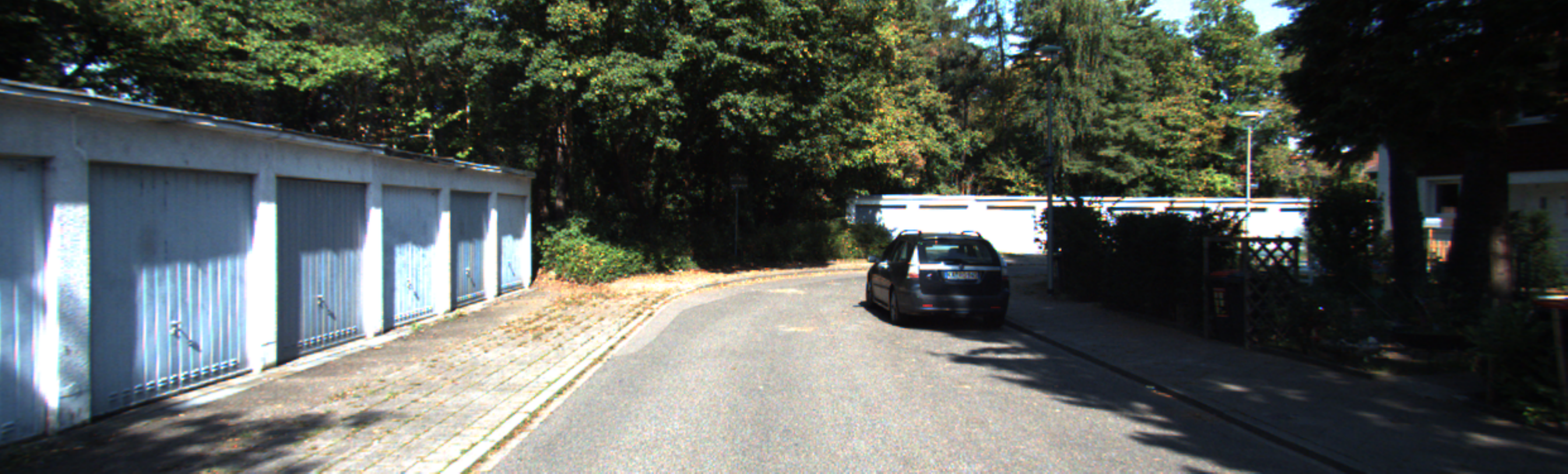}
}\hspace{12pt}
\subfloat[Ground truth of object detection.]{
    \includegraphics[width=0.35\textwidth]{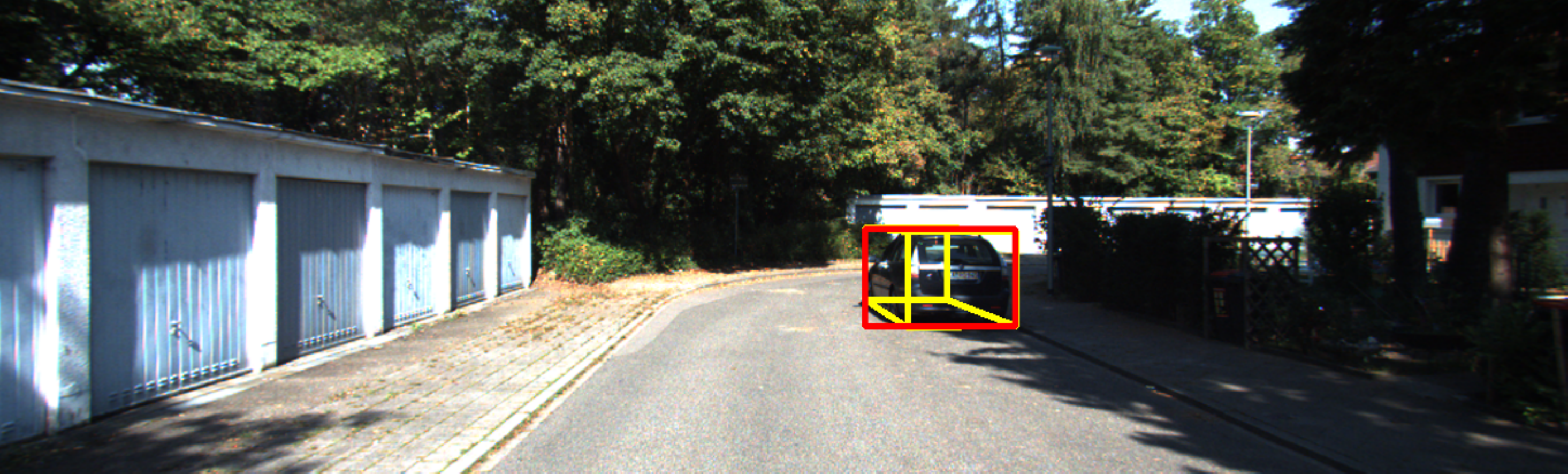}
} \\
\subfloat[Detection results of DSGN without attack.]{
    \includegraphics[width=0.35\textwidth]{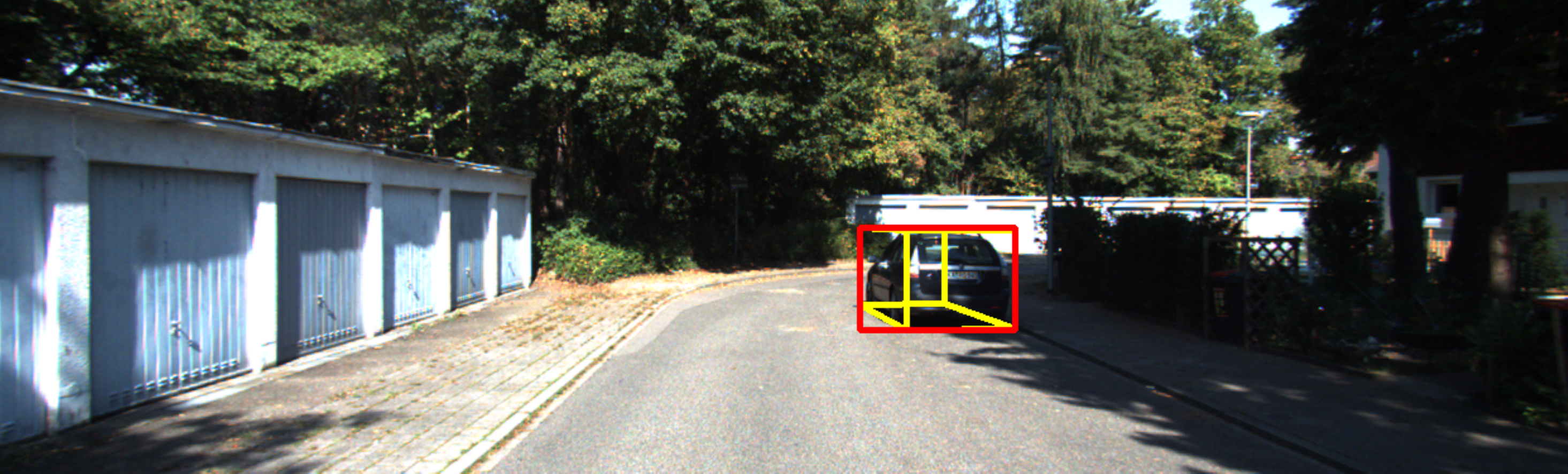}
}\hspace{12pt}
\subfloat[Detection results of DSGN under attack.]{
    \includegraphics[width=0.35\textwidth]{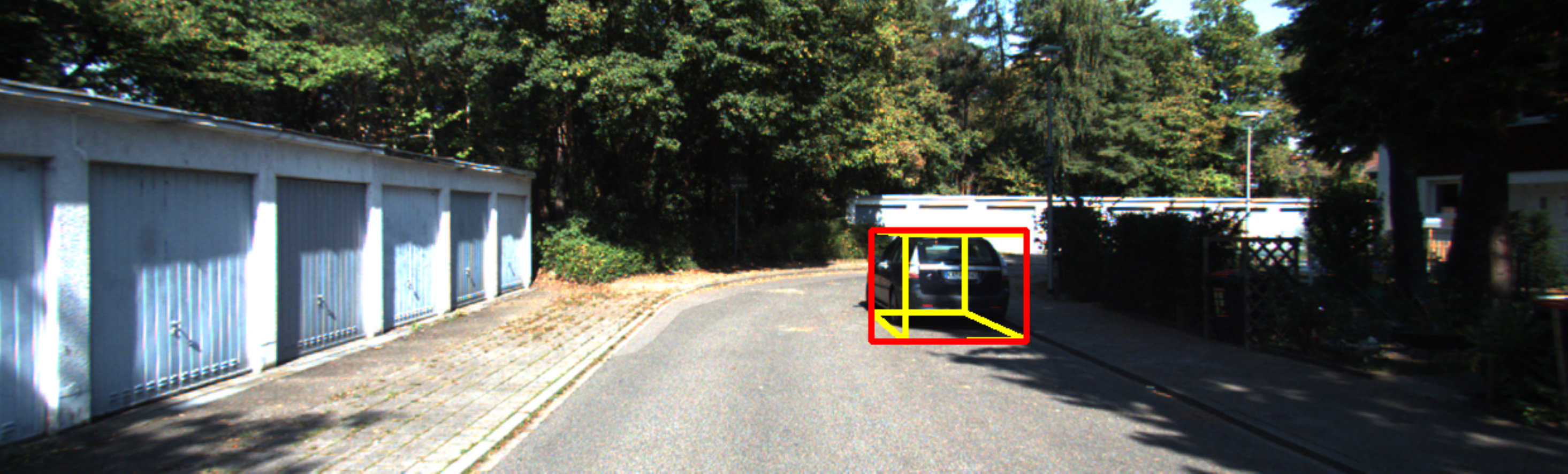}
} \\
\subfloat[Detection results of Stereo R-CNN without attack.]{
    \includegraphics[width=0.35\textwidth]{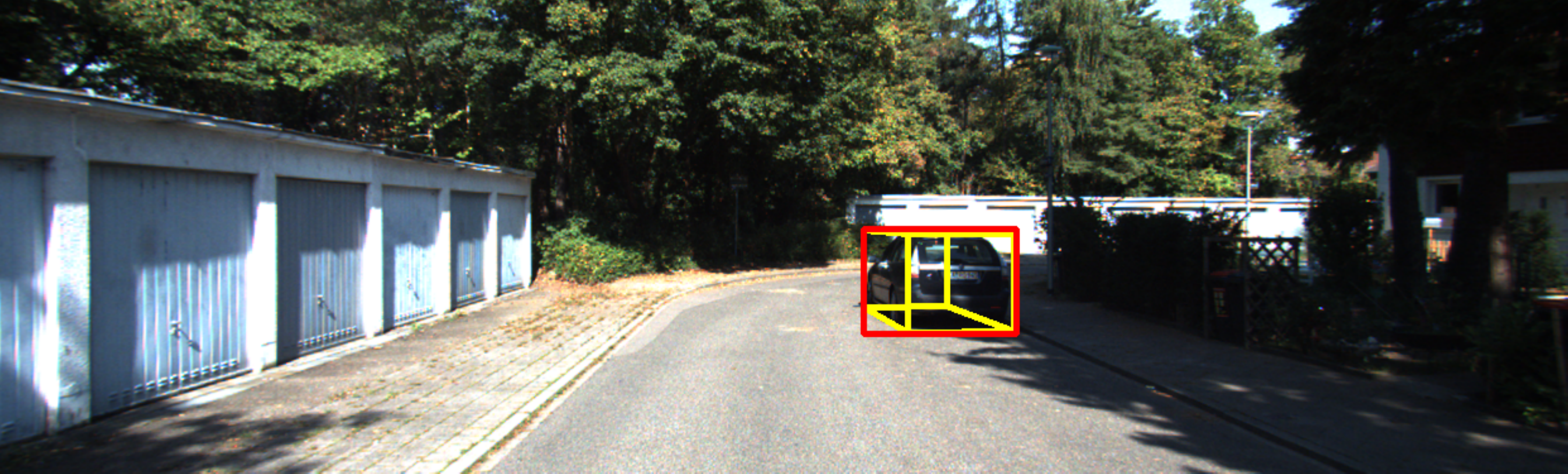}
}\hspace{12pt}
\subfloat[Detection results of Stereo R-CNN under attack.]{
    \includegraphics[width=0.35\textwidth]{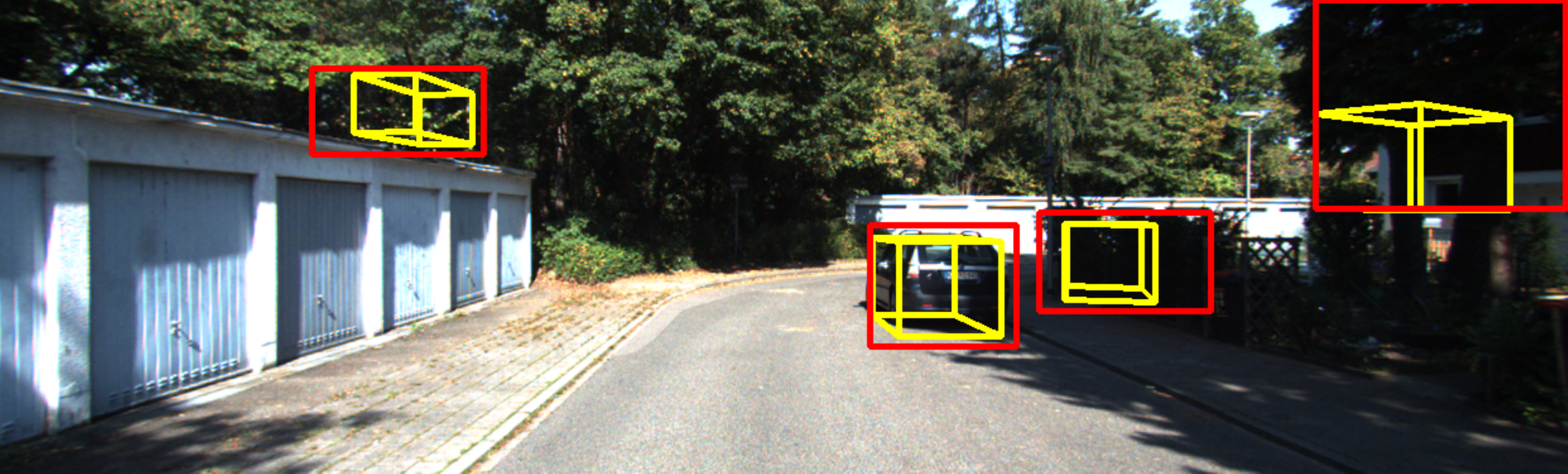}
} \\
\subfloat[Clean image input.]{
    \includegraphics[width=0.35\textwidth]{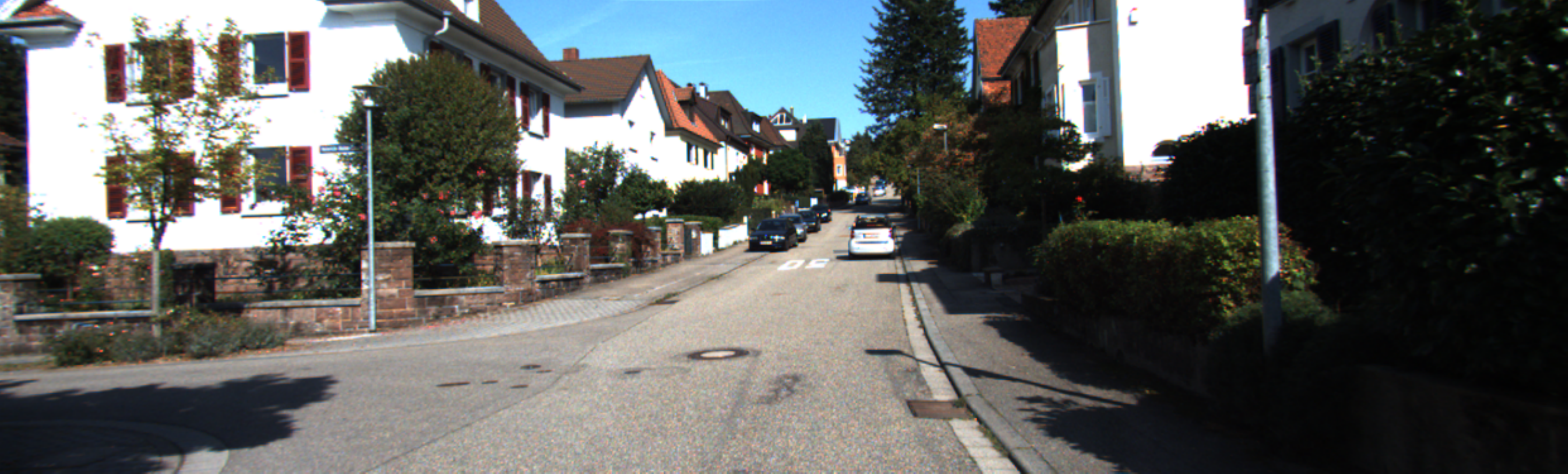}
}\hspace{12pt}
\subfloat[Ground truth of object detection.]{
    \includegraphics[width=0.35\textwidth]{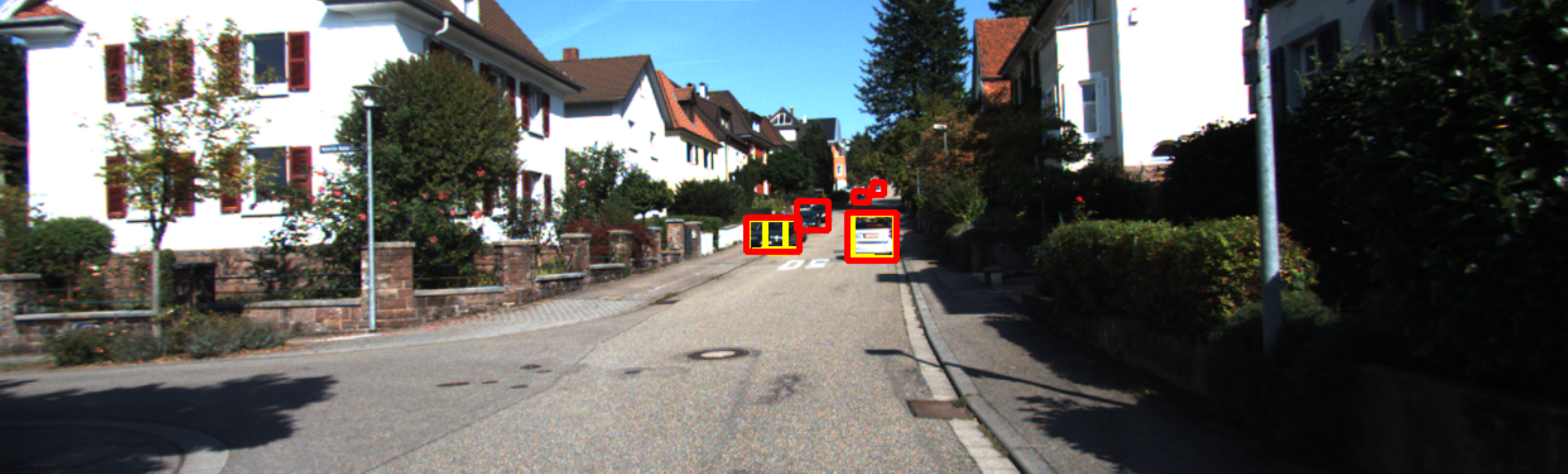}
} \\
\subfloat[Detection results of DSGN without attack.]{
    \includegraphics[width=0.35\textwidth]{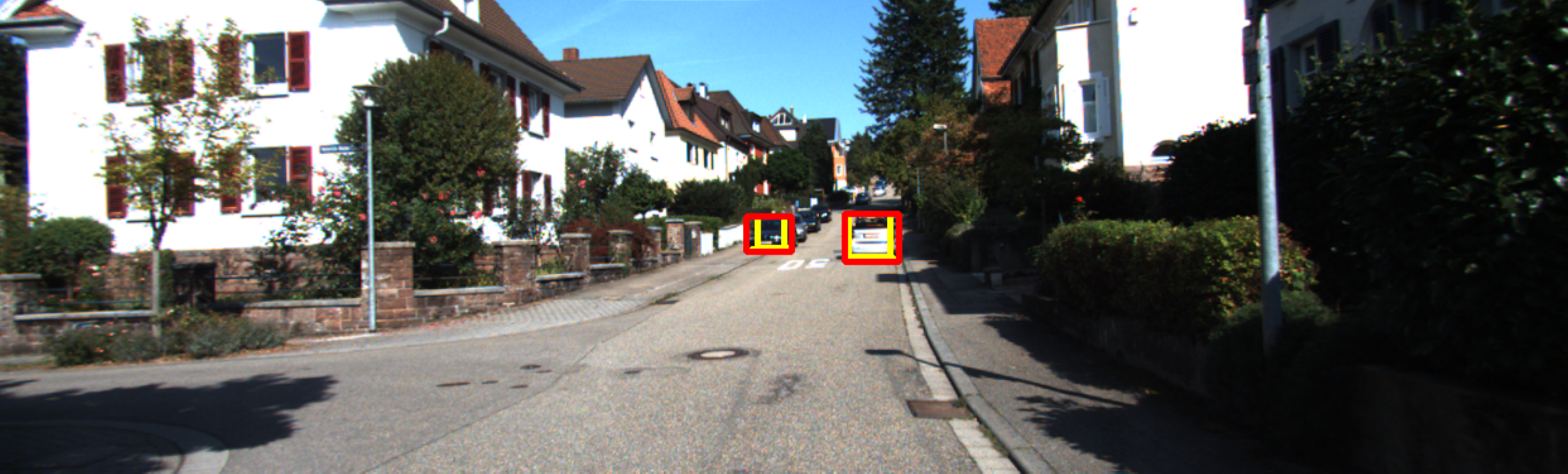}
}\hspace{12pt}
\subfloat[Detection results of DSGN under attack.]{
    \includegraphics[width=0.35\textwidth]{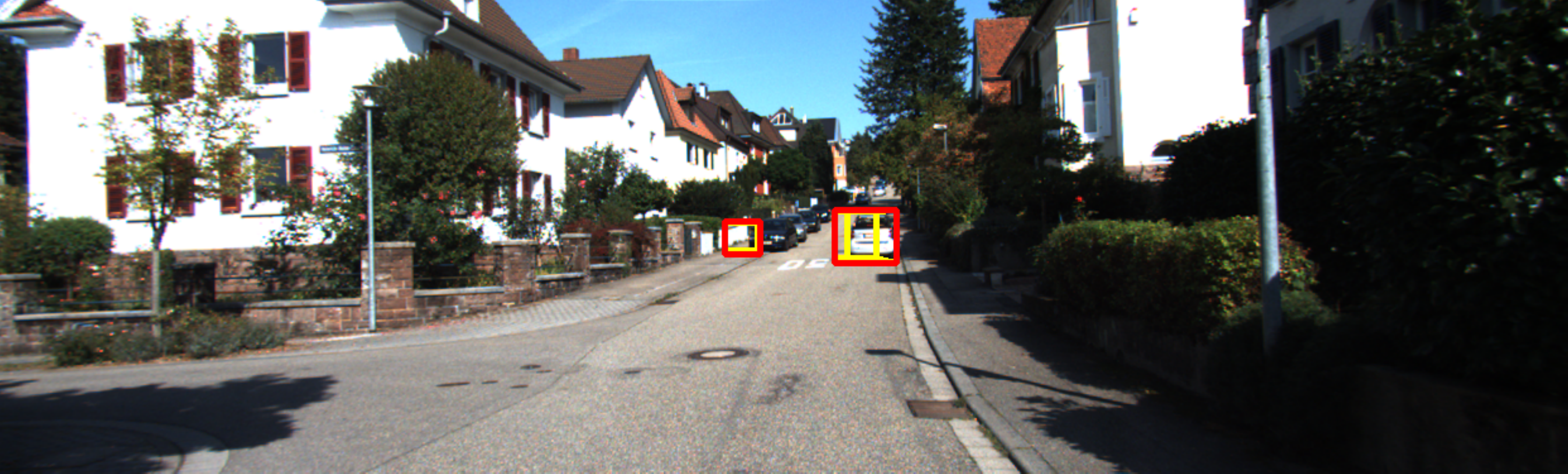}
} \\
\subfloat[Detection results of Stereo R-CNN without attack.]{
    \includegraphics[width=0.35\textwidth]{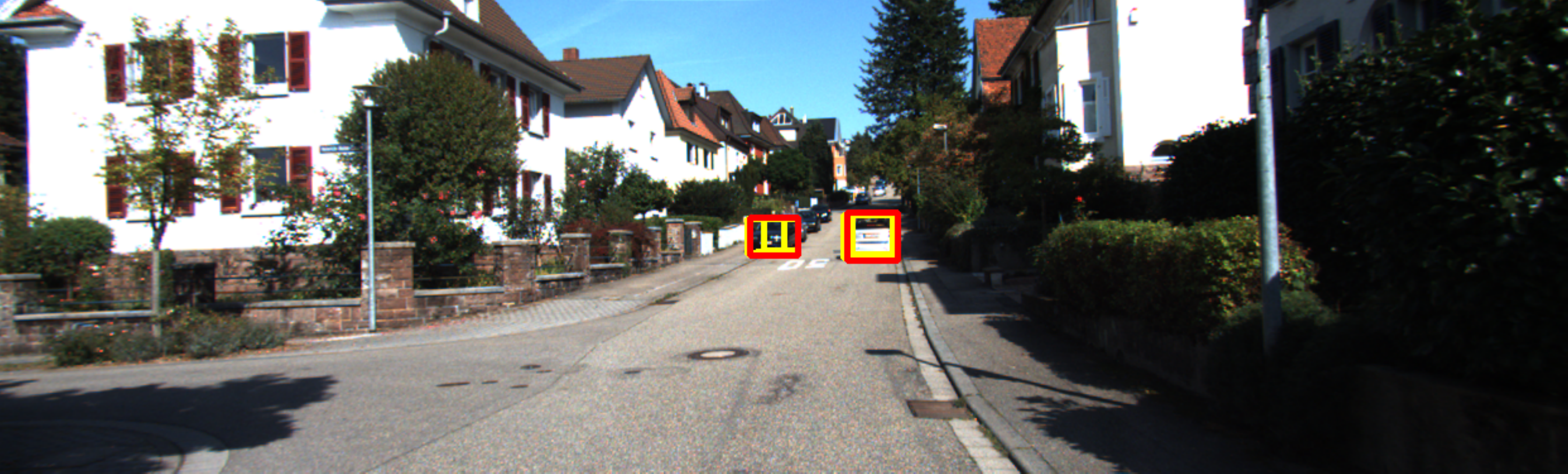}
}\hspace{12pt}
\subfloat[Detection results of Stereo R-CNN under attack.]{
    \includegraphics[width=0.35\textwidth]{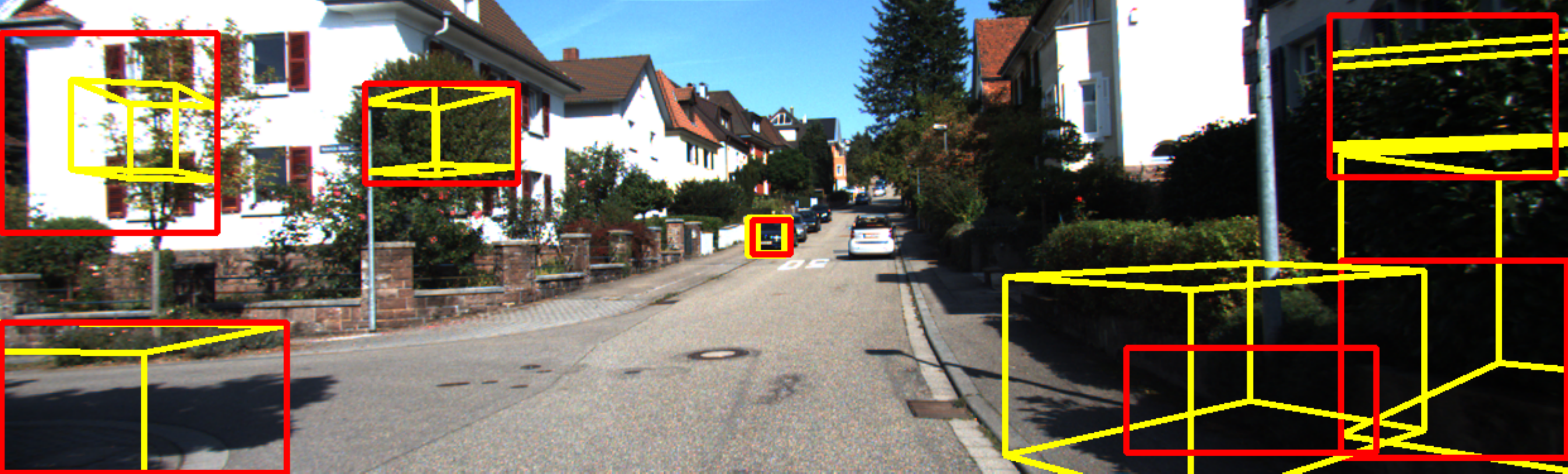}
} \\
\caption{When there is no attack, both Stereo R-CNN and DSGN can detect objects accurately as shown in (c), (e), (i), and (k). When the perturbation attack is launched, the two models produce erroneous object detection results including inaccurate detection of real objects in (d), (j), and false detection of ghost objects in (f), (l).}
\label{fig:plot_perturbation}
\end{figure*}

To tackle the first challenge, we train a CNN-based moving object classifier $C(\cdot, \cdot, \cdot)$ to distinguish between dynamic and static objects by leveraging continuous frames of image data. We manually label each object with the ground truth indicating whether this object is moving or not. By doing this, we associate the object detection results with dynamic information. We denote this process as $\vec{b}=C(b, I_{l}, I_{r})$.

To address the second challenge, we train another CNN-based model $S(\cdot, \cdot)$ with road type labels as the driving constraint selector, so that it can classify the road type of driving scenarios and select proper real driving constraints for the evaluation. We denote this part as $(s_{0}, r, d)=S(I_{l}, I_{r})$, where $s_{0}$ is the initial vehicle state, $r$ is the allowed speed range, and $d$ represents the vehicle dynamics. In this chapter, we define a vehicle state $s := (p, v, \varphi, \omega)$ as a combination of position $p$, velocity $v$, orientation $\varphi$, and steering angle $\omega$ at a specific moment. Note that both the two aforementioned models, $C(\cdot, \cdot, \cdot)$ and $S(\cdot, \cdot)$, are trained on KITTI raw dataset~\cite{geiger13vision}.

Then, together with goal region $g$ and cost function $c$, we combine the processed results of both the moving object classifier and the driving constraint selector to form a planning scenario. After that, the scenario is fed to the motion planning module $M(\cdot)$ that outputs a temporal sequence of planned vehicle states $\{s_{t}\}$ (a trajectory with planned driving motions), which is denoted as $\{s_{t}\} = M(\vec{b}, s_{0}, r, g, c, d)$, where $1 \leqslant t \leqslant T$.

The final assessment of driving safety is conducted by the evaluation module based on processing a large number of driving scenarios. Specifically, the evaluation module incorporates the planned trajectory into the planning scenario and detects collision for each driving scenario in the dataset. Then, it generates driving safety performance scores based on all detected collisions. Note that we refer to a collision as the physical contact of objects. In this chapter, we evaluate driving safety on the KITTI object detection dataset~\cite{geiger12are}.

Next, we introduce the driving safety performance metrics and present the details of the framework implementation.

\subsection{Driving Safety Performance Metrics}
We define the driving safety as the vehicular ability to safely moves without causing dangers or harms to passengers, the ego vehicle, and surrounding objects. And we use the probabilities of dangers and accidents to quantify it. Therefore, to evaluate the driving safety of the vision-based autonomous driving system in a quantitative manner, we define a set of driving safety performance metrics as follows.

\begin{itemize}

\item
\textbf{Successful planning rate}. In some scenarios, the motion planning module may not be able to generate a trajectory solution, which imposes a risk in driving safety. Thus, we define the successful planning rate as $m_{suc}=\frac{k_{trj}}{k_{dts}}$, where $k_{dts}$ is the total number of scenarios in a dataset, and $k_{trj}$ is the number of scenarios in that dataset where a trajectory can be successfully generated, no matter whether it is collision-free or not. For the sake of simplicity, this metric is referred to as \textit{the success rate}.

\item
\textbf{Collision rate}. We define the collision rate, $m_{cls}$, as the percentage of scenarios in all successfully planned trajectories where a collision occurs. Let $m_{cls}=\frac{k_{cls}}{k_{trj}}$, where $k_{cls}$ is the number of scenarios with collision occurrence. Collision rate approximately reflects the collision probability under different levels of threats.

\item
\textbf{Safe driving rate}. The safe driving rate, $m_{saf}$, is defined as the percentage of scenarios in a dataset where a collision-free trajectory can be produced by the motion planning module. We denote it as $m_{saf}=\frac{k_{trj}-k_{cls}}{k_{dts}}=m_{suc}-\frac{k_{cls}}{k_{trj}}\frac{k_{trj}}{k_{dts}}=m_{suc}-m_{cls}m_{suc}=(1-m_{cls})m_{suc}$. 
\end{itemize}

In this chapter, we only focus on fatal driving risks when referring to the driving safety. By measuring successful planning rate and collision rate, we capture the two most risky driving scenarios in autonomous driving, i.e., the failure of path planning and collision.

Note that successful planning rate and collision rate are also common performance metrics measuring the quality of motion planning~\cite{elbanhawi14sampling}.
Furthermore, safe driving rate is jointly determined by both successful planning rate and collision rate, which is a more direct measure of driving safety.

\begin{figure*}[!t]
    \centering
    
    \subfloat[DSGN, $\alpha=0.4$]{
        \includegraphics[width=0.19\textwidth]{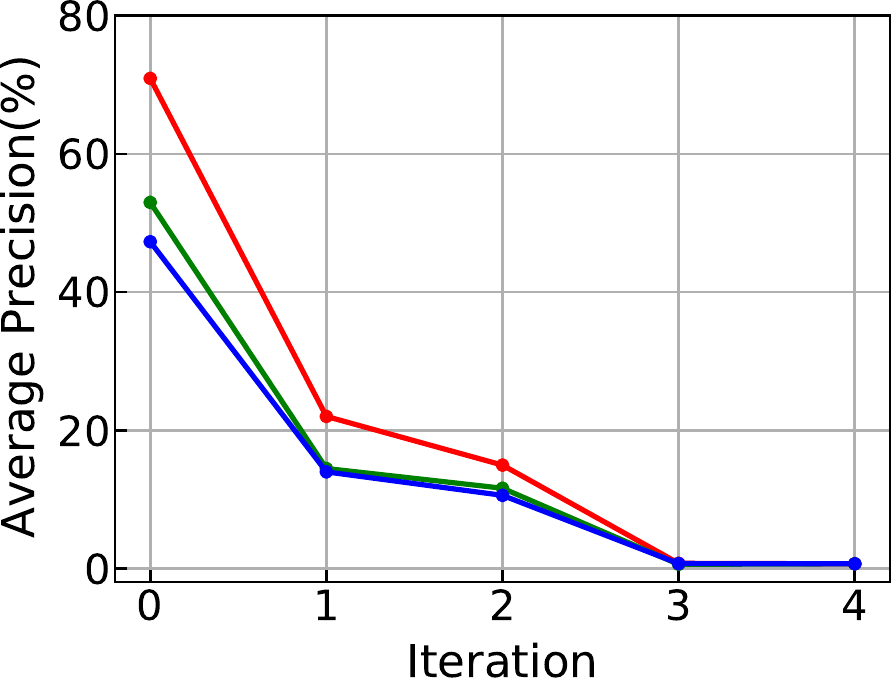}
    }\hspace{12pt}
    \subfloat[Stereo R-CNN, $\alpha=0.4$]{
        \includegraphics[width=0.19\textwidth]{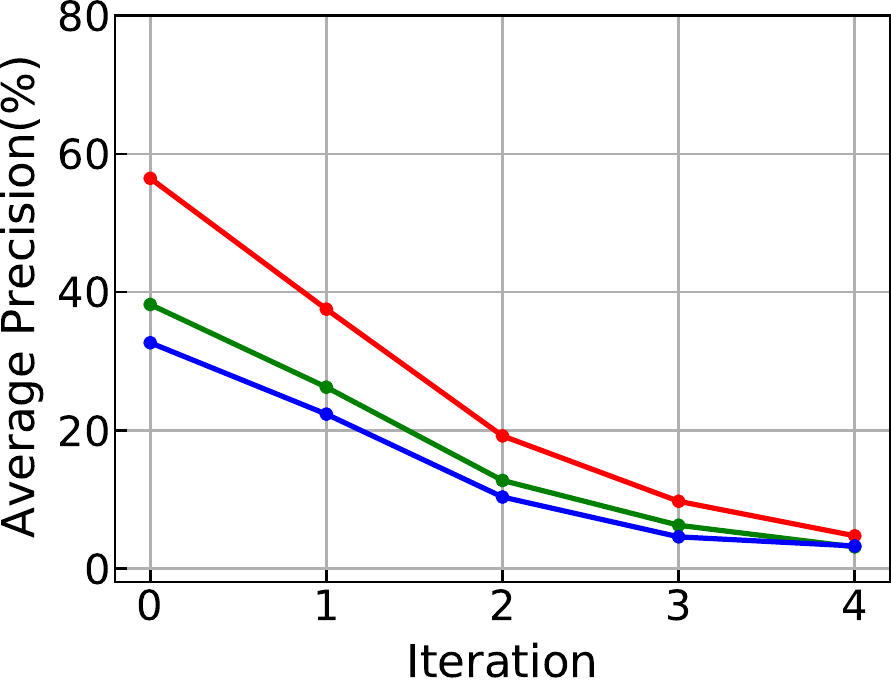}
    }\hspace{12pt}
    \subfloat[DSGN, $\alpha=1$]{
        \includegraphics[width=0.19\textwidth]{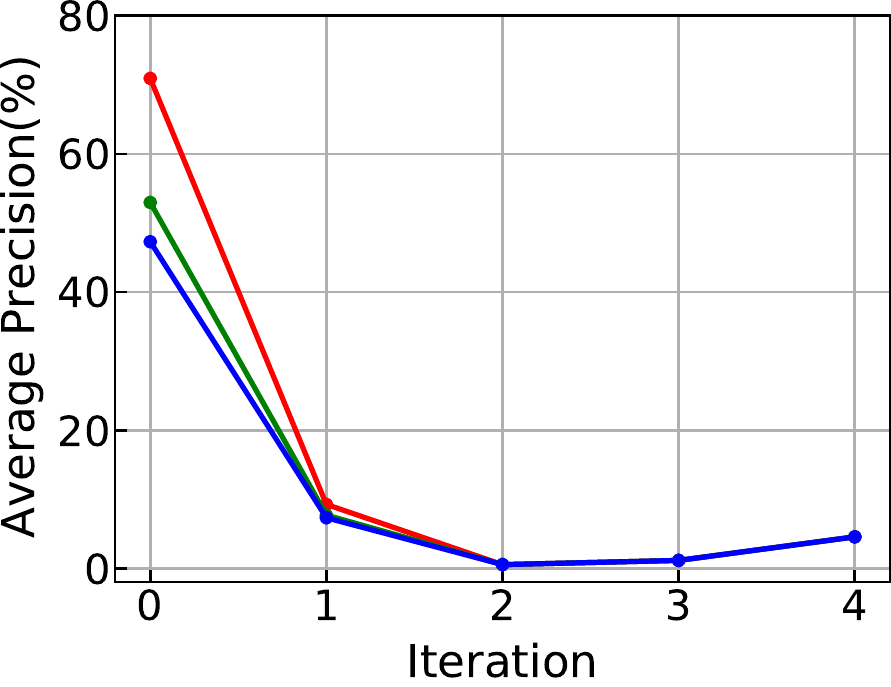}
    }\hspace{12pt}
    \subfloat[Stereo R-CNN, $\alpha=1$]{
        \includegraphics[width=0.19\textwidth]{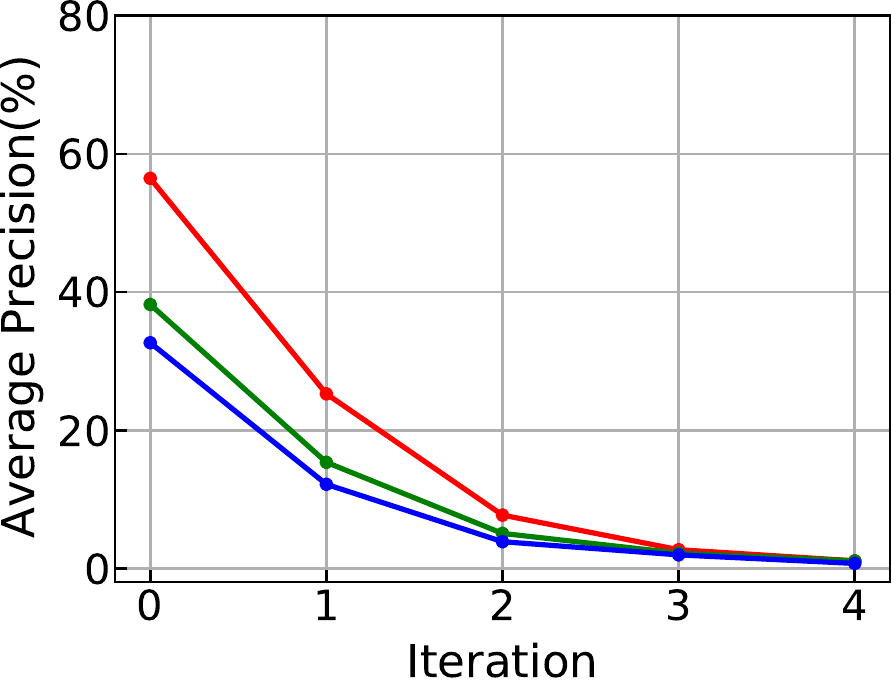}
    } \\
    \includegraphics[width=0.5\textwidth]{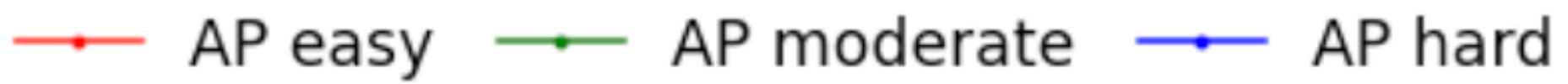}
    
    \caption{Average precision for 3D object detection under the perturbation attack.}
    \label{fig:ap_perturbation}
\end{figure*}

\begin{figure*}[!t]
    \centering
    
    \text{Changing to left lane}\par\medskip
    \begin{minipage}[t]{\textwidth}
    \centering 
    \includegraphics[width=0.22\textwidth]{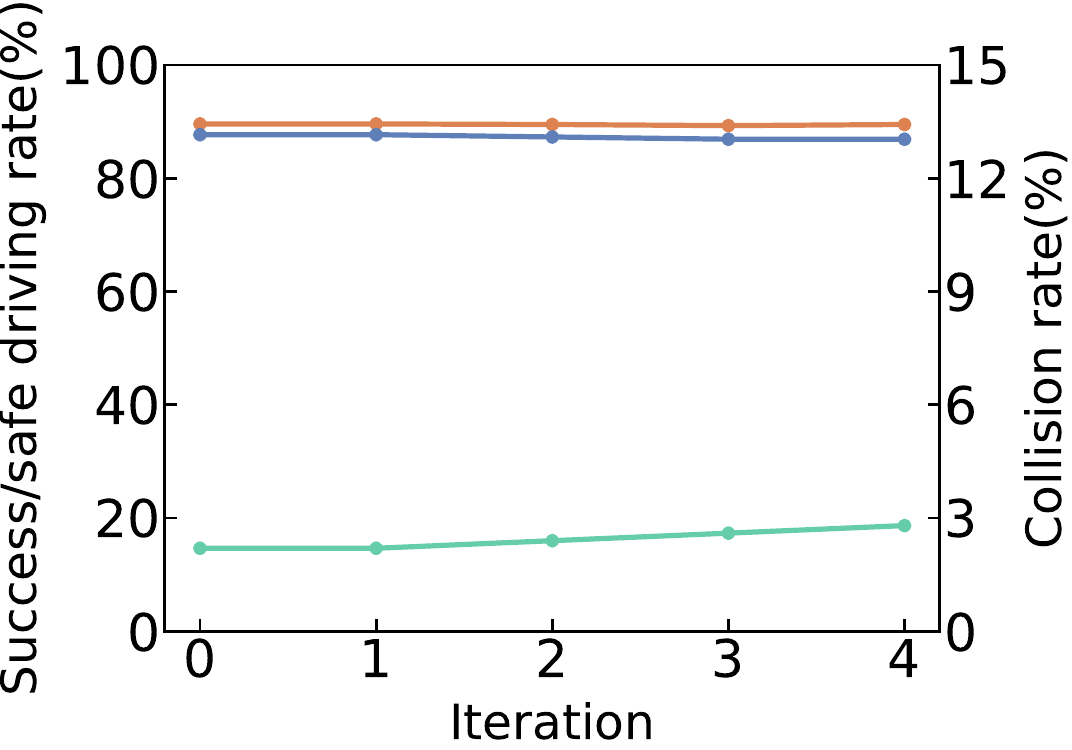}\hspace{6pt}
    \includegraphics[width=0.22\textwidth]{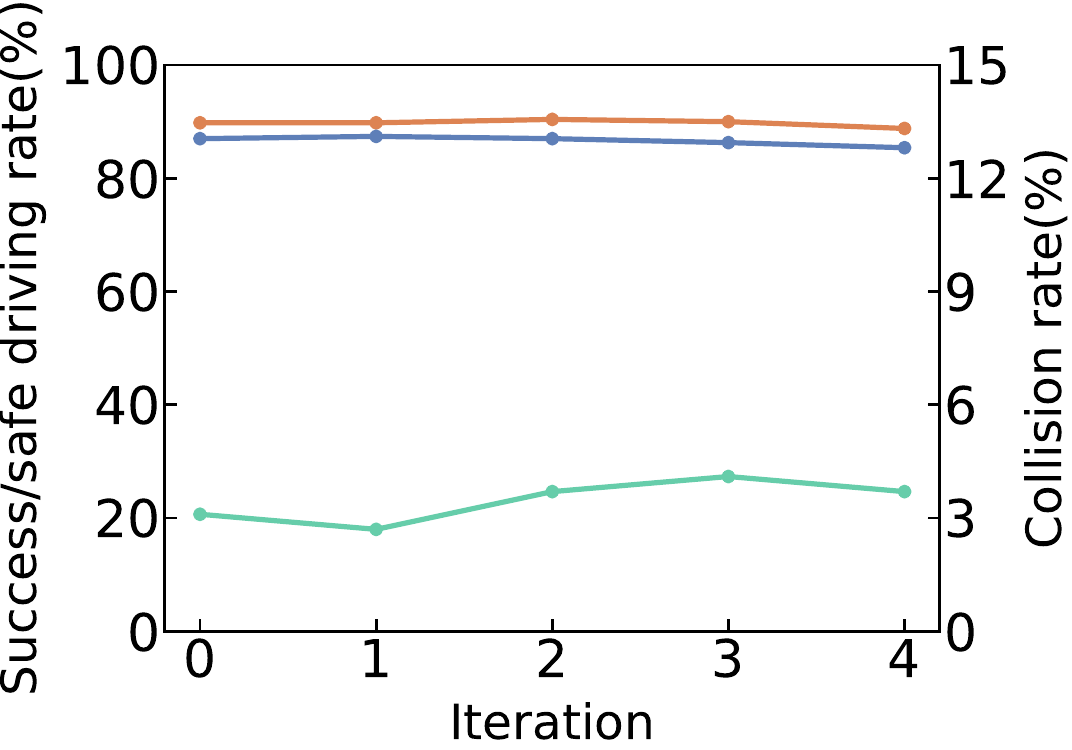}\hspace{6pt}
    \includegraphics[width=0.22\textwidth]{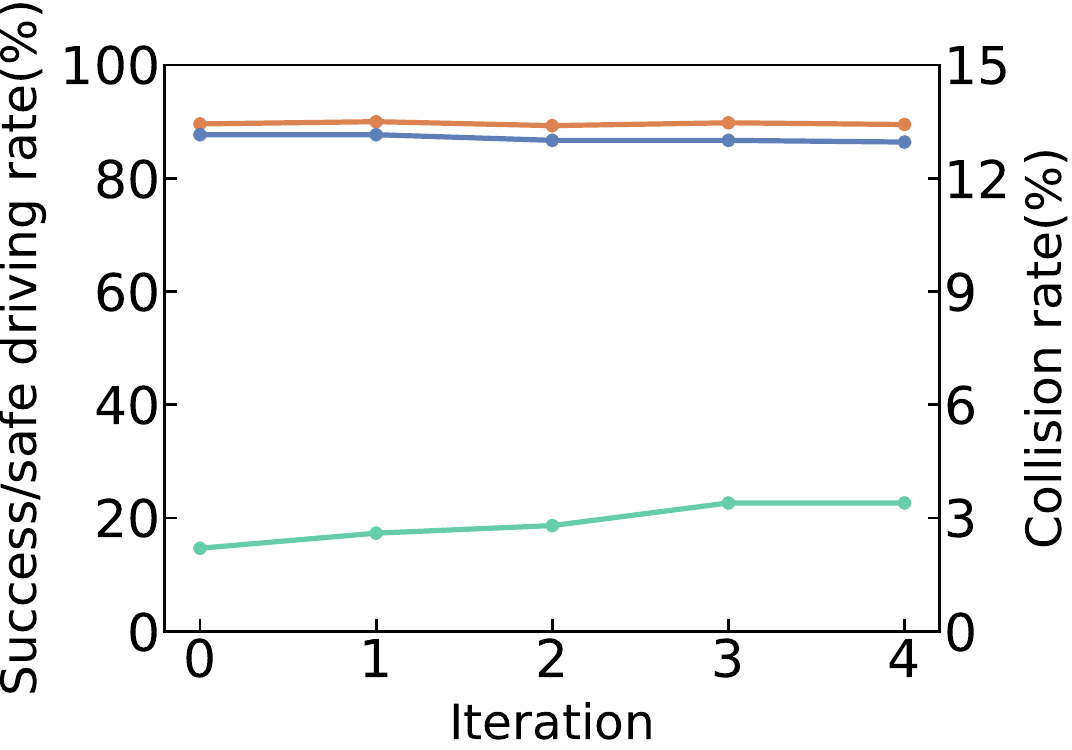}\hspace{6pt}
    \includegraphics[width=0.22\textwidth]{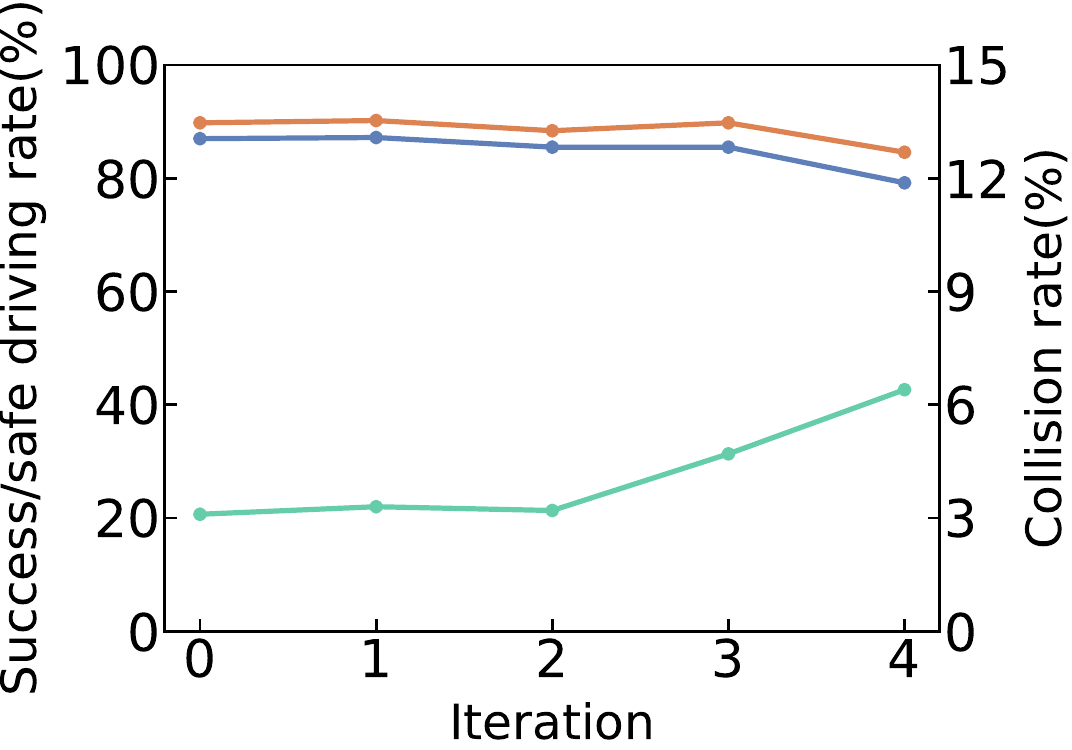}
    \end{minipage}
    
    \text{Keeping lane}\par\medskip
    \begin{minipage}[t]{\textwidth}
    \centering 
    \includegraphics[width=0.22\textwidth]{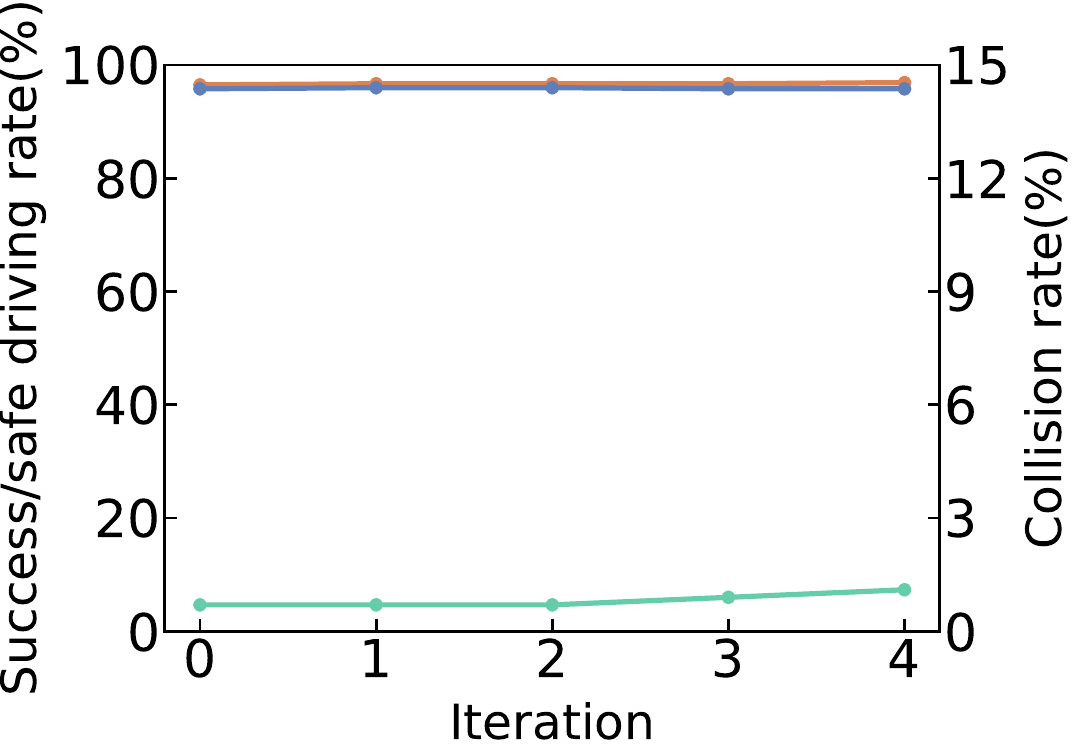}\hspace{6pt}
    \includegraphics[width=0.22\textwidth]{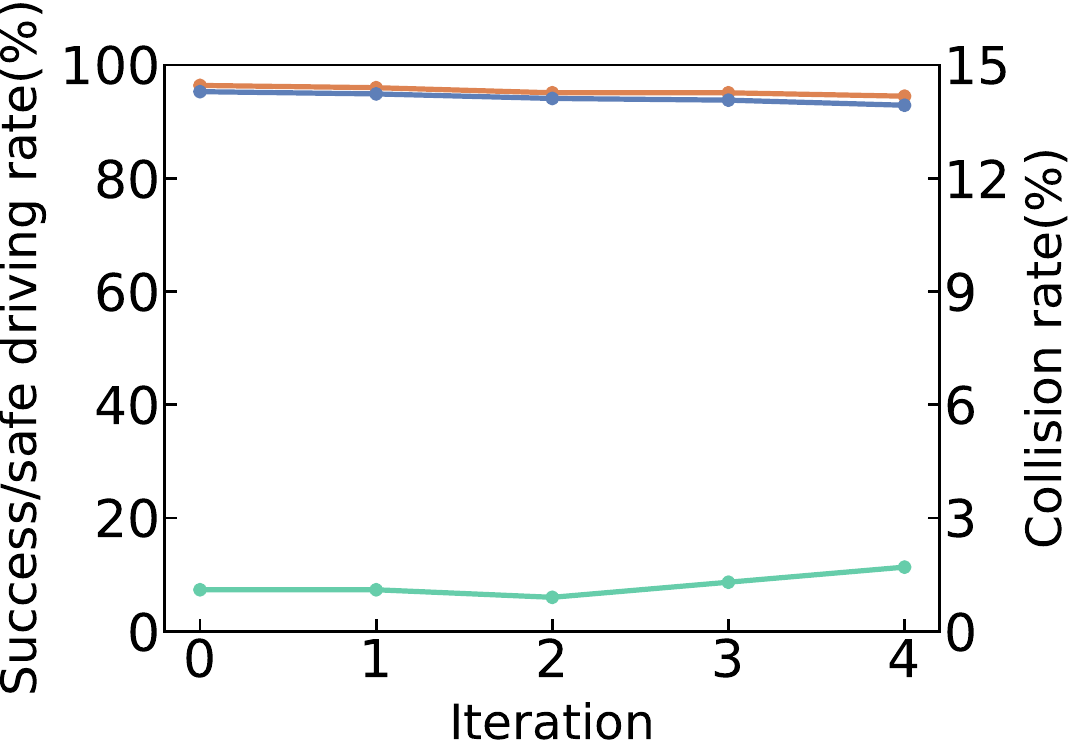}\hspace{6pt}
    \includegraphics[width=0.22\textwidth]{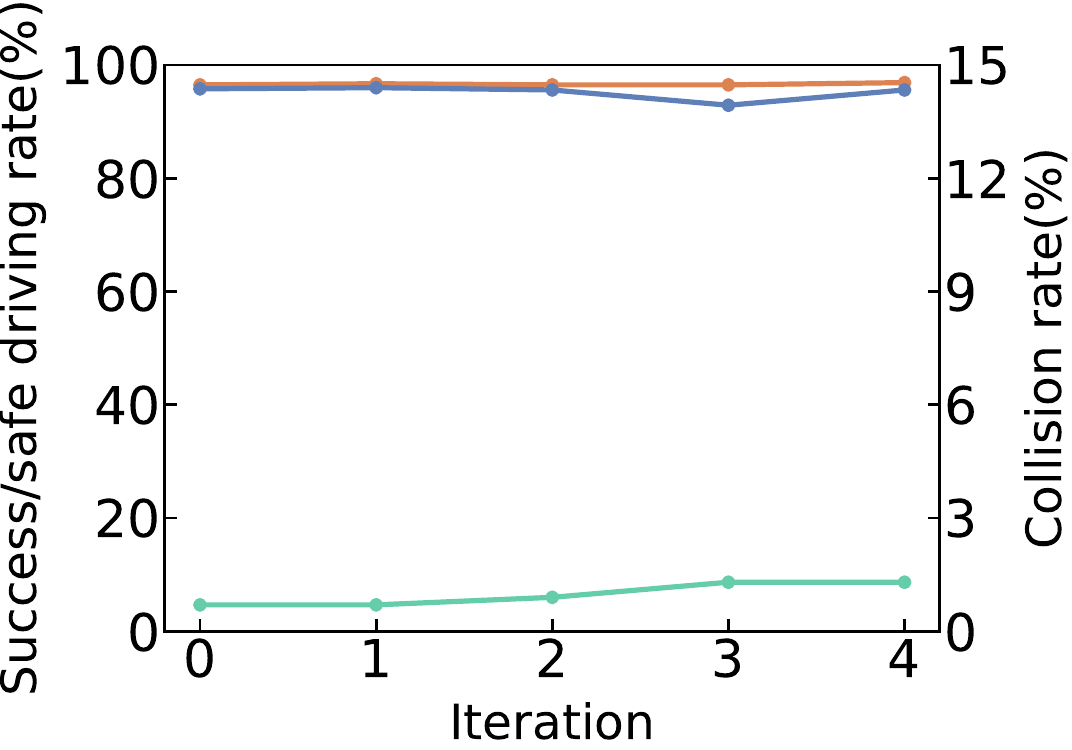}\hspace{6pt}
    \includegraphics[width=0.22\textwidth]{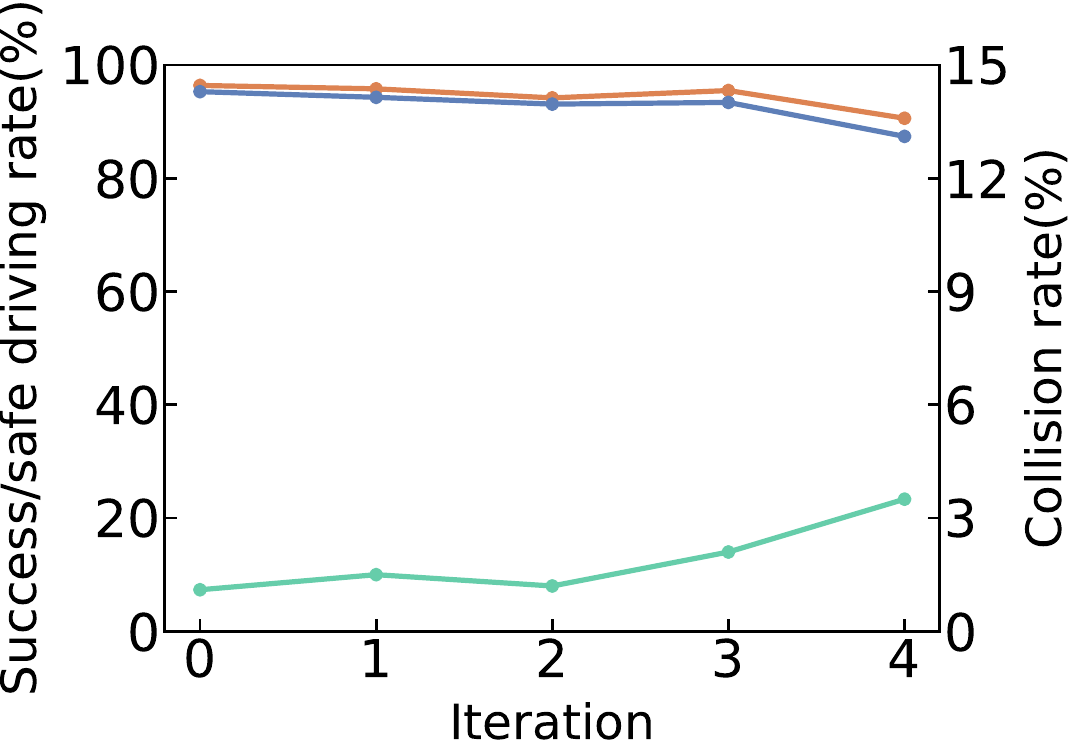}
    \end{minipage}
    
    \text{Changing to right lane}\par\medskip
    \subfloat[DSGN, $\alpha=0.4$]{
        \includegraphics[width=0.22\textwidth]{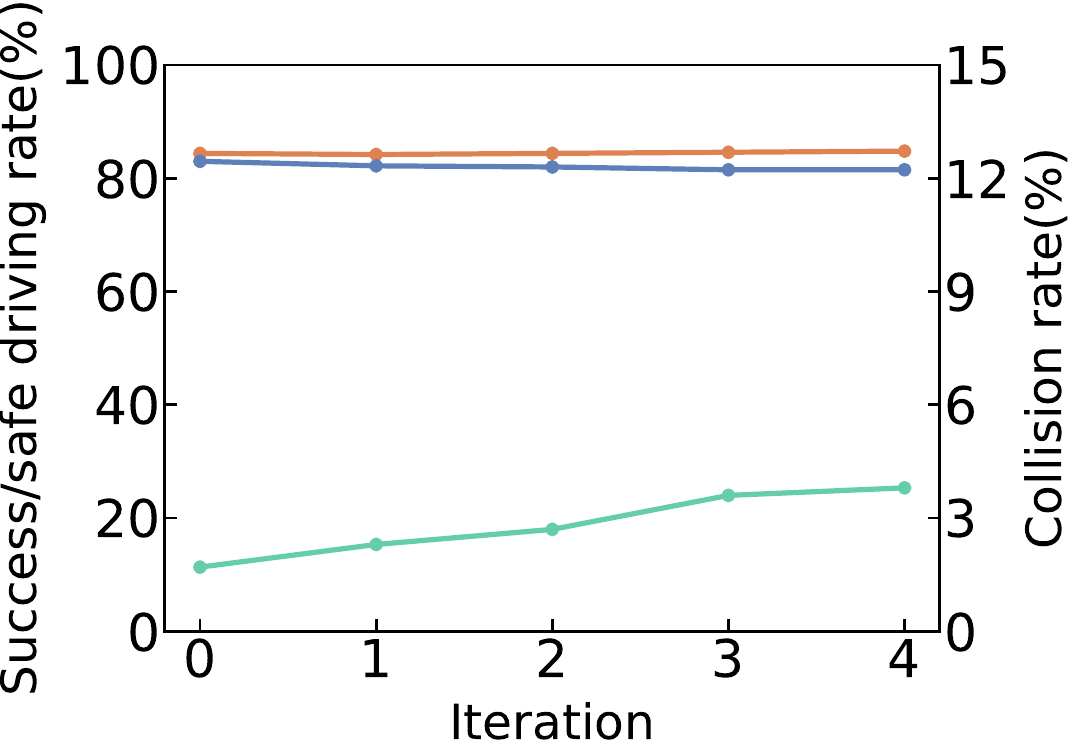}
    }\hspace{1.5pt}
    \subfloat[Stereo R-CNN, $\alpha=0.4$]{
        \includegraphics[width=0.22\textwidth]{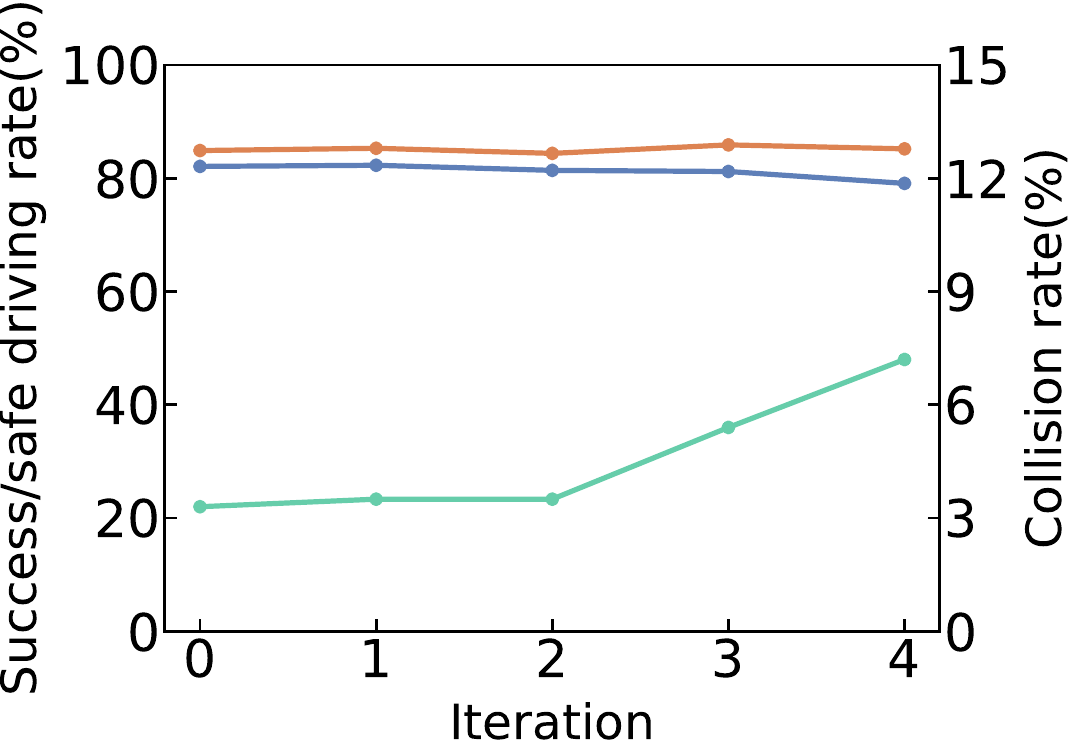}
    }\hspace{1.5pt}
    \subfloat[DSGN, $\alpha=1$]{
        \includegraphics[width=0.22\textwidth]{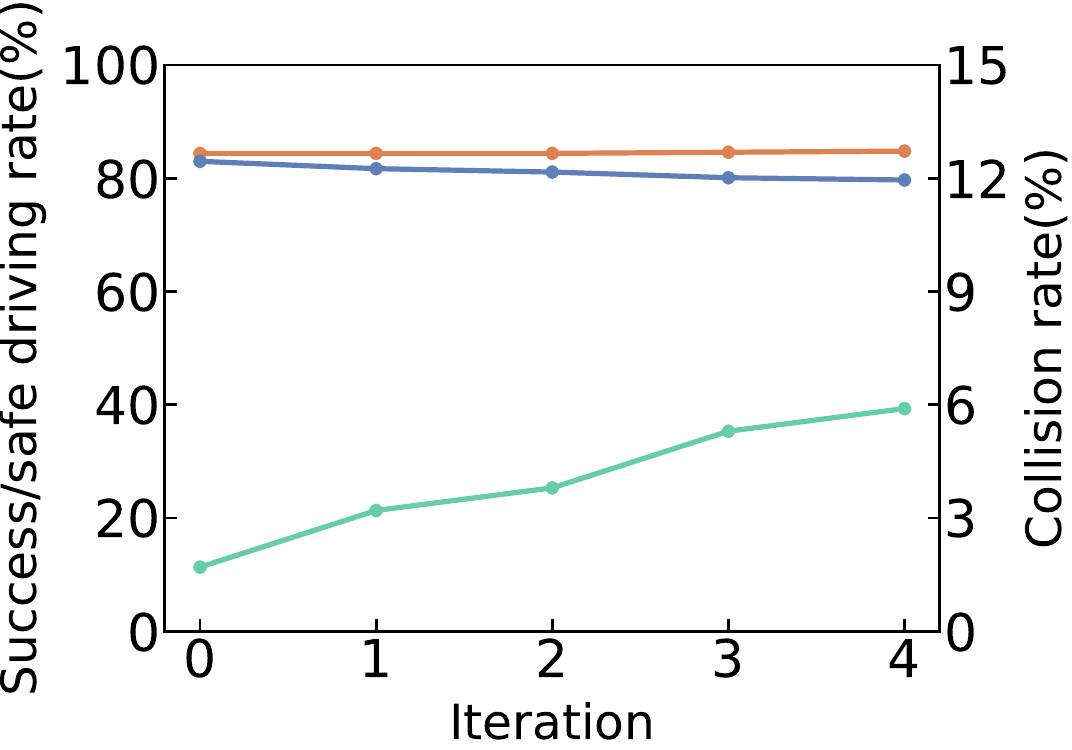}
    }\hspace{1.5pt}
    \subfloat[Stereo R-CNN, $\alpha=1$]{
        \includegraphics[width=0.22\textwidth]{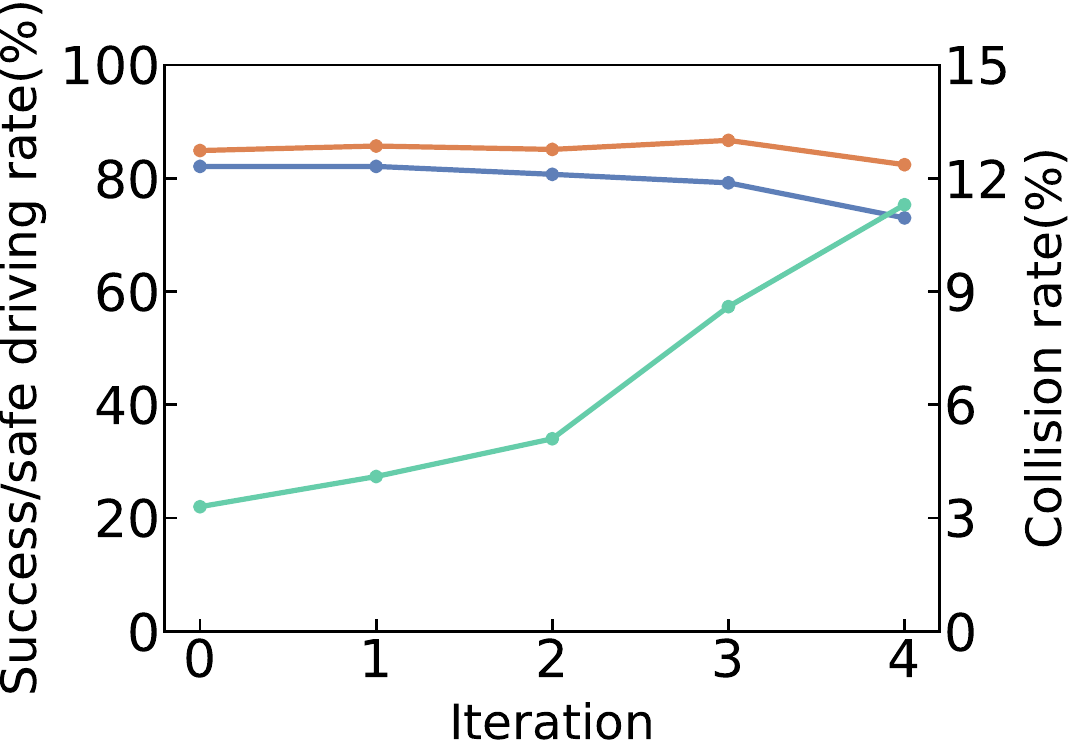}
    } \\
    \includegraphics[width=0.7\textwidth]{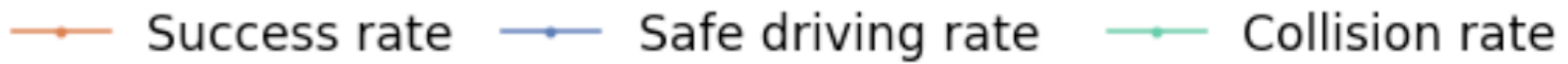}
    
    \caption{Driving safety performance metrics under the perturbation attack.}
    \label{fig:ds_perturbation}
\end{figure*}

\subsection{Implementation}

To implement this end-to-end driving safety evaluation framework for vision-based autonomous driving, we adopt two pre-trained models for the object detection module, namely, Stereo R-CNN~\cite{li19stereo} and DSGN~\cite{chen20dsgn}, which are currently two state-of-the-art methods in the area of vision-based 3D object detection. As for the motion planning module and the evaluation module, we use CommonRoad~\cite{althoff17commonroad} as the framework and leverage the built-in A*~\cite{hart68a} with sampled motion primitives as the motion planning method.

\begin{table*}[hbt!]
    \centering
    \caption{\\\textsc{Driving Safety Performance Metrics under the Perturbation Attack ($\alpha=0.4$)}}
    \resizebox{0.95\textwidth}{!}{
        \begin{tabular}{c|c|ccccc|ccccc}
        \hline
        \multirow{2}{*}{}          & Model & \multicolumn{5}{c|}{DSGN}             & \multicolumn{5}{c}{Stereo R-CNN}      \\ \cline{2-12} 
         & Iteration & Unattacked & 1     & 2     & 3     & 4     & Unattacked & 1     & 2     & 3     & 4     \\ \hline
        \multirow{3}{*}{Success rate ($\%$)}      & Left  & 89.6 & 89.6 & 89.5 & 89.3 & 89.5 & 89.8 & 89.8 & 90.4 & 90.0 & 88.8 \\
         & Straight  & 96.5      & 96.7 & 96.7 & 96.7 & 96.9 & 96.4      & 96.0 & 95.1 & 95.1 & 94.5 \\
         & Right     & 84.4      & 84.2 & 84.4 & 84.6 & 84.8 & 84.9      & 85.3 & 84.4 & 85.9 & 85.2 \\ \hline
        \multirow{3}{*}{Collision rate ($\%$)}    & Left  & 2.2 & 2.2 & 2.4 & 2.6 & 2.8 & 3.1 & 2.7 & 3.7 & 4.1 & 3.7 \\
         & Straight  & 0.7      & 0.7 & 0.7 & 0.9 & 1.1 & 1.1      & 1.1 & 0.9 & 1.3 & 1.7 \\
         & Right     & 1.7      & 2.3 & 2.7 & 3.6 & 3.8 & 3.3      & 3.5 & 3.5 & 5.4 & 7.2 \\ \hline
        \multirow{3}{*}{Safe driving rate ($\%$)} & Left  & 87.7 & 87.7 & 87.3 & 86.9 & 86.9 & 87.0 & 87.4 & 87.0 & 86.3 & 85.4 \\
         & Straight  & 95.8      & 96.0 & 96.0 & 95.8 & 95.8 & 95.3      & 94.9 & 94.1 & 93.8 & 92.9 \\
         & Right     & 83.0      & 82.2 & 82.0 & 81.5 & 81.5 & 82.1      & 82.3 & 81.4 & 81.2 & 79.1 \\ \hline
        \end{tabular}
    }
    \label{tab:ds_perturbation_0.4}
\end{table*}

\begin{table*}[hbt!]
    \centering
    \caption{\\\textsc{Driving Safety Performance Metrics under the Perturbation Attack ($\alpha=1$)}}
    \resizebox{0.95\textwidth}{!}{
        \begin{tabular}{c|c|ccccc|ccccc}
        \hline
        \multirow{2}{*}{}          & Model & \multicolumn{5}{c|}{DSGN}             & \multicolumn{5}{c}{Stereo R-CNN}      \\ \cline{2-12} 
         & Iteration & Unattacked & 1     & 2     & 3     & 4     & Unattacked & 1     & 2     & 3     & 4     \\ \hline
        \multirow{3}{*}{Success rate ($\%$)}      & Left  & 89.6 & 90.0 & 89.3 & 89.8 & 89.5 & 89.8 & 90.2 & 88.4 & 89.8 & 84.6 \\
         & Straight  & 96.5      & 96.7 & 96.5 & 96.5 & 96.9 & 96.4      & 95.8 & 94.2 & 95.5 & 90.6 \\
         & Right     & 84.4      & 84.4 & 84.4 & 84.6 & 84.8 & 84.9      & 85.7 & 85.1 & 86.7 & 82.4 \\ \hline
        \multirow{3}{*}{Collision rate ($\%$)}    & Left  & 2.2 & 2.6 & 2.8 & 3.4 & 3.4 & 3.1 & 3.3 & 3.2 & 4.7 & 6.4 \\
         & Straight  & 0.7      & 0.7 & 0.9 & 1.3 & 1.3 & 1.1      & 1.5 & 1.2 & 2.1 & 3.5 \\
         & Right     & 1.7    & 3.2 & 3.8 & 5.3 & 5.9 & 3.3      & 4.1 & 5.1 & 8.6 & 11.3 \\ \hline
        \multirow{3}{*}{Safe driving rate ($\%$)} & Left  & 87.7 & 87.7 & 86.7 & 86.7 & 86.4 & 87.0 & 87.2 & 85.5 & 85.5 & 79.2 \\
         & Straight  & 95.8      & 96.0 & 95.6 & 92.9 & 95.6 & 95.3      & 94.3 & 93.1 & 93.4 & 87.4 \\
         & Right     & 83.0      & 81.7 & 81.1 & 80.1 & 79.7 & 82.1      & 82.1 & 80.7 & 79.2 & 73.0 \\ \hline
        \end{tabular}
    }
    \label{tab:ds_perturbation_1}
\end{table*}

To implement the moving object classifier, we extract in total $600$ real driving scenarios from the KITTI raw dataset~\cite{geiger13vision} and manually label each object in each driving scenario with a moving/static property. We use 500 scenarios for training and 100 scenarios for validation. To determine whether an object is moving or not in a driving scenario, we refer to the previous and the subsequent image frames of that scenario. Though it is easy for human to judge the moving object from sequential frames, two people are assigned to manually label the scenarios independently in order to eliminate personal bias or errors in manual labelling. The independently produced labels are checked together for consistency, and no inconsistent labelling is found. We adopt the $16$-layer VGG net~\cite{simonyan14very} as the core network of the moving object classifier and replace its fully connected layers, i.e., $fc6$, $fc7$, and $fc8$, with a flatten layer, a new fully connected layer with a dropout layer and ReLU activation function, and another new fully connected layer with a dropout layer and a sigmoid activation function, respectively, to make sure that there is only one output score to indicate the probability of a moving object. The validation results suggest that the accuracy of the trained moving object classifier is 98.31\%.

In addition, to implement the driving constraint selector, we also leverage the KITTI raw dataset~\cite{geiger13vision} to train the model so that it can classify the road type of a scenario. Specifically, we divide the dataset into two subsets, i.e., \textbf{street} and \textbf{highway}. The street subset consists of city and residential scenarios where the traffic speed is relatively low, while the highway subset contains highway scenarios in which vehicles move much faster. Accordingly, we pre-define two sets of motion primitives for two road types so that the selector can pick the motion primitive with appropriate speed ranges and steering angle ranges after classifying the road type. The selector also chooses the dynamics constraints for moving vehicles predicted by the moving object classifier. The network architecture of the driving constraint selector consists of $5$ convolution layers connected by max-pooling layers and $1$ fully connected layer with dropout. Both convolution layers and the fully connected layer use ReLU as the activation function. After excluding the scenarios without cars, we select 444 scenarios as the training dataset and 112 scenarios as the validation dataset. The validation result indicates that the accuracy of the driving constraint selector achieves 94.64\%.

\begin{figure*}
\centering
\subfloat[Clean image input.]{
    \includegraphics[width=0.35\textwidth]{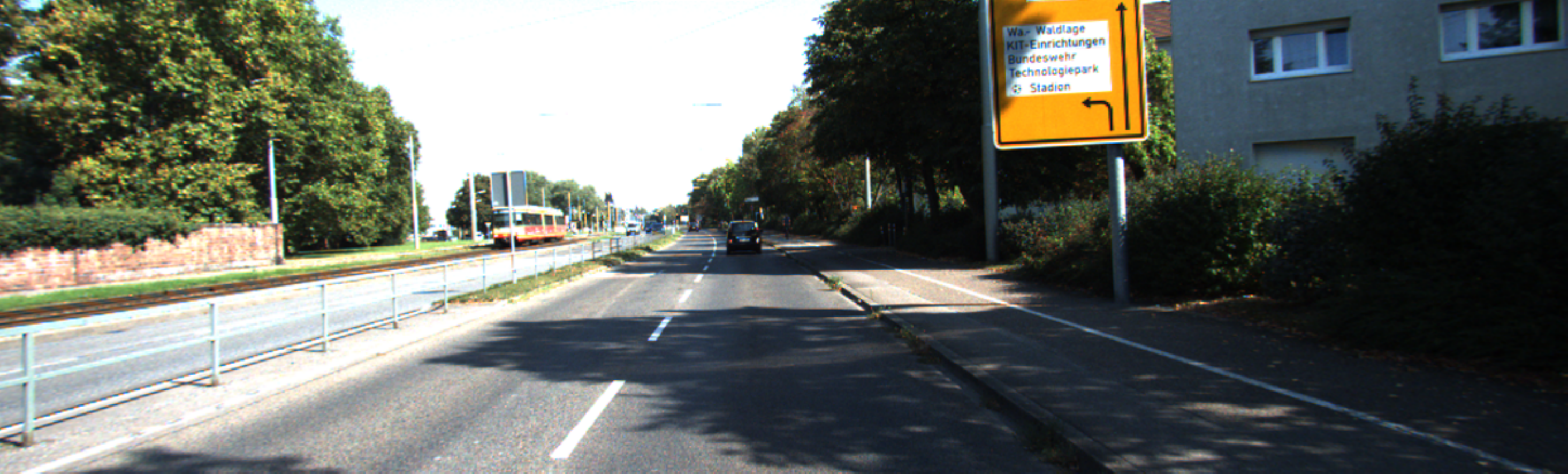}
}\hspace{12pt}
\subfloat[Ground truth of object detection.]{
    \includegraphics[width=0.35\textwidth]{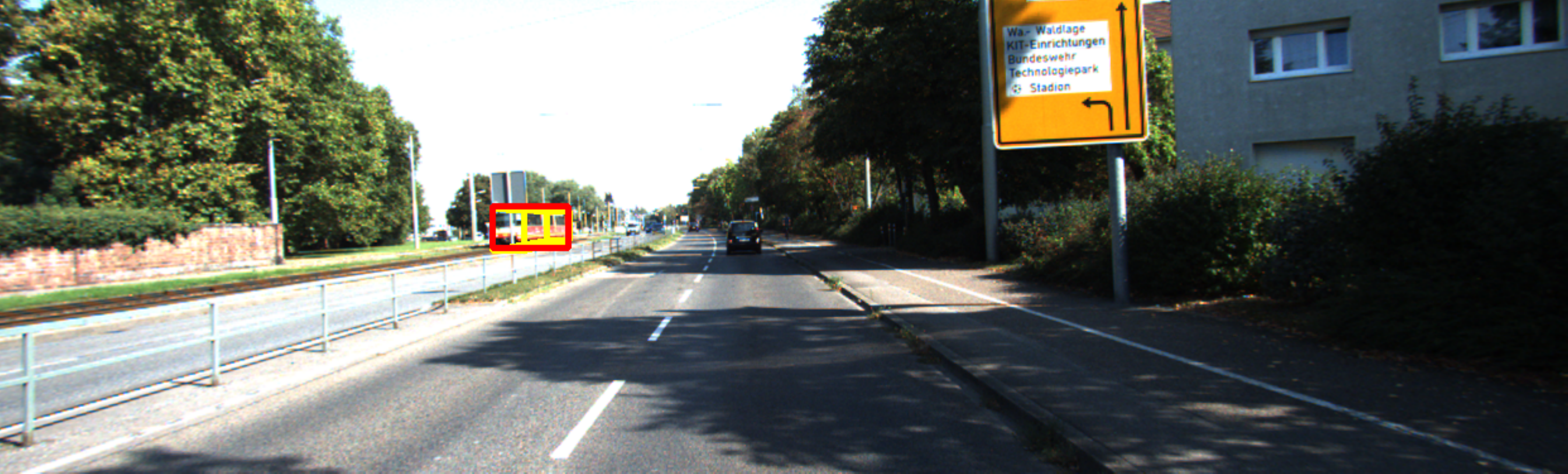}
} \\
\subfloat[Detection results of DSGN without attack.]{
    \includegraphics[width=0.35\textwidth]{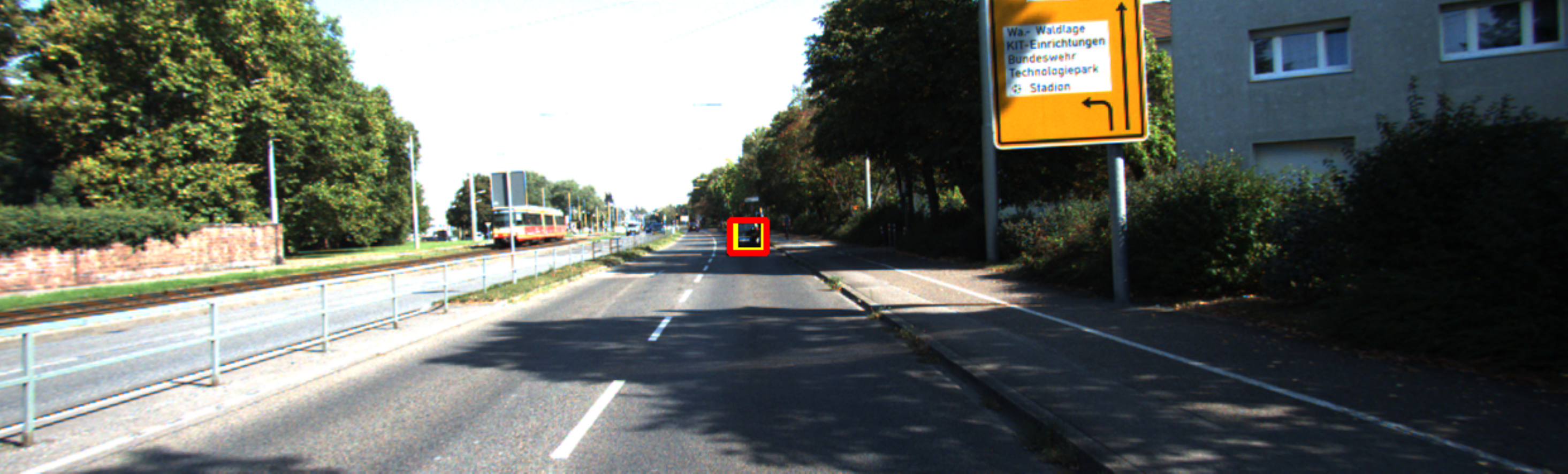}
}\hspace{12pt}
\subfloat[Detection results of DSGN under attack.]{
    \includegraphics[width=0.35\textwidth]{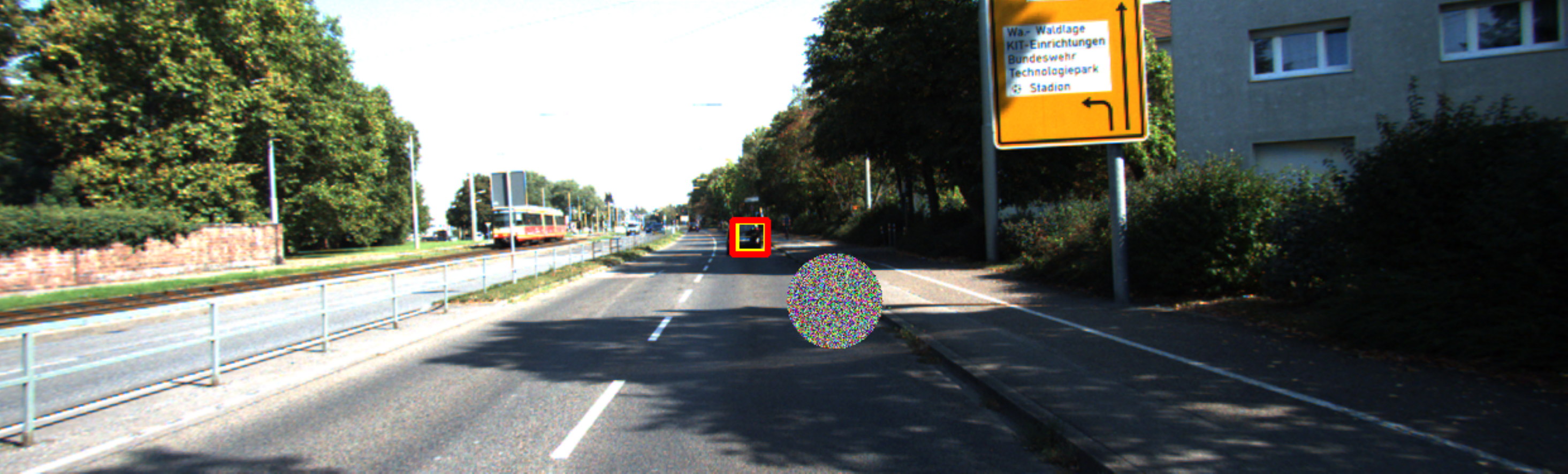}
} \\
\subfloat[Detection results of Stereo R-CNN without attack.]{
    \includegraphics[width=0.35\textwidth]{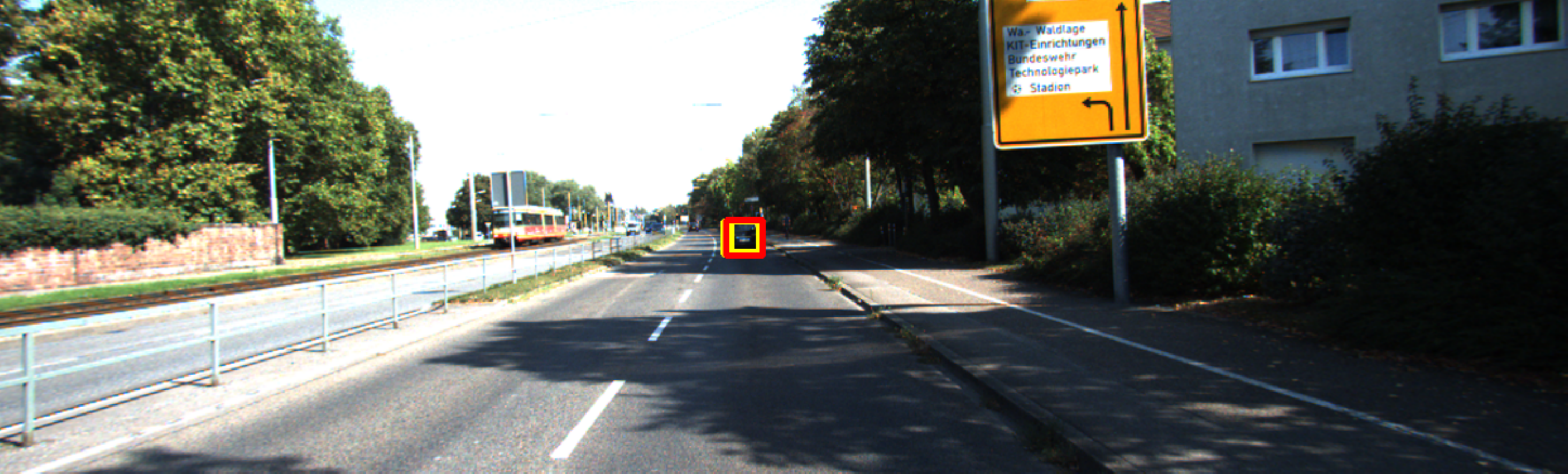}
}\hspace{12pt}
\subfloat[Detection results of Stereo R-CNN under attack.]{
    \includegraphics[width=0.35\textwidth]{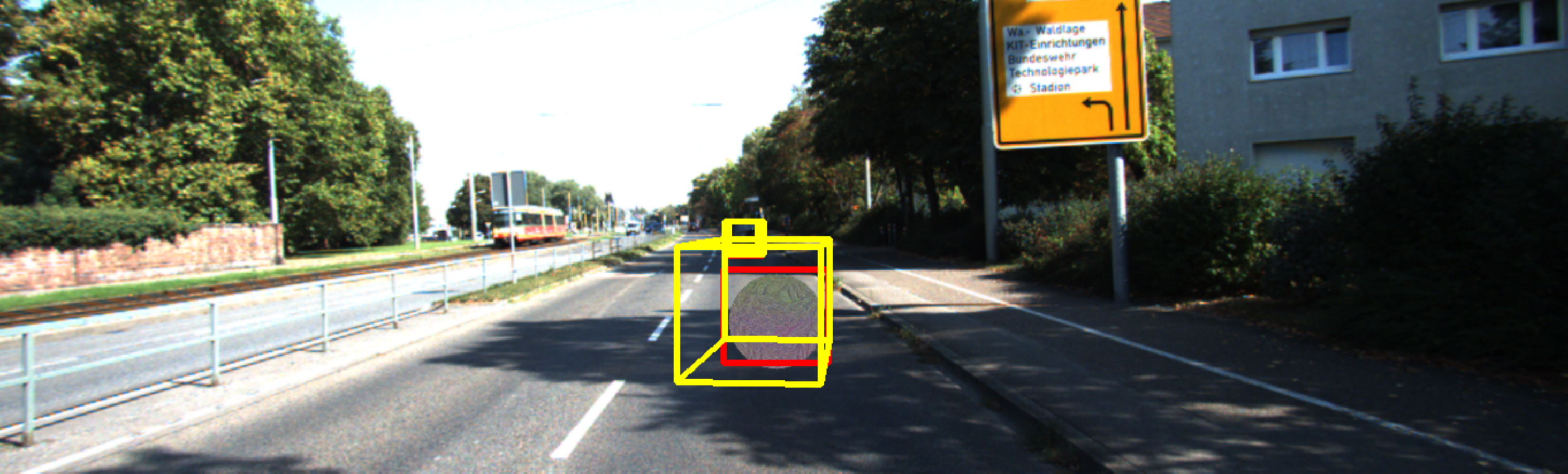}
} \\
\subfloat[Clean image input.]{
    \includegraphics[width=0.35\textwidth]{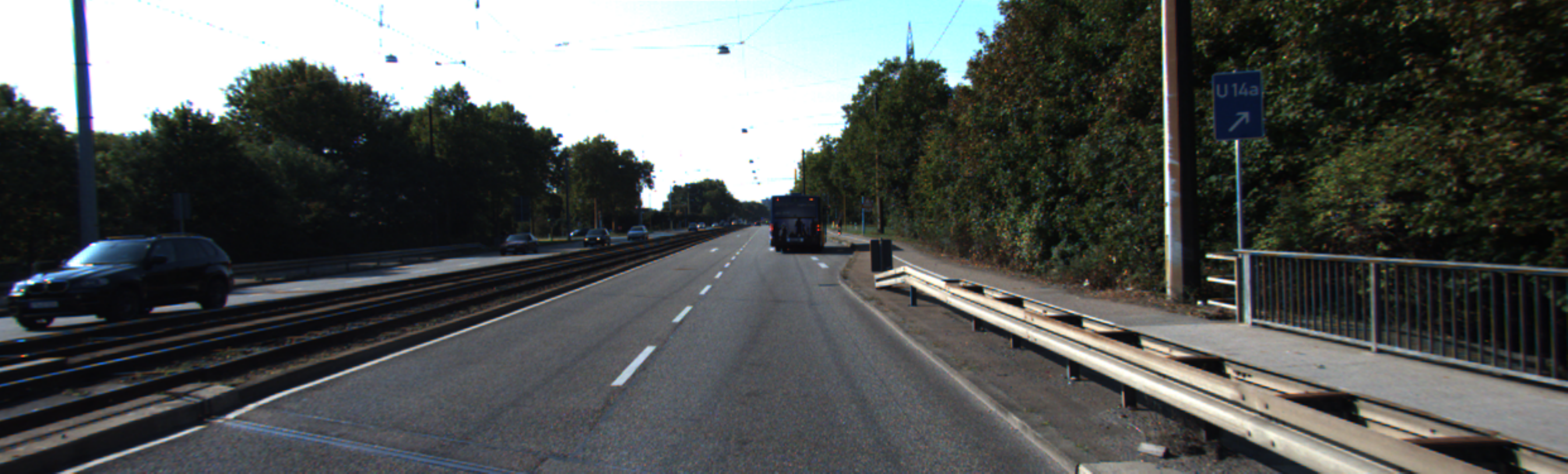}
}\hspace{12pt}
\subfloat[Ground truth of object detection.]{
    \includegraphics[width=0.35\textwidth]{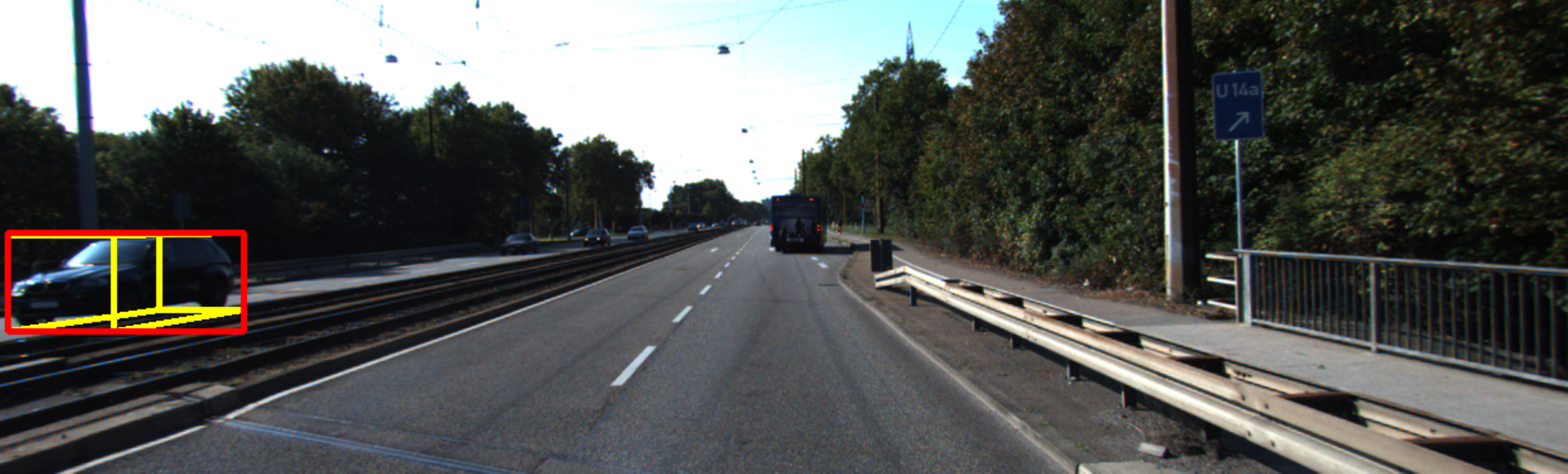}
} \\
\subfloat[Detection results of DSGN without attack.]{
    \includegraphics[width=0.35\textwidth]{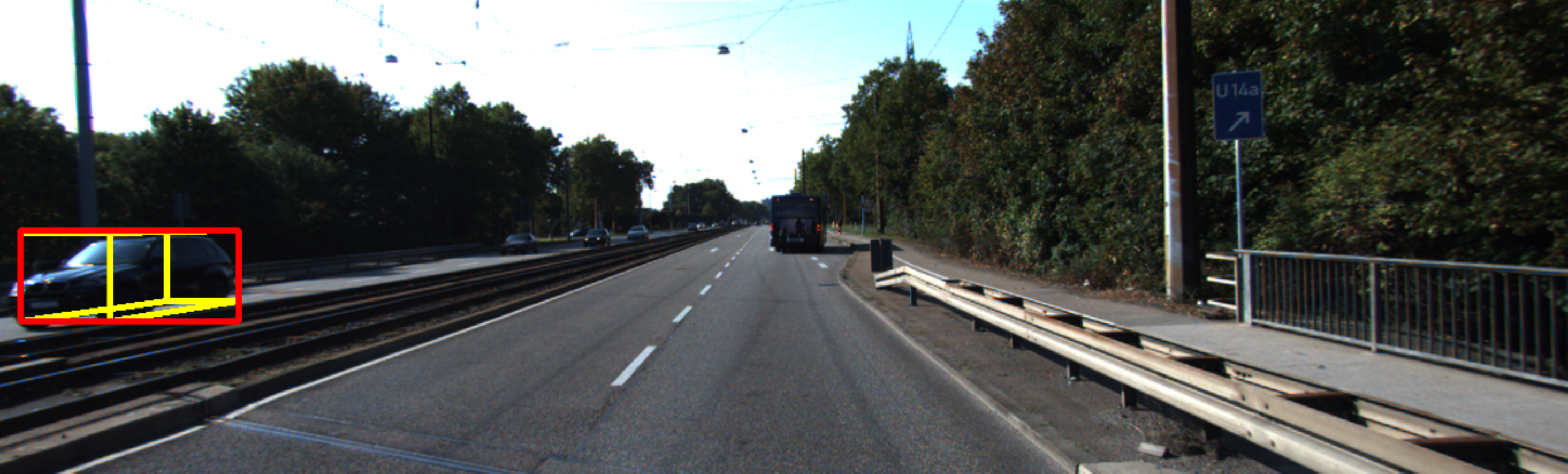}
}\hspace{12pt}
\subfloat[Detection results of DSGN under attack.]{
    \includegraphics[width=0.35\textwidth]{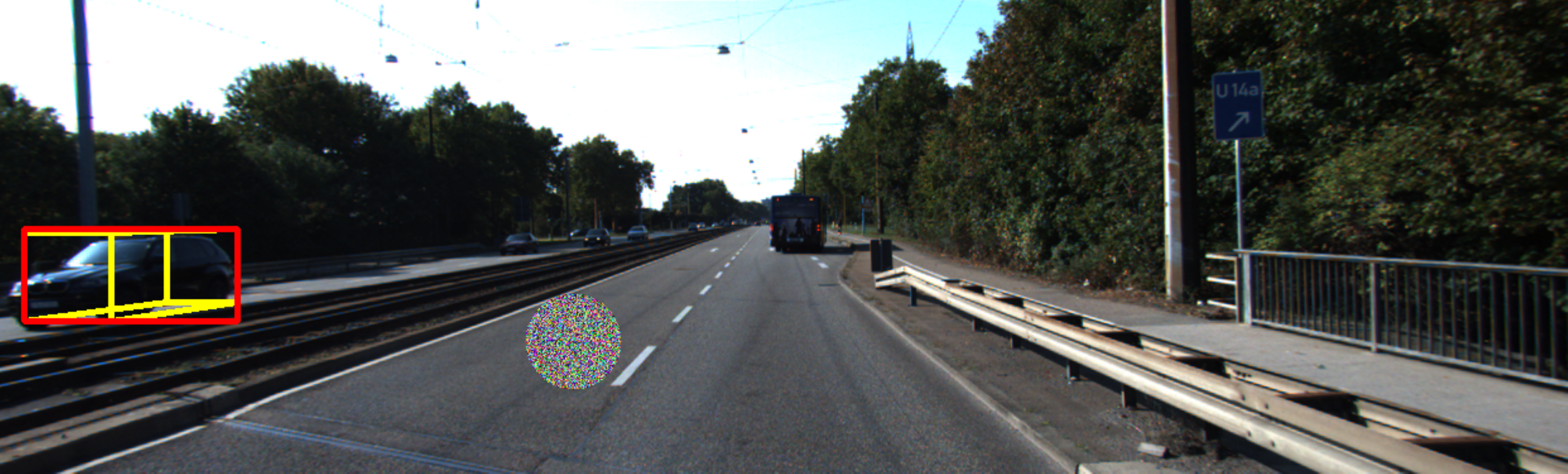}
} \\
\subfloat[Detection results of Stereo R-CNN without attack.]{
    \includegraphics[width=0.35\textwidth]{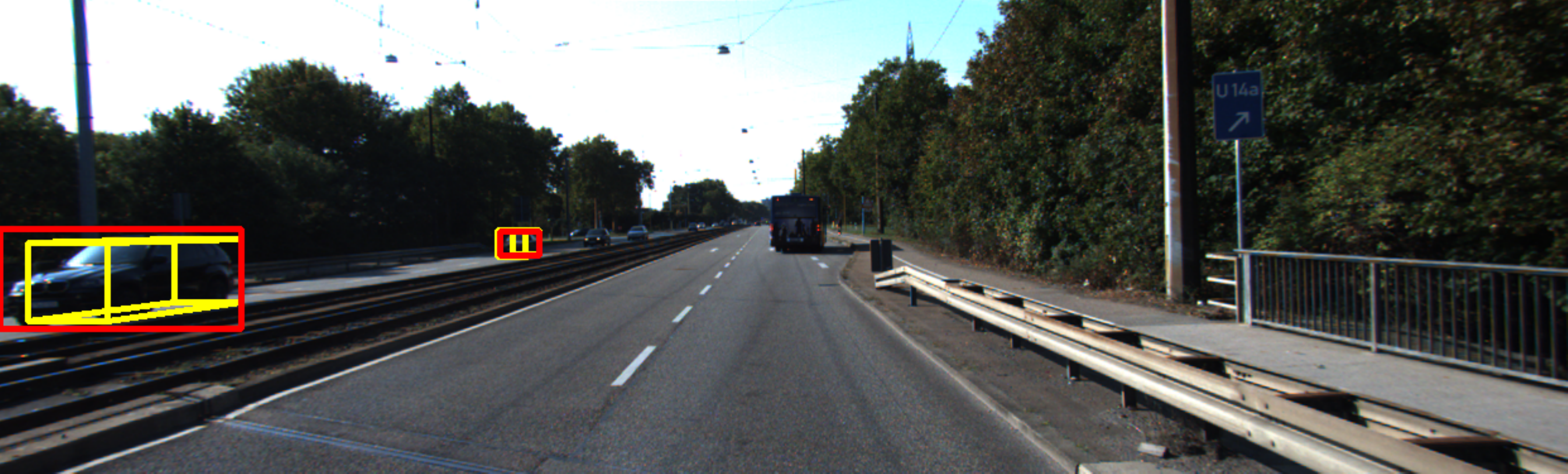}
}\hspace{12pt}
\subfloat[Detection results of Stereo R-CNN under attack.]{
    \includegraphics[width=0.35\textwidth]{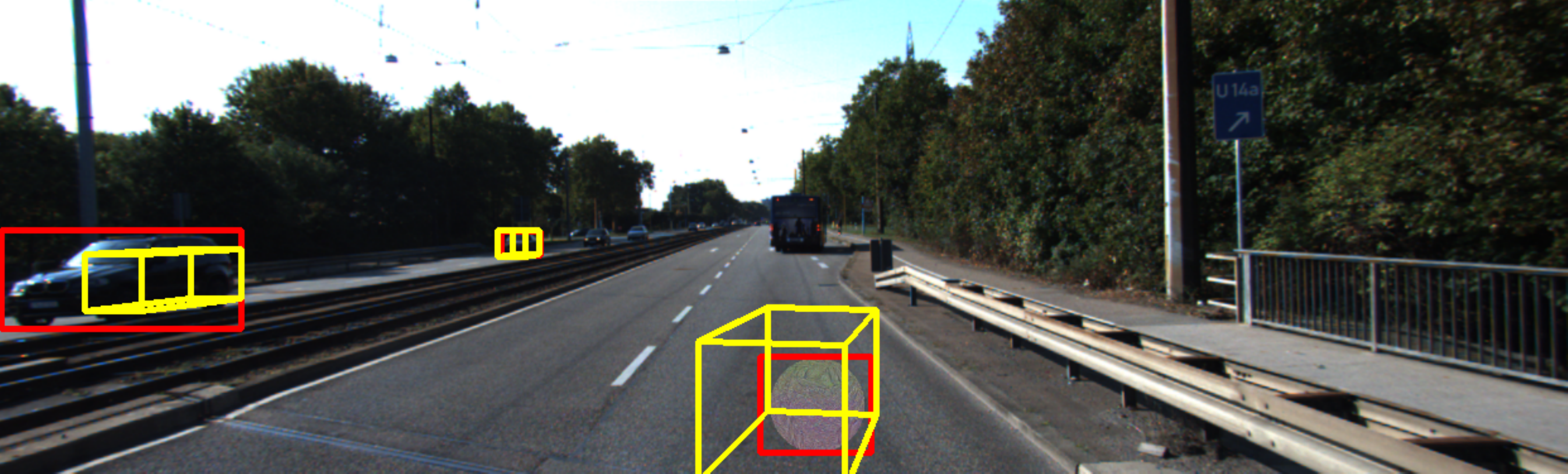}
} 
\caption{The patch attack triggers 3D object detectors to generate ghost bounding boxes at in the area of the patch as shown in (f) and (l). It has little influence on the detection of the objects away from the patch.}
\label{fig:plot_patch}
\end{figure*}

\section{Experiments}
\label{ch5:sec.experiments}

We conduct extensive experiments to investigate the impact of  perturbation attacks and patch attacks on driving safety of vision-based autonomous vehicles. We first introduce the common setup for all experiments, then elaborate on the specific settings for each attack experiment and present corresponding evaluation results. Finally, we summarize our findings at the end of this section.

\subsection{Common Setup}

In this chapter, we conduct all experiments by applying the two types of adversarial attacks with different settings in our driving safety evaluation framework. Specifically, the evaluation framework includes two object detection modules, Stereo R-CNN~\cite{li19stereo} and DSGN~\cite{chen20dsgn}. In order to assess the impact comprehensively, we gradually increase the attack intensity by changing the attack settings in fine-grained steps. For the same purpose, we also consider three driving intentions of the ego-vehicle for each scenario when planning the trajectory, namely, \textbf{changing to left lane}, \textbf{changing to right lane}, and \textbf{keeping lane}, which are abbreviated as \textbf{left}, \textbf{right}, and \textbf{straight}, respectively. For these three cases, the initial position of the ego-vehicle is the same and the goal region is located 15 meters away from the initial position but within three different adjacent lanes. Moreover, we randomly assign  an initial speed within the selected speed range to each moving vehicle, including the ego-car. Specifically, the  initial speed for moving vehicles in street scenarios is randomly assigned within the range of $[22, 29]$ km/h, considering the $30$ km/h speed limit in German cities, campus and residential areas. The initial speed in highway scenarios is randomly assigned within the range of $[40, 47]$ km/h, concerning the $50$ km/h speed limit of built-up roads in Germany. For each attack, after the framework processes all the scenarios and generates the motion planning results, it assesses the attack impact on the performance metrics of driving safety as well as on the accuracy of 3D object detector. By linking these two attack impacts together, we manage to obtain evaluation results that help answer the questions raised in Section~\ref{ch5:sec.introduction}. The models of Stereo R-CNN~\cite{li19stereo} and DSGN~\cite{chen20dsgn} are pretrained with 3712 data points from the KITTI object detection dataset~\cite{geiger12are}. For each experiment setting, we test 600 real driving scenarios. The platform that we use is a Ubuntu 18.04 server equipped with an Nvidia Tesla V100 GPU.

In our experiments, the evaluation of driving safety is based on the trajectory produced by the motion planning module and measured by the driving safety performance metrics. In terms of evaluating the accuracy of the vision-based 3D object detector, we adopt the KITTI object detection benchmark that tests the detector with a three-level standard, namely, \textbf{easy}, \textbf{moderate}, and \textbf{hard}~\cite{geiger12are}. We follow the standard to measure the \textbf{average precision} (AP) of the detector with \textbf{Intersection over Union} (IoU) larger than 70\%.

\begin{table*}[hbt!]
    \centering
    \caption{\\\textsc{Average Precision for 3D Object Detection under Patch Attack}}
    \resizebox{0.95\textwidth}{!}{
        \begin{tabular}{c|c|ccccc|ccccc}
        \hline
        \multirow{3}{*}{}        & Model    & \multicolumn{5}{c|}{DSGN}                & \multicolumn{5}{c}{Stereo R-CNN}          \\ \cline{2-12} 
         &
          \multirow{2}{*}{Scenario} &
          \multirow{2}{*}{Unattacked} &
          \multirow{1}{*}{Random} &
          \multicolumn{3}{c|}{Specific Attack} &
          \multirow{2}{*}{Unattacked} &
          \multirow{1}{*}{Random} &
          \multicolumn{3}{c}{Specific Attack} \\ \cline{5-7} \cline{10-12} 
                                         &          &       & Attack & Left  & Straight & Right &       & Attack & Left  & Straight & Right \\ \hline
        \multirow{1}{*}{3D} & easy     & 70.94 & 63.85 & 65.72 & 64.96    & 64.87 & 56.47 & 53.17 & 48.14 & 47.82    & 50.07 \\
                          Detection    & moderate & 52.98 & 48.20 & 51.47 & 50.77    & 51.63 & 38.20 & 37.07 & 36.27 & 35.23    & 38.02 \\
                                  AP ($\%$)   & hard     & 47.29 & 44.30 & 46.80 & 46.15    & 46.48 & 32.66 & 31.88 & 31.21 & 30.60    & 32.45 \\ \hline
        \end{tabular}
    }
    \label{tab:ap_patch}
\end{table*}

\begin{table*}[hbt!]
    \centering
    \caption{\\\textsc{Driving Safety Performance Metrics under the Patch Attack}}
    \resizebox{0.95\textwidth}{!}{
        \begin{tabular}{c|c|ccc|ccc}
        \hline
        \multirow{2}{*}{}              & Model    & \multicolumn{3}{c|}{DSGN}                    & \multicolumn{3}{c}{Stereo R-CNN}              \\ \cline{2-8} 
                                              & Scenario & Unattacked & Random Attack & Specific Attack & Unattacked & Random Attack & Specific Attack \\ \hline
        \multirow{3}{*}{Success rate(\%)}   & Left     & 89.6 & 90.0 & 90.0 & 89.8 & 70.3 & 83.7 \\
                                            & Straight & 96.5 & 96.5 & 96.7 & 96.4 & 81.2 & 57.5 \\
                                            & Right    & 84.4 & 84.6 & 84.8 & 84.9 & 68.7 & 74.4 \\ \hline
        \multirow{3}{*}{Collision rate(\%)} & Left     & 2.2  & 2.6  & 2.8  & 3.1  & 1.9  & 2.1  \\
                                            & Straight & 0.7  & 0.7  & 0.9  & 1.1  & 1.4  & 4.7  \\
                                            & Right    & 1.7  & 2.3  & 2.1  & 3.3  & 3.9  & 4.2  \\ \hline
        \multirow{3}{*}{Safe driving rate(\%)} & Left     & 87.7       & 87.7          & 87.5            & 87.0       & 68.9          & 81.9            \\
                                            & Straight & 95.8 & 95.8 & 95.8 & 95.3 & 80.1 & 54.7 \\
                                            & Right    & 83.0 & 82.6 & 83.0 & 82.1 & 65.9 & 71.2 \\ \hline
        \end{tabular}
    }
    \label{tab:ds_patch}
\end{table*}

\subsection{Perturbation Attack}

In order to perform the perturbation attack against autonomous driving systems at various intensities, we adjust two parameters $\alpha$ and $n$ in Eqn.~(\ref{eq:eq_2}). To ensure that the perturbation is imperceivable to human eyes, their values usually should be set as small as possible. Specifically, we set the value of $\alpha$ as $0.4$ and $1$, to represent medium to high attack intensities, respectively. The number of iterations $n$ changes from $1$ to $4$ accordingly, so that the modification on image pixel values is constrained within the range of $[0.4,4]$. We note that even the attack with the lowest attack intensity, i.e., $\alpha=0.4$ and $n=1$, can cause significant decline in the accuracy of 3D object detectors. Moreover, the produced perturbation and the input stereo images have the same dimension.

\begin{table*}[hbt!]
\centering
\caption{\\\textsc{Driving Safety Performance Metrics of Stereo R-CNN under the Patch Attack with Various Intentions}}
\resizebox{0.95\textwidth}{!}{
\begin{tabular}{c|c|ccc}
\hline
                       & Specific attack & \multicolumn{3}{c}{Random attack}                  \\ \cline{2-5} 
                       & —               & Same intentions$^1$ & Different intentions$^2$ & Unattacked  \\ \hline
Success rate (\%)      & 71.9            & 76.2            & 91.6                 & 90.4       \\ \hline
Collision rate (\%)    & 3.5             & 2.3             & 1.5                  & 2.4        \\ \hline
Safe driving rate (\%) & 69.3            & 74.4            & 90.1                 & 88.1       \\ \hline
\end{tabular}
}
\\
$^1$ "Same intentions" refers to cases where the attack intention and the driving intention are the same.
\\
$^2$ "Different intentions" refers to cases where the attack intention differs from the driving intention.
\label{tab:intentions}
\end{table*}

\begin{table}[hbt!]
\centering
\caption{\\\textsc{Safe Driving Rate using Different Planning Algorithms}}
\begin{tabular}{c|c|c|c}
\hline
\multirow{2}{*}{}                       & Planning algorithm & GBFS            & A*              \\ \cline{2-4} 
                                        & Scenario           & \multicolumn{2}{c}{Ground Truth} \\ \hline
\multirow{3}{*}{Safe driving rate (\%)} & Left               & 87.9            & 89.7            \\
                                        & Straight           & 98.0            & 98.0            \\
                                        & Right              & 82.3            & 85.2            \\ \hline
\end{tabular}
\label{tab:plan_alg}
\end{table}

\begin{table*}[hbt!]
\centering
\caption{\\\textsc{Safe Driving Rate with Different Inputs}}
\resizebox{0.95\textwidth}{!}{
\begin{tabular}{c|c|c|ccc|ccc}
\hline
\multirow{2}{*}{}                       & Model    & —    & \multicolumn{3}{c|}{DSGN} & \multicolumn{3}{c}{Stereo R-CNN} \\ \cline{2-9} 
 &
  Scenario &
  \begin{tabular}[c]{@{}c@{}}Ground \\ Truth\end{tabular} &
  Unattacked &
  \begin{tabular}[c]{@{}c@{}}Perturbation\\ Attack\end{tabular} &
  \begin{tabular}[c]{@{}c@{}}Patch\\ Attack\end{tabular} &
  Unattacked &
  \begin{tabular}[c]{@{}c@{}}Perturbation\\ Attack\end{tabular} &
  \begin{tabular}[c]{@{}c@{}}Patch\\ Attack\end{tabular} \\ \hline
\multirow{3}{*}{Safe driving rate (\%)} & Left     & 89.7 & 87.7    & 86.4   & 87.7   & 87.0      & 79.2      & 68.9      \\
                                        & Straight & 98.0 & 95.8    & 95.6   & 95.8   & 95.3      & 87.4      & 80.1      \\
                                        & Right    & 85.2 & 83.0    & 79.7   & 82.6   & 82.1      & 73.0      & 65.9      \\ \hline
\end{tabular}
}
\label{tab:atk_impact}
\end{table*}

\textbf{Evaluation.} The effect of the perturbation attack in some driving scenarios is shown in Figure~\ref{fig:plot_perturbation} where we can see that the attack causes inaccurate detection of real objects and false detection of ghost objects. We present the impact of the perturbation attack with different settings on average precision of 3D object detection and on driving safety metrics in Figure~\ref{fig:ap_perturbation} and Figure~\ref{fig:ds_perturbation}, respectively. The numerical results of driving safety scores can be found in Table~\ref{tab:ds_perturbation_0.4} and Table~\ref{tab:ds_perturbation_1}. When the number of iterations $n$ is 0, it indicates that there is no attack applied. From Figure~\ref{fig:ap_perturbation}, we can observe that with the enhanced attack intensity by increasing $\alpha$ and $n$, the average precision of both object detection models drops significantly, while the driving safety metrics only show very small changes. Take DSGN as an example. When $\alpha$ is 0.4 and $n$ is increased from $0$ (no attack) to $1$, the AP declines by more than half for all three levels of the benchmark standard, i.e., from 70.94\% to 21.99\% for the category of AP easy, from 52.98\% to 14.45\% for the category of AP moderate, and from 47.29\% to 13.96\% for the category of AP hard. When $n$ is 3, the AP of DSGN almost reaches 0. However, in the meantime, the driving safety performance metrics in Figure~\ref{fig:ds_perturbation} barely change for all three intention cases (e.g., collision rate is only changed from $1.7\%$ to $3.6\%$ for the case of changing to right lane). When $\alpha = 1$ and $n$ is increased from $0$ (no attack) to $4$, the AP of DSGN drops even more significantly, but the driving safety metrics only demonstrate slightly larger changes than that when $\alpha = 0.4$ (e.g., the safe driving rate drops by $0.8\%$ when $\alpha = 0.4$ and by $1.3\%$ when $\alpha = 1$ for the case of changing to left lane). The experiment results clearly indicate that the perturbation attack can dramatically affect the performance of 3D object detection methods, but does not have much influence on the driving safety. In other words, a larger precision decline of the vision-based 3D object detectors under the perturbation attack does not indicate higher risk of driving safety.

Moreover, by comparing DSGN and Stereo R-CNN in terms of driving safety under perturbation attacks, we can observe that the changes in driving safety metrics for Stereo R-CNN tend to be larger than the changes for DSGN when both of them are tested in the same driving intention scenarios and at the same intensity. Therefore, Stereo R-CNN is more prone to  perturbation attacks than DSGN in regard to driving safety.

\subsection{Patch Attack}

Different from  perturbation attacks, the size of a patch in a patch attack is much smaller than the size of an input image. In our patch attack experiments, the radius of the patch is limited to $38$ pixels. Here, the patch attack is launched as a white-box attack, which means that the patch is trained for Stereo R-CNN and DSGN, respectively. Specifically, we train the patch according to Eqn.~\ref{eq:eq_3} by placing the patch at a random position in stereo image pair and setting $b^{*}$ in Eqn.~\ref{eq:eq_4} accordingly for each training scenario. To ensure that the patch for Stereo R-CNN and the patch for DSGN are equally optimized, we use the same learning step size and the same number of epochs when training patches. 

We design two attack approaches, namely, \textbf{random attack} and \textbf{specific attack}. Random attacks are to place the trained patch at a random position within the entire image no matter which driving intention case it is. In other words, the attack intention may or may not be consistent with the driving intention in random attacks. Specific attacks are also to place the patch randomly but within a certain region of the image, depending on the driving intention case, e.g., if the driving intention is  changing to right lane, then the patch is placed in the right part of the image. In other words, the attack intention is always consistent with the driving intention in specific attacks. By designing two attack approaches, we can create different attack intensities for patch attacks. Specifically, the attack intensity of random attacks is lower than that of specific attacks.

\textbf{Evaluation.} The performance of patch attacks in some driving scenarios is shown in Figure~\ref{fig:plot_patch} in which we can observe that patch attacks cause false detection of ghost objects. The  impact of  patch attacks on object detection and driving safety are shown in Table~\ref{tab:ap_patch} and Table~\ref{tab:ds_patch} respectively. From the tables, we can observe that, when different attack approaches are applied, the average precision of both object detection models declines slightly, while some of the driving safety metrics degrade significantly. For example, when random patch attacks are applied to the Stereo R-CNN model, AP declines slightly for all three levels of the benchmark standard, i.e., from $56.47\%$ to $53.17\%$ for the case of AP easy, from $38.20\%$ to $37.07\%$ for the case of AP moderate, and  $32.66\%$ from to $31.88\%$ for the case of AP hard. However, the driving safety performance metrics of Stereo R-CNN have a relatively larger drop under random patch attacks (e.g., safe driving rate drops from $95.3\%$ to $80.1\%$ for the case of keeping lane). At the same time, for specific patch attacks, Stereo R-CNN shows the similar average precision decline which is only within the range of $[0.21\%, 8.65\%]$, while significant driving safety performance degradation can be observed (e.g., the safe driving rate decreases to half for the case of keeping lane). The experiment results suggest that a slight precision decline of the 3D object detectors under  patch attacks does not indicate mild risk of driving safety. 

Since the driving safety performance of Stereo R-CNN can be significantly affected by patch attacks, we further investigate the performance under the attacks where the attack intention is the same as the driving intention, and the attacks where the attack intention is different from the driving intention. The results are listed in Table~\ref{tab:intentions}. From Table~\ref{tab:intentions}, we can see that  the driving safety performance under the attacks where the driving intention and the attack intention are different is very similar to that in unattacked scenarios, and the performance under the attacks where the attack intention is the same as the driving intention is very close to that in specific attack scenarios.

Furthermore, the DSGN model again shows its much better robustness in object detection and driving safety under patch attacks. We can observe that,  even under the well-designed specific patch attacks, DSGN's average precision decline is only less than $6\%$, and the driving safety performance metrics almost remain unchanged, while Stereo R-CNN performs worse under both  random patch attacks and  specific patch attacks.

\begin{figure*}[!t]
\subfloat[No attack.]{
    \includegraphics[width=0.32\textwidth]{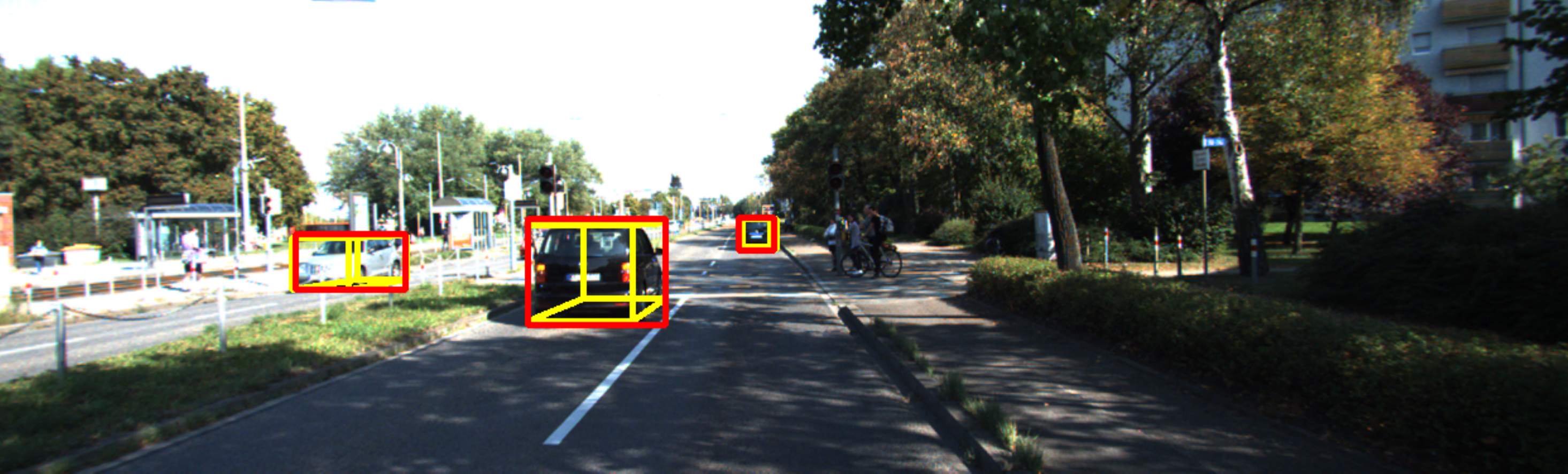}
    \label{fig:ab_1_stereo_1}
}
\subfloat[Attack applied.]{
    \includegraphics[width=0.32\textwidth]{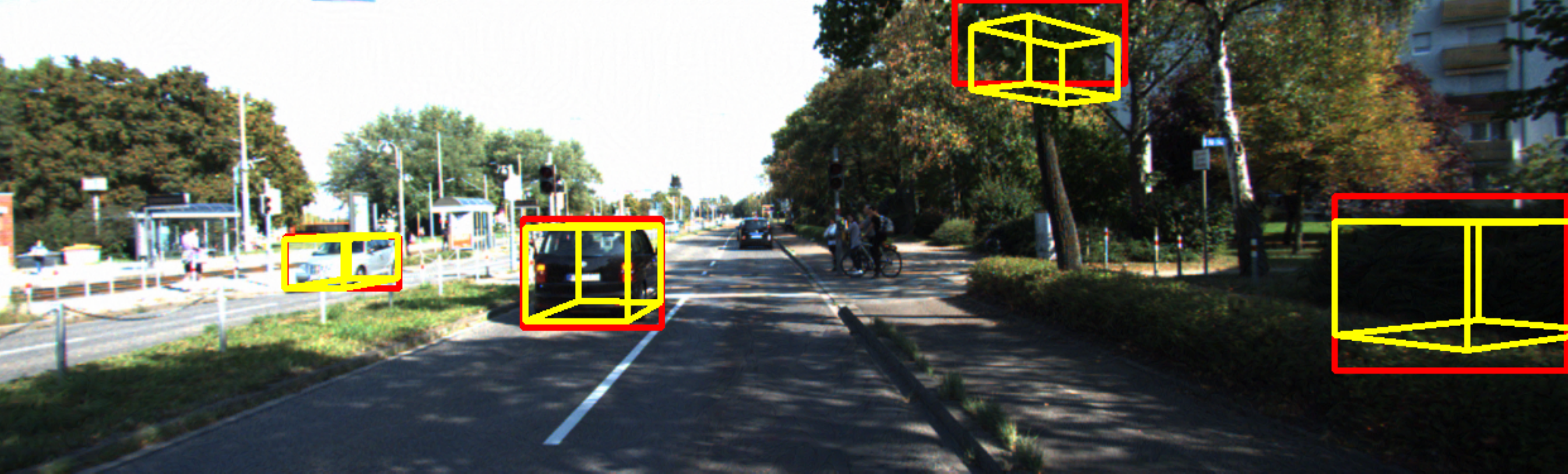}
    \label{fig:ab_1_stereo_2}
}
\subfloat[Attack applied and texture of side area replaced.]{
    \includegraphics[width=0.32\textwidth]{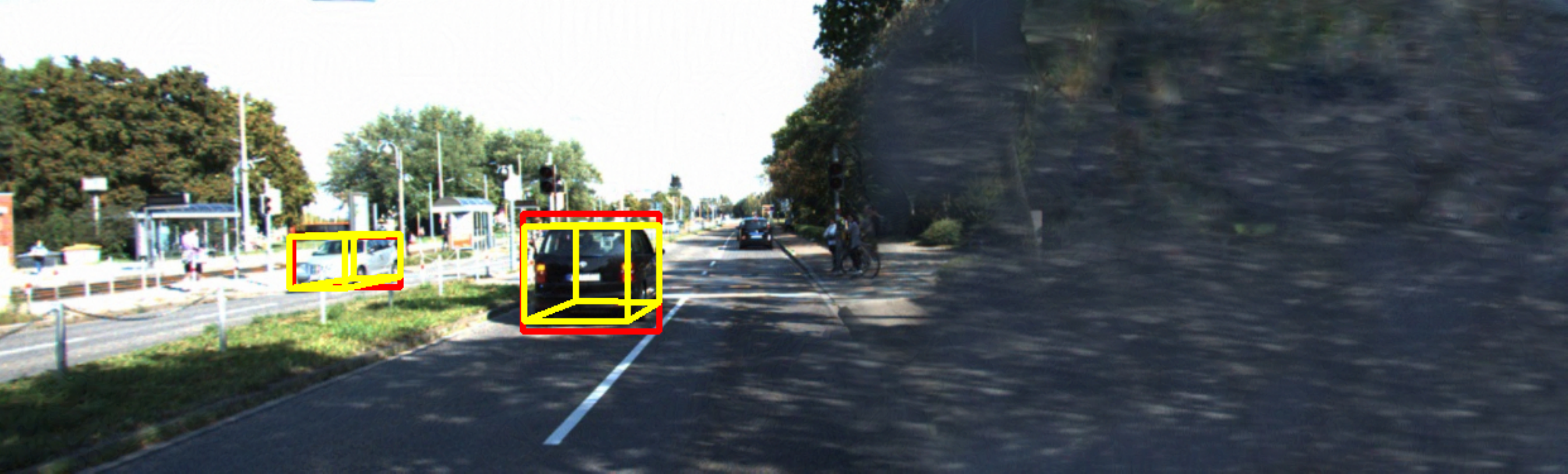}
    \label{fig:ab_1_stereo_3}
}
\caption{Results of the texture replacement experiment for Stereo R-CNN.}
\label{fig:ab_1_stereo}
\end{figure*}

\begin{figure*}[!t]
\centering
\subfloat[No attack.]{
    \includegraphics[width=0.31\textwidth]{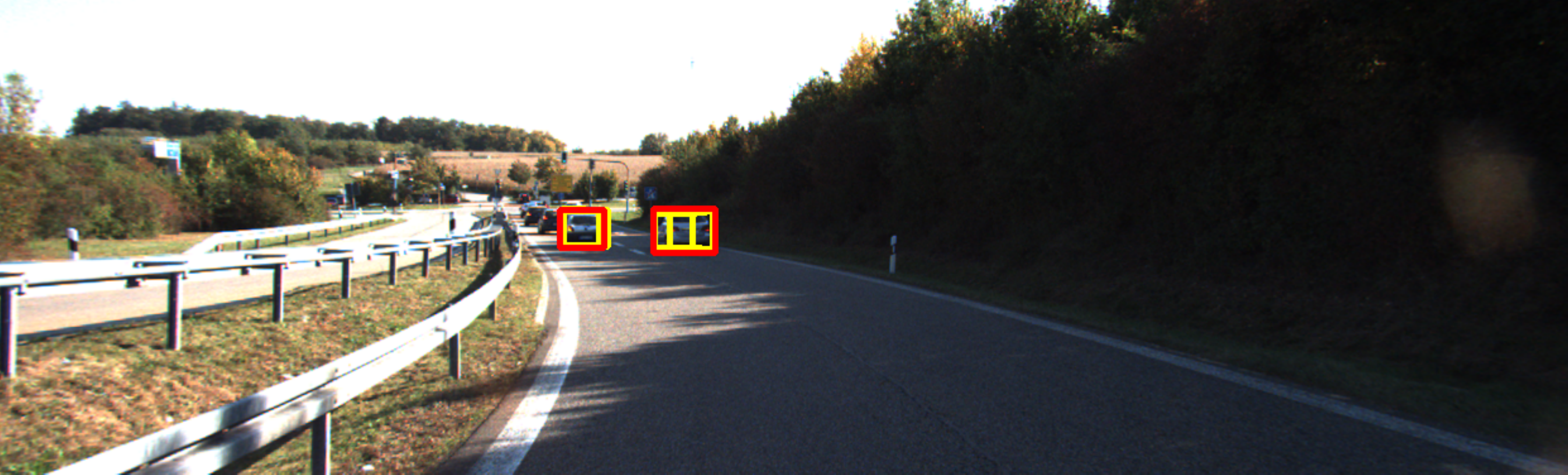}\label{fig:ab_1_dsgn_1}
}
\subfloat[Attack applied.]{
    \includegraphics[width=0.31\textwidth]{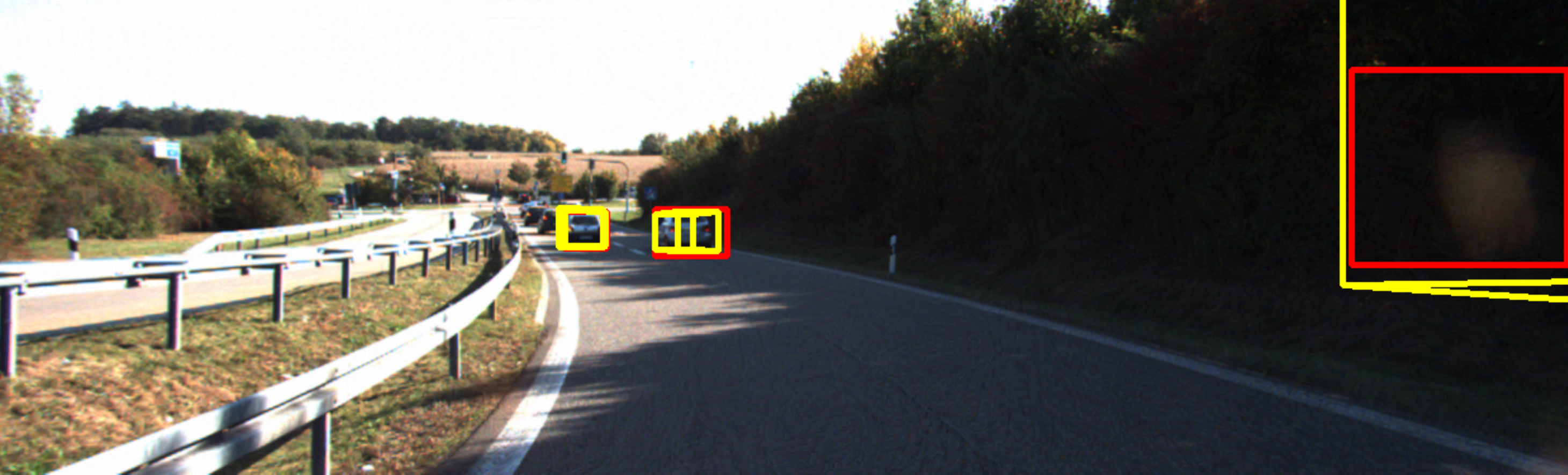}\label{fig:ab_1_dsgn_2}
}
\subfloat[Attack applied and texture of side area replaced.]{
    \includegraphics[width=0.31\textwidth]{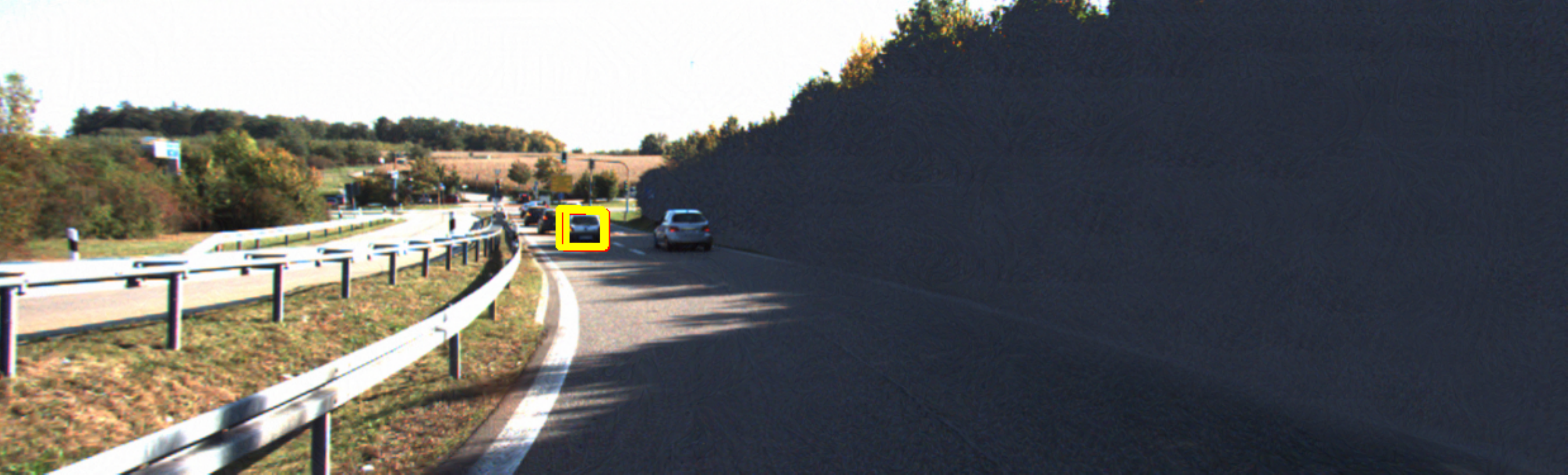}\label{fig:ab_1_dsgn_3}
}
\caption{Results of the texture replacement experiment for DSGN.}
\label{fig:ab_1_dsgn}
\end{figure*}

\subsection{Attack Impact Demonstration}

To demonstrate that the performance of different models under different attacks is mainly caused by adversarial attacks, not by the motion planning algorithms, we conduct two experiments.

We first evaluate the performance of the motion planning module using different inputs and then calculate the safe driving rates in different scenarios. Specifically, we first use the ground truth data of 3D object detection as the inputs to evaluate two popular motion planning algorithms, A* and Greedy-BFS, so as to show the impact of the motion planning module on the driving safety. The experimental results are summarized in Table~\ref{tab:plan_alg}. From Table~\ref{tab:plan_alg}, we can first observe that, when ground truth data are used, the A* planning algorithm can achieve the safe driving rates $89.7\%$, $98.0\%$, and $85.2\%$ for the three driving intention scenarios, respectively, while the Greedy-BFS algorithm can achieve the safe driving rates $87.9\%$, $98.0\%$, and $82.3\%$ for the three driving intentions scenarios, respectively. The performance of A* and the performance of the Greedy-BFS are very close. In other words, the performance variance demonstrated by DSGN and Stereo R-CNN under adversarial attacks is irrelevant to the selection of the motion planning algorithm. Since the performance of A* is slightly better than that of Greedy-BFS, we select A* as the motion planning algorithm for our driving safety evaluation framework.

We then use detection data without attacks (i.e., unattacked) and detection data under two types of attacks to demonstrate the impact of detection module and adversarial attacks on the driving safety. The results are shown in Table~\ref{tab:atk_impact}. From Table~\ref{tab:atk_impact}, we can observe that the safe driving rates produced by the detection data without attacks are slightly smaller than the safe driving rates when the ground truth data are used as inputs. This slight difference is caused by the accuracy of the two models. Finally, compared with the safe driving rates when the inputs are unattacked detection data, the safe driving rates under adversarial attacks are significantly declined in all driving intention scenarios. Since all experiments are conducted using the same motion planning algorithm, we can conclude that the declines in the driving safety performance metrics are primarily caused by adversarial attacks.

\subsection{Findings}

To briefly summarize, the findings from the experiments of perturbation attacks and patch attacks are listed as follows:

\begin{itemize}
\item
A larger precision decline of the attacked vision-based object detectors does not necessarily indicate a higher risk of driving safety. Similarly, a slight precision decline of the vision-based object detectors under attacks does not necessarily indicate a small risk of driving safety, either. Hence, the precision decline or the erroneous rate increase of the vision-based object detectors under attacks cannot represent their robustness with respect to driving safety of autonomous vehicles.

\item
Stereo R-CNN is less robust than DSGN in terms of driving safety and detection accuracy when the attacks launched on them are at the same intensity level.  Hence, DSGN is a better selection of the vision-based 3D object detection for its stronger robustness and higher detection precision.
\end{itemize}

In terms of these two findings, we further design more experiments to explain the real causes behind them in the next section.

\section{Ablation Study}
\label{ch5:sec.ablationstudy}

In this section, we first investigate the reason of the decoupling between the precision of 3D object detectors and the driving safety performance metrics under adversarial attacks. Second, we investigate the reason why the DSGN model is more robust than the Stereo R-CNN model.

\begin{figure*}[!t]
\centering
\begin{minipage}{0.070\textwidth}
    \centerline{\includegraphics[width=1\textwidth]{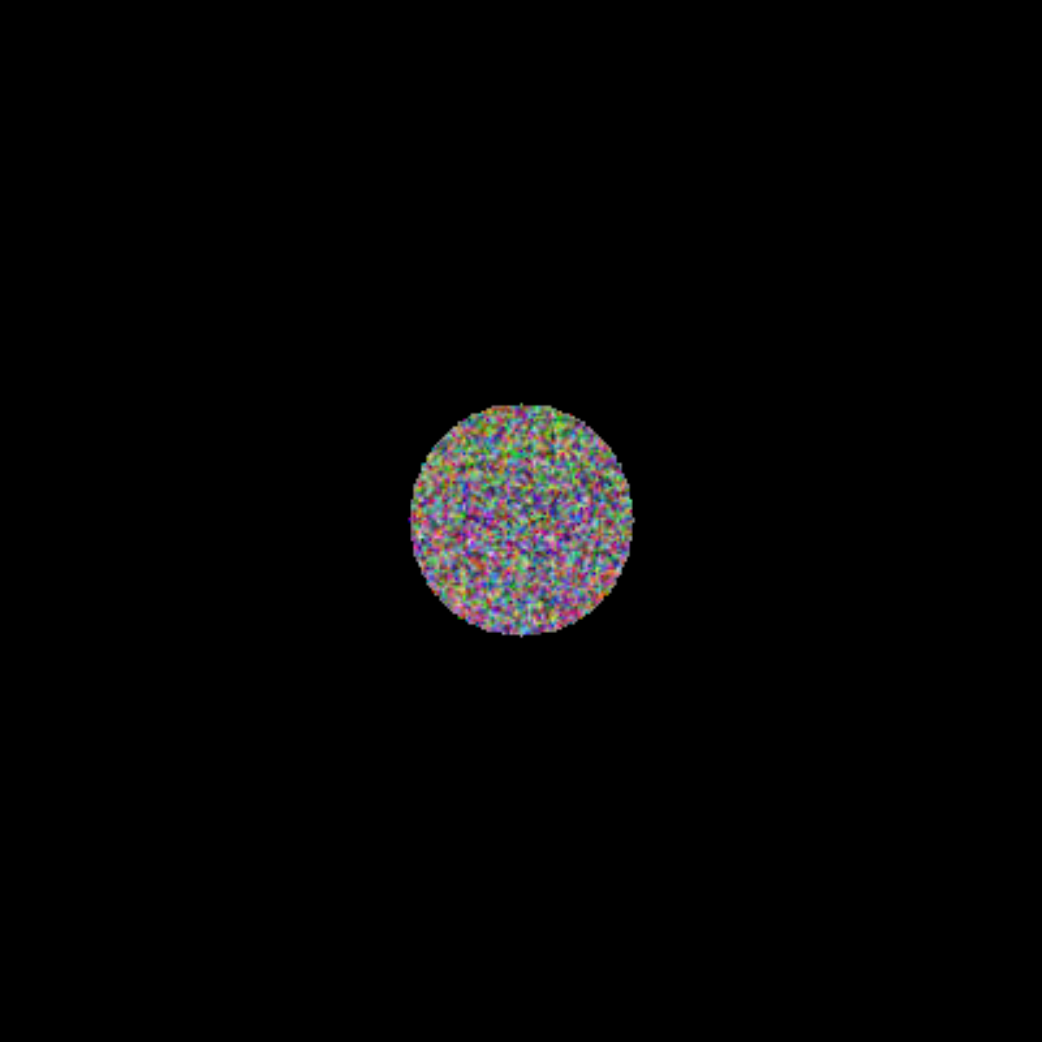}}
    \centerline{\scriptsize Input}
\end{minipage}
\begin{minipage}{0.070\textwidth}
    \centerline{\includegraphics[width=1\textwidth]{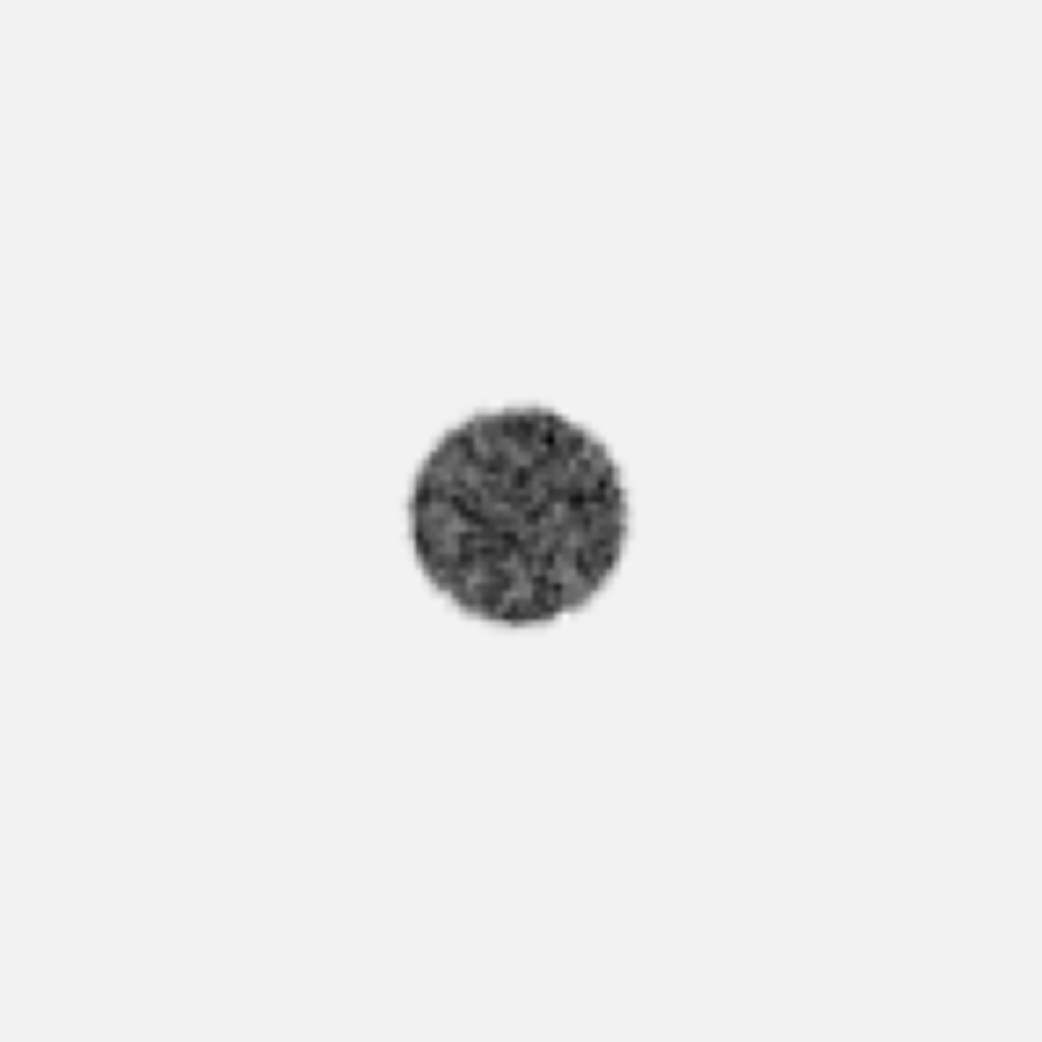}}
    \centerline{\scriptsize conv1}
\end{minipage}
\begin{minipage}{0.070\textwidth}
    \centerline{\includegraphics[width=1\textwidth]{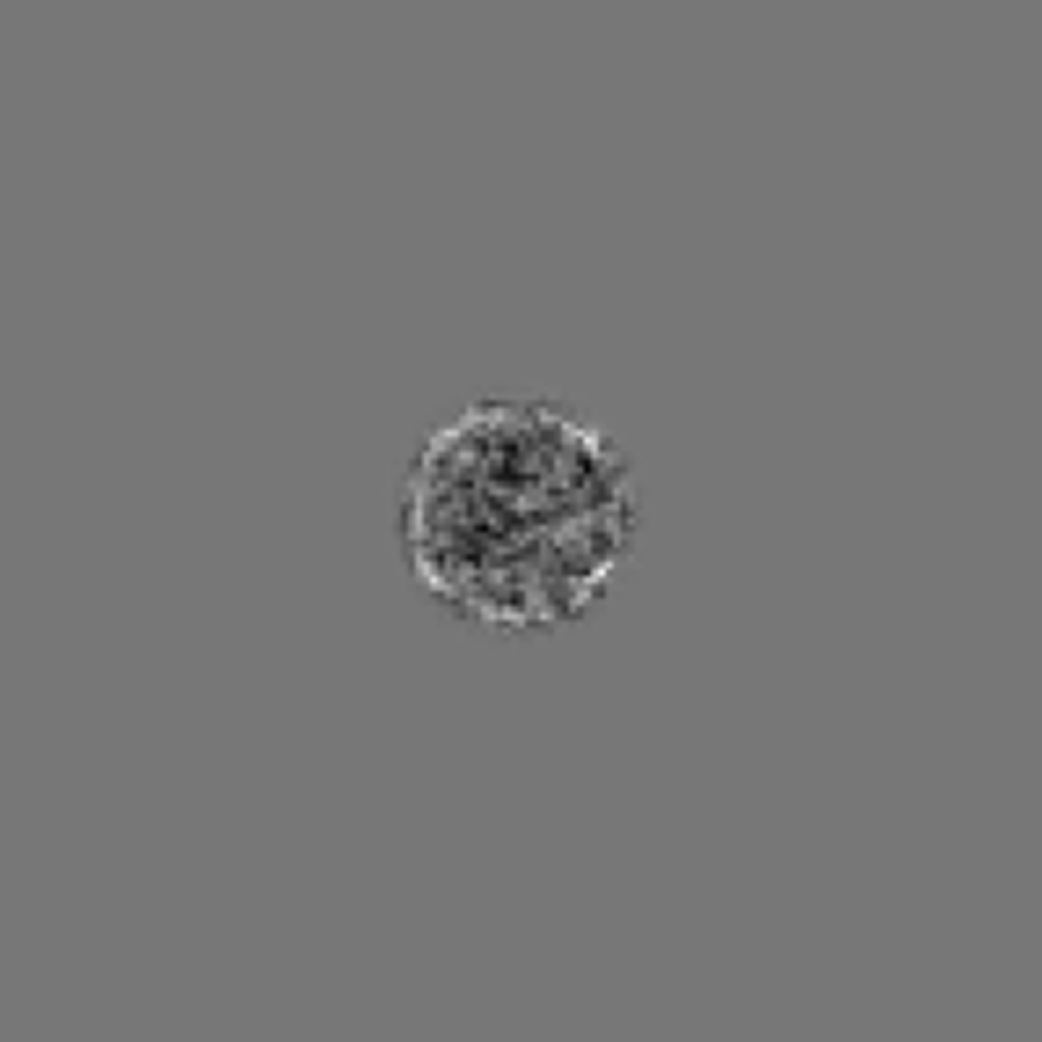}}
    \centerline{\scriptsize conv2}
\end{minipage}
\begin{minipage}{0.070\textwidth}
    \centerline{\includegraphics[width=1\textwidth]{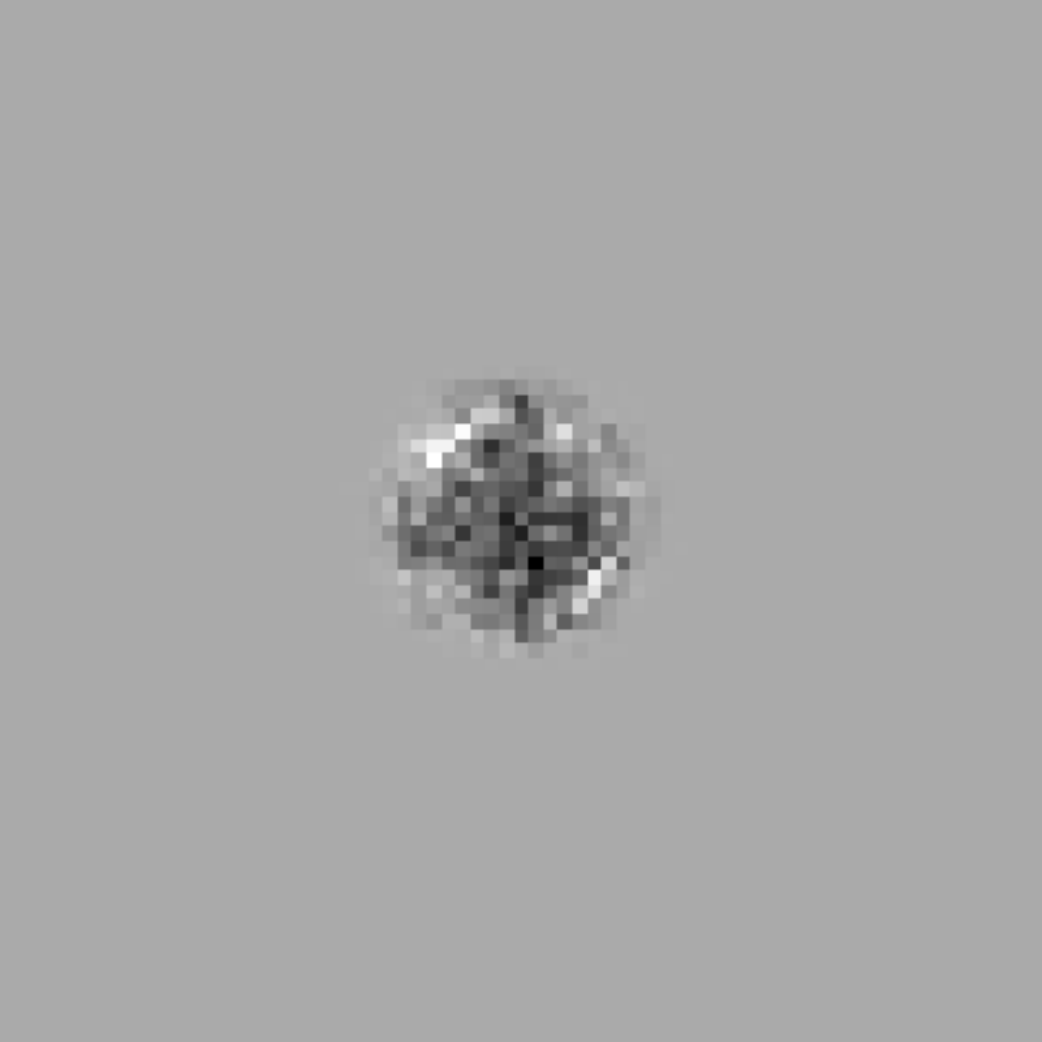}}
    \centerline{\scriptsize conv3}
\end{minipage}
\begin{minipage}{0.070\textwidth}
    \centerline{\includegraphics[width=1\textwidth]{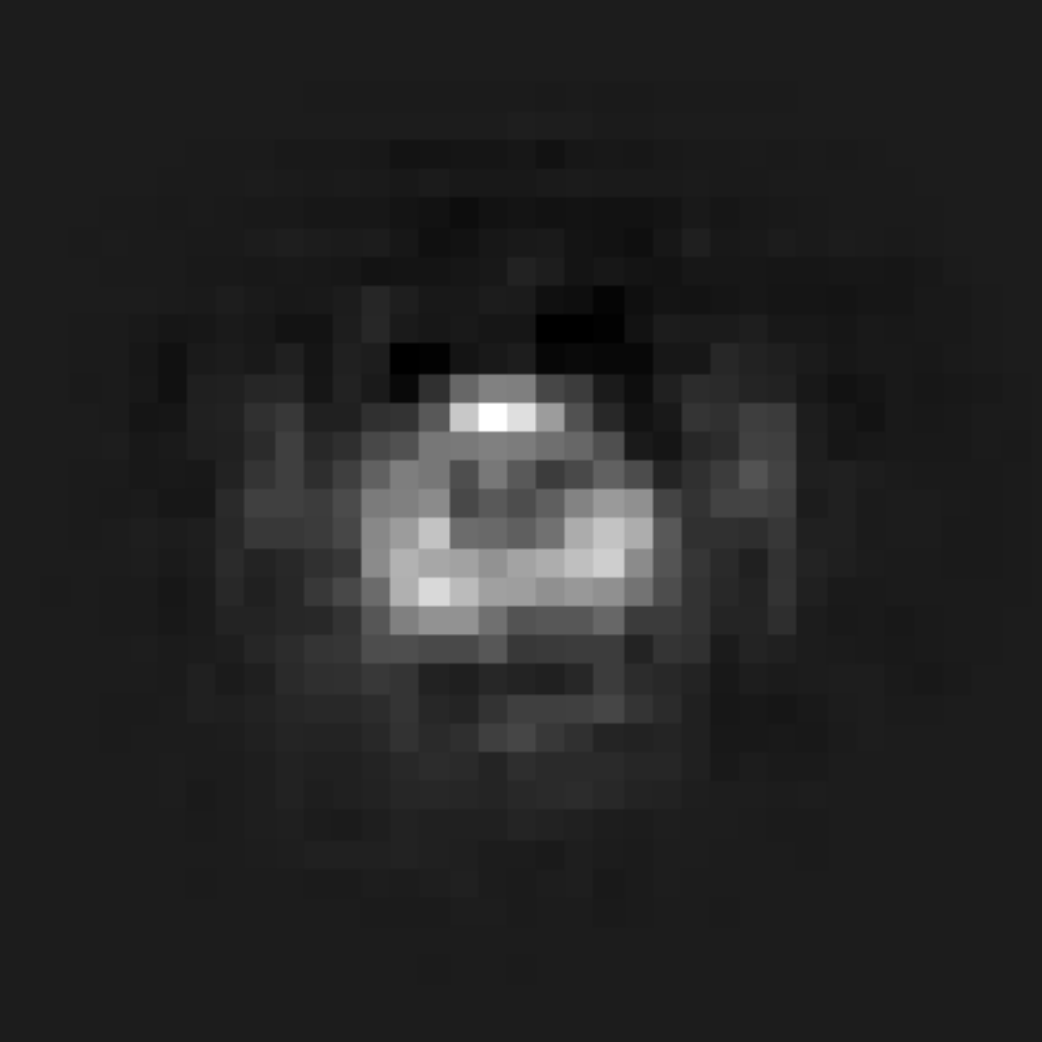}}
    \centerline{\scriptsize conv4}
\end{minipage}
\begin{minipage}{0.070\textwidth}
    \centerline{\includegraphics[width=1\textwidth]{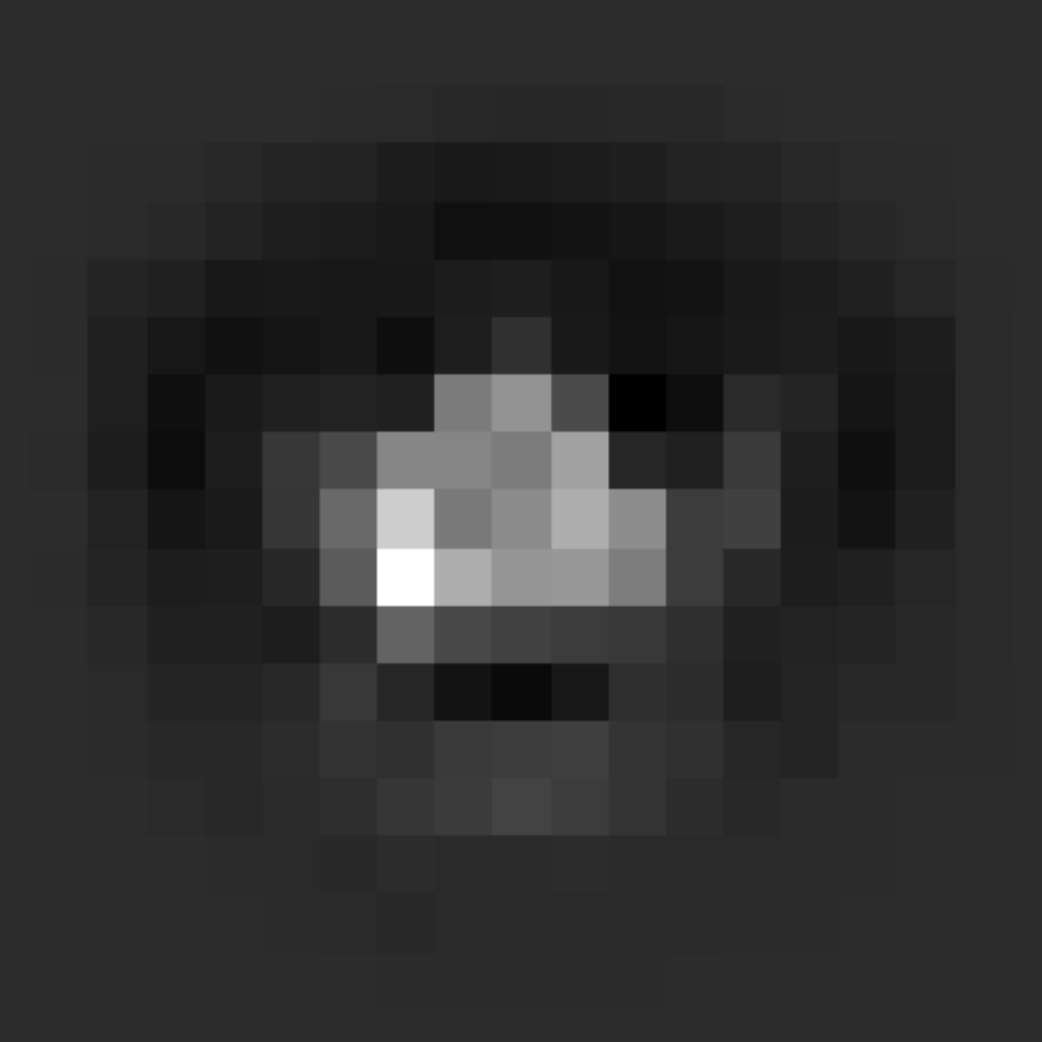}}
    \centerline{\scriptsize conv5}
\end{minipage}
\begin{minipage}{0.070\textwidth}
    \centerline{\includegraphics[width=1\textwidth]{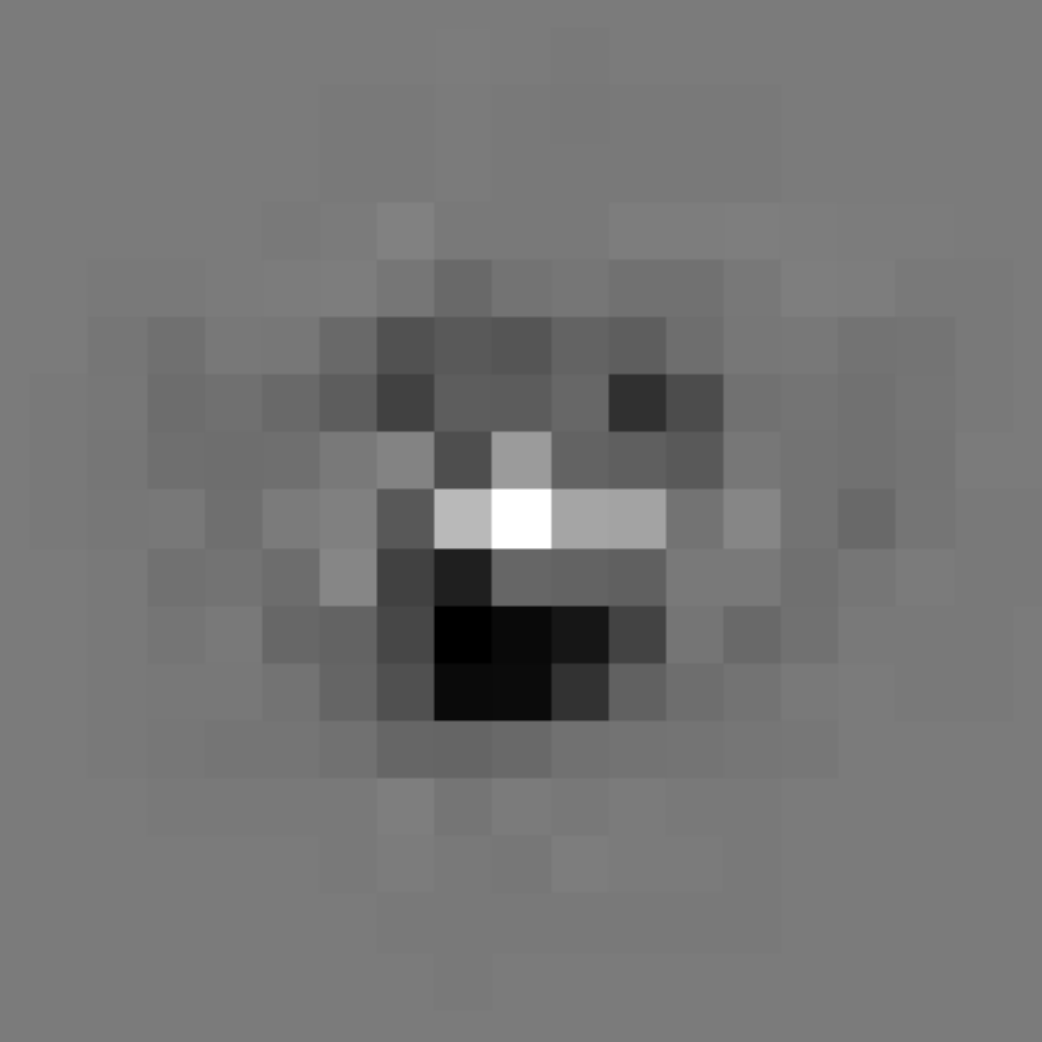}}
    \centerline{\scriptsize conv6}
\end{minipage}
\begin{minipage}{0.070\textwidth}
    \centerline{\includegraphics[width=1\textwidth]{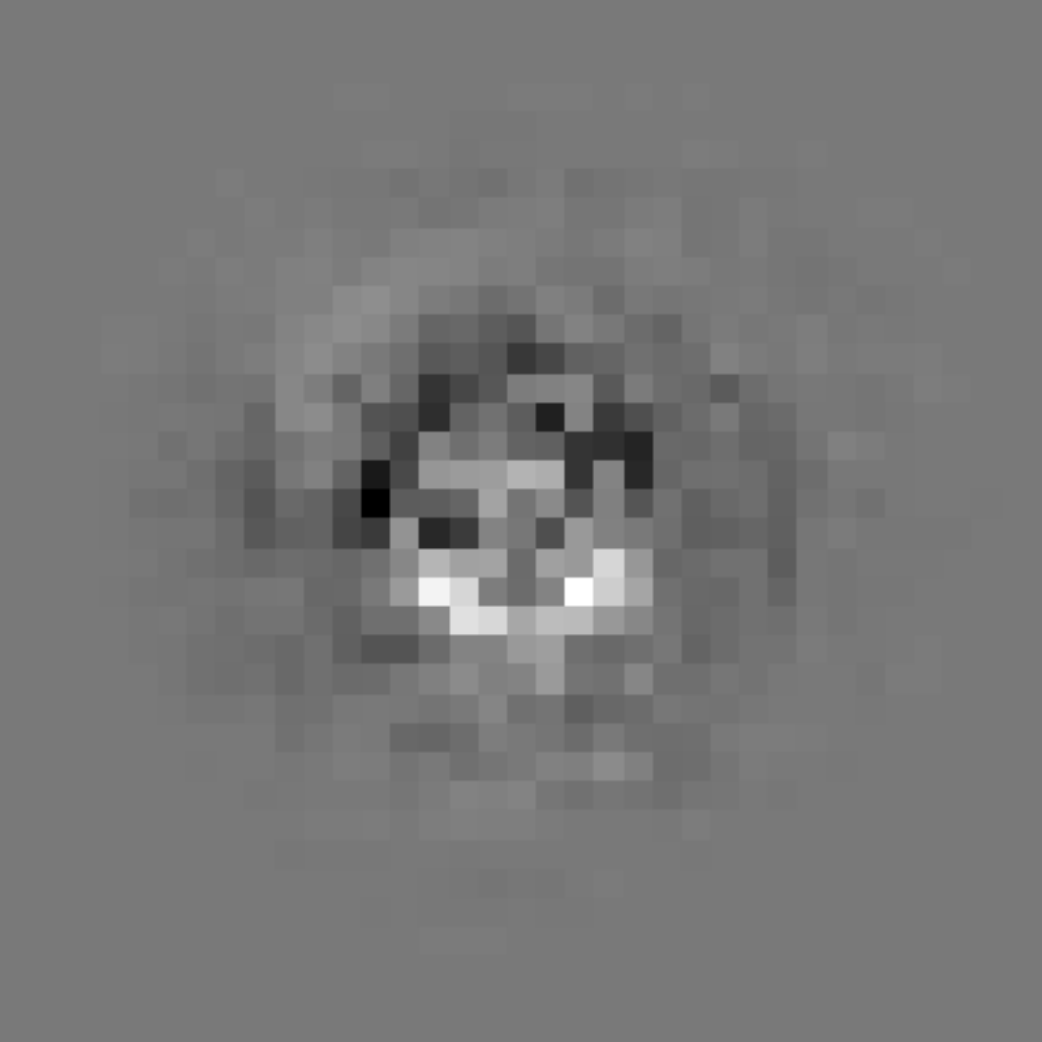}}
    \centerline{\scriptsize upsample1}
\end{minipage}
\begin{minipage}{0.070\textwidth}
    \centerline{\includegraphics[width=1\textwidth]{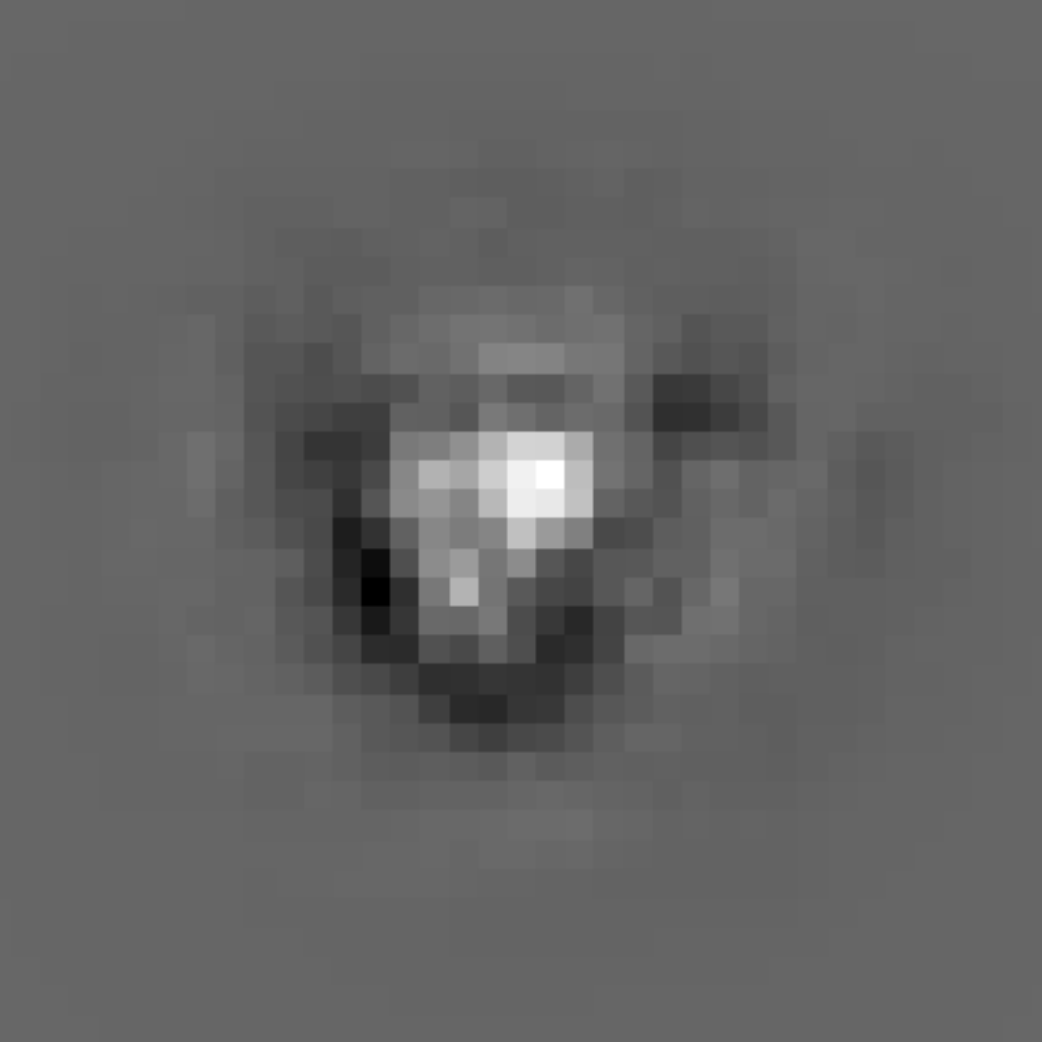}}
    \centerline{\scriptsize smooth1}
\end{minipage}
\begin{minipage}{0.070\textwidth}
    \centerline{\includegraphics[width=1\textwidth]{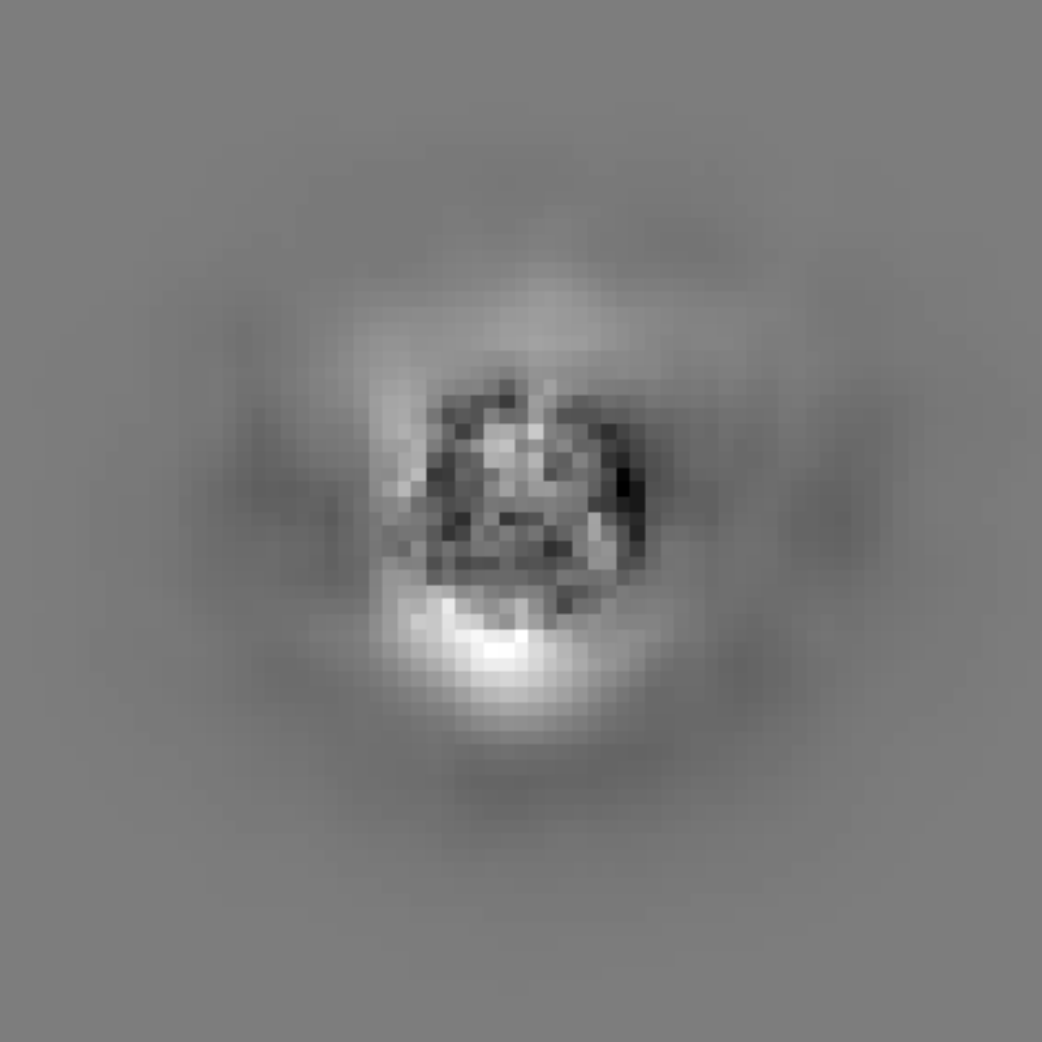}}
    \centerline{\scriptsize smooth2}
\end{minipage}
\begin{minipage}{0.070\textwidth}
    \centerline{\includegraphics[width=1\textwidth]{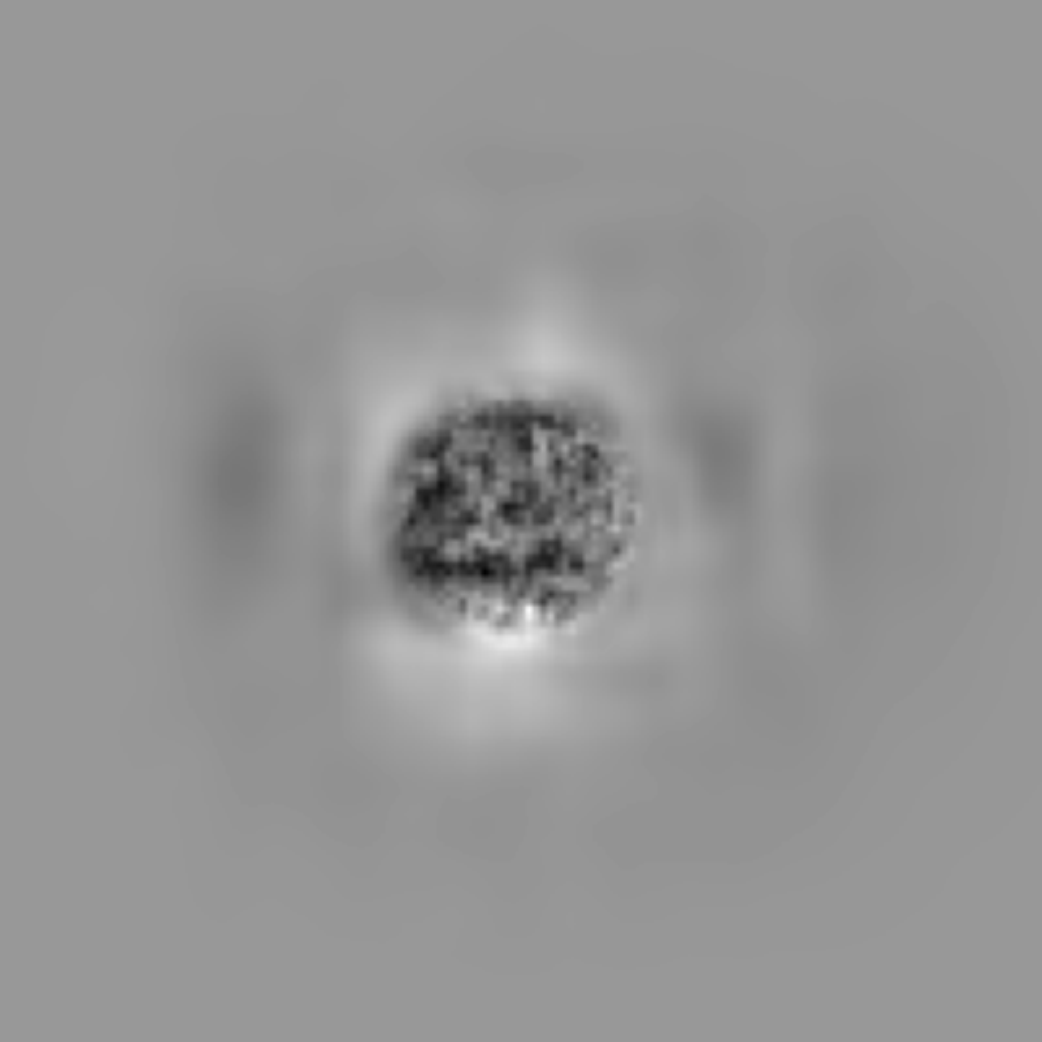}}
    \centerline{\scriptsize smooth3}
\end{minipage}
\begin{minipage}{0.070\textwidth}
    \centerline{\includegraphics[width=1\textwidth]{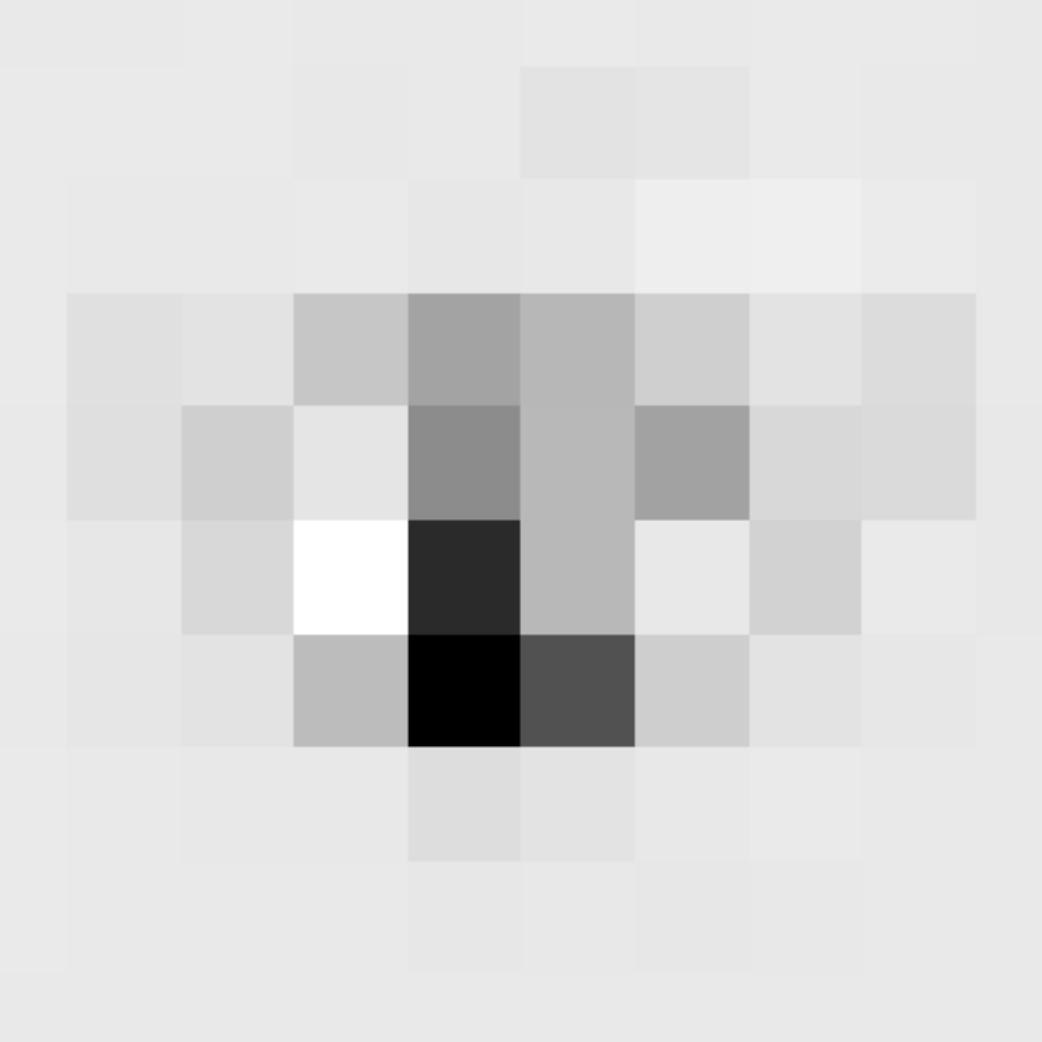}}
    \centerline{\scriptsize maxpool1}
\end{minipage} \\
\begin{minipage}{0.070\textwidth}
    \centerline{\includegraphics[width=1\textwidth]{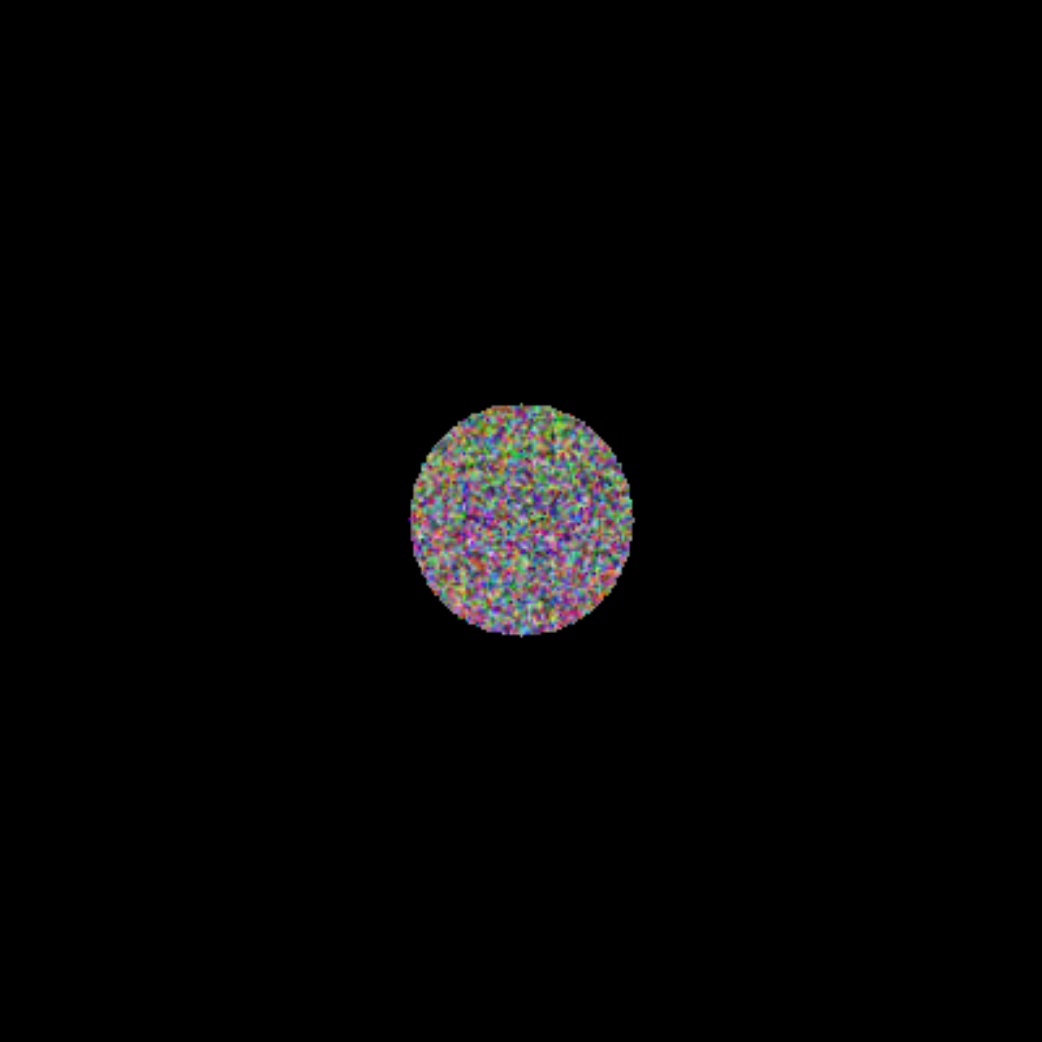}}
    \centerline{\scriptsize Input}
\end{minipage}
\begin{minipage}{0.070\textwidth}
    \centerline{\includegraphics[width=1\textwidth]{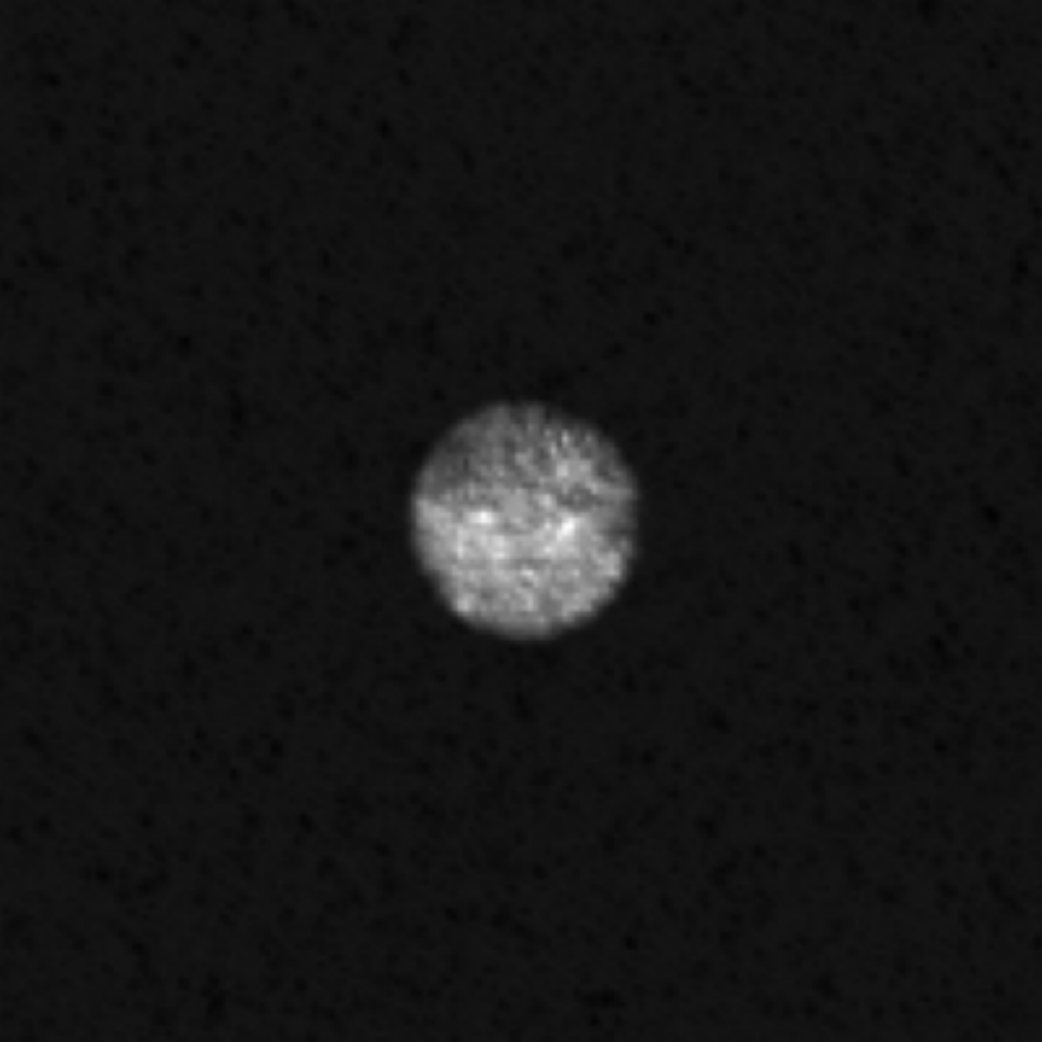}}
    \centerline{\scriptsize conv1}
\end{minipage}
\begin{minipage}{0.070\textwidth}
    \centerline{\includegraphics[width=1\textwidth]{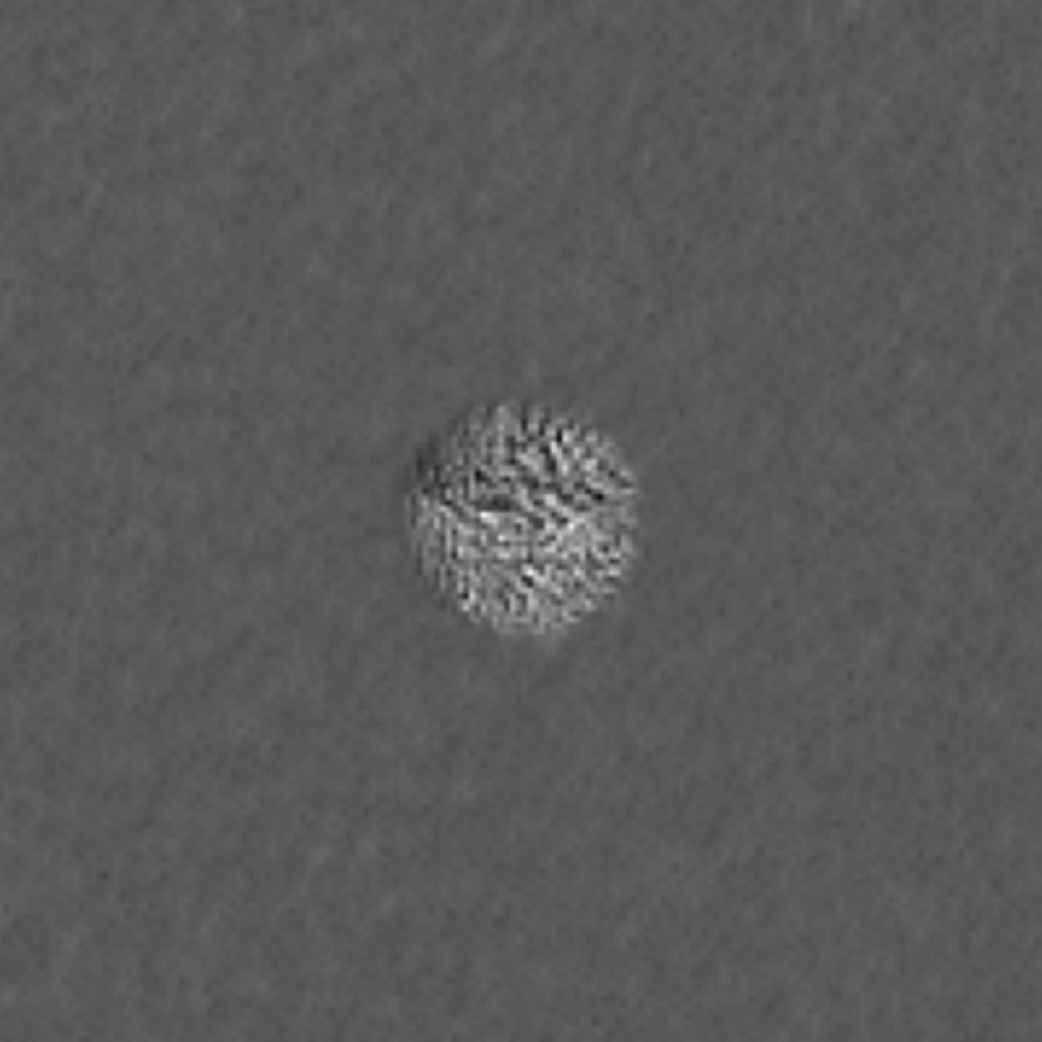}}
    \centerline{\scriptsize conv2}
\end{minipage}
\begin{minipage}{0.070\textwidth}
    \centerline{\includegraphics[width=1\textwidth]{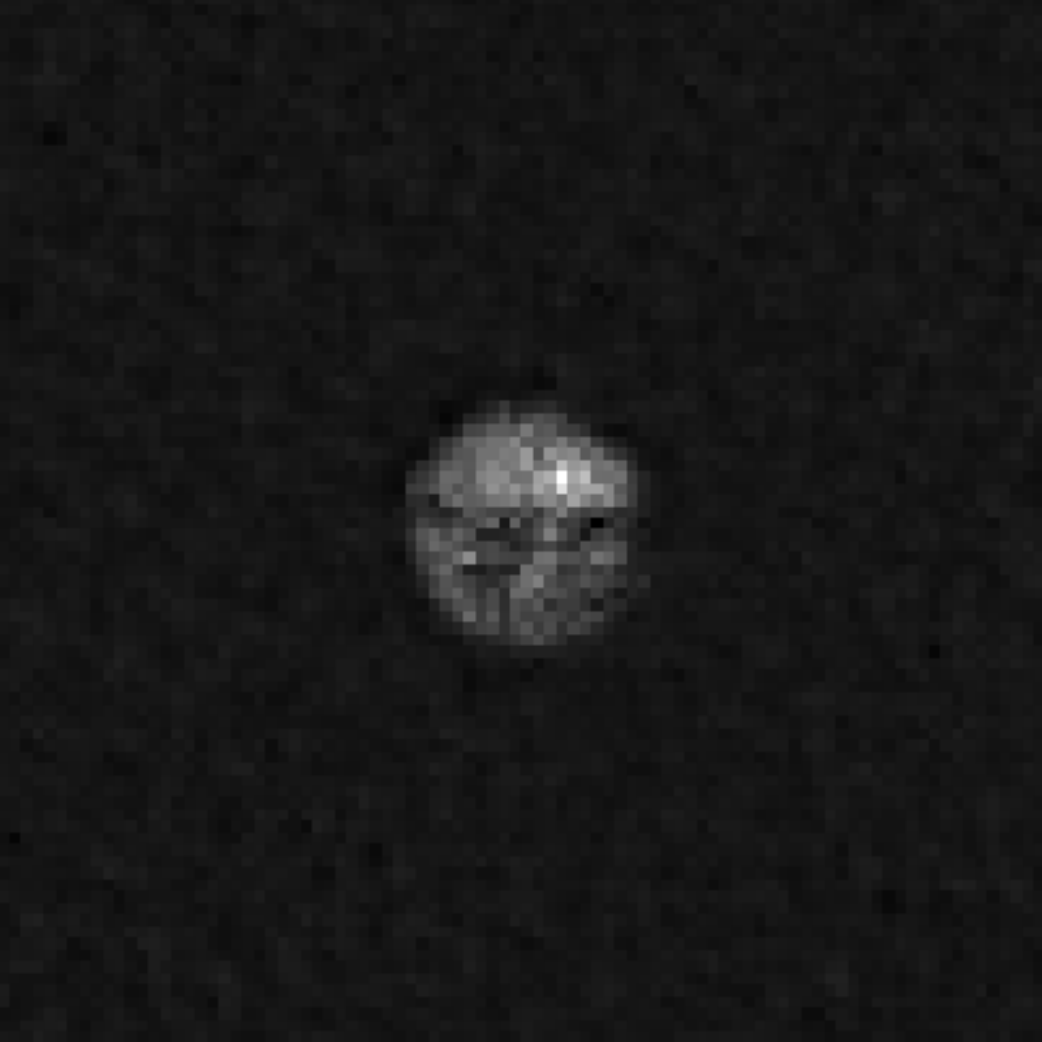}}
    \centerline{\scriptsize conv3}
\end{minipage}
\begin{minipage}{0.070\textwidth}
    \centerline{\includegraphics[width=1\textwidth]{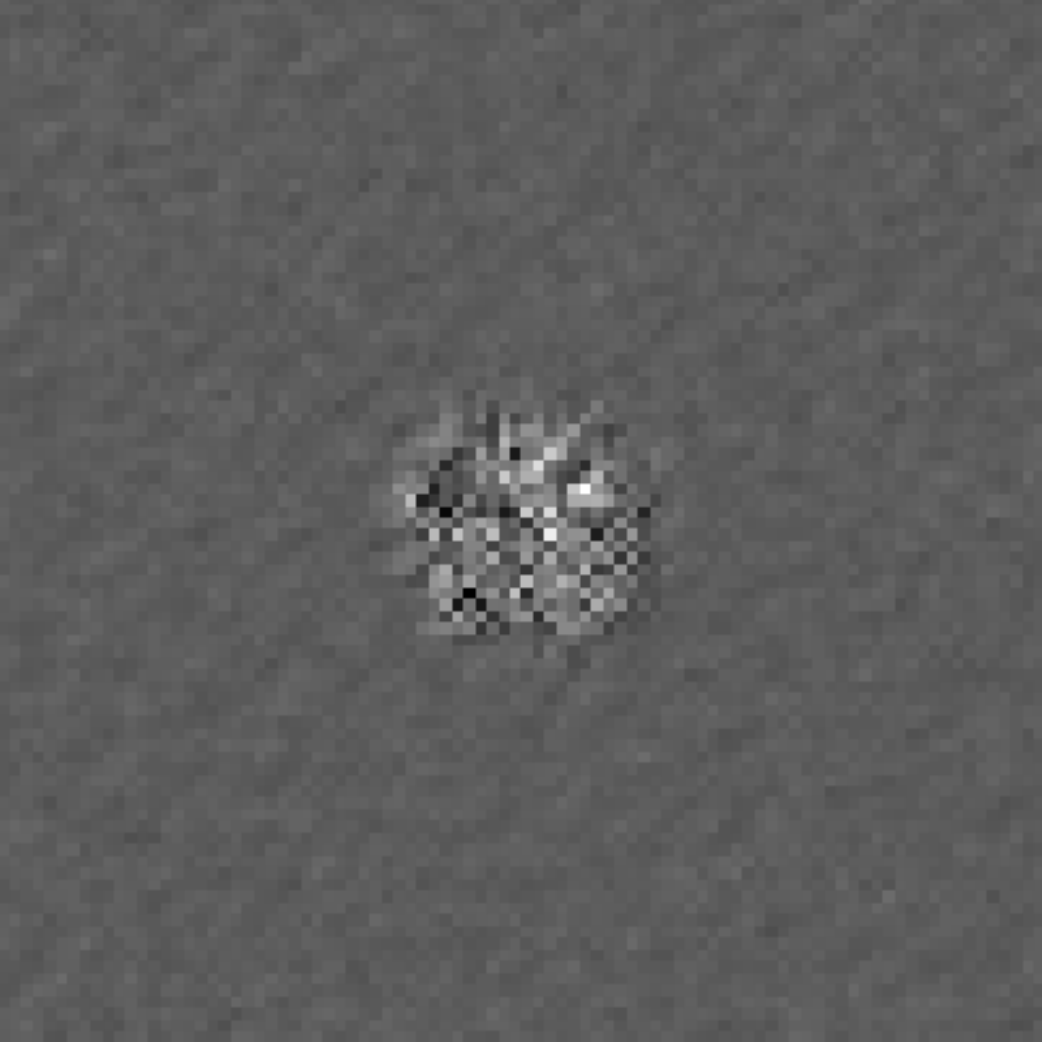}}
    \centerline{\scriptsize conv4}
\end{minipage}
\begin{minipage}{0.070\textwidth}
    \centerline{\includegraphics[width=1\textwidth]{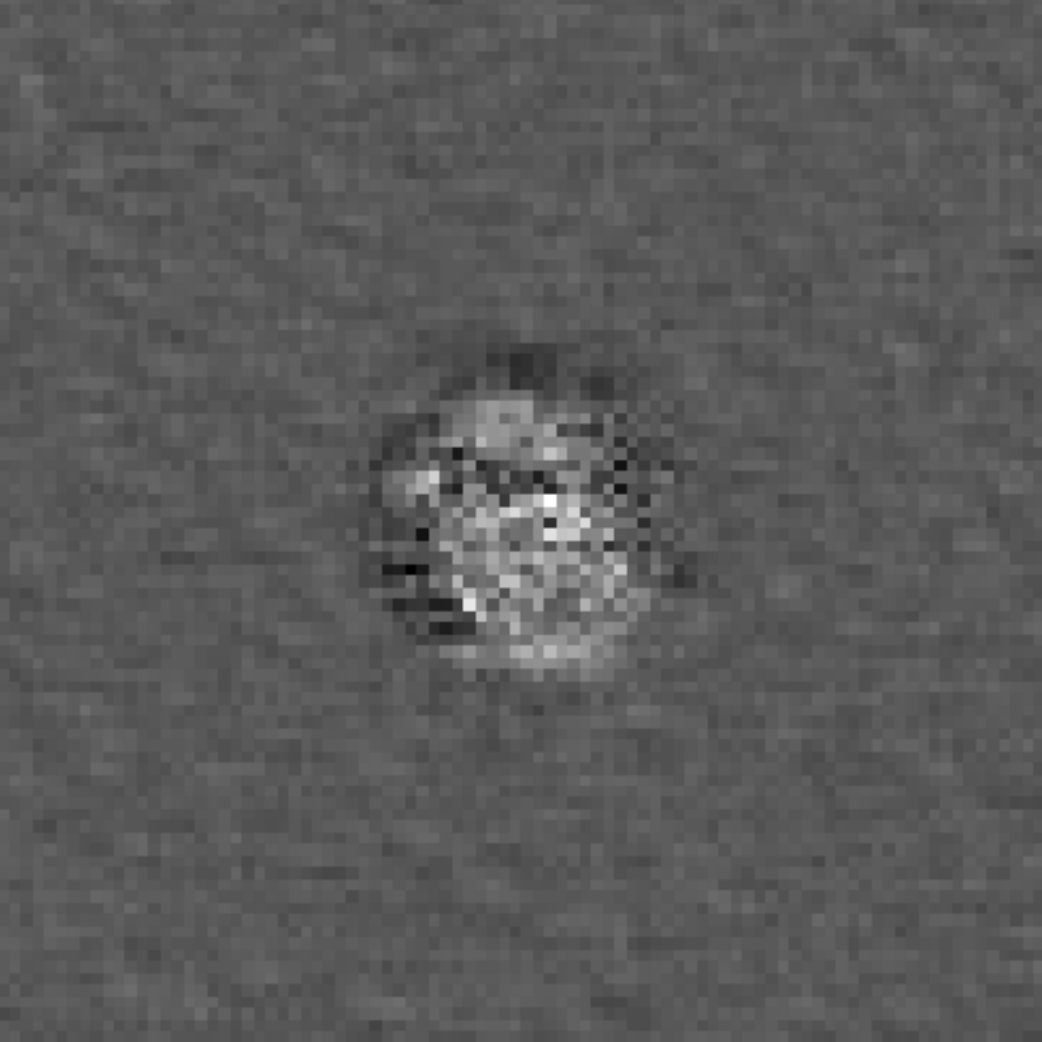}}
    \centerline{\scriptsize conv5}
\end{minipage}
\begin{minipage}{0.070\textwidth}
    \centerline{\includegraphics[width=1\textwidth]{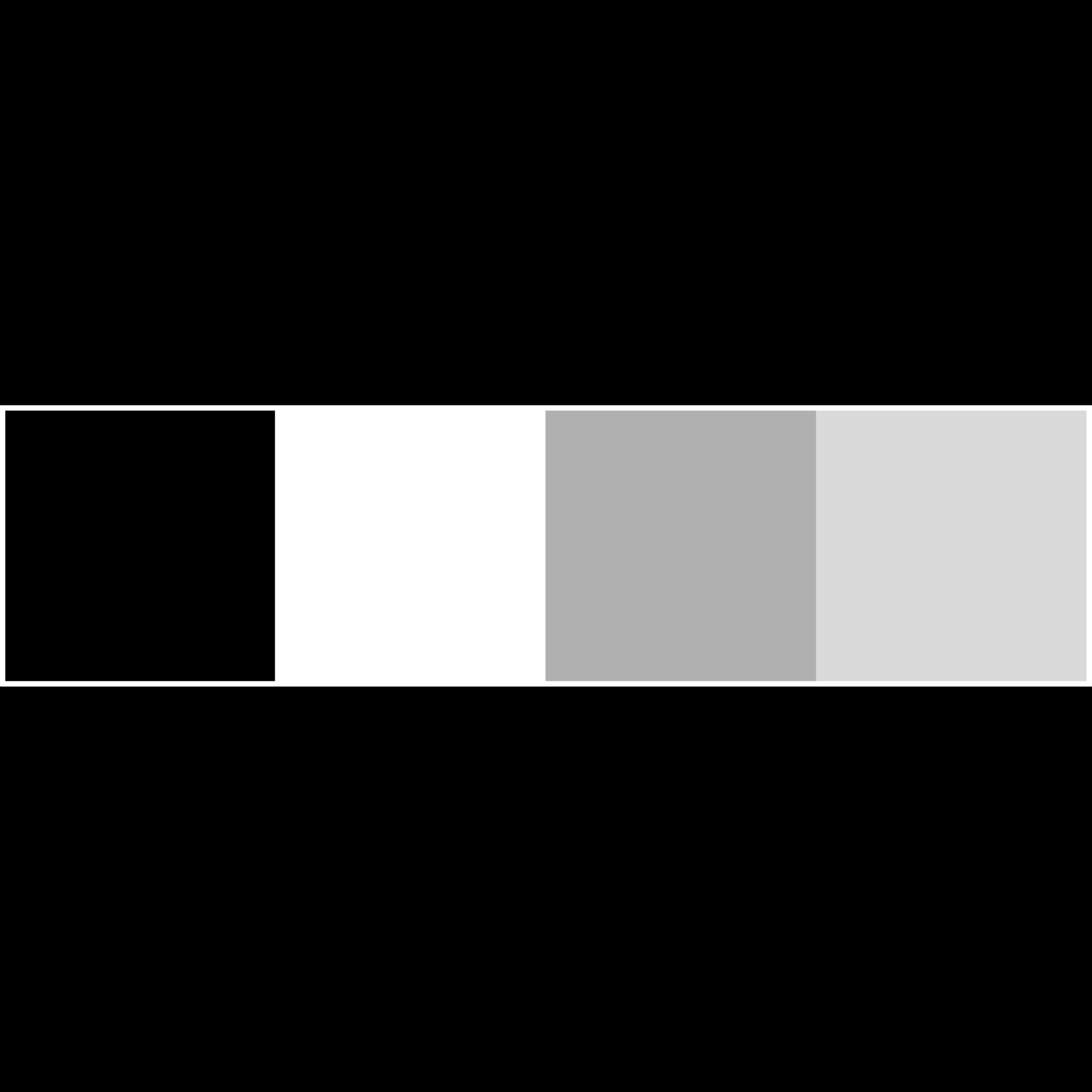}}
    \centerline{\scriptsize branch1}
\end{minipage}
\begin{minipage}{0.070\textwidth}
    \centerline{\includegraphics[width=1\textwidth]{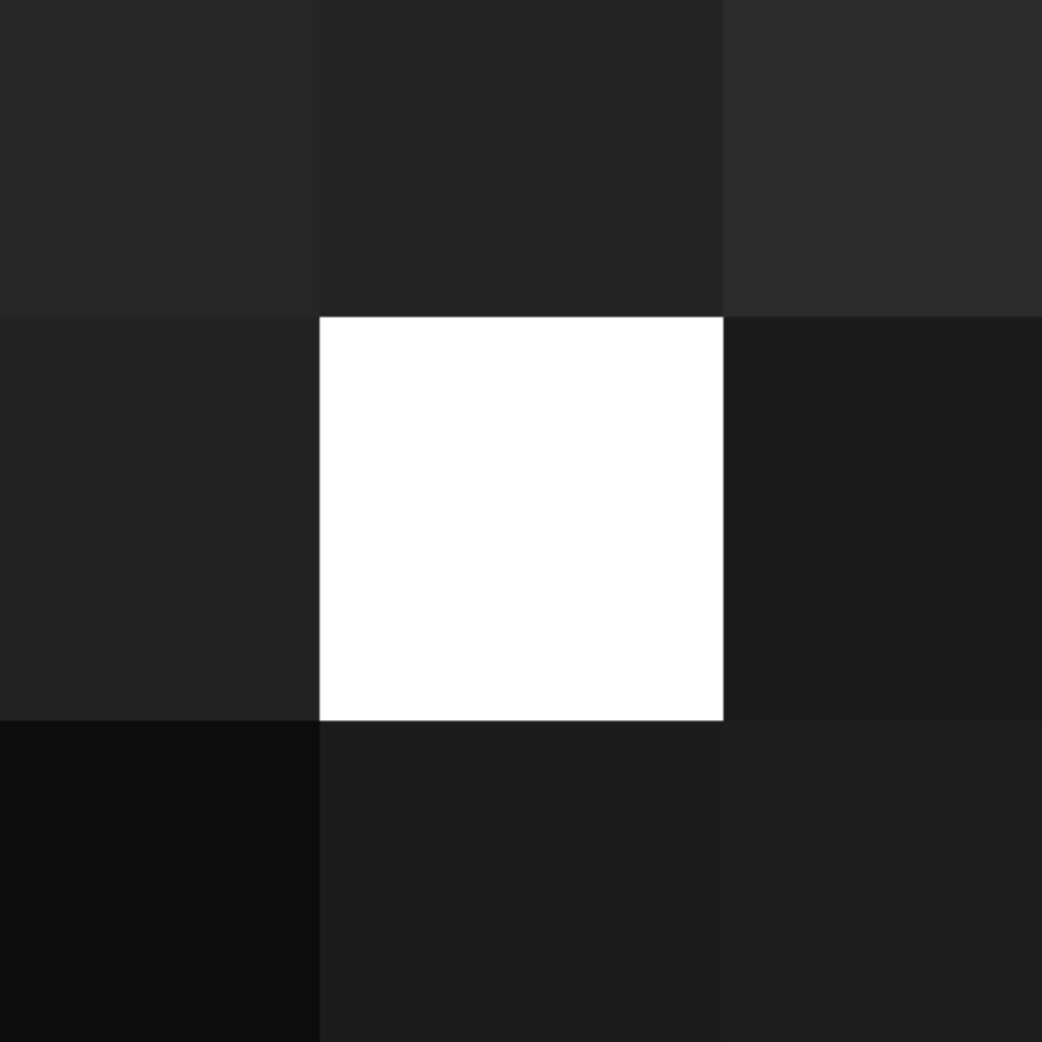}}
    \centerline{\scriptsize branch2}
\end{minipage}
\begin{minipage}{0.070\textwidth}
    \centerline{\includegraphics[width=1\textwidth]{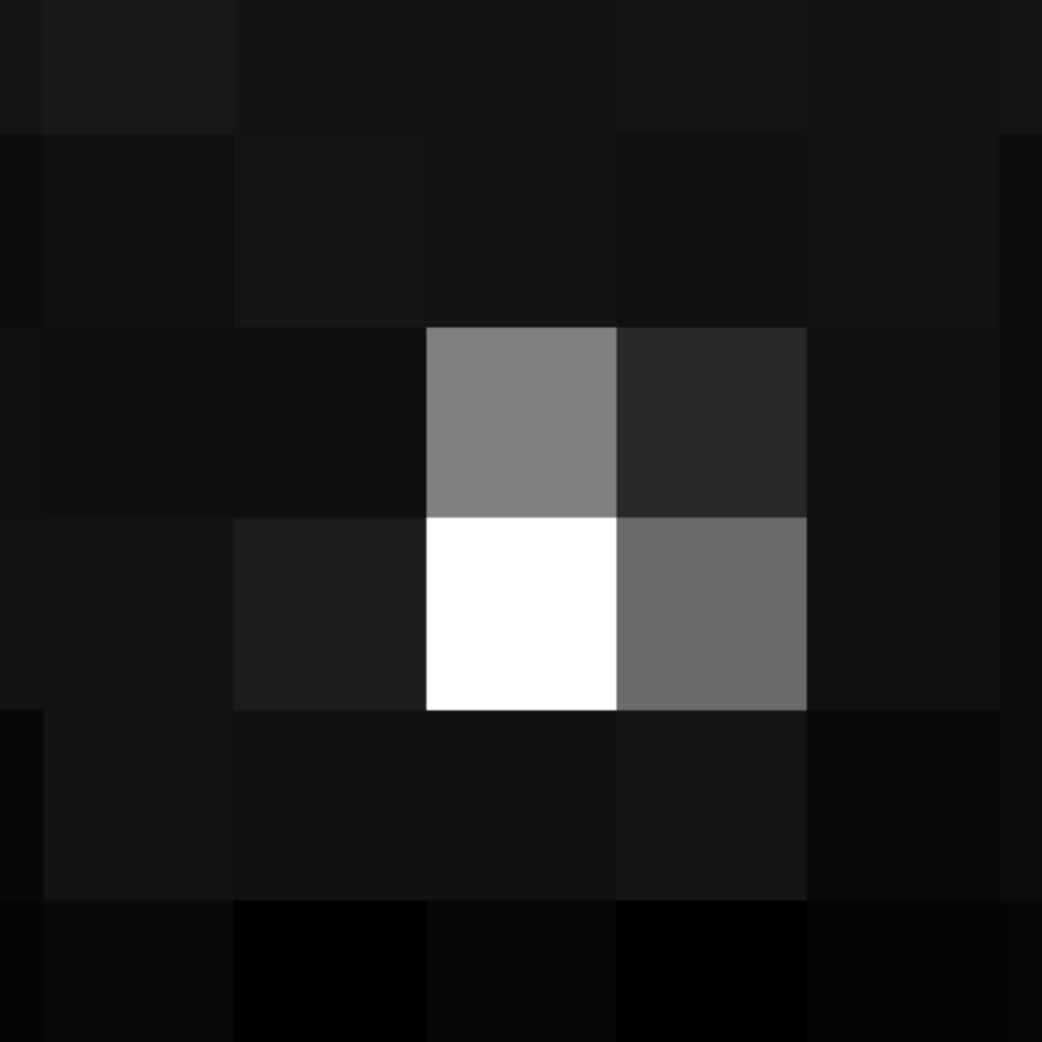}}
    \centerline{\scriptsize branch3}
\end{minipage}
\begin{minipage}{0.070\textwidth}
    \centerline{\includegraphics[width=1\textwidth]{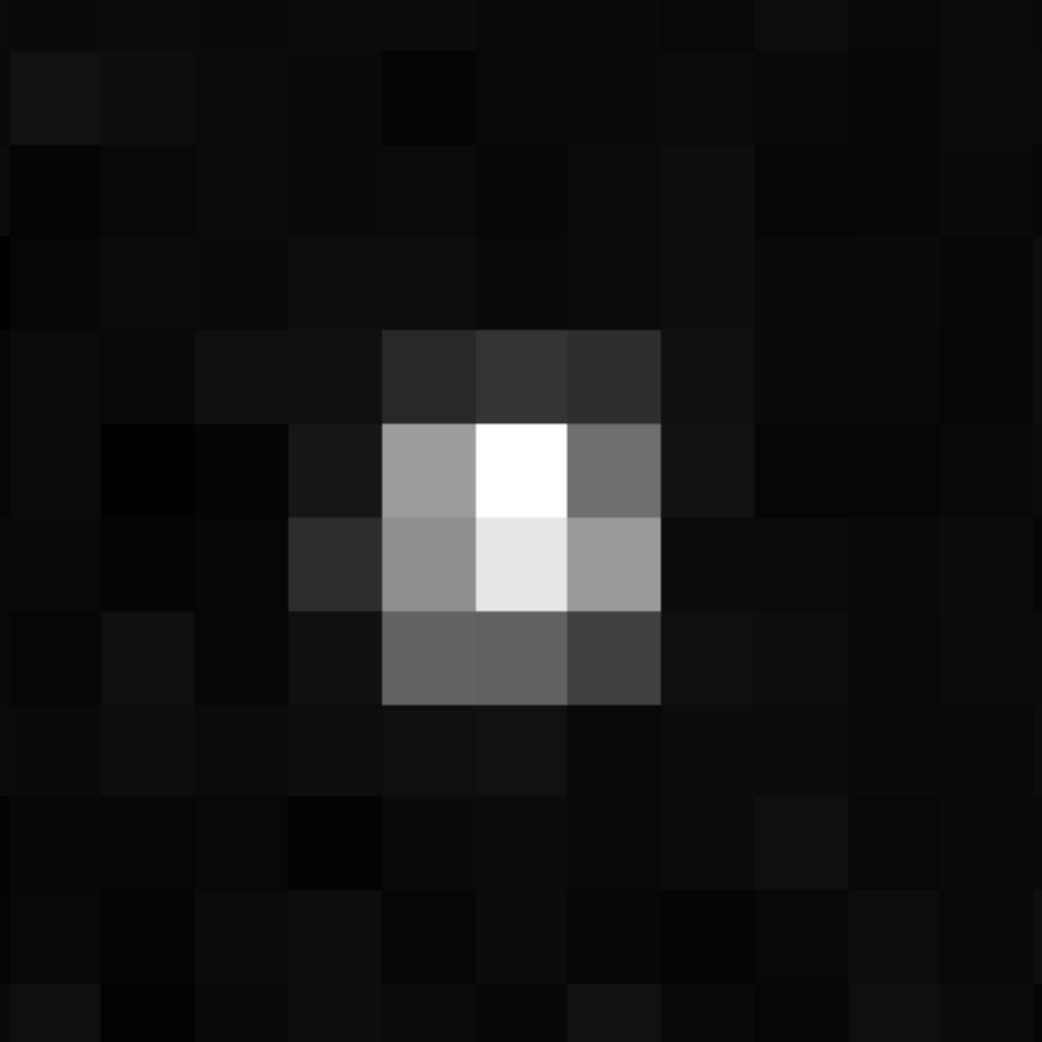}}
    \centerline{\scriptsize branch4}
\end{minipage}
\begin{minipage}{0.070\textwidth}
    \centerline{\includegraphics[width=1\textwidth]{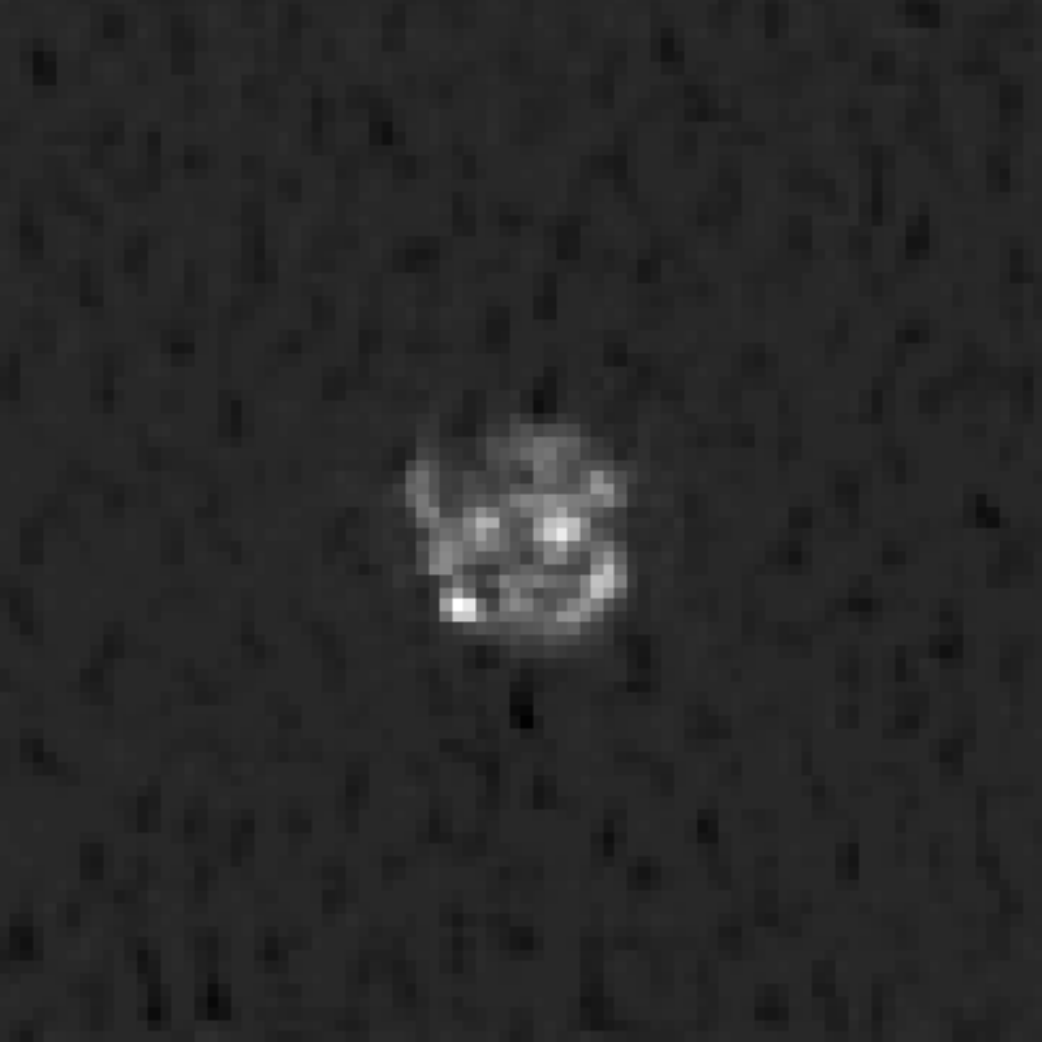}}
    \centerline{\scriptsize fusion1}
\end{minipage}
\begin{minipage}{0.070\textwidth}
    \centerline{\includegraphics[width=1\textwidth]{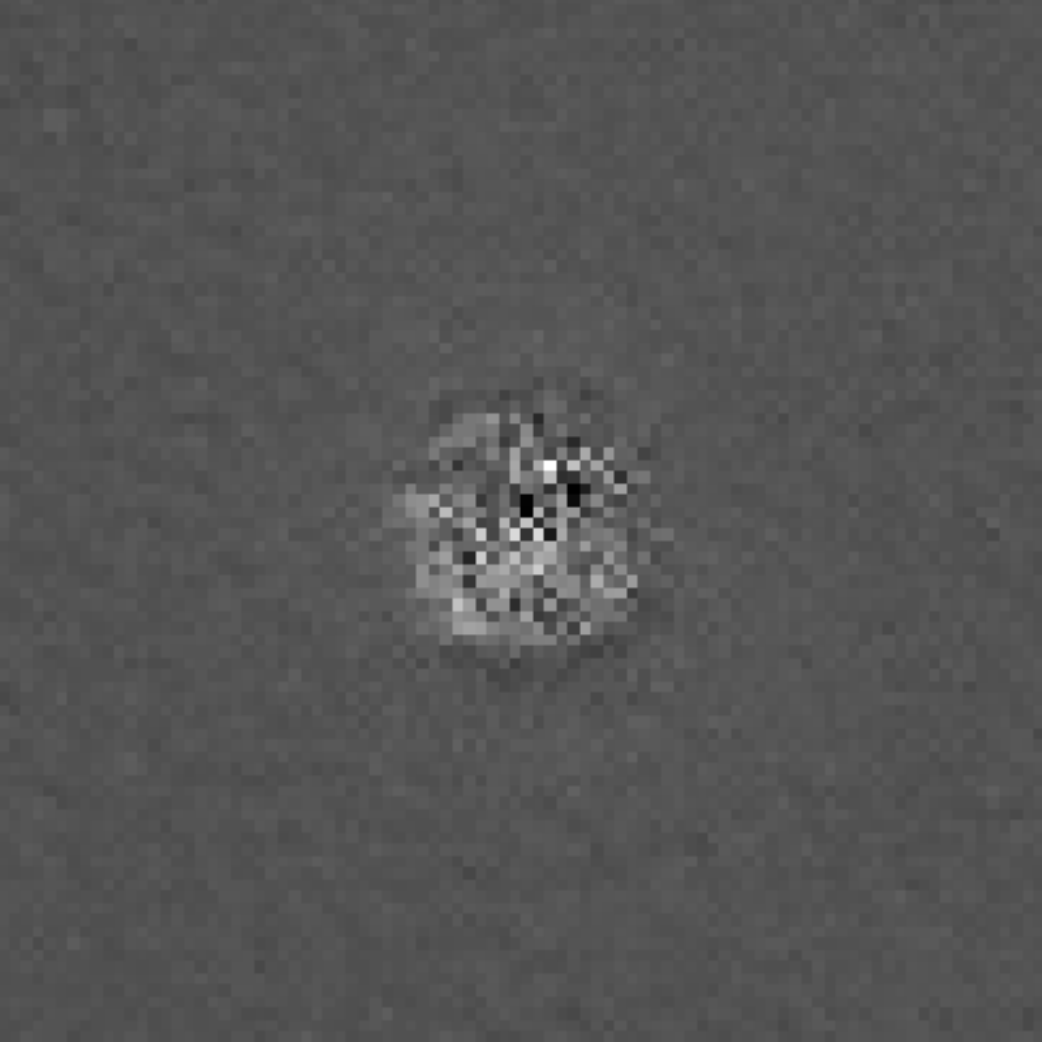}}
    \centerline{\scriptsize fusion2}
\end{minipage}
\caption{Visualization of the feature maps of Stereo R-CNN (the 1st row) and DSGN (the 2nd row) from shallow to deeper layers.}
\label{fig:diff_feat_map}
\end{figure*}

\subsection{The Cause of the Decoupling between the Precision of 3D Object Detectors and Driving Safety under Adversarial Attacks}

From the experiments in Section~\ref{ch5:sec.experiments}, we observe that  perturbation attacks cause a significant precision decline in 3D object detection, but  only a slight change in driving safety performance. Patch attacks cause a slight decline in 3D object detection precision, but a relatively significant performance degradation in driving safety. To figure out the reasons beneath these results, we take a closer look at the detection results of the attacked and unattacked 3D object detectors and compare them accordingly.

The reasons for the decoupling caused by patch attacks are relatively straightforward. First, we notice that the affected area of patch attacks is limited to the patch itself where the patch is usually quite small in order to make it difficult to be detected. Thus, patch attacks can only  trigger the object detectors to produce a very small number of ghost 3D bounding boxes inside the patch.  This is the reason why the object detection precision does not show a significant decline when detectors are under patch attacks. Second, since the adversarial patch is randomly placed in driving scenarios, the resulted ghost 3D bounding boxes have a fair chance to appear on the road surface and block the way of the ego-vehicle, which directly leads to noticeably driving safety performance degradation. These two reasons together explain  the decoupling between the precision of 3D object detectors and driving safety under patch attacks. In the rest of this section, we mainly focus on investigating the reasons for the decoupling caused by perturbation attacks.

Apart from the fact that the perturbation attacks cause slight drifts of 3D bounding boxes of real objects which are originally produced accurately when no attack is launched, the most significant consequence of a perturbation attack is that it triggers the object detectors to produce a lot of ghost 3D bounding boxes which do not circle any specific or meaningful object inside. In particular, almost all ghost boxes appear in the side areas of a road instead of on the surface of a road. Since the optimal trajectory generated by the motion planner most likely will not traverse the side areas of a road, the ghost objects will not affect the trajectory generated by the motion planner. In other words, the trajectories generated before and after perturbation attacks are essentially the same. Thus, driving safety is not affected dramatically by permutation attacks. 

We further investigate why the ghost 3D bounding boxes caused by perturbation attacks tend to appear in the side areas of a road. After inspecting the positions where ghost bounding boxes appear in a large number of driving scenarios, we hypothesize that the difference in the texture complexity between the side areas of a road and the road surface may be the cause of this. The reason is that the texture of the road surface is more regular than the texture of the side areas of a road. Thus,  it takes more ``efforts" for perturbation attacks to change the pixel values to generate ghost boxes on the road surface than that in the side areas of a road.

In order to validate our hypothesis, we design a texture replacement experiment. Specifically, for a driving scenario in which vision-based 3D object detectors under a perturbation attack produce ghost 3D bounding boxes in  the side areas of a road, we replace the texture of the side areas with the texture of the road surface, then feed this modified driving scenario to the attacked object detectors and check the detection results. If our hypothesis is correct,  we shall expect that the attacked object detectors do not produce any ghost boxes in the side areas of a road for the modified driving scenario.

The results of the texture replacement experiments are shown in Figure~\ref{fig:ab_1_stereo} for the Stereo R-CNN model and in Figure~\ref{fig:ab_1_dsgn} for the DSGN model. The attack setting of the perturbation attack applied here is set to be $\alpha=1$ and $n=4$ in Eqn.~(\ref{eq:eq_2}). From the figures, we can observe that both 3D object detectors can detect the object accurately when there is no perturbation attack applied (Figure~\ref{fig:ab_1_stereo_1}, \ref{fig:ab_1_dsgn_1}), and ghost 3D bounding boxes appear in the side areas of a road when the perturbation attack is launched on the same driving scenario (Figure~\ref{fig:ab_1_stereo_2}, \ref{fig:ab_1_dsgn_2}). More importantly, after we replace the texture of side area of road with the texture of road surface (Figure~\ref{fig:ab_1_stereo_3}, \ref{fig:ab_1_dsgn_3}), no more ghost boxes are produced in the side area of the road by the 3D object detectors, which matches our expectation. Hence, the texture replacement experiment results validate our hypothesis that the difference in the texture complexity between the side areas of a road and the road surface leads to the decoupling between the precision of 3D object detectors and the driving safety performance metrics when the 3D object detectors are attacked.

\subsection{The Cause of Difference in Robustness}

The experiment results in Section~\ref{ch5:sec.experiments} indicate that DSGN is more robust than Stereo R-CNN in terms of driving safety and object detection when they are under adversarial attacks. Especially, when  patch attacks are launched, DSGN is more robust than Stereo R-CNN.

To better understand the cause of such a difference in robustness, we conduct a contrast experiment by implementing the black-box patch attack where instead of training a patch for each model separately, we learn a patch $p$ that is jointly optimized for both the DSGN model and the Stereo R-CNN model using Eqn.~(\ref{eq:eq_3}). Thus, the patch is capable of attacking both models. To conduct this experiment, we also generate an image $I$ with uniformly distributed random noise and paste the patch $p$ on $I$ to form the attacked input image $\tilde{I}$. We then feed the two input images into the models to observe the intermediate results produced by their network architectures. In Figure~\ref{fig:diff_feat_map}, we visualize the difference between corresponding intermediate feature maps generated from $I$ and $\tilde{I}$ for both models respectively. In other words, the feature map $F_k$ in Figure~\ref{fig:diff_feat_map} refers to the average norm of the difference between the $k$-th layer output with an attack applied and the $k$-th layer output without any attack. For each model, we inspect the intermediate feature map of layers from shallow to deep in the feature extraction part of its network architecture. Each feature map is cropped so that only the central part is used for the propose of demonstration.

It is the spatial propagation of the patch activation area in feature maps of a model that implies the robustness of the model to patch attacks. Specifically, if the patch activation area in feature maps propagates along the data flow direction of the network architecture, then the network architecture amplifies the impact of patch attacks on the model, suggesting weak robustness of the model to patch attacks. In contrast, if the patch activation area in feature maps does not propagate or even contracts along the data flow, then the network architecture of the model is more resilient to patch attacks, indicating stronger robustness of the model to adversarial patches.

In the first row of Figure~\ref{fig:diff_feat_map}, we show how the patch activation area propagates layer by layer in the Stereo R-CNN model. In the first few convolution layers ($conv$\textless1, 2, 3\textgreater), the patch activation area  is bounded by the original region. However, starting from the last two layers of the feature extractor ($conv$\textless4, 5\textgreater), we observe that the activation area gradually propagates as we move on to deeper layers. After the $maxpool1$ layer, the patch activation area propagates to  almost the entire cropped image. Since the patch activation area keeps propagating through the network architecture, Stereo R-CNN shows poor robustness under patch attacks.

In the second row of Figure~\ref{fig:diff_feat_map}, the DSGN model shows less propagation of the patch activation in the first three convolution layers ($conv$\textless1, 2, 3\textgreater), but the activation area at the last two layers of the feature extractor ($conv$\textless4, 5\textgreater) are expanded slightly. According to the network structure of DSGN, the feature extractor is connected to the Spatial Pyramid Pooling (SPP) module, and the outputs of SPP branches ($branch$\textless1, 2, 3, 4\textgreater) are fused with features from the former layers for future prediction. Interestingly, we observe that the patch activation area shrinks to its original size after the SPP module ($fusion$\textless1, 2\textgreater). Hence, different from  Stereo R-CNN, the SPP module in the DSGN model restrains the propagation of the patch impact. This demonstrates that DSGN has strong robustness to adversarial patches due to the SPP module in its network architecture. A similar observation of the Spatial Pyramid structure can be found in another study~\cite{ranjan19attacking}. We can conclude that when the adversarial patch is used to attack the model of Stereo R-CNN whose network architecture is not equipped with the SPP module, the patch exploits the weakness of the network architecture and amplifies its impact on 3D object detection. For the DSGN model, the SPP module in its network architecture restrains the impact of the adversarial patch on 3D object detection. Therefore, DSGN and Stereo R-CNN have different robustness to patch attacks and demonstrate different performance on the average precision of 3D object detection and the driving safety metrics.

\section{Conclusion}
\label{ch5:sec.conclusion}

In this chapter, we have systematically investigated the impact of adversarial attacks not only on the object detection precision, but also on the driving safety of vision-based autonomous vehicles. Specifically, we proposed an end-to-end driving safety evaluation framework with designed performance metrics for the assessment of driving safety. Through extensive evaluation experiments, we found that a significant precision decline of 3D object detectors under the perturbation attack only leads to a slight decline in the driving safety performance metrics, but a mild precision decline of 3D object detectors under the patch attack can result in a significant performance degradation in driving safety. This finding suggests that it is desirable to evaluate the robustness of deep learning models in terms of driving safety rather than model precision. The proposed work can help guide the selection of deep learning models. The code of our evaluation framework is available upon request.

\chapter{Conclusion and Future Work}
\label{ch6:conclusion}

\section{Conclusion}

In this thesis, we focused on the sensor data validation for single autonomous vehicles and multiple connected autonomous vehicles, and the end-to-end impact evaluation of adversarial attacks on the driving safety of vision-based autonomous vehicles.

First, we proposed a data validation framework to help defend optical sensors for single autonomous vehicles. The framework is two-fold. At first, we detected the optical attacks in a three-sensor system by analyzing the distribution of information inconsistency among depth maps. Then, based on the detection scheme, we further developed a method capable of identifying up to $n − 2$ attacked sensors in a system with one LiDAR and $n$ cameras. We also presented the sensitivity analysis of our data validation framework.

Second, we proposed a data validation method to detect optical attacks against LiDARs for multiple connected autonomous vehicles. In our method, we leveraged multiple point clouds from neighboring vehicular nodes and completed the scan in selected validation regions by mirroring points based on object symmetry. Then, we discretized surface meshes generated from selected validation regions into 2D grids and measured their distances. Finally, we detected the optical attacks against LiDARs by thresholding the distances.

Third, we studies the linkage between attacked deep learning models and driving safety by evaluating the impact of adversarial attacks, i.e., perturbation attacks and patch attacks, on the driving safety of vision-based autonomous vehicles in an end-to-end fashion. To perform the evaluation, we proposed an end-to-end driving safety evaluation framework with a set of designed metrics. Through the experiments, we found the decoupling between the weakened detection precision of deep learning models and the vehicular driving safety when autonomous vehicles are under adversarial attacks. We further investigated the reasons for the decoupling in an ablation study.

\section{Future Work}

With regards to sensor security and deep learning model security, there is still a long way to go before reaching the goal of safe and secure autonomous driving technology. In this thesis, we attempted to defend optical sensors and evaluate the end-to-end impact of adversarial attacks on driving safety. In the future, we aim to not only defend key parts of autonomous vehicles from more security threats, but also recover from the possible attack damages:

\begin{itemize}
\item
We plan to study methods to further identify the damaged portions in image and point cloud, and perform data recovery for the damaged portions using intact data from other redundant sensors of the ego-vehicle, nearby infrastructure, or connected vehicles in vicinity.

\item
Since we used multiple single-vehicle point clouds to simulate the the scenario of multiple connected autonomous vehicles in Chapter~\ref{ch4:sensor}, we plan to first verify our current method using simulated multi-vehicles point clouds. Then, we will use LiDARs to harvest real multi-vehicles point clouds and improve our current method. In addition, we plan to further turn the step of comparison and thresholding into a learning-based step.

\item
Based on our experiments and discovered causes~\cite{zhang21evaluating} in Chapter~\ref{ch5:impact}, we plan to expand our study to the autonomous vehicles that fuse the information from both LiDARs and cameras, and consider adversarial examples in other types of data, such as adversarial examples of LiDAR point clouds.

\item
Since we found that the spatial pyramid structure is more robust under adversarial attacks~\cite{zhang21evaluating} in Chapter~\ref{ch5:impact}, we plan to design countermeasures for deep learning models against adversarial attacks by leveraging the spatial pyramid structure. Our future studies will also be conducted in an end-to-end fashion.
\end{itemize}

\newpage

\renewcommand{\bibname}{References}
\addcontentsline{toc}{chapter}{References}
\bibliographystyle{plain}
\bibliography{reference}

\appendix

\chapter{List of Publications}

\begin{enumerate}
\item
\textbf{Jindi Zhang}, Yifan Zhang, Kejie Lu, Jianping Wang, Kui Wu, Xiaohua Jia, and Bin Liu. ``Detecting and Identifying Optical Signal Attacks on Autonomous Driving Systems.'' \textit{IEEE Internet of Things Journal}, vol. 8, no. 2, pp. 1140-1153, 15 Jan.15, 2021.

\item
\textbf{Jindi Zhang}, Yang Lou, Jianping Wang, Kui Wu, Kejie Lu, and Xiaohua Jia. ``Evaluating Adversarial Attacks on Driving Safety in Vision-Based Autonomous Vehicles.'' \textit{IEEE Internet of Things Journal}, 2021.

\item
Yifan Zhang, Jinghuai Zhang, \textbf{Jindi Zhang}, Jianping Wang, Kejie Lu, and Jeff Hong. ``Integrating Algorithmic Sampling-based Motion Planning with Learning in Autonomous Driving.'' \textit{ACM Transactions on Intelligent Systems and Technology}, vol. 13, no. 3, pp. 1-27, 2022.

\item
Yifan Zhang, Jinghuai Zhang, \textbf{Jindi Zhang}, Jianping Wang, Kejie Lu, and Jeff Hong. ``A novel learning framework for sampling-based motion planning in autonomous driving.'' \textit{Proceedings of the AAAI Conference on Artificial Intelligence}, 2020.
\end{enumerate}

\end{document}